\documentclass[english]{article}

\usepackage[T1]{fontenc}
\usepackage[latin9]{inputenc}
\usepackage{geometry}
\usepackage{babel}
\usepackage{verbatim}
\usepackage{dsfont}
\usepackage{float}
\usepackage{bm}
\usepackage{amsmath}
\usepackage{amssymb}
\usepackage{graphicx}
\usepackage{hyperref}
\hypersetup{
 colorlinks,linkcolor=red,anchorcolor=blue,citecolor=blue}
\usepackage{breakurl}

\usepackage{subcaption}

\usepackage{titlesec}

\usepackage{euscript,amsfonts,amssymb,amsmath,amscd,amsthm,enumerate,hyperref}
\usepackage{bbm}
\usepackage{bm}
\usepackage{framed}
\usepackage{mathtools}
\usepackage{comment}
\usepackage{todonotes}
\usepackage{mathrsfs}
\usepackage{graphicx}
\usepackage{color}
\usepackage{pdfpages}
\usepackage{framed}

\newtheorem{theorem}{Theorem}[section]
\newtheorem*{theorem*}{Theorem}
\newtheorem{lemma}[theorem]{Lemma}
\newtheorem{definition}[theorem]{Definition}

\newtheorem{corollary}[theorem]{Corollary}
\newtheorem{assumption}[theorem]{Assumption}
\newtheorem{remark}[theorem]{Remark}

\newcommand{\eps}{\varepsilon}

\newcommand{\N}{\mathcal N}
\newcommand{\E}{\mathbb E}
\newcommand{\EE}{\mathcal E}
\newcommand{\R}{\mathbb R}

\renewcommand{\P}{\mathbb P}
\newcommand{\PP}{\mathcal P}

\newcommand{\frob}{\mathrm{F}}

\newcommand{\orth}{\mathsf{orth}}

\newcommand{\grad}{\nabla}
\newcommand{\norm}[1]{\left \| #1 \right \|}
\newcommand{\inprod}[2]{\left\langle #1, #2 \right\rangle}

\newcommand{\var}{\mathsf{Var}}

\newcommand{\sphere}{\mathbb{S}}
\newcommand{\ot}{\otimes}

\DeclareMathOperator*{\argmin}{arg\,min}
\renewcommand{\(}{\left(}
\renewcommand{\)}{\right)}
\renewcommand{\]}{\right ]}
\newcommand{\diag}{\mathsf{diag}}
\renewcommand{\[}{\left [}

\newcommand{\m}{\left( m \right)}

\renewcommand{\vec}[1]{\underline{#1}}

\newcommand{\lammax}{\lambda_{\max}}
\newcommand{\lammin}{\lambda_{\min}}
\renewcommand{\vector}{\mathsf{vec} }
\newcommand{\seq}{\mathsf{seq}}

\newcommand{\clean}{\mathsf{clean} }

\newcommand{\offdiag}{\mathsf{off}\text{-}\mathsf{diag}}

\newcommand{\loo}{\mathsf{loo}}

\newcommand{\op}{\mathsf{op}}
\newcommand{\proj}{\mathsf{proj}}
\newcommand{\local}{\mathsf{local}}
\newcommand{\mode}{\mathsf{mode}}

\usepackage{makecell}
\usepackage[outdir=./]{epstopdf}

\DeclareMathOperator{\ind}{\mathds{1}}  % Indicator

\definecolor{yxc}{RGB}{255,0,0}
\definecolor{cxc}{RGB}{0, 160, 0}

\allowdisplaybreaks

\usepackage{arydshln,enumitem,letltxmacro}
\setlist[itemize]{leftmargin=1.5em}
\setlist[enumerate]{leftmargin=1.5em}
\usepackage{booktabs}

\usepackage[noend]{algpseudocode}
\usepackage{algorithmicx,algorithm}

\title{Nonconvex Low-Rank Tensor Completion \\ from Noisy Data\footnotetext{Corresponding
author: Yuxin Chen. A short version of this work has appeared in NeurIPS 2019 \cite{cai2019nonconvex}.}}

\author{Changxiao Cai\thanks{Department of Electrical Engineering, Princeton University, Princeton,
NJ 08544, USA; Email: \texttt{\{ccai,poor,yuxin.chen\}@princeton.edu}.}
	\and Gen Li\thanks{Department of Electronic Engineering, Tsinghua University, Beijing,
10084, China; Email: \texttt{g-li16@mails.tsinghua.edu.cn}.}
	\and H.~Vincent Poor\footnotemark[1]
	\and Yuxin Chen\footnotemark[1] }

\begin{document}

\maketitle

\begin{abstract}
We study a noisy tensor completion problem of broad practical interest, namely, the reconstruction of a low-rank tensor from highly incomplete and randomly corrupted observations of its entries.  While a variety of prior work has been dedicated to this problem,   prior algorithms either are computationally too expensive for large-scale applications, or come with sub-optimal statistical guarantees.  Focusing on  ``incoherent'' and well-conditioned tensors of a constant CP rank, we propose a two-stage nonconvex algorithm --- (vanilla) gradient descent following a rough initialization --- that  achieves the best of both worlds. Specifically, the proposed nonconvex algorithm faithfully completes the tensor and retrieves all individual tensor factors within nearly linear time, while at the same time enjoying near-optimal statistical guarantees (i.e.~minimal sample complexity and optimal estimation accuracy).  The estimation errors are evenly spread out across all entries, thus achieving  optimal $\ell_{\infty}$ statistical accuracy. We have also discussed how to extend our approach to accommodate asymmetric tensors.
The insight conveyed through our analysis of nonconvex optimization might have implications for other tensor estimation problems.

\end{abstract}


\tableofcontents

\section{Introduction and motivation}

%\cxc{AE said that there too many subsections in the introduction}
%\yxc{Let's split it into multiple sections then.}

\subsection{Tensor completion from noisy entries}

Estimation of low-complexity models from highly incomplete observations is a fundamental task that spans a diverse array of science and engineering applications. Arguably one of the most extensively studied problems of this kind is matrix completion,  where one wishes to recover a low-rank matrix given only partial entries \cite{davenport2016overview,chen2018harnessing}.  
Moving beyond matrix-type data, a natural higher-order generalization  is  {\em low-rank tensor completion}, which aims to reconstruct a low-rank tensor when the vast majority of its entries are unseen.  
There is certainly no shortage of applications  that motivate the investigation of tensor completion (e.g.~personalized medicine \cite{soroushmehr2016transforming,pawlowski2019machine}, medical imaging \cite{gandy2011tensor,semerci2014tensor,cheng2017comprehensive}, seismic data analysis \cite{kreimer2013tensor,ely20135d}, multi-dimensional harmonic retrieval \cite{chen2014robust,ying2017hankel}). One concrete example in operations research arises when  
learning the preference of individual customers for a collection of products on the basis of historical transactions \cite{farias2019learning,mivsic2020data}. Given the limited availability of transaction data (e.g.~each customer might only have purchased very few products before), it is crucial to exploit multi-way customer-product interactions (e.g.~users' browsing and searching histories) in order to better predict the likelihood of a customer purchasing a new product. Clearly, the presence of missing data and the need of exploiting multi-way structure result in the task of tensor completion. Additionally, tensor completion finds important applications in visual data in-painting \cite{liu2013tensor,li2017low}, where one wishes to reconstruct video data (or a sequence of images) from incomplete measurements. The video data consist of at least two spatial variables and one temporal variable, whose intrinsic connections are often modeled via certain low-complexity tensors.

For the sake of clarity, we phrase the problem formally  before we proceed, focusing on a simple model that already captures the intrinsic difficulty of tensor completion in many aspects.\footnote{We focus on symmetric order-3 tensors primarily for simplicity of presentation. Many of our  findings naturally extend to the more general case with asymmetric tensors of possibly higher order. Detailed discussions are deferred to Appendix~\ref{sec:asymmetry} due to the space limits. 
}  Imagine we are asked to estimate a  symmetric order-three tensor\footnote{Here, a tensor $\bm{T}\in \mathbb{R}^{d\times d\times d}$ is said to be symmetric if $T_{j, k, l}=T_{k, j, l}=T_{k, l, j}=T_{l, k, j}=T_{j, l, k}=T_{l, j, k}$ for all $1\leq j,k,l\leq d$.} $\bm{T}^{\star} \in \mathbb{R}^{d\times d\times d}$  from a small number of noisy entries 
%that is, we observe a few (possibly corrupted) entries from an index subset $\Omega \subseteq \{1,\cdots,d\}^3$: 
%
\begin{align}
	\label{eq:observation-model}
	T_{j,k,l} = T_{j,k,l}^{\star} + E_{j,k,l}, \qquad \forall (j,k,l)\in \Omega,
\end{align}
where $T_{j,k,l}$ is the observed noisy entry at location $(j,k,l)$,  $E_{j,k,l}$ stands for the associated noise, and $\Omega \subseteq \{1,\cdots,d\}^3$ is a symmetric index subset to sample from. For notational simplicity, we  set $\bm{T}= [T_{j,k,l}]_{1\leq j,k,l\leq d}$ and $\bm{E}= [E_{j,k,l}]_{1\leq j,k,l\leq d}$, with $T_{j,k,l}=E_{j,k,l}=0$ for any $(j,k,l)\notin \Omega$. 
We adopt a  {\em random sampling} model such that  each index $(j,k,l)$ ($j\leq k\leq l$) is included in $\Omega$ independently with probability  $p$. 
%
%\begin{align}
%	\text{each index is included in } \Omega ~ \text{independently with probability }p. 
%\end{align}
%
In addition,  we know {\em a priori} that the unknown tensor $\bm{T}^{\star} \in  \mathbb{R}^{d\times d\times d}$ is  a superposition of $r$ rank-one tensors (often termed canonical polyadic (CP) decomposition if $r$ is minimal)
\begin{align}
	\label{eq:true-tensor}
	\bm{T}^{\star} = \sum_{i=1}^r \bm{u}_i^{\star} \ot \bm{u}_i^{\star} \ot \bm{u}_i^{\star}, \qquad \text{or more concisely,} \qquad \bm{T}^{\star} = \sum_{i=1}^r \bm{u}_i^{\star\, \ot 3},    
\end{align}
where each $\bm{u}_i^{\star} \in \mathbb{R}^d$ represents one of the $r$ low-rank tensor components\,/\,factors. 
Here and throughout, for any vectors $\bm{a},\bm{b},\bm{c}\in \mathbb{R}^d$, the tensor $\bm{a}\ot \bm{b} \ot \bm{c}$ is a $d\times d\times d$ array whose $(j,k,l)$-th entry is given by $a_j b_k c_l$.   
The primary question is this: can we hope to faithfully estimate $\bm{T}^{\star}$, as well as the individual tensor factors $\{\bm{u}_i^{\star} \}_{1\leq i\leq r}$, from the partially revealed entries \eqref{eq:observation-model}, assuming that  
 $r$  is reasonably small?

\subsection{Computational and statistical challenges}

Even though tensor completion conceptually resembles matrix completion in various ways, it is considerably more challenging than the matrix counterpart. 
This is perhaps not surprising,  given that a plethora of natural tensor problems (e.g.~computing the spectral norm, finding the best low-rank approximation) are all notoriously hard \cite{hillar2013most}.  
As a notable example, while matrix completion is often  efficiently solvable under nearly minimal sample complexity \cite{ExactMC09,Gross2011recovering},  all polynomial-time algorithms developed so far for tensor completion --- even in the noise-free case --- require a sample size at least exceeding the order of $rd^{3/2}$, which is substantially larger than the degrees of freedom (i.e.~$rd$) underlying the model \eqref{eq:true-tensor}. In fact, it is widely conjectured that there exists a large computational barrier away from the information-theoretic sampling limits \cite{barak2016noisy}.

With this fundamental gap in mind, the current paper  focuses on the regime (in terms of the sample size) that enables reliable tensor completion in polynomial time. A variety of algorithms have been proposed that enjoy some sort of theoretical guarantees in (at least part of) this regime, including but not limited to spectral methods \cite{montanari2018spectral,cai2019subspace}, sum-of-squares hierarchy \cite{barak2016noisy,potechin2017exact}, nonconvex algorithms \cite{jain2014provable,xia2017polynomial},  
and also convex relaxation (based on proper unfolding) \cite{gandy2011tensor,huang2015provable,romera2013new,goldfarb2014robust}. While these are all polynomial-time algorithms, most of the computational complexities supported by prior theory  remain prohibitively high  when dealing with large-scale tensor data --- a point that we shall elaborate on later. The only exception is the unfolding-based spectral method, which, however, fails to achieve exact recovery as the noise vanishes. This leads to a critical question:  
{
\setlist{rightmargin=\leftmargin}
\begin{itemize}
	\item[] {\em {\bf Q1:} Is there any linear-time algorithm that is guaranteed to work for low-rank tensor completion? }
\end{itemize}
}

Going beyond such computational concerns,  one might naturally wonder whether it is also possible for a fast algorithm to achieve a nearly un-improvable statistical accuracy  in the presence of noise. Towards this end, intriguing stability guarantees have been established for sum-of-squares hierarchy  in the noisy settings \cite{barak2016noisy}, although this paradigm is computationally expensive for large-scale data. The recent work  \cite{xia2017statistically} came up with a two-stage algorithm (i.e.~a spectral method followed by tensor power iterations) for noisy tensor completion. Its estimation accuracy, however, falls short of achieving exact recovery in the absence of noise.  This gives rise to another  question of fundamental importance: 
{
\setlist{rightmargin=\leftmargin}
\begin{itemize}
	\item[] {\em {\bf Q2:} Can we achieve near-optimal statistical accuracy without compromising computational efficiency? }
\end{itemize}
}
\noindent In this paper, we aim to address the above two questions by developing a nonconvex algorithm that achieves optimal computational efficiency and statistical accuracy all at once.

\section{Algorithm and main results}

\subsection{A two-stage nonconvex algorithm}

To address the above-mentioned challenges, a first impulse  is to resort to the following least squares problem:
\begin{align}
	\underset{\bm{u}_{1,}\cdots,\bm{u}_{r}\in\mathbb{R}^{d}}{\text{minimize}}\quad\sum_{j,k,l\in\Omega}\Big(\left[\sum\nolimits _{i=1}^{r} \bm{u}_{i}^{\otimes3}\right]_{j,k,l}-T_{j,k,l} \Big)^{2},
\end{align}
or more concisely (up to proper re-scaling), 
\begin{align}
	\label{eq:loss-function}
	\underset{\bm{U}\in \mathbb{R}^{d\times r}}{\text{minimize}}\quad  f ( \bm{U} ) := \frac{1}{6p} \, \Big\| \PP_{\Omega} \Big( \sum\nolimits_{i=1}^r \bm{u}_i^{\ot 3} - \bm{T} \Big) \Big\|^2_{\frob}
\end{align}
if we take $\bm{U} := \[\bm{u}_1, \dots, \bm{u}_r \] \in \R^{d \times r}$. Here,  we denote by $\PP_{\Omega}(\bm{T})$  the orthogonal projection of any tensor $\bm{T}$ onto the subspace  of tensors which vanish outside of the index set $\Omega$. 
This optimization problem, however, is  highly nonconvex (which involves minimizing a degree-6 polynomial), thus resulting in computational intractability in general.

Fortunately, not all nonconvex problems are as daunting to solve as they may seem. For example, recent years have seen a flurry of activity in low-rank matrix factorization via nonconvex optimization,  which provably achieves optimal statistical accuracy and computational efficiency at once; see \cite{chi2018nonconvex} for an overview of recent advances. Motivated by this strand of work, we propose to solve  \eqref{eq:loss-function} via a two-stage nonconvex paradigm, presented below in reverse order. The whole procedure is summarized in Algorithms~\ref{alg:gd}-\ref{alg:localization}.

\paragraph{Gradient descent (GD).} 
Arguably one of the simplest optimization algorithms is gradient descent, which adopts a gradient update rule 
\begin{align}
	\bm{U}^{t+1} = \bm{U}^t-\eta_t \grad f(\bm{U}^t), \qquad t=0,1,\cdots
\end{align}
where $\eta_t$ is the learning rate or the stepsize, and $\bm{U}^t\in \mathbb{R}^{d\times r}$ is the estimate in the $t$-th iteration. The main computational burden in each iteration lies in  gradient evaluation, which, in this case, can be performed in time proportional to that taken to read the data. 

Despite the simplicity of this algorithm, two critical issues stand out and might significantly affect its  efficiency, which we shall bear in mind throughout the algorithmic and theoretical development.  
%
%\begin{itemize}
%	\item 
		
\bigskip
		\noindent {\em (i) Local stationary points and initialization.}  As is well known, GD is guaranteed to find an approximate local stationary point, provided that the learning rates do not exceed  the inverse Lipschitz constant of the gradient \cite{bubeck2015convex}. There exist, however, local stationary points (e.g.~saddle points or spurious local minima) that might fall short of the desired statistical properties.  This requires us to properly avoid  such undesired points, while retaining computational efficiency.  To address this issue, one strategy is to first identify a rough initial guess within a local region surrounding the global solution (which often helps rule out bad local minima), in order to guarantee proper convergence of subsequent optimization procedures \cite{li2017convex,jain2014provable}.   As a side remark, while careful initialization might not be crucial for several matrix recovery cases \cite{chen2018gradient,gilboa2018efficient,tan2019online}, it does seem to be critical in various tensor problems \cite{richard2014statistical}. We shall elucidate this point in Section \ref{sec:alternatives}.   
		
%	\item 
		
\bigskip
\noindent {\em (ii) Learning rates and regularization.}  Learning rates play a pivotal role in determining the convergence properties of  GD.  The challenge, however, is that the loss function \eqref{eq:loss-function} is overall not sufficiently smooth (i.e.~its gradient often has an exceedingly large Lipschitz constant), and hence generic optimization theory recommends a pessimistically slow update rule (i.e.~an extremely small learning rate) so as to  guard against over-shooting. This, however,  slows down the algorithm significantly, thus destroying the main computational advantage of GD (i.e.~low per-iteration cost). With this issue in mind, prior literature suggests carefully designed regularization steps (e.g.~proper projection, regularized loss functions) in order to improve the geometry of the optimization landscape \cite{xia2017polynomial}. In contrast, we  argue that one is allowed to take a constant learning rate --- which is as aggressive as it can possibly be --- even without enforcing any  regularization procedures.

\begin{algorithm}[t]
\caption{Gradient descent for nonconvex tensor completion}
\label{alg:gd}
\begin{algorithmic}[1]
	\State Generate an initial estimate $\bm{U}^0 \in \mathbb{R}^{d\times r}$ via Algorithm \ref{alg:init}. 
  \For{$t = 0, 1, \dots, t_0-1$}
      \State $\bm{U}^{t+1} = \bm{U}^t-\eta_t \grad f(\bm{U}^t)
	= \bm{U}^t- \frac{\eta_t}{p} \mathcal{P}_{\Omega}\big( \sum\nolimits _{i=1}^{r} \big( \bm{u}_{i}^{t} \big)  ^{\otimes3} -\bm{T}\big) \times_1^\seq \bm{U}^t \times_2^\seq \bm{U}^t$,
      where $\times_1^\seq$ and $\times_2^\seq$ are defined in Section~\ref{sec:Notation-1}. 
   \EndFor
\end{algorithmic}
\end{algorithm}

\paragraph{Initialization.} 

Motivated by the above-mentioned issue (i), we develop a procedure that guarantees a reasonable initial estimate. In a nutshell, the proposed procedure consists of two steps: 
\begin{itemize}
	\item[(a)] Estimate the subspace spanned by the $r$ low-rank tensor factors $\{\bm{u}_i^{\star}\}_{1\leq i\leq r}$ via a spectral method;
	\item[(b)] Disentangle  individual  low-rank tensor factors from this subspace estimate.   
\end{itemize}
		As we shall see momentarily, the total computational complexity of the proposed initialization is $O(pd^3 )$ when $r=O(1)$, $\kappa = O(1)$ and $p \geq 1/d^2$ (where $\kappa$ is a sort of ``condition number'' defined later), which is a linear-time algorithm. 
Note, however, that these two steps in the initialization procedure are relatively more complicated to describe. To improve the flow of the current paper, we postpone the details  to Section \ref{sec:algorithm-main}. The readers can catch a glimpse of these  procedures in Algorithms~\ref{alg:init}-\ref{alg:localization}.

\begin{algorithm}[t]
\caption{Spectral initialization for nonconvex tensor completion}
\label{alg:init}
\begin{algorithmic}[1]
  \State Let $\bm{U} \bm{\Lambda} \bm{U}^\top$ be the rank-$r$ eigen-decomposition of
  \begin{align}
  \label{B_alg}
	  \bm{B} := \PP_{\mathsf{off}\text{-}\mathsf{diag}} (\bm{A} \bm{A}^\top) ,
  \end{align} 
  where $\bm{A}=\mathsf{unfold} \big( p^{-1} \bm{T} \big)$ is the mode-1 matricization of $p^{-1} \bm{T}$, and $\PP_{\mathsf{off}\text{-}\mathsf{diag}}(\bm{Z})$ extracts out the off-diagonal entries of  $\bm{Z}$. 
  \State {\bf Output:} an initial estimate $\bm{U}^0 \in \mathbb{R}^{d\times r}$ on the basis of $\bm{U}\in \mathbb{R}^{d\times r}$ using Algorithm~\ref{alg:localization}. 
\end{algorithmic}
\end{algorithm}

%\begin{algorithm}[t]
%\caption{Retrieval of tensor components from a given subspace.}
%\label{alg:localization}
%\begin{algorithmic}[1]
%  \Require{sampling  set $\Omega$, observed entries $\{ T_{i, j, k} \mid (i, j, k) \in \Omega \}$,  sampling rate $p$, subspace estimate $\bm{U}\in \mathbb{R}^{d\times r}$.}
%  \For{$\tau = 1, \dots, r$}
%   	\State{Generate an independent Gaussian vector $\bm{g}^\tau \sim \N(0, \bm{I}_d)$, and compute 
%	\begin{subequations}
%	\begin{align}
%		\bm{W}^\tau &= \big[ \bm{w}^1, \dots, \bm{w}^{\tau-1} \big] \qquad &(\text{and } \bm{W}^1=\bm{0}), \\
%		\bm{\theta}^\tau &= \mathcal{P}_{\bm{U}}\mathcal{P}_{\bm{W}_{\perp}^{\tau}}(\bm{g}^{\tau}), \\
%		\bm{M}^\tau &= p^{-1} \bm{T} \times_3 \bm{\theta}^\tau, \label{eq:defn-Mtau}
%	\end{align} 
%	\end{subequations}}
%	where the projection operator is defined in \eqref{eq:defn-PW}, and $\times_3$ is defined in Section \ref{sec:Notation-1}. 
%	\State Let $\bm{w}^\tau$ be the leading left singular vector of $\bm{M}^\tau$ obeying $\langle \bm{T}, \bm{w}^{\tau \ot 3} \rangle  \geq 0$, and 
%$\gamma_\tau$  the associated singular value. Set $\lambda_\tau = \big| \gamma_\tau / \langle \bm{w}^\tau, \bm{\theta}^\tau \rangle \big|^{1/3}$. 
%	%Here, $\tsym$ is defined in Section \ref{sec:notation}. 
%  \EndFor
%	\Ensure{$\bm{U}^0 = \big[ \lambda_1 \bm{w}^1, \dots, \lambda_r \bm{w}^r \big].$}
%\end{algorithmic}
%\end{algorithm}

\begin{algorithm}[t]
\caption{Retrieval of low-rank tensor factors from a given subspace estimate.}
\label{alg:localization}
\begin{algorithmic}[1]
%\State {\bf Input:} sampling  set $\Omega$, observed entries $\{ T_{i, j, k} \mid (i, j, k) \in \Omega \}$,  sampling rate $p$, number of restarts $L$, pruning threshold $\epsilon_{\mathsf{th}}$, subspace estimate $\bm{U}\in \mathbb{R}^{d\times r}$.
\State {\bf Input:} number of restarts $L$, pruning threshold $\epsilon_{\mathsf{th}}$, subspace estimate $\bm{U}\in \mathbb{R}^{d\times r}$ given by Algorithm~\ref{alg:init}.
  \For{$\tau = 1, \dots, L$}
%  \State{Generate tensor factor estimate tuple $ \big(\bm{\nu}^\tau, \lambda_\tau, \sigma_1 (\bm{M}^\tau) - \sigma_2 (\bm{M}^\tau) \big)$ using Function~\ref{alg:one-factor}.}
\State Generate an independent Gaussian vector $\bm{g}^\tau \sim \N(0, \bm{I}_d)$.
  \State 
	$\big(\bm{\nu}^\tau, \lambda_\tau, \mathsf{spec}\text{-}\mathsf{gap}_{\tau} \big) \gets \Call{Retrieve-one-tensor-factor}{\bm{T}, p, \bm{U}, \bm{g}^\tau}$.
  %Generate estimate tuple 
  %\begin{align*}
   
   %\end{align*}
    \EndFor
	\State Generate tensor factor estimates $\big\{ (\bm{w}^1, \lambda_1), \dots, (\bm{w}^r,\lambda_r)  \big\} \gets \Call{Prune}{\big\{ \big( \bm{\nu}^{\tau}, \lambda_\tau, \mathsf{spec}\text{-}\mathsf{gap}_{\tau}  \big) \big\}_{\tau = 1}^L, \epsilon_{\mathsf{th}}}$.
    %\begin{align*}	
    %
    %\end{align*}
%    \For{$i = 1, \dots, r$}
%    \State{Let $\tau(i) := \argmax_{1 \leq \tau \leq L} \big\{ \sigma_1 (\bm{M}^\tau) - \sigma_2 (\bm{M}^\tau) \big\}$, $\bm{w}^i = \bm{\nu}^{\tau(i)}$ and $\lambda_i = \langle p^{-1} \bm{T}, (\bm{w}^i)^{\ot 3} \rangle$.}
%    \State{Remove $\bm{\nu}^\tau$ if $ \left| \langle \bm{w}^i, \bm{\nu}^\tau \rangle \right| > 1- \epsilon_{\mathsf{th}}$.}
%    \EndFor
%\State Generate tensor factor estimates $\big\{ \lambda_1 \bm{w}^1, \dots, \lambda_r \bm{w}^r \big\}$ from the set of estimate tuples $\big\{ \big( \bm{\nu}^{\tau}, \lambda_\tau, \sigma_1 (\bm{M}^\tau) - \sigma_2 (\bm{M}^\tau) \big) \big\}_{\tau = 1}^L$ using Function~\ref{alg:prune}.
\State {\bf Output:} initial estimate $\bm{U}^0 = \big[ \lambda_1^{1/3} \bm{w}^1, \dots, \lambda_r^{1/3} \bm{w}^r \big].$
\end{algorithmic}
\end{algorithm}

\begin{algorithm}[htb]
%\floatname{algorithm}{Function}
%\caption{Retrieval of one tensor factor}
\label{alg:one-factor}
\begin{algorithmic}[1]
%\Require{sampling  set $\Omega$, observed entries $\{ T_{i, j, k} \mid (i, j, k) \in \Omega \}$, subspace estimate $\bm{U}\in \mathbb{R}^{d\times r}$.}
\Function{Retrieve-one-tensor-factor}{$\bm{T}, p, \bm{U}, \bm{g}$}
%\State {\bf Input:} sampling  set $\Omega$, observed entries $\{ T_{i, j, k} \mid (i, j, k) \in \Omega \}$, subspace estimate $\bm{U}\in \mathbb{R}^{d\times r}$.
\State Compute 
	\begin{subequations}
	\begin{align}
		\bm{\theta} &= \bm{U} \bm{U}^{\top} \bm{g}  =:\mathcal{P}_{\bm{U}} (\bm{g} ), \\
		\bm{M} &= p^{-1} \bm{T} \times_3 \bm{\theta}, \label{eq:defn-Mtau}
	\end{align} 
	\end{subequations} where 
	%the projection operator is defined in \eqref{eq:defn-PW}, and 
	$\times_3$ is defined in Section~\ref{sec:Notation-1}.
	\State Let $\bm{\nu}$ be the leading singular vector of $\bm{M}$ obeying $\langle \bm{T},  \bm{\nu} ^{\ot 3} \rangle  \geq 0$, and set $\lambda = \langle p^{-1} \bm{T}, \bm{\nu}^{\ot 3} \rangle$.
% \Ensure{tensor factor estimate tuple $\big(\bm{\nu}, \lambda, \sigma_1 (\bm{M}) - \sigma_2 (\bm{M}) \big)$.}
\State \Return $\big(\bm{\nu}, \lambda, \sigma_1 (\bm{M}) - \sigma_2 (\bm{M}) \big)$.
\EndFunction
%\State {\bf Output:} tensor factor estimate tuple $\big(\bm{\nu}, \lambda, \sigma_1 (\bm{M}) - \sigma_2 (\bm{M}) \big)$.
\end{algorithmic}
\end{algorithm}

\begin{algorithm}[htb]
%\floatname{algorithm}{Function}
%\caption{Pruning}
\label{alg:prune}
\begin{algorithmic}[1]
%\State {\bf Input:} tensor factor estimate tuples $\Theta = \big\{ \big( \bm{\nu}^{\tau}, \lambda_\tau, \sigma_1 (\bm{M}^\tau) - \sigma_2 (\bm{M}^\tau) \big) \big\}_{\tau = 1}^L$, threshold $\epsilon_{\mathsf{th}}$.
\Function{Prune}{$\big\{ \big( \bm{\nu}^{\tau}, \lambda_\tau, \mathsf{spec}\text{-}\mathsf{gap}_{\tau} \big) \big\}_{\tau = 1}^L, \epsilon_{\mathsf{th}}$}
	\State Set $\Theta = \big\{ \big( \bm{\nu}^{\tau}, \lambda_\tau, \mathsf{spec}\text{-}\mathsf{gap}_{\tau} \big) \big\}_{\tau = 1}^L.$
\For{$i = 1, \dots, r$}
	\State Choose $(\bm{\nu}^{\tau}, \lambda_{\tau}, \mathsf{spec}\text{-}\mathsf{gap}_{\tau} )$ from $\Theta$ with the largest $\mathsf{spec}\text{-}\mathsf{gap}_{\tau}$; set $\bm{w}^i = \bm{\nu}^{\tau}$ and $\lambda_i = \lambda_\tau$.
    \State Update $\Theta \gets \Theta \setminus \left\{ \big( \bm{\nu}^{\tau}, \lambda_\tau, \mathsf{spec}\text{-}\mathsf{gap}_{\tau} \big) \in \Theta : | \langle \bm{\nu}^\tau , \bm{w}^i \rangle | > 1- \epsilon_{\mathsf{th}} \right\}$.
 \EndFor
	\State \Return $\big\{ (\bm{w}^1, \lambda_1),  \dots,  (\bm{w}^r, \lambda_r) \big\}.$
  \EndFunction
% \State {\bf Output:} tensor factor estimates $\big\{ \lambda_1 \bm{w}^1, \dots, \lambda_r \bm{w}^r \big\}.$
\end{algorithmic}
\end{algorithm}

\subsection{Main results}

Encouragingly, the proposed nonconvex algorithm  provably achieves the best of both worlds --- in terms of statistical accuracy and computational efficiency --- for a class of low-rank, well-conditioned, and ``incoherent'' problem instances. This subsection summarizes our main findings.  

Before continuing, we note that one cannot hope to recover an arbitrary tensor from highly sub-sampled and arbitrarily corrupted entries.  In order to enable provably valid recovery, the present paper focuses on a tractable model by imposing  the following assumptions.

%\paragraph{.} 
%
\begin{definition}[Incohrence and well-conditionedness]
\label{asmp_U}
	% The low-rank tensor factors $\{\bm{u}^{\star}_i\}_{1\leq i\leq r}$ satisfy
	Define the incoherence parameters and the condition number of $\bm{T}^{\star}$ as follows 
\begin{subequations}
\begin{align}
%\label{asmp_unit_norm} &(\asmp{1} ) \qquad \norm{\bm{u}_i^*}_2 = 1,  & & 1 \leq i \leq d ; \\
% \label{asmp_max} &( \asmp{1} ) \qquad \norm{\bm{T}^{\star}}_{\infty}  \leq \sqrt{\frac{\mu_0}{d^{3}}}  \norm{\bm{T}^{\star}}_{\mathrm{F}}; \\
% \label{asmp_incoh} &( \asmp{2} ) \qquad \norm{\bm{u}_i^{\star}}_{\infty}  \leq \sqrt{\frac{\mu_1}{d}}  \norm{\bm{u}_i^{\star}}_2, & & 1 \leq i \leq d; \\
% \label{asmp_corr} &( \asmp{3} ) \qquad \left| \left \langle \bm{u}_i^{\star}, \bm{u}_j^{\star} \right \rangle \right|  \leq \sqrt{\frac{\mu_2}{d}} \norm{\bm{u}_i^{\star}}_2 \norm{\bm{u}_j^{\star}}_2,  & & 1 \leq i \neq j \leq d; \\
% 	\label{asmp_coef} &( \asmp{4} ) \qquad \kappa := \frac{ \max_{i} \left\| \bm{u}_i^{\star} \right\|_2^3 }{\min_{i} \left\| \bm{u}_i^{\star} \right\|_2^3 } = O(1).
\mu_{0} & :=\frac{d^3 \, \|\bm{T}^{\star}\|_{\infty}^{2}}{\|\bm{T}^{\star}\|_{\mathrm{F}}^{2}},\label{asmp_max}\\
\mu_{1} & := \frac{d\|\bm{u}_{i}^{\star}\|_{\infty}^{2}}{ \, \|\bm{u}_{i}^{\star}\|_{2}^{2}},\label{asmp_incoh}\\
\mu_{2} & := \frac{d\langle\bm{u}_{i}^{\star},\bm{u}_{j}^{\star}\rangle^{2}}{ \, \|\bm{u}_{i}^{\star}\|_{2}^{2}\|\bm{u}_{j}^{\star}\|_{2}^{2}} ,\label{asmp_corr}\\
\kappa & :=\frac{\max_{i} \|\bm{u}_{i}^{\star}\|_{2} }{\min_{i} \|\bm{u}_{i}^{\star}\|_{2} }.\label{asmp_coef}
%\label{asmp_rank} &(\asmp{5} ) \qquad r  = O \( d^{1/2}/\polylog \(d\)\).
\end{align}
\end{subequations}
\end{definition}
%
% \begin{remark} Here, $\mu_0$, $\mu_1$ and $\mu_2$ are termed the incoherence parameters.
% % both of which are allowed to depend on $d$ and $r$. 
% Assumptions $\asmp{1}$, $\asmp{2}$ and $\asmp{3}$ can be viewed as some sort of incoherence conditions for the tensor. For instance, when $\mu_0, \mu_1$ and $\mu_2$ are small, these conditions say that  (1) the energy of tensor $\bm{T}^\star$ is (nearly) evenly spread across all entries; (2) each factor $\bm{u}_i^{\star}$ is de-localized; (3) the factors $\{\bm{u}_i^{\star}\}$ are nearly orthogonal to each other.   Assumption $\asmp{4}$ is concerned with the ``well-conditionedness'' of the tensor,  meaning that each rank-1 component is of roughly the same size. \end{remark}
\begin{remark} Here, $\mu_0$, $\mu_1$ and $\mu_2$ are termed the incoherence parameters.
%
% both of which are allowed to depend on $d$ and $r$.
%
	Definitions~\eqref{asmp_max}-\eqref{asmp_corr} can be viewed as some sort of incoherence conditions for the tensor. For instance, when $\mu_0, \mu_1$ and $\mu_2$ are small, these conditions say that  (1) the energy of tensor $\bm{T}^\star$ is (nearly) evenly spread across all entries; (2) each factor $\bm{u}_i^{\star}$ is de-localized; (3) the factors $\{\bm{u}_i^{\star}\}$ are nearly orthogonal to each other.   Definition~\eqref{asmp_coef} is concerned with the ``well-conditionedness'' of the tensor,  meaning that each rank-1 component is of roughly the same size. In particular, we note that an assumption on pairwise correlation (i.e.~a constraint on $\mu_2$) is often assumed in the literature of tensor decomposition\,/\,factorization (e.g.~\cite{anandkumar2014guaranteed,sun2017provable,hao2018sparse}).

\end{remark}
For notational simplicity, we shall set
\begin{align}
\mu := \max \left\{ \mu_0, \, \mu_1, \, \mu_2 \right\}. \label{def:mu}
\end{align}
Note that our theory allows $\mu$ to grow with the problem dimension $d$ (in fact, $\mu$ can be as large as $d/\mathrm{poly}\log(d)$).  

%\begin{remark}
%	Throughout this paper,  $f(n) \asymp g(n)$ means that there exist absolute constants $c_1,c_2>0$ such that $c_1 g(n) \leq f(n) \leq c_2 g(n)$ for all integer $n>0$.    
%\end{remark}
%
\begin{assumption}[Random noise] 
\label{asmp:random-noise}
	Suppose that  $\bm{E}$ is a symmetric random tensor, where $\{E_{j,k,l}\}_{1\leq j\leq k\leq l \leq d}$ (cf.~\eqref{eq:observation-model}) are independently generated sub-Gaussian random variables with mean zero and variance $\mathsf{Var}(E_{j,k,l}) \leq \sigma^2$. 
\end{assumption}

In addition, recognizing that there is a global permutational ambiguity
issue (namely, one cannot distinguish $\bm{u}_1^\star,\cdots,\bm{u}_r^\star$ from an arbitrary permutation of them), we introduce the following loss metrics to account for this ambiguity:
\begin{subequations}
\begin{align}
\mathsf{dist}_{\mathrm{F}}(\bm{U},\bm{U}^{\star}) & := \min_{\bm{\Pi}\in\mathsf{perm}_r}\|\bm{U}\bm{\Pi}-\bm{U}^{\star}\|_{\mathrm{F}},\\
\mathsf{dist}_{\infty}(\bm{U},\bm{U}^{\star}) & := \min_{\bm{\Pi}\in\mathsf{perm}_r}\|\bm{U}\bm{\Pi}-\bm{U}^{\star}\|_{\infty},\\
\mathsf{dist}_{2,\infty}(\bm{U},\bm{U}^{\star}) & := \min_{\bm{\Pi}\in\mathsf{perm}_r}\|\bm{U}\bm{\Pi}-\bm{U}^{\star}\|_{2,\infty},
\end{align}
\end{subequations}
where $\mathsf{perm}_r$ stands for the set of $r\times r$ permutation matrices. For notational simplicity, we also take
\begin{align}
	\label{eq:defn-lambda-min-lambda-max}
	\lambda_{\min}^{\star} := \min_{1\leq i\leq r} \left\|\bm{u}_i^{\star} \right\|_2^3 \qquad \text{and} \qquad	
	\lambda_{\max}^{\star} := \max_{1\leq i\leq r} \left\|\bm{u}_i^{\star} \right\|_2^3. 
\end{align}

With these notations in place, we are ready to present our main results. For simplicity of presentation, we shall start with the setting where $r,\mu,\kappa \asymp 1$.  
\begin{theorem}
\label{thm:main-rank1}
%	Suppose the sample size obeys $d^3 p \geq c_1 \mu_1^7 r^{12} d^{3/2} \log^7  d $ for some sufficiently large constant $c_1>0$, the noise satisfies $\sigma \leq c_2 \frac{\lammin^{\star 3}}{\mu_1 r^4 \log d}\sqrt{\frac{p}{ d^{3/2} }} $ and the rank satisfies $r \leq c_3 \Big( \frac{\sqrt{d} \log d}{\mu_1^{3/2} \sqrt{\mu_2}} \Big)^{1/5}$ for some sufficiently small constants $c_2,c_3>0$.
	%\yxc{Need to specify $\epsilon_{\mathsf{th}}$. }
	Fix an arbitrary small constant $\delta > 0$. Suppose that $r, \kappa, \mu =O(1)$, 
	\begin{align*}
		& p \geq c_0 \frac{\log^4 d}{d^{3/2}}, \qquad \frac{\sigma}{\lambda_{\min}^\star} \leq c_1 \frac{\sqrt{p}}{ d^{3/4} \log^{2} d }, \\
		& L = c_2 \quad \text{and} \quad \epsilon_{\mathsf{th}} = c_3 \bigg( \frac{ \log d}{d\sqrt{p}}+\frac{\sigma}{\lambda_{\min}^\star}\sqrt{\frac{d\log^{2}d}{p}}+\sqrt{\frac{ \log d}{d}} \bigg)
	\end{align*}
	for some sufficiently large constants $c_0, c_2 >0$ and some sufficiently small constants $c_1, c_3>0$. 
The learning rate $\eta_t \equiv \eta$ is taken to be a constant obeying $0 < \eta \leq \lammin^{\star 4/3}  / \big (32 \lammax^{\star 8/3} \big)$.  Then with probability at least $1- \delta$,
\begin{subequations}
\begin{align}
\label{claim:thm:U_frob_loss} 
	 \mathsf{dist}_{\mathrm{F}}(\bm{U}^t,\bm{U}^{\star}) &\leq  \Bigg( C_1 \rho^t  + C_2 \frac{\sigma}{\lammin^{\star}} \sqrt{\frac{ d \log d}{p}}   \Bigg) \norm{\bm{U}^\star}_{\frob}, \\
\label{claim:thm:U_2inf_loss} 
	 \mathsf{dist}_{\infty}(\bm{U}^t,\bm{U}^{\star}) &\leq \mathsf{dist}_{2,\infty}(\bm{U}^t,\bm{U}^{\star})  
	  \leq \Bigg( C_3 \rho^t + C_4 \frac{\sigma}{\lammin^{\star}} \sqrt{\frac{ d \log d}{p}}   \Bigg) \norm{\bm{U}^\star}_{2, \infty}
\end{align}
\end{subequations}
hold simultaneously for all $0 \leq t \leq t_0 = d^5$. Here, $0<C_1, C_3, \rho<1$ and $C_2, C_4>0$ are some absolute constants. 
\end{theorem}
\begin{remark}
	The theorem holds unchanged if $d^5$ is replaced by $d^{c}$ for an arbitrarily large constant $c>0$. 
\end{remark}
%
% \cxc{upper bound on the iteration count by reviewer 3}
\begin{remark}
	The upper bound $t_0$ on the iteration count arises from the leave-one-out analysis when handling noisy observations. In short, the leave-one-out argument can only provide high-probability bounds for each iteration, thus requiring an upper bound on the iteration count if we desire a uniform bound across iterations.  Note that in the noiseless case, our results and analysis hold for an arbitrarily large number of iterations.
\end{remark}
As an immediate consequence of Theorem \ref{thm:main-rank1},  we obtain appealing $\ell_{\infty}$ statistical guarantees for estimating tensor entries, which are previously rarely available (see Table \ref{tab:comparison}). Specifically, let our tensor estimate in the $t$-th iteration be
\begin{align}
	\bm{T}^{t}:=\sum_{i=1}^{r}\bm{u}_{i}^{t}\otimes\bm{u}_{i}^{t}\otimes\bm{u}_{i}^{t},\qquad\text{where }\bm{U}^{t}=[\bm{u}_{1}^{t},\cdots,\bm{u}_{r}^{t}]\in\mathbb{R}^{d\times r}.
	\label{defn:T-t}
\end{align}
Then our result is this:
\begin{corollary}
	\label{corollary:tensor-entries-rank1}
	Fix an arbitrarily small constant $\delta>0$. Instate the assumptions of Theorem \ref{thm:main-rank1}. Then with probability at least $1-\delta$, 
	\begin{subequations}
	\begin{align}
		\norm{\bm{T}^t-\bm{T}^\star}_\frob 	&\lesssim  \( C_1 \rho^t  + C_2 \frac{\sigma}{\lammin^{\star}} \sqrt{\frac{ d \log d}{p}}   \) \norm{\bm{T}^\star}_{\frob},   \\
		\norm{\bm{T}^t-\bm{T}^\star}_{\infty} 	&\lesssim \( C_3 \rho^t + C_4 \frac{\sigma}{\lammin^{\star}} \sqrt{\frac{ d \log d}{p}}   \) \norm{\bm{T}^\star}_{\infty}
	\end{align}
	\end{subequations}
hold simultaneously for all $0 \leq t \leq t_0 = d^5$. Here, $0<C_1, C_3, \rho<1$ and $C_2, C_4>0$ are some absolute constants. 
\end{corollary}
%
%In words, this corollary leads to optimal entrywise error control for noisy tensor completion. 

Several important implications are provided as follows.  
%focuses on the regime where $\max \{r, \, \mu \}\asymp 1$, and 
The discussion below assumes $\lambda_{\max}^\star \asymp \lambda_{\min}^\star \asymp 1$ for notational simplicity.  

\begin{enumerate}
	\item {\em Linear convergence.} In the absence of noise, the proposed algorithm converges linearly, namely, it provably attains $\varepsilon$ accuracy within $O(\log (1/\varepsilon))$ iterations. Given  the inexpensiveness of each gradient iteration, this algorithm can be viewed as a linear-time algorithm, which can almost be implemented as long as we can read the data. In the noisy setting, the algorithm  reaches an appealing statistical accuracy within a logarithmic number of iterations. 
		
	\item {\em Near-optimal sample complexity.} The fast convergence is guaranteed as soon as the sample size exceeds the order of $d^{3/2} \mathrm{poly}\log d$. This matches the minimal sample complexity --- modulo some logarithmic factor --- known so far for any polynomial-time algorithm.    

	\item {\em Near-optimal statistical accuracy.} The proposed algorithm converges geometrically fast to a point with Euclidean error 
		$O\big( \sigma \sqrt{ ( d \log d ) /p}  \big)$. This matches the lower bound established in  \cite[Theorem 5]{xia2017statistically} up to some logarithmic factor, thus justifying the statistical optimality of the proposed nonconvex algorithm.

	\item {\em Entrywise estimation accuracy.} In addition to the Euclidean statistical guarantees, we have also established an entrywise error bound, which, to the best of our knowledge, has not been established in any of the prior work. When $t$ is sufficiently large, the iterates reach an entrywise error bound   $O\big( \sigma \sqrt{ (\log d) /p}   \big) $.  This entrywise error bound is about an order of $\sqrt{d}$ times smaller than the above $\ell_2$ error bound, thereby implying that the estimation errors are evenly spread out across all entries.

	\item {\em Noise size.} The above theory operates in the regime where $\sigma \lesssim \sqrt{ \frac{p}{d^{3/2}} }$ (modulo some log factor). Given that we have $\|\bm{T}^{\star}\|_{\infty} \asymp  d^{-3/2}$ in this case, our noise size constraint can be equivalently written as (up to some log factor)
	\begin{align}
		\frac{\sigma}{\|\bm{T}^{\star}\|_{\infty}}\lesssim\sqrt{pd^{3/2}}.
	\end{align}
		Since the sampling rate needs to satisfy $p\gg d^{-3/2}$, this condition essentially allows the typical  size of each noise component to be considerably larger than the size of the corresponding entry of the truth, which covers a broad range of practical scenarios.   

	\item {\em Implicit regularization.} One appealing feature of our finding is the simplicity of the iterative refinement stage of the algorithm.  All of the above statistical and computational benefits hold for  vanilla gradient descent (when properly initialized).  This should be contrasted with prior work (e.g.~\cite{xia2017polynomial}) that relies on extra regularization terms to stabilize the optimization landscape. In principle, vanilla gradient descent implicitly constrains itself within a region of well-conditioned landscape,  thus enabling fast convergence without explicit regularization.

	\item {\em No need of sample splitting.} The theory  developed herein does not require fresh samples in each iteration. We note that sample splitting has been frequently adopted in other context primarily to simplify mathematical analysis. Nevertheless,  it typically does not exploit the data in an efficient manner (i.e.~each data sample is used only once),  thus resulting in the need of a much larger sample size in practice.  
\end{enumerate}

\begin{table}
\centering
    {\small
    \begin{tabular}{ c | c | c | c | c | c | c }
    \hline
	    & algorithm & \makecell[c]{sample \\ complexity} & \makecell[c]{computational \\ complexity} & \begin{tabular}{@{}c@{}} $\ell_2$ error   \\ (noisy) \end{tabular} & \begin{tabular}{@{}c@{}} $\ell_{\infty}$ error  \\ (noisy) \end{tabular} & \begin{tabular}{@{}c@{}} recovery type \\ (noiseless) \end{tabular} \\ \hline
		    \begin{tabular}{@{}c@{}} our \\ theory \end{tabular} & \makecell[c]{spectral method \\ + (vanilla) GD} & $ d^{1.5} $ & $ pd^3 $  & $ \sigma \sqrt{ \frac{d}{ p}} $ & $\sigma \sqrt{ \frac{1}{ p}}$  & exact \\ \hline
		  \cite{xia2017statistically} & \makecell[c]{spectral initialization \\ + tensor power method} & $d^{1.5}$ & $p d^3 $   &  $( \| \bm{T}^\star\|_\infty + \sigma ) \sqrt{ \frac{d}{ p}}$ & n/a  & approximate \\ \hline
	    \cite{xia2017polynomial} & \makecell[c]{spectral method \\ + GD on manifold} & $ d^{1.5}  $  & $\mathrm{poly}(d)$  & n/a & n/a & exact \\ \hline
	% \cite{jain2014provable}    (TPM + alt-min) & $d^3$ & $  p d^3  $ & n/a \\ \hline
	    \cite{montanari2018spectral}   & {spectral method} & $  d^{1.5} $  & $ d^{3}  $ & n/a & n/a & approximate  \\ \hline
	\cite{barak2016noisy}   & sum-of-squares & $  d^{1.5} $  & $ d^{15}  $ & $  \frac{\| \bm{T}^\star\|_{\mathrm{F}}}{\sqrt{p d^{1.5}}} +  \sigma d^{1.5} $ & n/a &  approximate  \\ \hline
	\cite{potechin2017exact}  & sum-of-squares & $d^{1.5}$  & $d^{10}$   &  n/a & n/a &  exact \\ \hline
	    %\cite{shah2019iterative}  & \makecell[c]{collaborative filtering \\  (nearest neighbors)}  & $d^{1.5 + o(1)}$  & $p d^{6}$   &   & n/a & approximate \\ \hline
	    \begin{tabular}{@{}c@{}} \cite{yuan2016tensor} \\ \cite{yuan2017incoherent} \end{tabular} & \makecell[c]{tensor nuclear norm\\ minimization} & $ d  $  & NP-hard & n/a   & n/a & exact \\ \hline
   % convex relaxation + unfolding & $d^2$ & $ d^5 $ &  n/a  \\ \hline
  % Xia, Yuan, Zhang [64] & $ r d^{1.5} \log^5 d + r^2 d \log^{10} d$  & $ O (p^2 d^6 + r d^3 )$ & $ \sigma \ll \sqrt{\frac{p}{rd^{2} \log d}} $ & No \\    \hline
  \end{tabular}
     }
	\caption{Comparison with prior theory for existing methods when $r,\mu,\kappa \asymp 1$ (neglecting logarithmic factors). \label{tab:comparison}}
\end{table}

We shall take a moment to discuss the merits of our approach in comparison to prior work. One of the best-known polynomial-time algorithms is the degree-6 level of the sum-of-squares (SoS) hierarchy, which seems to match the computationally feasible limit in terms of the sample complexity \cite{barak2016noisy}.  However, this approach has a well-documented limitation in that it involves solving a semidefinite program of dimensions $d^3 \times d^3$, which  requires enormous storage and computation power. The work \cite{montanari2018spectral} alleviates  this computational burden by resorting to a clever unfolding-based spectral algorithm; it is a nearly linear-time procedure that enables near-minimal sample complexity (among polynomial-time algorithms), although it does not achieve exact recovery even in the absence of noise. The two-stage algorithm developed by  \cite{xia2017statistically} --- which is based on spectral initialization followed by tensor power methods --- shares similar advantages and drawbacks as \cite{montanari2018spectral}. Further, the recent work \cite{xia2017polynomial} proposes a polynomial-time nonconvex algorithm based on gradient descent over Grassmann manifold (with a properly regularized objective function), which is an extension of the nonconvex matrix completion algorithm proposed by \cite{keshavan2010matrix,Se2010Noisy} to tensor data. The theory provided in \cite{xia2017polynomial}, however, does not provide explicit computational complexities.  The recent work \cite{shah2019iterative} attempts tensor estimation via an interesting algorithm adapted from collaborative filtering and investigates both $\ell_2$ and $\ell_{\infty}$ estimation accuracy. This approach, however, does not guarantee exact recovery in the absence of noise. 
We summarize and compare several prior results in Table \ref{tab:comparison} (omitting logarithmic factors). 

%\yxc{Cite C. Yu; maybe brief discussion + table.}

Thus far, we have concentrated on the low-rank, well-conditioned, and incoherent case. Our main theory can be extended to cover a broader class of scenarios, as stated below. %
\begin{theorem}
\label{thm:main}
%	Suppose the sample size obeys $d^3 p \geq c_1 \mu_1^7 r^{12} d^{3/2} \log^7  d $ for some sufficiently large constant $c_1>0$, the noise satisfies $\sigma \leq c_2 \frac{\lammin^{\star 3}}{\mu_1 r^4 \log d}\sqrt{\frac{p}{ d^{3/2} }} $ and the rank satisfies $r \leq c_3 \Big( \frac{\sqrt{d} \log d}{\mu_1^{3/2} \sqrt{\mu_2}} \Big)^{1/5}$ for some sufficiently small constants $c_2,c_3>0$. 
%\yxc{Need to specify $\epsilon_{\mathsf{th}}$. }
Fix an arbitrary small constant $\delta > 0$. Suppose that $\kappa \asymp 1$, 
	\begin{align*}
	 & p \geq c_0 \frac{\mu^4 r^4 \log^4 d}{d^{3/2}}, \qquad \frac{\sigma}{\lambda_{\min}^\star} \leq c_1 \frac{\sqrt{p}}{\mu r^{3/2} d^{3/4} \log^{2} d}, \quad r \leq c_2 \( \frac{d}{\mu^6 \log^6 d} \)^{1/6}, \\
		& L = c_3 r^{2\kappa^2} \log^{3/2} r \qquad \text{and} \qquad \epsilon_{\mathsf{th}} = c_4 \bigg( \frac{\mu r \log d}{d\sqrt{p}}+\frac{\sigma}{\lambda_{\min}^\star}\sqrt{\frac{rd\log^{2}d}{p}}+\sqrt{\frac{\mu r  \log d}{d}} \bigg)
	\end{align*}
	for some sufficiently large constants $c_0, c_3 >0$ and some sufficiently small constants $c_1,c_2, c_4>0$. 
The learning rate $\eta_t \equiv \eta$ is taken to be a constant obeying $0 < \eta \leq \lammin^{\star 4/3}  / \big (32 \lammax^{\star 8/3} \big)$.  Then with probability at least $1- \delta$,
\begin{subequations}
\begin{align}
\label{claim:thm:U_frob_loss} 
	 \mathsf{dist}_{\mathrm{F}}(\bm{U}^t,\bm{U}^{\star}) &\leq  \( C_1 \rho^t  + C_2 \frac{\sigma}{\lammin^{\star}} \sqrt{\frac{ d \log d}{p}}   \) \norm{\bm{U}^\star}_{\frob} \\
\label{claim:thm:U_2inf_loss} 
	 \mathsf{dist}_{\infty}(\bm{U}^t,\bm{U}^{\star}) &\leq \mathsf{dist}_{2,\infty}(\bm{U}^t,\bm{U}^{\star})  
	  \leq \( C_3 \rho^t + C_4 \frac{\sigma}{\lammin^{\star}} \sqrt{\frac{ d \log d}{p}}   \) \norm{\bm{U}^\star}_{2, \infty}
\end{align}
\end{subequations}
hold simultaneously for all $0 \leq t \leq t_0 = d^5$. 
Here, $0<C_1, C_3, \rho<1$ and $C_2, C_4>0$ are some absolute constants. 
\end{theorem}
\begin{corollary}
	\label{corollary:tensor-entries}
	Fix an arbitrarily small constant $\delta>0$. Instate the assumptions of Theorem \ref{thm:main}. Then with probability at least $1-\delta$, 
	\begin{subequations}
	\begin{align}
		\norm{\bm{T}^t-\bm{T}^\star}_\frob 	&\lesssim  \( C_1 \rho^t  + C_2 \frac{\sigma}{\lammin^{\star}} \sqrt{\frac{ d \log d}{p}}   \) \norm{\bm{T}^\star}_{\frob},   \\
		\norm{\bm{T}^t-\bm{T}^\star}_{\infty} 	&\lesssim \( C_3 \rho^t + C_4 \frac{\sigma}{\lammin^{\star}} \sqrt{\frac{\mu^3 r d \log d}{p}}   \) \norm{\bm{T}^\star}_{\infty}
	\end{align}
	\end{subequations}
hold simultaneously for all $0 \leq t \leq t_0 = d^5$. Here, $0<C_1, C_3, \rho<1$ and $C_2, C_4>0$ are some absolute constants. 
\end{corollary}
\begin{remark} Clearly, Theorem~\ref{thm:main} and Corollary~\ref{corollary:tensor-entries} subsume Theorem~\ref{thm:main-rank1} and Corollary~\ref{corollary:tensor-entries-rank1} as a special case respectively. \end{remark}
	\begin{remark} Our theorems require the rank $r$ to not exceed $o(d^{1/6})$, which, we believe, is an artifact of the current nonconvex analysis (particularly for the initialization stage). For instance, our local convergence analysis is built upon strong convexity and smoothness, which holds only within a sufficiently small neighborhood  surrounding the truth; given that the diameter of this neighborhood is no more than $o(1/r)$, our analysis requires an initial guess with higher accuracy than expected, thus leading to our rank constraint.  It might be possible to improve the rank dependency via more refined analysis, and we leave it to future investigation. 
\end{remark}
In a nutshell, this theorem reveals intriguing theoretical support (including both $\ell_{\mathrm{F}}$ and $\ell_{2,\infty}$ bounds) for more general settings.  Assuming that the condition number $\kappa \asymp 1$, the nonconvex algorithm we propose is guaranteed to succeed in polynomial time.   
Note, however, that our theoretical dependency (including both sample and computational complexities) on the rank $r$ and the incoherence parameter $\mu$ are likely loose and sub-optimal. 
%In addition, the computational bottleneck lies in the initialization stage, which requires a number of restarts. 
In addition, if $\kappa$ is allowed to grow with $d$, then the current theory requires a large number of restart attempts during the initialization stage, resulting in a very high computational burden. 
Improving these aspects, however, calls for a much more refined analysis framework, which we leave for future investigation.

\subsection{Numerical experiments} 

%\yxc{Changxiao: include all experiments we have mentioned in the rebuttal.}

We carry out a series of numerical experiments to corroborate our theoretical findings.  Before proceeding, recall that Theorem \ref{thm:main} only guarantees successful recovery with probability $1-\delta$ for some small constant $\delta$; this means that we shall not anticipate a very high success rate (e.g.~$1-O(d^{-5})$) as in the matrix recovery case. As we shall make clear shortly, this happens mainly because the initialization stage  works only with probability $1-\delta$, where the uncertainty largely depends on the random vectors $\{\bm{g}^{\tau}\}_{1\leq \tau \leq L}$.  With this observation in mind, we recommend the following modification to improve the empirical success rate:   
\begin{itemize}
	\item Run  Algorithm \ref{alg:init} independently for $t_{\mathsf{init}}= 5$ times to obtain multiple initial estimates (denoted by $\bm{U}^{0}_{[1]},\cdots,\bm{U}^{0}_{[t_{\mathsf{init}}]}$);  select the one achieving the smallest empirical loss, namely
	\begin{align}
		\bm{U}^{0}_{\mathsf{best}} = \underset{{\bm{U}\in \{\bm{U}^{0}_{[i]} \}}_{1\leq i\leq t_{\mathsf{init}}}}{\arg\min} f(\bm{U}) .
	\end{align}
	\item Run Algorithm \ref{alg:gd} with the initial point $\bm{U}^0$ set to be $\bm{U}^{0}_{\mathsf{best}}$. 
\end{itemize}
The final estimates for the low-rank factor and the whole tensor are denoted respectively by
\begin{align}
	\widehat{\bm{U}} = \bm{U}^{t_0}\qquad   \text{and} \qquad \widehat{\bm{T}} = \sum\nolimits_{i=1}^r \bm{u}^{t_0}_i \otimes \bm{u}^{t_0}_i \otimes \bm{u}^{t_0}_i, 
\end{align}
where $\bm{U}^{t_0} = [\bm{u}^{t_0}_1, \cdots, \bm{u}^{t_0}_r] \in \mathbb{R}^{d\times r}$ is the iterate returned by Algorithm \ref{alg:gd}, with 
$t_0$ the total number of gradient iterations. In the sequel, we generate the true  tensor $\bm{T}^{\star} = \sum_{1\leq i\leq r} \bm{u}_i^{\star\,\ot 3}$  randomly in such a way that  $\bm{u}_i^{\star} \overset{\mathrm{i.i.d.}}{\sim} \mathcal{N}(\bm{0}, \bm{I}_d)$. 
The learning rates are taken to be $\eta_t\equiv 0.2$ unless otherwise noted.

We start with numerical convergence rates of our algorithm in the absence of noise. Set $d = 100$, $r = 4$, $p = 0.1$, $L = 16$ and $\epsilon_{\mathsf{th}} = 0.4$. Figure \ref{fig:simultation-TC}(a) the numerical estimation errors vs.~iteration count $t$ in a typical Monte Carlo trial. Here, four kinds of estimation errors are reported: (1) the relative Frobenius norm error $\frac{\mathsf{dist}_{\mathrm{F}}(\bm{U}^{t},\bm{U}^{\star})}{\|\bm{U}^{\star}\|_{\mathrm{F}}}$; (2) the relative $\|\cdot\|_{2,\infty}$  error $\frac{\mathsf{dist}_{2,\infty}( \bm{U}^t, \bm{U}^{\star})}{\|\bm{U}^{\star}\|_{2,\infty}}$; (3) the relative Frobenius norm error $\frac{\|\bm{T}^{t}-\bm{T}^{\star}\|_{\mathrm{F}}}{\|\bm{T}^{\star}\|_{\mathrm{F}}}$; (4) the relative $\ell_{\infty}$  error $\frac{\|\bm{T}^{t}-\bm{T}^{\star}\|_{\infty}}{\|\bm{T}^{\star}\|_{\infty}}$. Here, $\bm{T}^{t}=\sum_{i=1}^{r}\bm{u}_{i}^{t}\otimes\bm{u}_{i}^{t}\otimes\bm{u}_{i}^{t}$ with $\bm{U}^t = [\bm{u}^t_1,\cdots, \bm{u}^t_r]$.  For all  these metrics, the numerical estimation errors decay geometrically fast.

Next, we study the phase transition (in terms of the success rates for exact recovery) in the noise-free settings. Set $d = 100$, $r = 4$, $L = 16$ and $\epsilon_{\mathsf{th}} = 0.4$. For the sake of comparisons, we also report the numerical
performance of the tensor power method (TPM) followed by gradient descent. When running the tensor power method, we set both the number of iterations and the restart number to be $16$.
%Set  $r = 4$ and vary $d$.  
Each trial is claimed to succeed if the relative $\ell_2$ error obeys $\frac{\mathsf{dist}_{\mathrm{F}}(\widehat{\bm{U}},\bm{U}^{\star})}{\|\bm{U}^{\star}\|_{\mathrm{F}}}\leq 0.01$. Figure \ref{fig:simultation-TC}(b) plots the empirical success rates over 100 independent Monte Carlo trials. 
As can be seen, our initialization algorithm outperforms the tensor power method.
%As can be seen, the proposed algorithm almost always succeeds once the sampling rate $p$ exceeds a certain point.

The third series of experiments is concerned with the dependence of the success rate on the rank $r$. Let us set $p = r d^{-3/2} \log^2 d$, $L = r^2$ and $\epsilon_{\mathsf{th}} = 0.4$, and the success recovery criterion is the same as above. Figure \ref{fig:simultation-TC}(c) depicts the empirical success rates (over $100$ independent Monte Carlo trials) as the rank $r$ varies. As can be seen from the plots,  the proposed algorithm is able to achieve exact reconstruction as long as the rank $r$ is sufficiently small compared to $d$. The plausible range of $r$, however, seems  and seems to be larger than our theoretic requirement $r=o(d^{1/6})$. This, once again, suggests the need of future investigation to pin down the best possible dependency on $r$.

Finally, we consider the numerical estimation accuracy of our algorithm. Take $t_0 = 100$,  $d = 100$, $r = 4$, $p = 0.1$, $L = 16$ and $\epsilon_{\mathsf{th}} = 0.4$. Define the signal-to-noise ratio (SNR) to be $\mathrm{SNR} =  \frac{\norm{\bm{T}^\star}_\frob^2/d^3}{\sigma^2} $. We report in Figure \ref{fig:simultation-TC}(d) three types of squared relative errors (namely,  $\frac{\mathsf{dist}^2_{\mathrm{F}}(\widehat{\bm{U}},\bm{U}^{\star})}{\|\bm{U}^{\star}\|^2_{\mathrm{F}}}$, $\frac{\mathsf{dist}^2_{2,\infty}(\widehat{\bm{U}},\bm{U}^{\star})}{\|\bm{U}^{\star}\|^2_{2,\infty}}$ and $\frac{\|\widehat{\bm{T}}-\bm{T}^{\star}\|^2_{\infty}}{\|\bm{T}^{\star}\|^2_{\infty}}$) vs.~SNR. 
%, where both  SNR and relative squared errors are displayed on a dB scale, i.e.~the values of $10\log_{10}(\mathrm{SNR})$ and $10\log_{10}(\mathrm{rel.~squared~error})$ and are plotted. Here,  the SNR varies from 5dB (i.e.~$\sqrt{10}$) to 25dB (i.e.~$100\sqrt{10}$).   
Figure \ref{fig:simultation-TC}(d) illustrates that all three types of relative squared errors scale inversely proportional to the SNR (since the slope in the figure is roughly $-1$), which is consistent with our statistical guarantees.

% \begin{figure}
%  \includegraphics[width=0.48\linewidth]{Simulation/succ_rate.eps}
%	\caption{Empirical success rate vs.~sampling rate, where $d = 100, r = 4, p = 0.1$ and $\eta_t = 0.2$. Each point is averaged over $100$ Monte Carlo trials.}
%  \label{fig:succ_rate}
% \end{figure} 

% \begin{figure}
% \begin{tabular}{ccc}
% \includegraphics[width=0.32\linewidth]{convergence.eps} & \includegraphics[width=0.32\linewidth]{phase-transition.eps} &  \includegraphics[width=0.32\linewidth]{mse.eps} \tabularnewline
% 	(a) & (b) & (c) \tabularnewline
% \end{tabular}
% 	\caption{(a) relative errors of the estimates $\bm{U}^{t}$ and $\bm{T}^t$ vs.~iteration count $t$ for noiseless tensor completion, where $d = 100$, $r = 4$, $p = 0.1$; (b) empirical success rate vs.~sampling rate, where $d = 100$, $r = 4$; (c) squared relative errors vs.~SNR for noisy settings, where $d = 100$, $r = 4$ and $p = 0.1$. Each point in (b) and (c) is averaged over $100$ independent Monte Carlo trials. \label{fig:simultation-TC}
% }

\begin{figure}
\begin{tabular}{cccc}
\includegraphics[width=0.23\linewidth]{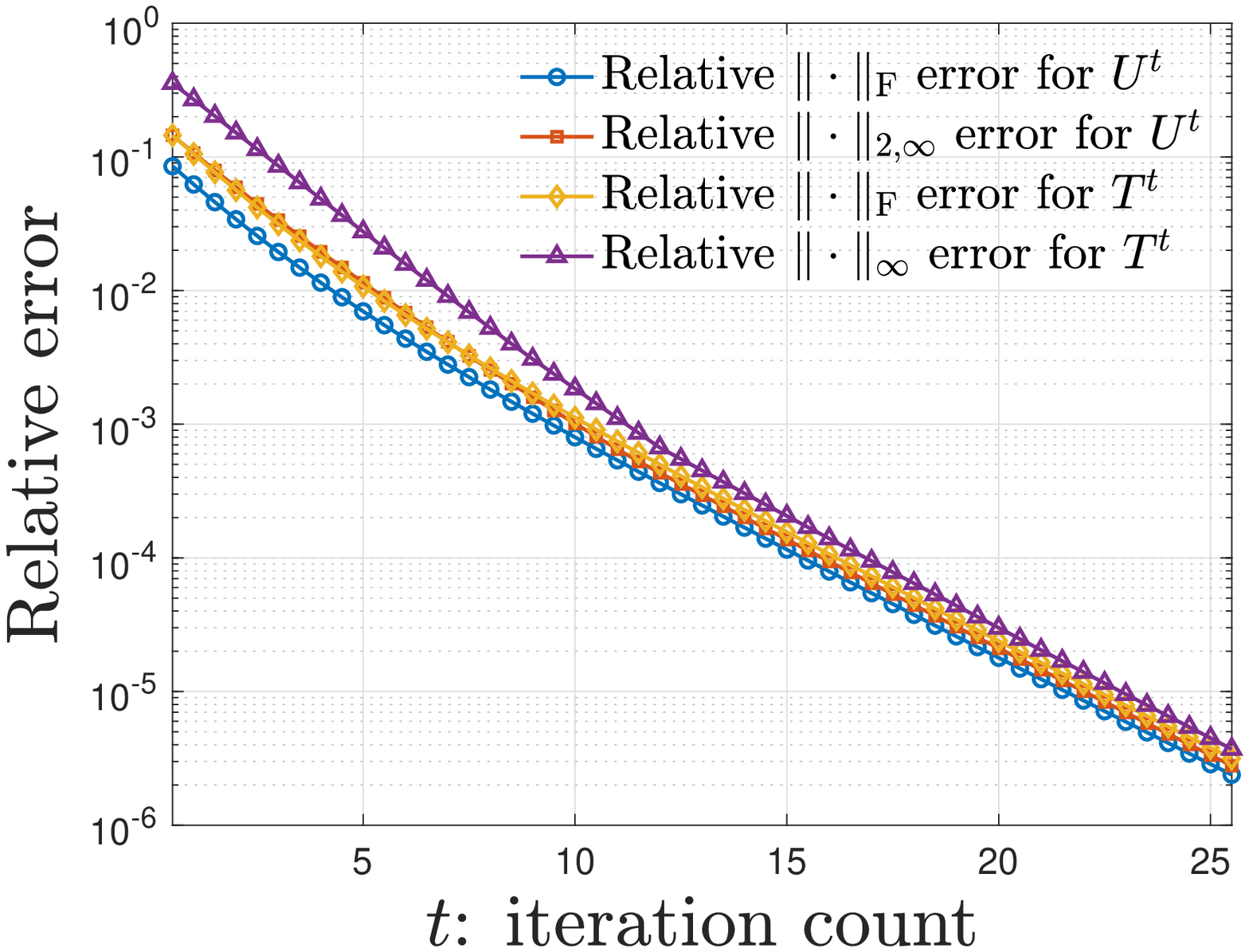} & \includegraphics[width=0.23\linewidth]{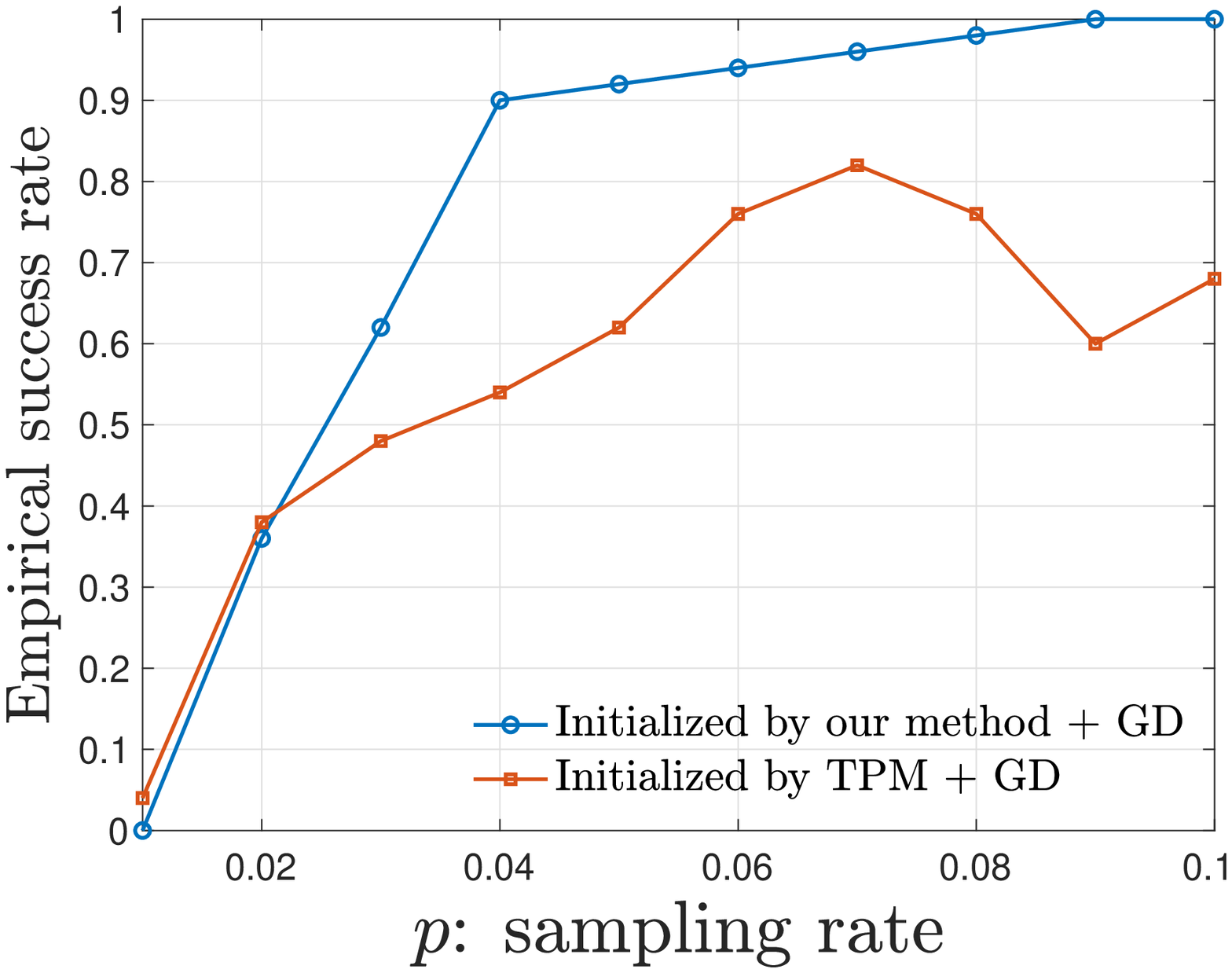} & \includegraphics[width=0.23\linewidth]{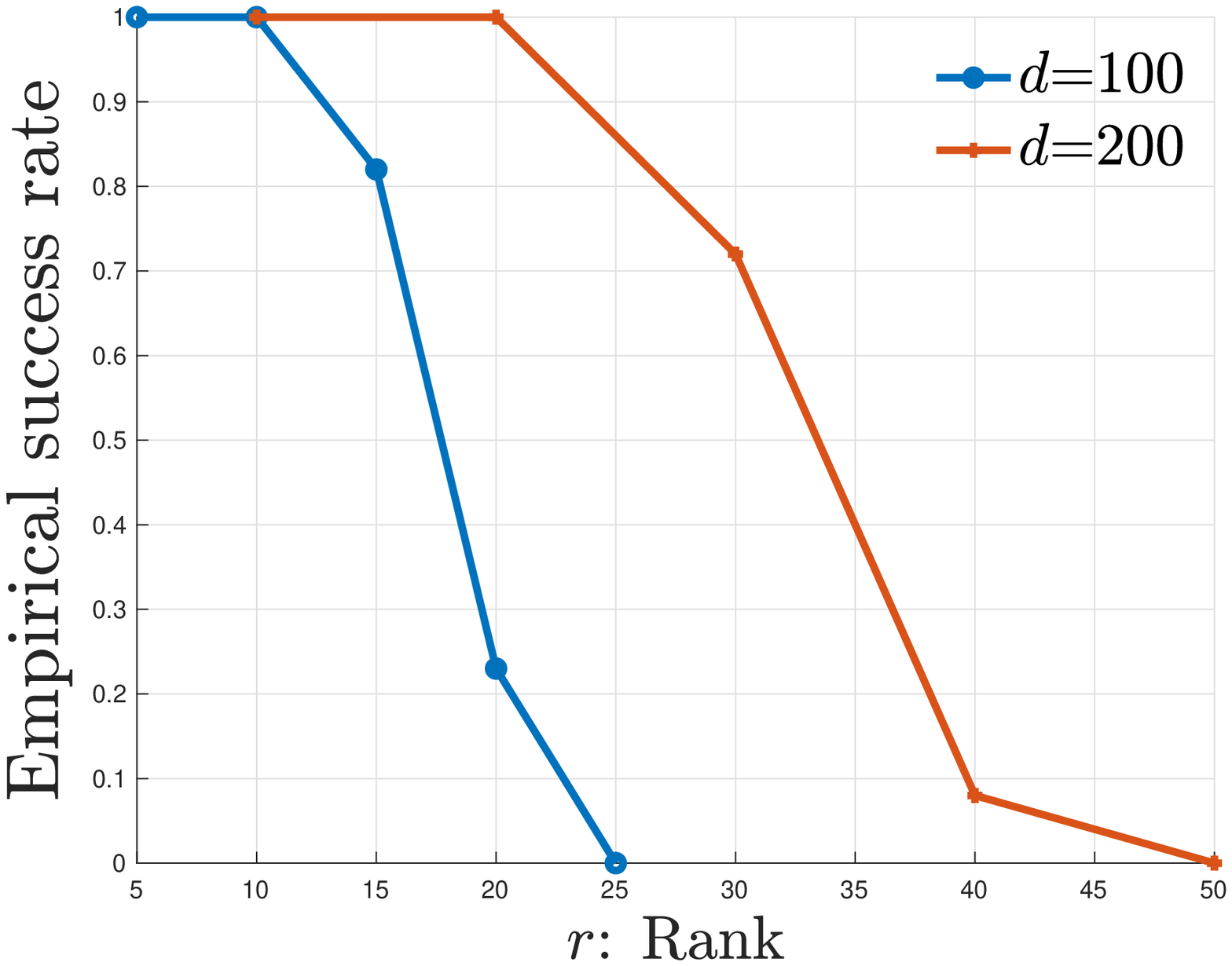} & \includegraphics[width=0.23\linewidth]{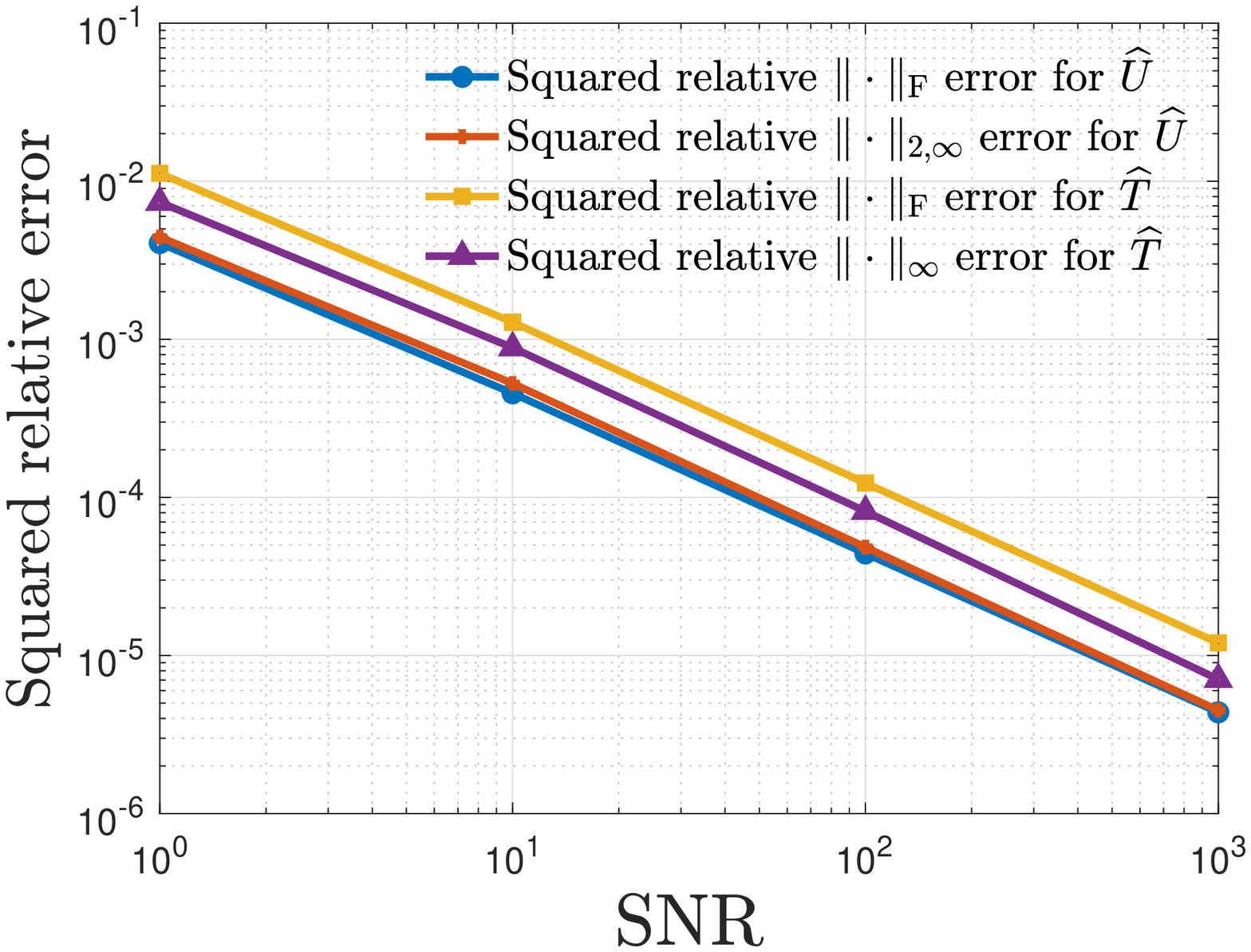} \tabularnewline
	(a) & (b) & (c) & (d) \tabularnewline
\end{tabular}
	\caption{(a) relative errors of the estimates $\bm{U}^{t}$ and $\bm{T}^t$ vs.~iteration count $t$ for noiseless tensor completion, where $d = 100$, $r = 4$, $p = 0.1$; (b) empirical success rate vs.~sampling rate, where $d = 100$, $r = 4$; (c) empirical success rate vs.~rank, where $p = r d^{-3/2} \log^2 d$; (d) squared relative errors vs.~SNR for noisy settings, where $d = 100$, $r = 4$ and $p = 0.1$. Each point in (b), (c) and (d) is averaged over $100$ independent Monte Carlo trials. \label{fig:simultation-TC}
}
\end{figure}

% \begin{figure}
% \begin{tabular}{ccc}
% \includegraphics[width=0.44\linewidth]{convergence.eps} &  \includegraphics[width=0.44\linewidth]{mse.eps} \tabularnewline
% 	 (a) & (b) \tabularnewline
% \end{tabular}
% 	\caption{(a) empirical success rate vs.~rank, where $p = r d^{-3/2} \log^2 d$; (b) squared relative errors vs.~SNR for noisy settings, where $d = 100$, $r = 4$ and $p = 0.1$. Each point in (a) and (b) is averaged over $100$ independent Monte Carlo trials. \label{fig:simultation-TC-2}
% }
% \end{figure}

\subsection{Notation}
\label{sec:Notation-1}

Before proceeding, we gather a few notations that will be used throughout this paper. 
First of all, for any matrix $\bm{M} \in \R^{d \times d}$, we let $\norm{\bm{M}}$ and $\norm{\bm{M}}_\frob$ denote the operator norm (or the spectral norm) and the Frobenius norm of $\bm{M}$, respectively, and let $\bm{M}_{i, :}$ and $\bm{M}_{:, i}$ denote the $i$-th row and $i$-th column, respectively. In addition, $\lambda_1(\bm{M}) \geq \lambda_2(\bm{M}) \geq \dots \geq \lambda_d(\bm{M})$ denote the eigenvalues of $\bm{M}$ and $\sigma_1(\bm{M}) \geq \sigma_2(\bm{M}) \geq \dots \geq \sigma_d(\bm{M})$ denote the singular values of $\bm{M}$. 
%For any left-invertible matrix $\bm{W}\in \mathbb{R}^{d\times r}$, we define the following projection operators
%
%\begin{align}
%	\label{eq:defn-PW}
%	\mathcal{P}_{\bm{W}}(\bm{z})\triangleq\bm{W}(\bm{W}^{\top}\bm{W})^{-1}\bm{W}^{\top}\bm{z}\qquad\text{and}\qquad\mathcal{P}_{\bm{W}_{\perp}}(\bm{z})=\bm{z}-\mathcal{P}_{\bm{W}}(\bm{z}). 
%\end{align}
%

For any tensor $\bm{T} \in \mathbb{R}^{d\times d\times d}$, let $\bm{T}_{i, :, :} \in \mathbb{R}^{d \times d}$ denote the mode-$1$ $i$-slice with entries $\( \bm{T}_{i, :, :} \)_{j, k} = T_{i, j, k}$, and $\bm{T}_{:, i, :}$ and $\bm{T}_{:, :, i}$ are defined in a similar way.
For any tensors $\bm{T},\bm{R}\in \mathbb{R}^{d\times d\times d}$, 
the inner product is defined as 
$\langle \bm{T}, \bm{R} \rangle := \sum_{j,k,l} T_{j, k, l} R_{j, k, l}$. 
The Frobenius norm of $\bm{T}$ is defined as
$\norm{\bm{T}}_\frob := \sqrt{ \langle \bm{T}, \bm{T} \rangle }$.  
For any  vectors $\bm{u}, \bm{v} \in \R^d$, we define the vector products of a tensor $\bm{T} \in \R^{d \times d \times d}$ --- denoted by $\bm{T} \times_3 \bm{u} \in \R^{d \times d}$ and $\bm{T} \times_1 \bm{u} \times_2 \bm{v} \in \R^d$ --- such that
\begin{subequations}
\begin{align}
	\big[\bm{T} \times_3 \bm{u} \big]_{ i j } ~&:=~ \sum\nolimits_{1\leq k \leq d} T_{i, j, k} u_k,  \qquad & 1\leq i,j\leq d;\\ 
	\big[\bm{T} \times_1 \bm{u} \times_2 \bm{v} \big]_{ k } ~&:=~ \sum\nolimits_{1\leq i, j \leq d} T_{i, j, k} u_i v_j,  \qquad & 1\leq k \leq d.
% \\ \bm{T} \times_1 \bm{u} \times_2 \bm{v} \times_3 \bm{w} ~&:=~ \sum\nolimits_{1\leq i, j, k \leq d} T_{i, j, k} u_i v_j w_k.
\end{align}
\end{subequations}
%
%The products  $\bm{T} \times_2 \bm{u} \in \R^{d\times d}$, $\bm{T} \times_3 \bm{u} \in \R^{d\times d}$, $\bm{T} \times_1 \bm{u} \times_3 \bm{v} \in \R^d$ and $\bm{T} \times_2 \bm{u} \times_3 \bm{v} \in \R^d$ are  defined in a similar manner.
%
The products  $\bm{T} \times_2 \bm{u} \in \R^{d\times d}$, $\bm{T} \times_3 \bm{u} \in \R^{d\times d}$, $\bm{T} \times_1 \bm{u} \times_3 \bm{v} \in \R^d$ and $\bm{T} \times_2 \bm{u} \times_3 \bm{v} \in \R^d$ are  defined in a similar manner.
For any $\bm{U}=[\bm{u}_1,\cdots,\bm{u}_r]\in \mathbb{R}^{d\times r}$ and $\bm{V}=[\bm{v}_1,\cdots,\bm{v}_r]\in \mathbb{R}^{d\times r}$, we further define
\begin{align}
	\label{eq:T-v-u-seq-defn}
	\bm{T} \times_1^\seq \bm{U} \times_2^\seq \bm{V} := [\bm{T} \times_1 \bm{u}_i \times_2 \bm{v}_i ]_{1\leq i\leq r} \in \mathbb{R}^{d\times r}.
\end{align}
In addition, the operator norm of $\bm{T}$ is defined as
\begin{align}
\norm{\bm{T}} := \sup_{\bm{u}, \bm{v}, \bm{w} \,\in \,\sphere^{d-1}} \left\langle \bm{T},  \bm{u} \otimes \bm{v} \otimes \bm{w} \right\rangle,
\end{align}
where $\sphere^{d-1} := \{\bm{u} \in \mathbb{R}^d \mid \|\bm{u}\|_2=1 \}$ indicates the unit sphere in $\R^d$.

Further, $f(n)\lesssim g(n)$ or $f(n) = O(g(n))$ means that $|f(n)/g(n)| \leq C_1$ for some constant $C_1>0$; $f(n)\gtrsim g(n)$  means that $|f(n)/g(n)| \geq C_2$ for some constant $C_2>0$;   $f(n)\asymp g(n)$ means that $C_1 \leq |f(n)/g(n)| \leq C_2$ for some constants $C_1, C_2>0$; $f(n)= o( g(n))$ means that $\lim_{n\rightarrow \infty} f(n)/g(n)=0$.  
%We also use $a \vee b$ to denote $\max \{ a, b \}$.  
In addition, $f(n)\ll g(n)$ means that $f(n)\leq c_1 g(n)$ for some sufficiently small constant $c_1>0$, and  $f(n)\gg g(n)$ means that $f(n)\geq c_2 g(n)$ for some sufficiently large constant $c_2>0$. 

\section{Initialization}
\label{sec:algorithm-main}

This section presents formal details of the proposed two-step initialization, accompanied by some intuition. Recall that the proposed initialization procedure consists of two steps, which we discuss separately.

%\subsection{Initialization} 

%Motivated by the above-mentioned issue (i), we develop a procedure that guarantees a reasonably good initial estimate. In a nutshell, the proposed procedure starts by estimating the subspace spanned by the $r$ tensor components via a spectral method,  and then proceeds  to retrieve each individual tensor component from this subspace.  

%\subsection{Details of the initialization procedure}

\subsection{Step 1: subspace estimation via a spectral method} 
\label{sec:subspace-estimate}

%\subsubsection{Procedure}

The spectral algorithm is often applied in conjunction with simple ``unfolding'' (or ``matricization'') to estimate the {\em subspace} spanned by the $r$ factors $\{\bm{u}_i^\star\}_{1\leq i\leq r}$. This strategy is partly motivated by prior approaches developed for covariance estimation with missing data \cite{lounici2014high,montanari2018spectral,cai2019subspace}.  We provide a brief introduction below. 
%, and refer the interested reader to \cite{cai2019subspace} for in-depth discussions.    
%that is, one can first reshape the observed tensor data into a $d\times d^2$ matrix, followed by
% estimation of the $r$-dimensional column subspace of this unfolded matrix.  

Let  
\begin{align}
	\label{eq:defn-A-subspace}
	\bm{A} = \mathsf{unfold}^{1\times 2} \big( \tfrac{1}{p} \bm{T} \big) \in \mathbb{R}^{d\times d^2}, \quad \text{or more concisely} ~~ \bm{A} = \mathsf{unfold} \big( \tfrac{1}{p} \bm{T} \big) \in \mathbb{R}^{d\times d^2}
\end{align}
be the mode-1 matricization of $p^{-1} \bm{T}$  (namely, $\frac{1}{p}T_{i,j,k}=A_{i,(j-1)d+k}$ for any $1\leq i,j,k\leq d$) \cite{kolda2009tensor}. The rationale of this step is that:  under our model, the unfolded matrix $\bm{A}$ obeys
\begin{align}
	\label{definition:Astar}
	\mathbb{E}[\bm{A}] = \mathsf{unfold} \big(\bm{T}^\star \big) = \sum_{i=1}^{r}\bm{u}_{i}^{\star}\left(\bm{u}_{i}^{\star} \otimes \bm{u}_{i}^{\star}\right)^{\top} =: \bm{A}^{\star}, 
\end{align}
whose column space is precisely the span of $\{\bm{u}^\star\}_{1\leq i\leq r}$.  This motivates one to estimate the $r$-dimensional column space of $\mathbb{E}[\bm{A}]$ from 
$\bm{A}$. Towards this, a natural strategy is to look at the principal subspace of $\bm{A} \bm{A}^\top$. However, the diagonal entries of $\bm{A} \bm{A}^\top$  bear too much influence on the principal directions and need to be properly down-weighed. The current paper chooses to work with the principal subspace of the following matrix that zeros out all diagonal components: 
%would rather go with the principal subspace of a reweighted matrix
%
\begin{align}
  \bm{B} := \PP_{\mathsf{off}\text{-}\mathsf{diag}} (\bm{A} \bm{A}^\top) ,
	%+ p\PP_{\mathsf{diag}} ( \bm{A} \bm{A}^\top) ,
 \end{align} 
%
%where $\PP_{\mathsf{diag}}(\bm{Z})$ (resp.~$\PP_{\mathsf{off}\text{-}\mathsf{diag}}(\bm{Z})$) extracts out the diagonal (resp.~off-diagonal) entries of a squared matrix $\bm{Z}$. 
where $\PP_{\mathsf{off}\text{-}\mathsf{diag}}(\bm{Z})$ extracts out the off-diagonal entries of a squared matrix $\bm{Z}$. 
If we let $\bm{U}\in \mathbb{R}^{d\times r}$ be an orthonormal matrix whose columns are the top-$r$ eigenvectors of $\bm{B}$, then $\bm{U}$ serves as our subspace estimate. 
See  Algorithm~\ref{alg:init} for a summary of the procedure. 
%This simple spectral algorithm has been studied in detail in our companion paper \cite{cai2019subspace}. 
%In the interest of a self-contained exposition, we include some intuitive explanation in Appendix \ref{sec:intuition-step1-init}. 

%\subsubsection

\subsection{Step 2: retrieval of low-rank tensor factors from the subspace estimate}

\subsubsection{Procedure}

As it turns out, it is possible to obtain rough (but reasonable) estimates of all individual low-rank tensor factors $\{\bm{u}_i^\star\}_{1\leq i\leq r}$ --- up to global permutation --- given a reliable subspace estimate $\bm{U}$. This is in 
stark contrast to the low-rank matrix recovery case, where there exists some global rotational ambiguity that prevents us from disentangling the $r$ factors of interest. 
%In what follows,  we propose to iteratively retrieve\,/\,estimate the tensor factors one by one.  

We begin by describing how to retrieve {\em one} tensor factor from the subspace estimate --- a procedure summarized in $\Call{Retrieve-one-tensor-factor}$.  
 Let us generate a random vector from the provided subspace $\bm{U}$ (which has orthonormal columns), that is, 
\begin{align}
	\bm{\theta} = \hspace{-2em} \underset{\text{projection of }\bm{g}\text{ onto }\bm{U}}{\underbrace{\bm{U}\bm{U}^{\top} \bm{g}}} \hspace{-2em} , \qquad \bm{g} \sim \mathcal{N}(\bm{0},\bm{I}_d).
\end{align}
The rescaled tensor data $p^{-1} \bm{T}$ is then transformed into a matrix via proper ``projection'' along this random direction  $\bm{\theta}$, namely, 
\begin{align}
	\label{eq:projected-matrix-M1}
	\bm{M} = \tfrac{1}{p} \bm{T} \times_3 \bm{\theta} \in \mathbb{R}^{d\times d}.
\end{align}
Our estimate for a tensor factor is then given by  $\lambda^{1/3} \bm{\nu}$, where 
$\bm{\nu}$ is the leading singular vector of $\bm{M}$ obeying $\langle \bm{T}, \bm{\nu}  ^{\ot 3} \rangle\geq 0$, and $\lambda$ is taken as $\lambda= \big \langle p^{-1} \bm{T}, \bm{\nu}  ^{\otimes 3} \big \rangle$.  Informally,  $\bm{\nu}$ reflects the direction of the component $\bm{u}_i^\star$ that exhibits the largest correlation with the random direction $\bm{\theta}$, and $\lambda$  forms an estimate of the corresponding size $\|\bm{u}^{\star}_i\|_2$. 
%We shall provide intuition  in Section \ref{sec:intuition}.  

A challenge remains, however, as there are oftentimes more than one tensor factors to estimate. To address this issue, we propose to re-run the aforementioned procedure multiple times, so as to ensure that we get to retrieve each tensor factor of interest at least once. We will then apply a careful pruning procedure (i.e.~$\Call{Prune}$) to remove redundancy.

\subsubsection{Intuition}

%In this subsection, we develop some intuitive understanding about the procedure introduced above. 

%\paragraph{(a) Subspace estimation via a spectral method.} 

%\paragraph{(b) Retrieval of tensor components from subspace estimates.}

To develop some intuition about the above procedure, consider the ``heuristic'' case where $\bm{\theta} = \bm{U}^\star ( \bm{U}^{\star\top} \bm{U}^{\star} )^{-1} \bm{U}^{\star\top} \bm{g}$, namely, the idealistic scenario where the subspace estimate $\bm{U}$ is accurate.  Averaging out the randomness in the sampling pattern and the noise, we see that the expected projected matrix \eqref{eq:projected-matrix-M1} takes the following form:
\begin{align*}
	\mathbb{E} \big[ \bm{M} \mid \bm{\theta} \big] = \bm{T}^{\star}  \times_3 \bm{\theta} = 
	\sum_{i=1}^{r} \langle \bm{\theta}, \bm{u}_{i}^{\star} \rangle \bm{u}_{i}^{\star}\bm{u}_{i}^{\star\top} .
\end{align*}  
%
%idealistic case
%where $\bm{T} = \bm{T}^*$, which corresponds to the noiseless case without missing data. In this scenario, the projected matrix \eqref{eq:projected-matrix-M1} we obtain takes the following form
%
%\begin{align*}
%	\bm{M}^1 = \sum_{i=1}^{r} \langle \bm{\theta}^{1}, \bm{u}_{i}^{*} \rangle \bm{u}_{i}^{*}\bm{u}_{i}^{*\top}
%\end{align*}  
%
As a result, in the incoherent case where $\{\bm{u}_j^\star\}$ are nearly orthogonal to each other, 
the leading singular vector of $\mathbb{E} \big[ \bm{M} \mid \bm{\theta} \big]$ --- and hence that of $\bm{M}$ (i.e.~$\bm{w}$) --- is expected to be reasonably close to the factor $\bm{u}_i^\star$ that enjoys the largest projected coefficient. 
%$\bm{\theta}^{1\top}\bm{u}_{i}^{*}$ among all $\{\bm{\theta}^{1\top}\bm{u}_{j}^{*}\}_{1\leq j\leq r}$, 
In other words, we expect 
\begin{equation}
\bm{\nu} \approx\frac{1}{\left\|\bm{u}_{i}^{\star} \right\|_{2}}\bm{u}_{i}^{\star},
	%\qquad\sigma_{1}\approx\big|\langle\bm{\theta}^{1},\bm{u}_{i}^{*}\rangle\big|\cdot\|\bm{u}_{i}^{*}\|_{2}^{2}\approx\big|\langle\bm{\theta}^{1},\bm{w}^{1}\rangle\big|\cdot\|\bm{u}_{i}^{*}\|_{2}^{3},
	\qquad \text{where }i=\arg\max_{1\leq j\leq r}\big|\langle\bm{\theta},\bm{u}_{j}^{\star}\rangle\big|. \label{eq:nu-intuition}
\end{equation}
%
%where $\bm{w}^1$ was defined above in the algorithm. 
%In the mean time, it can be easily verified that the leading singular value of $\bm{M}^1$ satisfies 
%%
%\begin{align}
%	\sigma_{1}\approx\big|\langle\bm{\theta}^{1},\bm{u}_{i}^{\star}\rangle\big|\cdot\left\|\bm{u}_{i}^{\star} \right\|_{2}^{2}\approx\big|\langle\bm{\theta}^{1},\bm{w}^{1}\rangle\big|\cdot\left\|\bm{u}_{i}^{\star} \right\|_{2}^{3}
%\end{align}
%%
%as long as $\bm{M}^1 \approx \mathbb{E} \big[ \bm{M}^1 \mid \bm{\theta}^{1} \big]$. This simple calculation immediately suggests that 
%%
%\begin{align}
%\left\|\bm{u}_{i}^{\star} \right\|_{2}\approx\big|\sigma_{1}/\langle\bm{\theta}^{1},\bm{w}^{1}\rangle\big|^{1/3},  
%\end{align}
In the mean time, armed with (\ref{eq:nu-intuition}) and the incoherence assumption (such that $\bm{u}_i^{\star}$ and $\bm{u}_j^{\star}$ are nearly orthogonal for $i\neq j$), one might have
\begin{align}
	\lambda = \big \langle \bm{T}^\star, \bm{\nu} ^{ \otimes 3} \big \rangle 
	\approx  \frac{1}{\left\|\bm{u}_{i}^{\star} \right\|_{2}^3} \left \langle \bm{T}^\star, \bm{u}_i^{\star \otimes 3} \right \rangle 
	\approx  \frac{1}{\left\|\bm{u}_{i}^{\star} \right\|_{2}^3} \left \langle \bm{u}_i^{\star \otimes 3}, \bm{u}_i^{\star \otimes 3} \right \rangle 
	= \left\|\bm{u}_{i}^{\star} \right\|_{2}^3,
\end{align}
thus explaining our choice of $\lambda$ in the proposed procedure. These arguments hint at the ability of our procedure in retrieving one tensor factor in each round.

The above intuitive argument, however, does not explain why we need to first project a random vector $\bm{g}$ onto the (approximate) column space of $\bm{U}^\star$. While we won't go into detailed calculations here, we remark in passing a crucial high variability issue:  without proper projection, the perturbation incurred by both the missing data and the noise might far exceed the strength of the true signal.  As a result, it is advised to first project the data onto the desired subspace, in the hope of amplifying the signal-to-noise ratio.

% One may ask if we can simply set $\bm{\theta} = \bm{g}$ without the projection onto the column space of $\bm{U}^\star$. Simple heuristic arguments suggest this approach cannot work. Let us consider the noiseless case. It is easy to verify that $\| \bm{T}^{\star}  \times_3 \bm{g} \| \approx \max_i \norm{\bm{u}_i}_2^3$ while the typical size of $\| (p^{-1}\PP_{\Omega} ( \bm{T}^\star ) - \bm{T}^\star) \times_3 \bm{g} \| $ is of order $ \frac{\poly \log d}{\sqrt{dp}} \max_i \norm{\bm{u}_i}_2^3$ (cf.~Lemma~\ref{lemma:T_loss_times3_op_norm}). In this scenario, the perturbation incurred by missing data dominates the true signal strength unless $ p \gtrsim \frac{\poly \log d}{d}$, which is much larger than the computational limit. Therefore, we need to carefully choose the projection vector to amplify the signal-to-noise ratio in order to recover the tensor factors.
%The remaining steps we employ to calculate $\bm{w}^j$ and $\lambda_j$ ($j\geq 2$) can be interpreted in a similar manner. 

\subsection{Other alternatives?}
\label{sec:alternatives}

The careful reader may naturally wonder whether a careful initialization is pivotal in achieving fast convergence.  While a thorough answer to this has yet to be developed, we shall point out some alternatives that seem sub-optimal in both theory and practice. 
To simplify the presentation, the current subsection focuses on the rank-1 noiseless case, where
\begin{align}
	\label{eq:rank-1-noiseless}
	\bm{T}^{\star}=\bm{u}^{\star\,\otimes3}, \qquad \bm{T}= \tfrac{1}{p} \mathcal{P}_{\Omega}( \bm{T}^{\star} ) , \qquad \|\bm{u}^\star \|_2=1. 
\end{align}
% 
%(\ref{eq:rank-1-noiseless}). 
Since the decision variable is now  a $d$-dimensional vector, we shall employ the conventional notation $\bm{u}^t$ to represent $\bm{U}^t$.

\paragraph{Random initialization.}

We  find it instrumental to begin with the population-level analysis, which corresponds to the scenario with no missing data and noise ($p=1$ and $\sigma = 0$).  A little calculation gives 
\begin{align}
	\label{eq:E-u1}
	\mathbb{E}\left[\bm{u}^{1} \mid \bm{u}^{0} \right] & = \mathbb{E}\left[\bm{u}^{0} - \eta \nabla f(\bm{u}^0) \mid \bm{u}^{0} \right] =\left(1-\eta\|\bm{u}^{0}\|_{2}^{4}\right)\bm{u}^{0}+\eta\langle\bm{u}^{0},\bm{u}^{\star}\rangle^{2}\bm{u}^{\star} .
\end{align}
As an immediate consequence, the expected correlation between the next iterate and the truth obeys
\begin{align*}
	\mathbb{E}\left[\langle\bm{u}^{1},\bm{u}^{\star}\rangle  \mid \bm{u}^{0} \right] & =\left\{ 1-\eta\|\bm{u}^{0}\|_{2}^{4}+\eta\langle\bm{u}^{0},\bm{u}^{\star}\rangle\|\bm{u}^{\star}\|_{2}^2 \right\} \langle\bm{u}^{0},\bm{u}^{\star}\rangle .
\end{align*}
This means that if $\bm{u}^0$ and $\bm{u}^\star$ are positively correlated and 
if the initial guess $\bm{u}^{0}$ is sufficiently small,\footnote{In fact, if a random initialization $\bm{u}^{0}$ is not small, then one can easily show that, with high probability, the $\ell_2$ norm of $\bm{u}^t$ is going to drop geometrically fast at the beginning. } then one has
\begin{align}
	\label{eq:expected-increment}
\mathbb{E}\left[\langle\bm{u}^{1},\bm{u}^{\star}\rangle  \mid \bm{u}^{0} \right] & \approx \left(1 +\eta\langle\bm{u}^{0},\bm{u}^{\star}\rangle\|\bm{u}^{\star}\|_{2}^2\right)\langle\bm{u}^{0},\bm{u}^{\star}\rangle ;
\end{align}
a similar recursion holds for $\bm{u}^t$. 
As a result, the GD iterates are expected to get increasingly more aligned with the truth,  
at least at the population level. Caution needs to be exercised, however, that this population-level analysis alone fails to capture what is happening in the finite-sample case.  
In what follows, we point out  potential issues with  random initialization.

Consider the case where $\bm{u}^0$ is generated as a vector of i.i.d.~Gaussian random variables. Suppose that 
$\bm{u}^0$ and $\bm{u}^\star$ are positively correlated and that $\|\bm{u}^0\|_2$ is sufficiently small. 
It is easily seen that, with high probability, the expected increment is on the order of (cf.~\eqref{eq:expected-increment})
\begin{align}
	\label{eq:mean-increment}
	\mathbb{E}\left[\langle\bm{u}^{1},\bm{u}^{\star}\rangle \mid \bm{u}^0 \right] & -\langle\bm{u}^{0},\bm{u}^{\star}\rangle\approx \eta\langle\bm{u}^{0},\bm{u}^{\star}\rangle^{2}\|\bm{u}^{\star}\|_{2}^2 \lesssim\frac{\eta\,\mathrm{poly}\log (d)}{d} \| \bm{u}^{\star} \|_2^4 \| \bm{u}^{0} \|_2^2 ,
\end{align}
which could be quite small as it depends quadratically on the current correlation $\langle\bm{u}^{0},\bm{u}^{\star}\rangle$. 
 
If we were to hope that the favorable population-level analysis captures more or less the finite-sample  dynamics, we would need to ensure that the variability of the gradient update  is well-controlled. Towards this, let us compute the variance of $\langle\bm{u}^{1},\bm{u}^{\star}\rangle$, assuming that $\frac{\|\bm{u}^0\|_{\infty}}{\|\bm{u}^0\|_2} \asymp \frac{\|\bm{u}^\star\|_{\infty}}{\|\bm{u}^\star\|_{2}} \asymp \frac{\mathrm{poly}\log (d)}{\sqrt{d}}$:  
%In this case, simple calculation shows that 
%
\begin{align*}
	\mathsf{Var}\left(\langle\bm{u}^{1},\bm{u}^{\star}\rangle \mid \bm{u}^0 \right) 
	%& =\mathsf{Var}\left(\frac{3\eta}{p}\mathcal{P}_{\Omega}\left(\bm{u}^{0\,\otimes3}-\bm{u}^{\star\,\otimes3}\right)\times_{1}\bm{u}^{0}\times_{2}\bm{u}^{0}\times_{3}\bm{u}^{\star}\right)\\
 & \asymp \mathsf{Var}\Bigg( \frac{\eta}{p} \sum_{1\leq j,k,l\leq d}\big(\chi_{jkl}-p\big)\left(u_{j}^{0}u_{k}^{0}u_{l}^{0}-u_{j}^{\star}u_{k}^{\star}u_{l}^{\star}\right)u_{j}^{0}u_{k}^{0}u_{l}^{\star}\Bigg)\\
 & \asymp\frac{\eta^{2}}{p}\sum_{1\leq j,k,l \leq d}\left(u_{j}^{0}u_{k}^{0}u_{l}^{0}+u_{j}^{\star}u_{k}^{\star}u_{l}^{\star}\right)^{2}\left(u_{j}^{0}u_{k}^{0}u_{l}^{\star}\right)^{2}
	\asymp\frac{\eta^{2} \mathrm{poly}\log (d) }{pd^{3}} \| \bm{u}^{\star} \|_2^8 \| \bm{u}^{0} \|_2^4 .
\end{align*}
In other words, the typical size of the variability of $\langle\bm{u}^{1},\bm{u}^{\star}\rangle$
is about the order of $\frac{\eta \mathrm{poly}\log (d)}{\sqrt{pd^{3}}} \| \bm{u}^{\star} \|_2^4 \| \bm{u}^{0} \|_2^2$, which dominates (in fact, is order-of-magnitudes larger than) the mean increment \eqref{eq:mean-increment} unless 
\begin{equation}
	\label{eq:sample-size-random-init}
	p\gtrsim \frac{\mathrm{poly}\log (d)}{d}.
\end{equation}
The sample size corresponding to \eqref{eq:sample-size-random-init} is, however,  considerably larger than the computation limit $p\asymp \frac{\mathrm{poly}\log (d)}{d^{1.5}}$. The presence of a large variance implies highly volatile dynamics of randomly initialized GD, thus casting doubt on its efficiency in the most challenging sample-starved regime.

In summary, the main issue stems from the quadratic dependence of  the expected  increment \eqref{eq:mean-increment} on the correlation $\langle \bm{u}^0, \bm{u}^{\star} \rangle$, which can be exceedingly small if $\bm{u}^0$ is randomly initialized.

%
%\begin{align*}
%\mathbb{E}\left[\bm{u}_{\perp}^{1}  \right] & =\left(1-3\eta\|\bm{u}^{0}\|_{2}^{2}\right)\bm{u}_{\perp}^{0}
%\end{align*}
%

\paragraph{Initialization via the tensor power method (TPM).} Another alternative for initialization is the tensor power method, which has recently gained popularity in the context of learning latent-variable models \cite{anandkumar2014tensor,anandkumar2017analyzing}. Nevertheless, the TPM (with random initialization) suffers from the same high-volatility issue as randomly initialized GD.  The argument for this would be nearly identical to the one presented above, and is hence omitted.  Instead, we invoke a perturbation analysis result in \cite[Theorem 5.1]{anandkumar2014tensor} to illustrate the insufficiency of the TPM.  

Recall that $\frac{1}{p}\bm{T} = \bm{T}^{\star} + \big( \frac{1}{p}\bm{T} - \bm{T}^{\star} \big) $. 
%To begin with,  
%\cite[Theorem 5.1]{anandkumar2014tensor} requires a polynomially large number of restart attempts \yxc{explain to me how large it is}, thus imposing enormous computational burdens. 
%Leaving out this matter, a more 
A critical issue is 
that  the perturbation bound in \cite[Theorem 5.1]{anandkumar2014tensor} requires the tensor perturbation to be exceedingly small, namely, 
\begin{equation}
	\big\| \tfrac{1}{p}\bm{T} - \bm{T}^{\star}   \big\| \lesssim 1/d. 
\end{equation}
This, however, cannot possibly hold if the sample size is merely $p\asymp \frac{\mathrm{poly}\log (d)}{d^{1.5}}$  (in which case one only expects a spectral norm bound on the order of $\|  p^{-1}\bm{T} - \bm{T}^{\star}  \| \lesssim \frac{1}{\mathrm{poly}\log (d)} $ shown in Corollary~\ref{cor:bound-PE} even in the absence of noise).  In light of all this, existing stability analysis of the TPM does not imply either sample efficiency or computational efficiency.

\section{Related work}

One of the most natural ideas for solving tensor completion is to first unfold the tensor data into matrices, followed by proper convex relaxation commonly adopted for low-rank matrix completion. 
Given that there are more than one ways to matricize a tensor, several prior work has  explored the design of matrix norms that can exploit the tensor structure more effectively \cite{tomioka2010estimation,gandy2011tensor,liu2013tensor,romera2013new,lu2016tensor,mu2014square}. Such algorithms have  been robustified to enable reliable recovery against sparse outliers as well \cite{goldfarb2014robust}. 
For the most part, however, such unfolding-based convex relaxation necessarily incur loss of structural information, which is particularly severe when handling odd-order tensors. The sample complexity developed for this paradigm is often sub-optimal vis-a-vis the computational limits (namely, minimal sample complexity achievable by polynomial-time algorithms).

% \cxc{not sure if is appropriate to add reference \cite{li2017convex} here or in the related work (which develops local convergence theories of using gradient descent for over-complete and non-orthogonal symmetric tensor problems with a general objective function) according to Reviewer 1}

% The local convergence theories of gradient descent~/~alternating minimization has been investigated in \cite{anandkumar2015learning,li2017convex} in different settings of tensor problems.

%unfolding\,/\,matricization operations necessarily incur loss of structural information

Motivated by the above sub-optimality issue,  \cite{yuan2016tensor,yuan2017incoherent} proposed to minimize instead the tensor nuclear norm subject to data constraints, which provably allows for reduced sample complexity. The issue, however, is that computing the tensor nuclear norm itself is already computationally intractable, thus limiting its applicability to even moderate-dimensional problems. 
Similar findings have also been discovered for tensor atomic norm minimization 
\cite{driggs2019tensor}. 
When restricted to polynomial-time algorithms, 
the best statistical guarantees  are often attained via convex relaxation tailored to the  sum-of-squares hierarchy  \cite{barak2016noisy}; the resulting computational cost, however, remains prohibitively high for practical large-scale problems. Another matrix nuclear norm minimization algorithm has been proposed based on promoting certain structures on certain factor matrices \cite{liu2014factor}.  Developing statistical guarantees is, however, not the focal point of this work.

Moving beyond convex relaxation,  a number of prior papers have developed nonconvex algorithms for tensor completion, examples including iterative hard thresholding \cite{rauhut2017low}, alternating minimization \cite{jain2014provable,wang2016tensor,xu2015parallel}, tensor SVD \cite{zhang2017exact}, optimization on manifold \cite{xia2017polynomial,kasai2016low,steinlechner2016riemannian}, proximal average algorithm with nonconvex regularizer \cite{yao2018scalable}, and  block coordinate decent  \cite{ji2016tensor,xu2013block}.  When it comes to the model considered herein, these algorithms either lack optimal statistical guarantees, or come with a computational cost that is significantly higher than a linear-time algorithm. 
% We also note that pairwise 

 The algorithm and theory that we develop are largely inspired by the recent advances  of nonconvex optimization algorithms for low-rank matrix recovery problems \cite{keshavan2010matrix,Se2010Noisy,candes2015phase,ChenCandes15solving,sun2016guaranteed,yi2016fast,chen2015fast}.  The main theoretical tool --- the leave-one-out analysis --- is a powerful technique that has proved successful in various other statistical problems \cite{el2015impact,chen2017spectral,abbe2017entrywise,ma2017implicit,zhong2017near,chen2018gradient,chen2019noisy,li2019non,chen2019inference,ding2018leave,pananjady2019value}. 
There are several major differences between the analysis of nonconvex tensor completion and that of nonconvex matrix recovery. For instance, our initialization scheme is substantially more complicated than the matrix recovery counterpart, thus requiring much more sophisticated analysis; in addition, the local convergence stage of tensor completion does not suffer from rotational ambiguity (which often appears in nonconvex matrix completion), and hence we only need to handle permutational ambiguity.

% \cxc{major difference} Compared to these earlier works, our main contribution lies in the initialization analysis, which is also one of the major technical difficulties of tensor completion compared to its matrix counterpart. 

% In constrast,     
% 

% \begin{align*}
% \text{minimize}_{\bm{T}}\quad & \sum_{i=1}^{3}\big\|\mathsf{unfold}^{(i)}(\bm{T})\big\|_{*}\\
% \text{subject to}\quad & \mathcal{P}(\bm{T})=\mathcal{P}(\bm{T}^{\star}).
% \end{align*}
%

In addition, the current paper focuses on non-adaptive uniform random sampling. If there is freedom in designing the sampling mechanism, then one can often expect improved performance; see \cite{krishnamurthy2013low, zhang2019cross} as examples.  Fundamental criteria that enable perfect low-CP-rank tensor completion have been studied in \cite{ashraphijuo2017fundamental}.

%\cite{bhaskara2014smoothed}

Tensor completion is simply a special example of the tensor recovery literature.  There is a large body of results tackling 
various other tensor recovery and estimation problems, including but not limited to tensor decomposition  \cite{kolda2001orthogonal,kolda2009tensor,anandkumar2014tensor,anandkumar2014guaranteed,tang2015guaranteed,kim2013robust,hopkins2016fast,ge2015escaping,zoubir2018robust,sidiropoulos2017tensor,sun2017provable,ge2017optimization},  tensor SVD and factorization \cite{zhang2018tensor,kilmer2013third,zhang2017exact}, and tensor regression and sketching \cite{rauhut2017low,hao2018sparse,chen2019non,hao2019sparse}.  The algorithmic ideas explored in this paper might have implications for these tensor-related problems as well.

% Finally, we remark that compared to the conference version \cite{cai2019nonconvex}, the current paper (1) extends the results presented therein 
% to a more general case where both the rank $r$ and the incoherence parameter $\mu$ are allowed to grow with $d$, (2) discusses how to handle asymmetric tensors, and (3) explains in detail the inadequacy of other initialization schemes (including both random initialization and tensor power methods). More numerical experiments have also been carried out and reported. 

% \yxc{Note that this paragraph will not appear in our Arxiv version.}

\section{Analysis}
\label{sec:analysis}

In this section, we outline the proof of Theorem~\ref{thm:main}. The proof of Corollary~\ref{corollary:tensor-entries} is deferred to Appendix~\ref{sec:corollary}. The analysis is divided into three parts:
\begin{itemize}
\item In Section~\ref{sec:analysis-GD}, we show that given an initial estimate sufficiently close to the ground truth, vanilla gradient descent converges linearly. These are formalized in Lemmas~\ref{lemma:U_frob_loss} and \ref{lemma:U_2inf_loss}.
\item Sections \ref{sec:analysis-subspace}-\ref{sec:analysis-retrieval} provide statistical guarantees for the two steps of the initialization procedure; see Theorems~\ref{thm:init}.
\item Under the assumptions of Theorem~\ref{thm:main}, one can see that the initialization satisfies the requirement of linear convergence of vanilla gradient descent. Therefore, Theorem~\ref{thm:main} immediately follows from the results in Sections~\ref{sec:analysis-GD}-\ref{sec:analysis-retrieval}.
\end{itemize}

\subsection{Analysis for local convergence of GD}
\label{sec:analysis-GD}

%\yxc{Needs to make clear about the global permutation issue.}

In this section, we demonstrate that: if the initialization is reasonably good, then vanilla gradient descent converges linearly to a solution with the desired statistical accuracy. We postpone the analysis for initialization to Sections \ref{sec:analysis-subspace}-\ref{sec:analysis-retrieval} for convenience of presentation.

\subsubsection{Preliminaries: gradient and Hessian calculation}

%First of all, using our notation $\tsym$ defined in \eqref{def:T_otimes_sym_square} and \eqref{def:T_otimes_sym_square-U}, we can write
First of all, using our notation $\times^\seq$ defined in \eqref{eq:T-v-u-seq-defn}, we can write
\begin{align}
%	\grad f \( \bm{U} \) &= \frac{1}{p} \PP_{\Omega} \Big( \sum\nolimits_{1\leq i \leq r} \bm{u}_i^{\ot 3} - \bm{T}^\star - \bm{E} \Big) \tsym \bm{U} \tsym \bm{U}.
	\grad f \( \bm{U} \) &= \frac{1}{ p} \PP_{\Omega} \Big( \sum\nolimits_{1\leq i \leq r} \bm{u}_i^{\ot 3} - \bm{T}^\star - \bm{E} \Big) \times^\seq_1 \bm{U} \times^\seq_2 \bm{U}.
\end{align}

Next, we find it convenient to define an auxiliary loss function    $f_\clean \( \bm{U} \) : \R^{d \times r} \to \R_+$ that corresponds to the noiseless case:
\begin{align}
	f_\clean \( \bm{U} \) = \frac{1}{6p} \, \Big\| \PP_{\Omega} \Big( \sum\nolimits_{1\leq i \leq r} \bm{u}_i^{\ot 3} - \bm{T}^\star \Big) \Big\|^2_{\frob}.
\end{align}
The gradient of $f_\clean$ w.r.t.~$\bm{u}_s$ ($1\leq s\leq r$) is thus given by
\begin{align}
	\grad_{\bm{u}_s} f_\clean \( \bm{U} \) &= \frac{1}{p} \, \PP_{\Omega} \Big( \sum\nolimits_{1\leq i \leq r} \bm{u}_i^{\ot 3} - \bm{T}^\star \Big) \times_1 \bm{u}_s \times_2 \bm{u}_s, \qquad 1\leq s \leq  r,
\end{align}
and hence one can write
\begin{align}
	\grad f_\clean \( \bm{U} \) &= \frac{1}{p} \, \PP_{\Omega} \Big( \sum\nolimits_{1\leq i \leq r} \bm{u}_i^{\ot 3} - \bm{T}^\star \Big) \times^\seq_1 \bm{U} \times^\seq_2 \bm{U}.
\end{align}
This clearly satisfies
\begin{align}
	\label{eq:gradient-grad-clean}
	\grad f \( \bm{U} \) = \grad f_\clean \( \bm{U} \) - \frac{1}{p} \mathcal{P}_{\Omega} (\bm{E}) \times^\seq_1 \bm{U} \times^\seq_2 \bm{U}.
\end{align}
Moreover, direct algebraic manipulations give that: for any matrix $\bm{V} = \[\bm{v}_1, \dots, \bm{v}_r \] \in \R^{d \times r}$, 
\begin{align}
\mathsf{vec} \( \bm{V} \)^\top \grad^2 f_\clean \(\bm{U} \) \mathsf{vec} \( \bm{V} \) & = \frac{1}{3p} \, \Big\| \PP_{\Omega} \Big( \sum\nolimits_{1\leq s \leq r} \bm{u}_s \ot \bm{u}_s \ot \bm{v}_s + \bm{u}_s \ot \bm{v}_s \otimes \bm{u}_s + \bm{v}_s \ot \bm{u}_s \ot \bm{u}_s \Big) \Big\|_\frob^2 \nonumber \\
& \quad + \frac{2}{p}
\Big \langle \PP_{\Omega} \Big( \sum\nolimits_{s \in [r]} \bm{u}_s^{\ot 3} - \bm{T}^\star \Big), \sum\nolimits_{s \in [r]} \bm{v}_s \otimes \bm{v}_s \otimes \bm{u}_s \Big \rangle, 
\label{hessian_qf}
\end{align}
where $\mathsf{vec} ( \bm{V} ) $ denotes the vectorization of $\bm{V}$.

\subsubsection{Local strong convexity and smoothness}

At the heart of our analysis is a crucial geometric property of the objective function, that is,  the noiseless loss function $f_\clean$ behaves like a locally strongly convex and smooth function. This fact, which is formally stated in the following lemma, is the key enabler of fast local convergence of vanilla GD. 

\begin{lemma}[{\bf Local strong convexity and smoothness}]
\label{lemma:RIC}
Suppose that the sample complexity and the rank satisfy
\begin{align}
\label{eq:RIC-asymp-sym}
p \geq c_0 \max \left\{ \frac{\log^3 d}{d^{3/2}}, \frac{\mu^2 r^2 \log d}{d^2} \right\} \qquad r \leq c_1 \sqrt{\frac{d}{\mu}}
\end{align}
for some sufficiently large (resp.~small) constant $c_0>0$ (resp.~$c_1 > 0$). Then with probability greater than $1- O(d^{-10})$,  
\begin{align}
	\frac{1}{2}\lambda^{\star 4/3}_{\min} \left\| \bm{V} \right\|_\frob^2 \leq
	 \mathsf{vec} \( \bm{V} \)^\top \grad^2 f_\clean \( \bm{U} \) \, \mathsf{vec} \( \bm{V} \)  \leq  
	%\lambda^{\star 4}_{\min} \leq \norm{\grad^2 f_\clean \( \bm{U} \)} \leq 
	4 \lambda^{\star 4/3}_{\max} \left\| \bm{V} \right\|_\frob^2
\end{align}
holds simultaneously for all $\bm{V}\in \mathbb{R}^{d\times r}$ and all $\bm{U}\in \mathbb{R}^{d\times r}$ obeying 
\begin{equation}
	\label{eq:property-U-cvx-smooth}
	\| \bm{U} - \bm{U}^\star \|_{\frob} \leq \delta \left\| \bm{U}^\star \right\|_{\frob} \quad \text{and} \quad \left\| \bm{U} - \bm{U}^\star \right\|_{2, \infty} \leq \delta  \left\| \bm{U}^\star \right\|_{2, \infty}.
\end{equation}
Here, $\delta \leq c_2 / ( \mu^{3/2}  r )$ for some sufficiently small constant $c_2 > 0$.  
%\yxc{check assumption on $\bm{V}$ later} \cxc{changed the assumption}
	%and any $\bm{x} \in \sphere^{d-1}$. 
\end{lemma}

\begin{proof}
See Appendix \ref{pf:lemma_RIC}.
\end{proof}

In order to invoke Lemma \ref{lemma:RIC}, one needs to make sure that the decision matrix $\bm{U}$ of interest (e.g.~$\bm{U}^t$ in the GD sequence) satisfies the condition \eqref{eq:property-U-cvx-smooth}. This, however, is a fairly stringent condition, as it requires $\bm{U}$ to be close to the truth in every single row.

\subsubsection{Leave-one-out gradient descent sequences}
\label{sec:LOO-sequence}

Motivated by the analytical framework developed for low-rank matrix recovery \cite{ma2017implicit,chen2019nonconvex}, we introduce the following leave-one-out sequences, which play a crucial role in guaranteeing that the entire trajectory $\{\bm{U}^t\}_{t \geq 0}$ satisfies the condition \eqref{eq:property-U-cvx-smooth}
 as required in Lemma \ref{lemma:RIC}.

%decoupling complicated statistical dependency in our  analysis. 

Specifically,  we define for each $1 \leq m \leq d$ the following auxiliary loss function:
\begin{align}
f^{\m} \( \bm{U} \) \triangleq \frac{1}{6p} \, \Big\| \PP_{\Omega_{-m}} \Big( \sum\nolimits_{1\leq s \leq r} \bm{u}_s^{\ot 3} - \bm{T}^\star - \bm{E} \Big) \Big\|_\frob^2 + \frac{1}{6} \, \Big\| \PP_{m} \Big( \sum\nolimits_{1\leq s \leq r} \bm{u}_s^{\ot 3} - \bm{T}^\star \Big) \Big\|_\frob^2,
	\label{eq:defn-fm-loo}
\end{align}
where
\begin{itemize}
	\item $\PP_{\Omega_m}$: the projection  onto the subspace of tensors supported on $\{(i, j, k) \in \Omega \colon i = m \text{ or } j = m \text{ or } k =m \}$;
	\item $\PP_{\Omega_{-m}}$: the projection  onto the subspace of tensors supported on $\{(i, j, k) \in \Omega \colon i \neq m \text{ and } j \neq m \text{ and } k \neq m \}$; 
	\item $\PP_{m}$: the projection  onto the subspace of tensors supported on $\{(i, j, k) \in [d]^3 \colon i = m \text{ or } j = m \text{ or } k =m \}$.
\end{itemize}
In words, this function is obtained by replacing all data at locations $\{(i, j, k) \in [d]^3 \colon i = m \text{ or } j = m \text{ or } k =m \}$ by their expected values, thus removing all randomness associated with this location  subset. 
The gradient of $f^{\m}(\bm{U})$ w.r.t.~$\bm{u}_s$ ($1\leq s\leq r$) can be computed as: 
\begin{align}
\label{grad_f_loo}
\begin{split}
\grad_{\bm{u}_s} f^{\m} (\bm{U}) &= \frac{1}{p} \PP_{\Omega_{-m}} \Big(\sum\nolimits_{1\leq s \leq r}\bm{u}_s^{\ot 3} - \bm{T}^\star - \bm{E} \Big) \times_1 \bm{u}_s \times_2 \bm{u}_s + \PP_{m} \Big(\sum\nolimits_{1\leq s \leq r} \bm{u}_s^{\ot 3} - \bm{T}^\star \Big) \times_1 \bm{u}_s \times_2 \bm{u}_s.
\end{split}
\end{align}

We then denote by $\big\{\bm{U}^{t, (m)} \big\}_{t \geq 0}$ the iterative sequence obtained by running gradient descent w.r.t.~the leave-one-out loss $f^{\m}(\cdot)$; see  Algorithm~\ref{alg:init_loo}. By construction, as long as $\bm{U}^{0, \m}$ is independent of the sampling locations and the noise associated with the locations $\{(i, j, k) \in \Omega \colon i = m \text{ or } j = m \text{ or } k =m \}$ (which holds true as detailed momentarily), then the entire  
 trajectory $\big\{\bm{U}^{t, (m)} \big\}_{t \geq 0}$ becomes statistically independent of such randomness.  This is a crucial property that allows us to decouple the complicated statistical dependency.  

%We define the following auxiliary tensor and matrices:
%\begin{align}
%\label{def:T_loo} &\bm{T}^{\m} = \PP_{-m} ( p^{-1}\bm{T} ) + \PP_{m} (\bm{T}^\star ) , \\
%\label{def:A_loo} &\bm{A}^{\m} = \text{mode-1 matricization of } \bm{T}^{\m} = \big[ \bm{A}^{\m}_1, \dots, \bm{A}^{\m}_d \big] \in \R^{d \times d^2}
%\end{align}
%
%Observe that $\bm{A}^{\m}_m = \bm{A}^\star_m$. We define $\bm{A}^{\m}_{\setm m} = \big[ \bm{A}^{\m}_1, \dots, \bm{A}^{\m}_{m-1}, \bm{A}^{\m}_{m+1}, \dots,\bm{A}^{\m}_d \big] \in \R^{d \times d(d-1)}$ and 
%\begin{align}
%\label{def:B_loo}&\bm{B}^{\m} =  \PP_{\diag} \Big( p \, \bm{A}^{\m}_{\setm m} \bm{A}^{\m \top}_{\setm m}  \Big) + (1-p) \bm{D}^{\m} +\PP_{\diag^{\perp}} \Big( \bm{A}^{\m}_{\setm m} \bm{A}^{\m \top}_{\setm m} \Big) + \bm{A}^{\star}_{m} \bm{A}^{\star \top}_{m}
%\end{align}
%where $\bm{D}^{\m}$ is a diagonal matrix in $\R^{d \times d}$ with entries 
%\begin{align*}
%& D^{\m}_{ii} = \sum_{k \neq m} \big( \bm{A}_{k}^{\star} \big)_{im}^2 , \qquad \text{for } i \neq m\\
%& D^{\m}_{mm} =  \sum_{k \neq m} \Big\| \big(\bm{A}_k^\star\big)_{m, :} \Big\|_2^2.
%\end{align*}

\begin{algorithm}[h]
\caption{The $m$-th leave-one-out sequence}
\label{alg:init_loo}
\begin{algorithmic}[1]
 % \Require{sampling  set $\Omega$, observed entries $\{ T_{i, j, k} \mid (i, j, k) \in \Omega \}$, true entries $\{T_{m, :, :}^\star, T_{:, m, :}^\star, T_{:, :, m}^\star \}$, sampling rate $p$, number of iterations $t_0$.}
%  \State {\bf Input:} sampling  set $\Omega$, observed entries $\{ T_{i, j, k} \mid (i, j, k) \in \Omega \}$, true entries $\{T_{m, :, :}^\star, T_{:, m, :}^\star, T_{:, :, m}^\star \}$, sampling rate $p$, number of iterations $t_0$.
  \State Generate an initial estimate $\bm{U}^{0,\m}$ via Algorithm \ref{alg:init_loo_subspace}.
  \For{$t = 0, 1, \dots, t_0-1$}
      \State $\bm{U}^{t+1, \m} = \bm{U}^{t, \m} - \eta_t \grad f^{\m} \big(\bm{U}^{t, \m} \big)$.
     \EndFor
\end{algorithmic}
\end{algorithm}

\subsubsection{Key lemmas}

The proof for local linear convergence of GD is inductive in nature, which proceeds on the basis of the following set of inductive hypotheses. 
As we shall see in Corollary~\ref{cor:init} in Section~\ref{sec:analysis-retrieval}, this set of inductive hypotheses --- modulo some global permutation --- is valid with high probability when $t=0$.  In order to simplify presentation, we remove the consideration of the global permutation factor throughout this section (namely, we assume that the following holds for $\bm{U}^0 \bm{\Pi}^0$ with some permutation matrix $\bm{\Pi}^0\in \mathbb{R}^{r\times r}$ obeying $\bm{\Pi}^0 = \bm{I}$.  Our inductive hypotheses are summarized as follows:  
%\yxc{Fill in all required conditions now.}
%

\bigskip
\noindent{\bf Key hypotheses for the gradient update stage: }
\begin{subequations}	
\label{hyp:t_step}
\begin{align}
\label{hyp:U_frob_loss}
 \big\| \bm{U}^{t} - \bm{U}^\star \big\|_\frob &\leq \( C_1 \rho^t \EE_\local + C_2 \frac{\sigma}{\lammin^{\star}} \sqrt{\frac{ d \log d }{p}}  \) \norm{\bm{U}^\star}_{\frob}; \\
\label{hyp:U_2inf_loss}
 \big\| \bm{U}^{t} - \bm{U}^\star \big\|_{2, \infty} &\leq \( C_3 \rho^t \EE_\local + C_4 \frac{\sigma}{\lammin^{\star}} \sqrt{\frac{ d \log d}{p}}   \) \norm{\bm{U}^\star}_{2, \infty}; \\
\label{hyp:U_loo_frob_diff}
 \big\| \bm{U}^{t} - \bm{U}^{t, \m} \big\|_{\frob} &\leq  \( C_5 \rho^t \EE_\local + C_6 \frac{\sigma}{\lammin^{\star}} \sqrt{\frac{ d \log d}{p}}   \) \norm{\bm{U}^\star}_{2, \infty}; \\
\label{hyp:U_loo_row_loss}
 \left\| \big( \bm{U}^{t, \m} - \bm{U}^{\star} \big)_{m, :} \right\|_{2} &\leq  \( C_7 \rho^t \EE_\local + C_8 \frac{\sigma}{\lammin^{\star}} \sqrt{\frac{ d \log d}{p}}   \) \norm{\bm{U}^\star}_{2, \infty};
\end{align}
\end{subequations} 	
for some quantity $\EE_\local>0$ (depending possibly on $\mu$ and $r$) and some constants $C_1,\cdots,C_8>0$.  
These exist a few straightforward consequences of  the hypotheses \eqref{hyp:t_step}, which we record in the following lemma. 
\begin{lemma}
\label{lemma:hyp_aux}
Assume that the hypotheses \eqref{hyp:t_step} hold, then we have
\begin{align}
\label{U_loo_frob_loss}
\big\| \bm{U}^{t, \m} - \bm{U}^\star \big\|_{\frob} & \leq  \( 2 C_1 \rho^t \EE_\local + 2 C_2 \frac{\sigma}{\lammin^{\star}} \sqrt{\frac{ d \log d }{p}}   \) \norm{\bm{U}^\star}_{\frob}, \\
\label{U_loo_2inf_loss}
\big\| \bm{U}^{t, \m} - \bm{U}^\star \big\|_{2, \infty} & \leq \( \( C_3 + C_5 \) \rho^t \EE_\local + \( C_4 + C_6 \) \frac{\sigma}{\lammin^{\star}} \sqrt{\frac{ d \log d}{p}}   \) \norm{\bm{U}^\star}_{2, \infty} .
\end{align}
\end{lemma}
\begin{proof}
See Appendix~\ref{pf:hyp_aux}.
\end{proof}

Our proof for the hypotheses \eqref{hyp:t_step} is inductive in nature: we would like to show that if the hypotheses in \eqref{hyp:t_step} hold for the $t$-th iteration, then they continue to be valid  for the $(t+1)$-th iteration.  
We shall justify each of the above hypotheses inductively through the following lemmas. 

\begin{lemma}
\label{lemma:U_frob_loss}
Suppose that  
\begin{align*}
p \geq c_0 \frac{\mu^3 r^2 \log^3 d}{d^{3/2}}, \quad \frac{\sigma}{\lambda_{\min}^\star} \leq c_1 \frac{\sqrt{p}}{\mu^{3/2} r \sqrt{d \log d}} , \quad \text{and} \quad r \leq c_2 \sqrt{\frac{d}{\mu}}
\end{align*}
for some sufficiently large constant $c_0 > 0$ and some sufficiently small constant $c_1, c_2 > 0$.
Assume that the hypotheses~\eqref{hyp:t_step} hold for the $t$-th iteration and $\EE_\local \leq c_3 / \big( \mu^{3/2} r \big)$ for some sufficiently small constant $c_3> 0$. Then with probability at least $1-O(d^{-10})$,
\begin{align}
\label{claim:U_frob_loss}
\left\| \bm{U}^{t+1} - \bm{U}^\star \right\|_\frob \leq \( C_1 \rho^{t+1} \EE_\local + C_2 \frac{\sigma}{\lammin^{\star}} \sqrt{\frac{ d \log d}{p}}   \) \norm{\bm{U}^\star}_{\frob} ,
\end{align}
provided that $0 < \eta \leq  \lammin^{\star 4/3 } / \big(32 \lammax^{\star 8 / 3} \big)$, $1 - \big( \lammin^{\star 4/3}  / 5 \big) \eta \leq \rho < 1$, and $C_2$ is sufficiently large.
\end{lemma}

\begin{proof}
See Appendix~\ref{pf:U_frob_loss}.
\end{proof}

\begin{lemma}
\label{lemma:U_loo_frob_diff}
Suppose that  
\begin{align*}
p \geq c_0 \frac{\mu^3 r^2 \log^3 d}{d^{3/2}}, \quad \frac{\sigma}{\lambda_{\min}^\star} \leq c_1 \frac{\sqrt{p}}{\mu^{3/2} r \sqrt{d \log d}} , \quad \text{and} \quad r \leq c_2 \sqrt{\frac{d}{\mu}}
\end{align*}
for some sufficiently large constant $c_0 > 0$ and some sufficiently small constant $c_1, c_2 > 0$.
Assume that the hypotheses~\eqref{hyp:t_step} hold for the $t$-th iteration and $\EE_\local \leq c_3 / \big( \mu^{3/2} r \big)$ for some sufficiently small constant $c_3> 0$. Then with probability at least $1-O(d^{-10})$, one has
\begin{align}
\label{claim:U_loo_frob_diff}
\big\| \bm{U}^{t+1, \m} - \bm{U}^{t+1} \big\|_{\frob} \leq  \( C_5 \rho^{t+1} \EE_\local + C_6 \frac{\sigma}{\lammin^{\star}} \sqrt{\frac{ d \log d}{p}}   \) \norm{\bm{U}^\star}_{2, \infty} ,
\end{align}
provided that $0 < \eta \leq  \lammin^{\star 4/3 } / \big(32 \lammax^{\star 8 / 3} \big)$, $1 -  \big( \lambda_{\min}^{\star 4/3} / 5 \big) \eta \leq \rho < 1$ and $C_6$ is sufficiently large.
\end{lemma}

\begin{proof}
See Appendix~\ref{pf:U_loo_frob_diff}.
\end{proof}

\begin{lemma}
\label{lemma:U_loo_row_loss}
Suppose that  
\begin{align*}
p \geq c_0 \frac{\mu^3 r^2 \log^3 d}{d^{3/2}}, \quad \frac{\sigma}{\lambda_{\min}^\star} \leq c_1 \frac{\sqrt{p}}{\mu^{3/2} r \sqrt{d \log d}} , \quad \text{and} \quad r \leq c_2 \sqrt{\frac{d}{\mu}} 
\end{align*}
for some sufficiently large constant $c_0 > 0$ and some sufficiently small constant $c_1, c_2 > 0$.
Assume that the hypotheses~\eqref{hyp:t_step} hold for the $t$-th iteration and $\EE_\local \leq c_3 / \big( \mu^{3/2} r \big)$ for some sufficiently small constant $c_3> 0$. Then with probability at least $1-O(d^{-10})$, one has
\begin{align}
\label{claim:U_loo_row_loss}
\left\| \big( \bm{U}^{t+1, \m} - \bm{U}^{\star} \big)_{m, :} \right\|_{2} \leq  \( C_7 \rho^{t+1} \EE_\local + C_8 \frac{\sigma}{\lammin^{\star}} \sqrt{\frac{ d \log d}{p}}   \) \norm{\bm{U}^\star}_{2, \infty} ,
\end{align}
provided that $0 < \eta \leq  \lammin^{\star 4/3 } / \big(32 \lammax^{\star 8 / 3} \big)$, $1 - \big( \lambda_{\min}^{\star 4/3} / 5 \big) \eta  \leq \rho < 1$, $C_7$ and $C_8$ are sufficiently large.
\end{lemma}

\begin{proof}
See Appendix~\ref{pf:U_loo_row_loss}.
\end{proof}

\begin{lemma}
\label{lemma:U_2inf_loss}
Suppose that  
\begin{align*}
p \geq c_0 \frac{\mu^3 r^2 \log^3 d}{d^{3/2}}, \quad \frac{\sigma}{\lambda_{\min}^\star} \leq c_1 \frac{\sqrt{p}}{\mu^{3/2} r \sqrt{d \log d}} , \quad \text{and} \quad r \leq c_2 \sqrt{\frac{d}{\mu}}
\end{align*}
for some sufficiently large constant $c_0 > 0$ and some sufficiently small constant $c_1, c_2 > 0$.
Assume that the hypotheses~\eqref{hyp:t_step} hold for the $t$-th iteration and $\EE_\local \leq c_3 / \big( \mu^{3/2} r \big)$ for some sufficiently small constant $c_3> 0$. Then with probability at least $1-O(d^{-10})$, one has
\begin{align}
\label{claim:U_2inf_loss}
\left\| \bm{U}^{t+1} - \bm{U}^\star \right\|_{2, \infty} \leq \( C_3 \rho^{t+1} \EE_\local + C_4 \frac{\sigma}{\lammin^{\star}} \sqrt{\frac{ d \log d}{p}}   \) \norm{\bm{U}^\star}_{2, \infty} ,
\end{align}
provided that $0 < \eta \leq  \lammin^{\star 4/3 } / \big(32 \lammax^{\star 8 / 3} \big)$, $1 - \big( \lambda_{\min}^{\star 4/3} / 5 \big) \eta  \leq \rho < 1$, $C_3 / \( C_5 + C_7 \)$ and $C_4 / \( C_6 + C_8 \)$ are both sufficiently large.
\end{lemma}

\begin{proof}
See Appendix~\ref{pf:U_2inf_loss}.
\end{proof}

The proofs of the above key lemmas are postponed to Appendix \ref{sec:proof-local-convergence}. 

\subsection{Analysis for initialization:~Part 1~(subspace estimation) }
\label{sec:analysis-subspace}

\subsubsection{Key results}

The aim of this subsection is to demonstrate that the subspace estimate $\bm{U}$ computed by Algorithm~\ref{alg:init} is sufficiently close to the space spanned by the true tensor factors. Given that the columns of $\bm{U}^\star=[\bm{u}_1^{\star},\cdots,\bm{u}_r^{\star}]$ are in general not orthogonal to each other, we shall define $\bm{U}_{\mathsf{orth}}^\star\in\mathbb{R}^{d\times r}$ as follows  (obtained by proper orthonormalization) :
\begin{align}
	\label{defn:U-orth}
	\bm{U}_{\mathsf{orth}}^{\star} :=
	{\bm{U}}^{\star}\big({\bm{U}}^{\star\top}{\bm{U}}^{\star}\big)^{-\frac{1}{2}}. \qquad 
	%\overline{\bm{U}}^{\star} =\Big[\tfrac{1}{\|\bm{u}_{1}^{\star}\|_{2}}\bm{u}_{1}^{\star},\cdots,\tfrac{1}{\|\bm{u}_{r}^{\star}\|_{2}}\bm{u}_{r}^{\star}\Big]\in\mathbb{R}^{d\times r}. 	
	%\bm{U}_{\mathsf{scaled}}^{\star}\big(\bm{U}_{\mathsf{scaled}}^{\star\top}\bm{U}_{\mathsf{scaled}}^{\star}\big)^{-\frac{1}{2}}
	%\qquad
	%\text{with }\bm{U}_{\mathsf{scaled}}^{\star} \triangleq\left[\|\bm{u}_{1}^{\star}\|_{2}^{2}\bm{u}_{1}^{\star},\cdots,\|\bm{u}_{r}^{\star}\|_{2}^{2}\bm{u}_{r}^{\star}\right]\in\mathbb{R}^{d\times r}.
\end{align}
This matrix $\bm{U}_{\mathsf{orth}}^{\star}$ reflects the rank-$r$ principal subspace of $\bm{A}^{\star}\bm{A}^{\star\top}=\sum_{i} \|\bm{u}_{i}^{\star}\|_{2}^{4}\bm{u}_{i}^{\star}\bm{u}_{i}^{\star\top} $, where we recall that $\bm{A}^{\star}\in \mathbb{R}^{d\times d^2}$ is the mode-1 matricization of $\bm{T}^{\star}$. In addition, we  define the rotation matrix
%We also define $\bm{H} = \bm{U}^\top \bm{U}^\star_\orth$ and its sign matrix $\sgn(\bm{H})$ as
%\begin{align}
%\sgn \(\bm{H} \) := \bm{W} \bm{V}^\top,
%\end{align}
%where columns of $\bm{W}$ and $\bm{V}$ are left and right singular vectors of $\bm{H}$. It can be checked that
%
\begin{align}
	\label{defn:R-rotation}
	\bm{R}  := \argmin_{\bm{Q} \in \mathcal{O}^{r \times r}} ~\norm{\bm{U} \bm{Q} - \bm{U}^\star_\orth}_{\mathrm{F}},
\end{align}
where $\mathcal{O}^{r \times r}$ stands for the set of $r\times r$ orthonormal matrices. This can be viewed as the global rotation matrix that best aligns the two subspaces represented by $\bm{U}$ and $\bm{U}^{\star}_\orth$ respectively. 

Equipped with the above notation, we can invoke \cite[Corollary~1]{cai2019subspace} to arrive at the following lemma, which  upper bounds  the distance between our subspace estimate $\bm{U}$ and the ground truth $\bm{U}^\star_\orth$. 
\begin{lemma}
\label{lemma:eigsp_dist_op}
There exist some universal constants $c_{0},c_{1},c_{2}>0$
such that if \begin{align*}
	p\geq c_{0} \frac{ \mu^2 r \log^{2}d}{d^{3/2}}, \quad \frac{\sigma}{\lambda_{\min}^\star}\leq c_{1} \frac{\sqrt{p}}{ d^{3/4} \sqrt{\log d}}, \quad \text{and} \quad r\leq c_2 \sqrt{\frac{d}{\mu}},
		\end{align*}
	then with probability  $1- O \( d^{-10} \)$, the subspace estimate $\bm{U}$ computed by Algorithm \ref{alg:init} obeys
\begin{subequations}
\begin{align}
\left\Vert \bm{U}\bm{R}-\bm{U}_{\mathsf{orth}}^\star \right\Vert  & \lesssim \mathcal{E}_{\mathsf{se}},\label{claim:eigsp_dist_op}\\
\left\Vert \bm{U}\bm{R}-\bm{U}_{\mathsf{orth}}^\star \right\Vert _{2,\infty} & \lesssim \mathcal{E}_{\mathsf{se}} \sqrt{\frac{\mu r}{d}} ,\label{claim:thm:eigsp_dist_2inf}
\end{align}
\end{subequations} where $\bm{U}_{\mathsf{orth}}^\star$ and $\bm{R}$ are defined respectively in \eqref{defn:U-orth} and \eqref{defn:R-rotation}, and
\begin{align}
\mathcal{E}_{\mathsf{se}} :=\frac{\mu^2 r\log d}{d^{3/2}p}+\sqrt{\frac{\mu^2 r\log d}{d^{2}p}}+\frac{\sigma^{2}}{\lambda_{\min}^{\star2}}\frac{d^{3/2}\log d}{p}+\frac{\sigma }{\lambda_{\min}^{\star}}\sqrt{\frac{d\log d}{p}}+\frac{\mu r}{d}.
	\label{def:err-sp}
\end{align}
\end{lemma}

In a nutshell, Lemma~\ref{lemma:eigsp_dist_op} asserts that: under our sample size, noise and rank conditions, Algorithm~\ref{alg:init} produces reliable estimates of the subspace spanned by the low-rank tensor factors $\{\bm{u}_i^{\star}\}_{1\leq i\leq r}$. The theorem quantifies the subspace distance in terms of both the spectral norm and  $\|\cdot\|_{2,\infty}$,  where the latter bound often reflects a considerably stronger sense of proximity compared to the former one.

As it turns out, in order to facilitate analysis for the subsequent stages, we need to introduce certain leave-one-out sequences as well,  which we detail in the next subsection.

\subsubsection{Leave-one-out sequences for subspace estimation}

The key idea of the leave-one-out analysis is to create auxiliary leave-one-out sequences that are (1) independent of a small fraction of the data; (2) sufficiently close to the true estimates.  
%Before proceeding, we first introduce the following notation. For any matrix $\bm{A} \in \R^{d \times d^2}$, we write 
%%
%\begin{align}
%	\bm{A} = [\bm{A}_1, \dots, \bm{A}_d] \in \mathbb{R}^{d \times d^2}, \qquad \bm{A}_m \in \mathbb{R}^{d\times d}~(1\leq m\leq d), 
%\end{align}
%%
%where we denote by  $\bm{A}_m\in \R^{d \times d}$ the $m$-th squared submatrix of $\bm{A}$.
% namely,  $[ ( \bm{A}_m)_{i, j} ]_{1 \leq i, j \leq d} = \[ A_{i, l} \]_{1 \leq i \leq d, (m-1) \cdot j + 1 \leq l \leq m \cdot j }$.  
%\begin{align}
%\[ \( \bm{A}_m\)_{i, j} \]_{1 \leq i, j \leq d} = \[ A_{i, l} \]_{1 \leq i \leq d, (m-1) \cdot j + 1 \leq l \leq m \cdot j }.
%\end{align}
%for each $m \in [d]$, i.e. $\bm{A} = [\bm{A}_1, \dots, \bm{A}_d]$.
%
%Recall the definition of $\PP_{-m}$ and $\PP_{m}$ in Section \ref{sec:LOO-sequence}. With these in place, we can introduce
% the following auxiliary tensor and $d\times d^2$-dimensional matrix for each $1\leq m \leq d$: 
%%
%\begin{align}
%\label{def:T_loo} \bm{T}^{\m} & := \PP_{-m} \( \bm{T} \) + p \PP_{m} \( \bm{T}^\star \) \in \R^{d \times d \times d}, \\
%	\label{def:A_loo} \bm{A}^{\m} & := \big [ \bm{A}^{\m}_1, \dots, \bm{A}^{\m}_d \big ] := \text{mode-1 matricization of } \tfrac{1}{p}\bm{T}^{\m} .
%\end{align}
We introduce the following auxiliary tensor and $d\times d^2$-dimensional matrix for each $1\leq m \leq d$: 
\begin{align}
\label{def:T_loo} \bm{T}^{\m} & := \PP_{\Omega_{-m}} \( \bm{T} \) + p \PP_{m} \( \bm{T}^\star \) \in \R^{d \times d \times d}, \\
	\label{def:A_loo} \bm{A}^{\m} & := \text{mode-1 matricization of } \tfrac{1}{p}\bm{T}^{\m} .
\end{align}
%
%where $\PP_{-m}$ and $\PP_{m}$ are defined in Appendix~\ref{notation_summary}.
%
By construction, $\bm{T}^{\m}$ and $\bm{A}^{\m}$ are independent of $\PP_{\Omega_m}\( \bm{E} \)$, where we recall that 
\begin{align}
	\Omega_{-m} &:= \{(i, j, k) \in \Omega \colon i \neq m \text{ and } j \neq m \text{ and } k \neq m \}, \\
	\Omega_m &:= \{(i, j, k) \in \Omega \colon i = m \text{ or } j = m \text{ or } k =m \}. 
\end{align}
%
%If we also introduce the following $d\times d^2$ matrix 
%%
%\begin{align}
%\label{def:A_true} \bm{A}^{\star} &= \big [ \bm{A}^{\star}_1, \dots, \bm{A}^{\star}_d \big ] = \text{mode-1 matricization of } \bm{T}^{\star} .
%\end{align}
%%
%then the above construction procedure guarantees that 
%%
%\begin{align}
%	\bm{A}_m^{(m)} = \bm{A}^\star_m, \qquad 1\leq m\leq d. \label{eq:Am-Astar-relation}
%\end{align}
%%
%Furthermore, we set
%%
%\begin{eqnarray}
%	& \bm{A}^{\m}_{\setm m} := \[ \bm{A}^{\m}_1, \dots, \bm{A}^{\m}_{m-1}, \bm{A}^{\m}_{m+1}, \dots,\bm{A}^{\m}_d \] &\in \mathbb{R}^{d\times d(d-1)}, \\
%	& \bm{A}^{\star}_{\setm m} := \[ \bm{A}^{\star}_1, \dots, \bm{A}^{\star}_{m-1}, \bm{A}^{\star}_{m+1}, \dots,\bm{A}^{\star}_d \] &\in \mathbb{R}^{d\times d(d-1)}, 
%\end{eqnarray}
%%
%which are obtained by dropping the squared submatrices $\bm{A}_m^{(m)}$ and $\bm{A}_m^{\star}$ from $\bm{A}^{(m)}$ and $\bm{A}^{\star}$, respectively. 
%

We are now ready to introduce the auxiliary leave-one-out procedure for subspace estimation.   Similar to the matrix $\bm{B}$ in Algorithm~\ref{alg:init} (whose eigenspace serves as an estimate of the column space of $\bm{U}^\star$), we define an auxiliary matrix $\bm{B}^{\m} \in \mathbb{R}^{d\times d}$ as follows:
%
%\begin{align}
%\label{def:B_loo}
%%	\bm{B}^{\m} =  \PP_{\diag} \( p \, \bm{A}^{\m}_{\setm m} \bm{A}^{\m \top}_{\setm m}  \) + \PP_{\offdiag} \( \bm{A}^{\m}_{\setm m} \bm{A}^{\m \top}_{\setm m} \)  + \bm{A}^\star_{m} \bm{A}^{\star \top}_{m} + \bm{D}^{\m}, \\
%	\bm{B}^{\m} =   \PP_{\offdiag} \( \bm{A}^{\m}_{\setm m} \bm{A}^{\m \top}_{\setm m} + \bm{A}^\star_{m} \bm{A}^{\star \top}_{m} \)  , 
%\end{align}
%
\begin{align}
\label{def:B_loo}
	\bm{B}^{\m} =   \PP_{\offdiag} \big( \bm{A}^{\m} \bm{A}^{\m \top} \big)  , 
\end{align}
where 
%$\PP_{\diag}(\cdot)$ and 
$\PP_{\offdiag}(\cdot)$ (as already defined in Section \ref{sec:subspace-estimate}) extracts out  off-diagonal entries from a matrix.  
%Here, 
% $\bm{D}^{\m}\in \R^{d \times d}$ is a diagonal matrix  with entries
% %
% \begin{align*}
% & D^{\m}_{i,i} = \( 1- p \) \sum_{k: k \neq m} \big( \bm{A}_{k}^\star \big)_{i,m}^2 -  \( d - 1 \)^2 \sigma^2, \qquad \text{for } i \neq m\\
% & D^{\m}_{m,m} 
% %= \( 1-p \)  \sum_{k \neq m} \Big\| \big(\bm{A}_k^\star \big)_{m, :} \Big\|_2^2.
% = \( 1 - p \) \big\| \(\bm{A}_{\setm m}^\star \)_{m, :} \big\|_2^2.
% \end{align*}
%
The rationale is simple:  it can be easily verified that 
\begin{align}
	\label{eq:defn-Bstar-EB}
	\E \big[ \bm{B}^{\m} \big] =\bm{B}^\star - \PP_{\diag} \( \bm{B}^\star \), \qquad \bm{B}^\star := \bm{A}^\star \bm{A}^{\star\top},
\end{align}
where  $\PP_{\diag}(\cdot)$ extracts out the diagonal entries of the matrix. 
This gives hope that the eigenspace of $\bm{B}^{\m}$ is also a reliable estimate of the column space of $\bm{U}^\star$, provided that the diagonal entries of $\bm{B}^\star$ are sufficiently small.
Consequently, we shall compute $\bm{U}^{0, \m} \in \mathbb{R}^{d\times r}$ --- a matrix whose columns are the top-$r$ leading eigenvectors of $\bm{B}^{\m}$.  The procedure is summarized in Algorithm~\ref{alg:init_loo_subspace}.

\begin{algorithm}[hbt!]
\caption{The $m$-th leave-one-out sequence for spectral initialization}
\label{alg:init_loo_subspace}
\begin{algorithmic}[1]
 % \Require{sampling  set $\Omega$, observed entries $\{ T_{i, j, k} \mid (i, j, k) \in \Omega \}$, true entries $\{ T^\star_{m, :, :}, T^\star_{:, m, :}, T^\star_{:, :, m} \}$, sampling rate $p$.}
%  \State {\bf Input:} sampling  set $\Omega$, observed entries $\{ T_{i, j, k} \mid (i, j, k) \in \Omega \}$, true entries $\{ T^\star_{m, :, :}, T^\star_{:, m, :}, T^\star_{:, :, m} \}$, sampling rate $p$.
  \State Let $\bm{U}^{\m} \bm{\Lambda}^{\m} \bm{U}^{\m \top}$ be the rank-$r$ eigen-decomposition of $\bm{B}^{\m}$ defined in \eqref{def:B_loo}.
  \State Generate the initial estimate $\bm{U}^{0, \m} \in \mathbb{R}^{d\times r}$ from $\bm{U}^{\m} \in \mathbb{R}^{d\times r}$ using Algorithm \ref{alg:localization_loo}. 
\end{algorithmic}
\end{algorithm}

The following lemma plays a crucial role in our analysis, which formalizes the fact that the leave-one-out version $\bm{U}^{\m}$ obtained by Algorithm~\ref{alg:init_loo_subspace} is extremely close to $\bm{U}$.
\begin{lemma}
\label{lemma:U_U_loo_dist_op}
There exist some universal constants $c_{0},c_{1},c_{2}>0$
such that if \begin{align*}
	p\geq c_{0} \frac{ \mu^2 r \log^{2}d}{d^{3/2}}, \quad \frac{\sigma}{\lambda_{\min}^\star}\leq c_{1} \frac{\sqrt{p}}{ d^{3/4} \sqrt{\log d}}, \quad \text{and} \quad r\leq c_2 \sqrt{\frac{d}{\mu}},
		\end{align*}
	then with probability  $1- O \( d^{-10} \)$, the subspace estimate $\bm{U}^{(m)}$ computed by Algorithm \ref{alg:init_loo_subspace} obeys
\begin{align}
	\big\| \bm{U} \bm{U}^{\top} -\bm{U}^{\m}\bm{U}^{\m \top} \big\|_{\mathrm{F}} & \lesssim \mathcal{E}_{\mathsf{loo}} \sqrt{\frac{\mu r}{d}}
\end{align}
simultaneously for all $1\leq m \leq  d$, where
\begin{align}
\mathcal{E}_{\mathsf{loo}} :=\frac{\mu^2 r  \log d}{d^{3/2}p}+\sqrt{\frac{\mu^2 r\log d}{d^{2}p}}+\frac{\sigma^{2}}{\lambda_{\min}^{\star2}}\frac{d^{3/2}\log d}{p}+\frac{\sigma }{\lambda_{\min}^{\star}}\sqrt{\frac{d\log d}{p}}.\label{def:err-loo}
\end{align}
\end{lemma}

Lemma~\ref{lemma:U_U_loo_dist_op} follows immediately from the analysis of \cite[Lemma~4]{cai2019subspace}. As a remark, the construction of the leave-one-out sequences herein is slightly different from the one in \cite{cai2019subspace}. 
%More specifically, the leave-one-out sequences in \cite{cai2019subspace} are obtained by replacing the mode-1 slice with its expectation value, while we replace all of three modes with their population counterparts here. 
However, it is straightforward to adapt the proof of \cite{cai2019subspace}  to the case considered herein.  We therefore omit the proof for the sake of brevity.

\subsection{Analysis for initialization:~Part 2~(retrieval of individual tensor factors)}
\label{sec:analysis-retrieval}

\subsubsection{Main results and leave-one-out sequences}

This section justifies that the procedure presented in Algorithm~\ref{alg:localization} allows to disentangle the tensor factors.  
For notational simplicity, we let
\begin{align}
	\overline{\bm{u}}_i^{\star} := \bm{u}_i^{\star} \,/ \left\| \bm{u}_i^{\star} \right\|_2, \qquad \lambda_i^{\star} := \left\| \bm{u}_i^{\star} \right\|_2^3, \qquad 1\leq i\leq d .  \label{eq:defn-ubar-lambda}
\end{align}
Our result is this: 
\begin{theorem}
\label{thm:init}
Fix any arbitrary small constant $\delta > 0$.  Assume that 
\begin{align}
& p\geq c_{0} \frac{ \mu^2  r^4 \log^{4}d}{d^{3/2}}, \quad \frac{\sigma}{\lambda_{\min}^\star}\leq c_{1} \frac{\sqrt{p}}{ r^{3/2} d^{3/4} \log^2 d}, \quad r\leq c_2 \(\frac{d}{\mu^6  \log^6 d}\)^{1/6}, \nonumber \\
& L = c_3 r^{2\kappa^2} \log^{3/2} r, \quad \epsilon_{\mathsf{th}} = c_4 \Bigg\{ \frac{\mu r \log d}{d\sqrt{p}}+\frac{\sigma}{\lambda_{\min}^\star}\sqrt{\frac{rd\log^{2}d}{p}}+\sqrt{\frac{\mu r  \log d}{d}} \Bigg\} \label{eq:init_cond}
\end{align}
for some sufficiently large universal constant $c_0, c_3 >0$ and some sufficiently small universal constants $c_1, c_2, c_4 >0$.
Then with probability exceeding $1- \delta$, there exists a permutation $\pi(\cdot): [d]\mapsto[d]$ such that for all $1 \leq i \leq r$, the tensor factors $\{\bm{w}^{i} \}_{i = 1}^r$ returned by Algorithm~\ref{alg:localization} satisfy
\begin{subequations}
\label{eq:thm:init-all}
\begin{align}
\label{eq:thm:init_2_loss}  
\big\| \bm{w}^{i} - \overline{\bm{u}}^\star_{\pi(i)} \big\|_2 & \lesssim \frac{\mu  r \log d}{d \sqrt{p}} + \frac{\sigma}{\lambda_{\min}^\star} \sqrt{\frac{r d \log^2 d}{p}} + \sqrt{\frac{\mu  r \log d}{d}}; \\
\label{eq:thm:init_inf_loss}  
\big\| \bm{w}^{i} - \overline{\bm{u}}^\star_{\pi(i)}  \big\|_\infty & \lesssim \Bigg\{\frac{\mu^2  r \log^{4}d}{d^{3/2}p}+\frac{\mu r \log^{3}d}{d\sqrt{p}}+\frac{\sigma^{2}}{\lambda_{\min}^{\star2}}\frac{d^{3/2}\log^{4}d}{p} +\frac{\sigma}{\lambda_{\min}^{\star}}\sqrt{\frac{ r d\log^{6}d}{p}}+\sqrt{\frac{\mu r \log^{2}d}{d}}\Bigg\}  \sqrt{\frac{\mu r}{d}}; \\
\label{eq:thm:init_coeff_loss} 
 \big| \lambda_{i} - \lambda_{\pi(i)}^\star \big| & \lesssim \Bigg\{ \frac{\mu r \log d}{d \sqrt{p}} + \frac{\sigma}{\lambda_{\min}^\star} \sqrt{\frac{r d \log^2 d}{p}} + \sqrt{\frac{\mu r \log d}{d}} \Bigg\} \lambda^\star_{\pi(i)}.
\end{align}
\end{subequations}
\end{theorem}

In short, this theorem asserts that the estimates returned by Algorithm~\ref{alg:localization} are --- up to global permutation --- reasonably close to the ground truth under our sample size and noise conditions. In order to establish this theorem and in order to provide initial guesses for the leave-one-out GD sequences,  we need to produce a leave-one-out sequence tailored to this part of the algorithm.  Such auxiliary sequences are generated in a similar spirit as the previous ones, and we  summarize them in Algorithm~\ref{alg:localization_loo}. As usual, the resulting leave-one-out estimates $\big\{ \lambda^{\m}_i, \bm{w}^{i, \m} \big\}_{i = 1}^r$ are statistically independent of $\PP_{\Omega_m} \( \bm{E} \)$.

In what follows, we gather a few key properties of the leave-one-out estimates,  which play a crucial role in the analysis.

\begin{algorithm}[hbt!]
\caption{The $m$-th leave-one-out sequence for retrieving individual tensor components}
\label{alg:localization_loo}
\begin{algorithmic}[1]
%\State {\bf Input:} sampling  set $\Omega$, observed entries $\{ T_{i, j, k} \mid (i, j, k) \in \Omega \}$, true entries $\{T_{m, :, :}^\star, T_{:, m, :}^\star, T_{:, :, m}^\star \}$, sampling rate $p$, restart number $L$, threshold $\epsilon_{\mathsf{th}}$, subspace estimate $\bm{U} ^{\m}$ given by Algorithm~\ref{alg:init_loo}.
\State {\bf Input:} restart number $L$, threshold $\epsilon_{\mathsf{th}}$, subspace estimate $\bm{U} ^{\m} \in \mathbb{R}^{d\times r}$ given by Algorithm~\ref{alg:init_loo_subspace}.
  \For{$\tau = 1, \dots, L$}
\State Recall the Gaussian vector $\bm{g}^\tau \sim \N(0, \bm{I}_d)$ generated in Algorithm~\ref{alg:localization}.
  \State 
	$\big(\bm{\nu}^{\tau, \m}, \lambda_{\tau}^{\m}, \mathsf{spec}\text{-}\mathsf{gap}_{\tau}^{\m} \big) \gets \Call{Retrieve-one-tensor-factor}{\bm{T}^{\m}, p, \bm{U}^{\m}, \bm{g}^{\tau}}$.
    \EndFor
    \State Generate tensor factor estimates \begin{align*}\big\{ \big(\bm{w}^{1, \m}, \lambda_{1}^{\m}), \dots, (\bm{w}^{r, \m},\lambda_{r}^{\m} \big)  \big\} \gets \Call{Prune}{\big\{ \big( \bm{\nu}^{\tau, \m}, \lambda_{\tau}^{\m}, \mathsf{spec}\text{-}\mathsf{gap}_{\tau}^{\m}  \big) \big\}_{\tau = 1}^L, \epsilon_{\mathsf{th}}}.\end{align*}
    \State {\bf Output:} an initial estimate $\bm{U}^{0, \m} = \big[ \big( \lambda_{1}^{\m} \big)^{1/3} \bm{w}^{1,\m}, \dots, \big( \lambda_{r}^{\m} \big)^{1/3} \bm{w}^{r, \m} \big].$
\end{algorithmic}
\end{algorithm}

\begin{theorem}
\label{thm:init-loo}
Fix any arbitrarily small constant $\delta > 0$. Instate the assumptions in Theorem~\ref{thm:init}.
%\begin{align*}
%& p\geq c_{0} \frac{ \mu^2  r^4 \log^{4}d}{d^{3/2}}, \quad \frac{\sigma}{\lambda_{\min}^\star}\leq c_{1} \frac{\sqrt{p}}{ r^{3/2} d^{3/4} \log^2 d}, \quad r\leq c_2 \(\frac{d}{\mu  \log^4 d}\)^{1/6}, \\
%& L = c_3 r^{2\kappa^2} \log^{3/2} r, \quad \epsilon_{\mathsf{th}} = c_4 \Bigg\{ \frac{\mu r \log d}{d\sqrt{p}}+\frac{\sigma}{\lambda_{\min}^\star}\sqrt{\frac{rd\log^{2}d}{p}}+\sqrt{\frac{\mu r  \log d}{d}} \Bigg\} 
%\end{align*}
%for some sufficiently large universal constant $c_0, c_3 >0$ and some sufficiently small universal constants $c_1, c_2, c_4>0$. 
With probability exceeding $1- \delta$, the permutation function stated in Theorem \ref{thm:init} obeys that: for all $1 \leq i \leq r$ and all $1 \leq m \leq d$:
\begin{subequations}
\begin{align}
\label{eq:thm:init_loo_loss}
 \big\| \bm{w}^{i} - \bm{w}^{i, \m} \big\|_2 &
%&  \lesssim \Bigg\{ \frac{\mu_1^{5/2} r^2 \log d }{d^{3/2} p}  \left\{ \sqrt{\log d} + \sqrt{\frac{dp}{\mu_1 r}} \right\} +  \frac{\sigma}{\lammin^{\star 3}} \frac{ \mu_1^{3/2} r^{3/2} \log d  }{p }  \left\{ \sqrt{\log d} + \sqrt{\frac{dp}{\mu_1 r}} \right\} \nonumber \\
%& \quad + \frac{\sigma^2}{\lammin^{\star 6}} \frac{ \sqrt{\mu_1} \, r d^{3/2} \log d}{p} \Bigg\} \big \| \overline{\bm{u}}_{\pi(\tau)}^\star \big\|_\infty ; 
\lesssim  \Bigg\{ \frac{\mu^2 r  \log^{3/2} d}{d^{3/2}p}+\frac{\mu  \sqrt{r} \, \log d}{d \sqrt{p}} +\frac{\sigma^{2}}{\lambda_{\min}^{\star2}}\frac{ d^{3/2}\log^{3/2} d}{p}+\frac{\sigma }{\lambda_{\min}^{\star}}\sqrt{\frac{ d\log^2 d}{p}} \Bigg\} \sqrt{\frac{\mu r}{d}}; \\
\label{eq:thm:init_loo_coef_loss}
 \big| \lambda_i - \lambda_i^{\m} \big| &
%&  \lesssim \Bigg\{ \frac{\mu_1^{5/2} r^2 \log d }{d^{3/2} p}  \left\{ \sqrt{\log d} + \sqrt{\frac{dp}{\mu_1 r}} \right\} +  \frac{\sigma}{\lammin^{\star 3}} \frac{ \mu_1^{3/2} r^{3/2} \log d  }{p }  \left\{ \sqrt{\log d} + \sqrt{\frac{dp}{\mu_1 r}} \right\} \nonumber \\
%& \quad + \frac{\sigma^2}{\lammin^{\star 6}} \frac{ \sqrt{\mu_1} \, r d^{3/2} \log d}{p} \Bigg\} \big \| \overline{\bm{u}}_{\pi(\tau)}^\star \big\|_\infty ; 
\lesssim  \Bigg\{ \frac{\mu^2 r  \log^{3/2} d}{d^{3/2}p}+\frac{\mu  \sqrt{r} \, \log d}{d \sqrt{p}} +\frac{\sigma^{2}}{\lambda_{\min}^{\star2}}\frac{ d^{3/2}\log^{3/2} d}{p}+\frac{\sigma }{\lambda_{\min}^{\star}}\sqrt{\frac{ d\log^2 d}{p}} \Bigg\} \sqrt{\frac{\mu r}{d}} \, \lambda_{\max}^\star; \\
\label{eq:thm:init_loo_entry_loss}
 \big| \big( \bm{w}^{i, \m} - \overline{\bm{u}}^\star_{\pi(i)} \big)_m \big| &
%& \lesssim \left\{ \frac{\mu_1^{3/2} r^{3} \log^{3}  d  }{d^{3/4} \sqrt{p}} + r^2 \sqrt{\frac{\mu_2 }{d} } \, \log d + \frac{\sigma}{\lammin^{\star 3}} r^2 \sqrt{\frac{d}{p}} \, \log^{7/2} d \right\} \big \| \overline{\bm{u}}_{\pi(\tau)}^\star \big\|_\infty  . 
\lesssim \Bigg\{ \frac{\sqrt{\mu r }  \log^{7/2}d}{d^{3/2}p}+\frac{\mu  r \log^{3}d}{d\sqrt{p}} + \frac{\sigma}{\lambda_{\min}^\star}\frac{\log^{4}d}{p}+\frac{\sigma}{\lambda_{\min}^\star}\sqrt{\frac{rd\log^{6}d}{p}}+\sqrt{\frac{\mu  r \log^2 d}{d}} \Bigg\} \sqrt{\frac{\mu r}{d}}.
\end{align}
\end{subequations}
\end{theorem}

With Theorems~\ref{thm:init}-\ref{thm:init-loo} in place, we can immediately establish a few desired properties (particularly those specified in Section~\ref{sec:analysis-GD}) of our initial estimate, as asserted in the following corollary.  
\begin{corollary}
\label{cor:init}
Fix any arbitrarily small constant $\delta > 0$. Instate the assumptions in Theorem~\ref{thm:main}. With probability exceeding $1- \delta$, the estimates $\bm{U}^0$ and $\bm{U}^{0, \m}$ returned by Algorithm~\ref{alg:localization} and Algorithm~\ref{alg:localization_loo} respectively satisfy the hypotheses \eqref{hyp:t_step} for $t = 0$.  
\end{corollary}
\begin{proof}
See Appendix~\ref{pf:cor:init}.
\end{proof}

%\subsubsection{The base case}
%
%To establish the main results for this stage, we start with  the base case, i.e.~the case with $\tau = 1$.   

\subsubsection{Analysis}

%First, we need to specify the permutation function $\pi(\cdot)$.  
Before we start with the proof, we first state the main idea. For the sake of clarify, we define
\begin{subequations}
\begin{align}
	\bm{\theta}^{\tau} &:= \bm{U}\bm{U}^{\top} \bm{g}^{\tau},  \\
	\bm{\theta}^{\tau,(m)} &:= \bm{U}^{(m)}\bm{U}^{(m)\top} \bm{g}^{\tau}, \\
	\bm{M}^{\tau} &:= p^{-1} \bm{T} \times_3 \bm{\theta}^{\tau},  \label{eq:projected-matrix-M-tau}\\ 
	\bm{M}^{\tau,(m)} &:= p^{-1} \bm{T}^{(m)} \times_3 \bm{\theta}^{\tau,(m)}. \label{eq:definition-M-tau-m}
\end{align}
\end{subequations}
In addition,  let $\bm{\nu}^{\tau}$ be the top singular vector of $\bm{M}^{\tau}$ obeying $\langle \bm{T}, (\bm{\nu}^{\tau})^{\otimes 3} \rangle \geq 0$, and $\bm{\nu}^{\tau,(m)}$ the top singular vector of $\bm{M}^{\tau,(m)}$ obeying $\langle \bm{T}^{(m)}, (\bm{\nu}^{\tau, (m)})^{\otimes 3} \rangle \geq 0$. Set 
\begin{align}
\label{def:lambda_tau}
\lambda_{\tau} := \langle p^{-1} \bm{T}, ( \bm{\nu}^{\tau}) ^{\otimes 3}  \rangle \qquad \text{and}  \qquad \lambda_{\tau}^{(m)} := \langle p^{-1} \bm{T}^{(m)}, ( \bm{\nu}^{\tau,(m)}) ^{\otimes 3}  \rangle.
\end{align}
 These  are all  computed in the function $\Call{Retrieve-one-tensor-factor}$ in the $\tau$-th round. 
\begin{enumerate}
\item We first show that for each $1\leq i \leq r$, there exists at least one trial $1 \leq \tau \leq L$ such that the $i$-th tensor factor $\overline{\bm{u}}_i^\star$ is the top singular vector of the population version of $\bm{T} \times_3 \bm{\theta}^{\tau}$ (with respect to the missing data and noise). In addition, the spectral gap is large enough to guarantee accurate estimates.
\item Next, we prove that given this spectral gap, the top singular vector $\bm{\nu}^{\tau}$ of $\bm{T} \times_3 \bm{\theta}^{\tau}$ is close to $\overline{\bm{u}}^\star_i$ both in the $\| \cdot \|_2$ and $\| \cdot \|_{\infty}$ norm. This also enables us to accurately estimate the magnitude of $\bm{u}^\star_i$.
\item Finally, we need to show that one can find those reliable estimates among $L$ random restarts. Combining the spectral gap information with the incoherence condition that tensor components are nearly orthogonal to each other, our selection procedure is guaranteed to recover all tensor factors.
\end{enumerate}

Now we proceed to the proof. Without loss of generality, we prove the case for $i = 1$ in the sequel, i.e.~there exists some $\tau \in \[ L \]$ such that $\bm{\nu}^{\tau}$ accurately recovers $\bm{u}^\star_1$. Together with the union bound, this shows that we can find reliable estimates
for all tensor factors. We then conclude the proof by showing that Algorithm~\ref{alg:localization} is able to find all of them without duplicates.

To this end, we find it convenient to introduce an auxiliary vector $\bm{\gamma}^{\star \tau} = \[ \gamma_1^{\star \tau}, \cdots, \gamma_r^{\star \tau} \]^{\top} \in \R^{r}$ and its leave-one-out versions $ \bm{\gamma}^{\star \tau, \m} = \big[ \gamma_1^{\star \tau, \m}, \cdots, \gamma_r^{\star \tau, \m} \big]^{\top}$ ($1\leq m \leq d$) for each $1\leq \tau \leq L$  as follows:
\begin{subequations}
\label{def:init_sigval}
\begin{align}
\label{def:gamma_star} \gamma_{i}^{\star\tau} & := \|\bm{u}_{i}^{\star}\|_{2}^{2}\, \big\langle\bm{u}_{i}^{\star},\bm{\theta}^{\tau}\big\rangle = \lambda_{i}^{\star} \big \langle\overline{\bm{u}}_{i}^{\star},\bm{\theta}^{\tau} \big\rangle , \\
\gamma_{i}^{\star\tau,(m)} & := \|\bm{u}_{i}^{\star}\|_{2}^{2} \, \big\langle\bm{u}_{i}^{\star},\bm{\theta}^{\tau,(m)}\big\rangle = \lambda_{i}^{\star} \big\langle\overline{\bm{u}}_{i}^{\star},\bm{\theta}^{\tau,(m)} \big\rangle ,
\end{align}
\end{subequations}
where $\overline{\bm{u}}_{i}^{\star}$  and $\lambda_{i}^{\star}$ are both defined in \eqref{eq:defn-ubar-lambda}. 
The idea is to let $\bm{\gamma}^{\star \tau}$ 
%(resp.~$\bm{\gamma}^{\star \tau,(m)}$) 
approximate the singular values of $\bm{T}^\star \times_3 \bm{\theta}^\tau$; this can be seen, for instance, via the following calculation:
\begin{align}
\ensuremath{\bm{T}^{\star}\times_{3}\bm{\theta}^{\tau}}=\sum_{i=1}^{r}\big\langle\bm{u}_{i}^{\star},\bm{\theta}^{\tau}\big\rangle\bm{u}_{i}^{\star}\bm{u}_{i}^{\star\top}=\sum_{i=1}^{r} \underset{=\gamma_{i}^{\star\tau}}{\underbrace{ \lambda_{i}^{\star}\big\langle\overline{\bm{u}}_{i}^{\star},\bm{\theta}^{\tau}\big\rangle }} \, \overline{\bm{u}}_{i}^{\star}\overline{\bm{u}}_{i}^{\star\top}, 	
\end{align} 
where $\{ \overline{\bm{u}}_i^{\star} \}_{i=1}^r$ --- which are assumed to be incoherent (or nearly orthogonal to each other) --- can be approximately viewed as the singular vectors of $\bm{T}^{\star}\times_{3}\bm{\theta}^{\tau}$. 

%Equipped with the above notation,  we shall define the first element of our permutation function $\pi(\cdot)$ as follows  
%%
% \begin{align}
%	 \label{eq:defn-pi-1}
%	 \pi(1)\triangleq\underset{1\leq i\leq r}{\arg\max} ~~\big|\gamma_{i}^{\star 1}\big|. 
% \end{align}
%%
% 
%Of course, if we want this definition \eqref{eq:defn-pi-1} to be unique, 
If we want our spectral estimate to be accurate, we would need to be assured that 
the two largest entries of $\bm{\gamma}^{\star \tau}$ (in magnitude) are sufficiently separated.

\begin{lemma} \label{lemma:init_spectra} Instate the assumptions
	of Theorem~\ref{thm:init}. Define $\Delta_1^\tau := \gamma_{1}^{\star\tau}-\max_{1<i\leq r}\big| \gamma_{i}^{\star\tau}\big|$ for each $1 \leq \tau \leq L$ and let $\Delta_1^{\( 1 \)} \geq \Delta_1^{\( 2 \)} \geq \dots \geq \Delta_1^{\( L \)}$ denote the order statistics of $\big\{ \Delta_1^\tau \big\}_{\tau = 1}^L$ (in descending order). Fix any arbitrary small constant $\delta>0$.
With probability greater than $1-\delta / r$, one has
\begin{subequations}
\label{eq:sigma-order-statistics}
\begin{align}
%\gamma_{1}^{\star\tau}-\max_{1<i\leq r}\big|\gamma_{i}^{\star\tau}\big| & \gtrsim\lambda_{\min}^{\star},\label{eq:sigma1-sigma2-gap-LB}\\
\Delta_1^{\( 1 \)} & \gtrsim\lambda_{\min}^{\star},\label{eq:sigma1-sigma2-gap-LB}\\
\Delta_1^{\( 1 \)} - \Delta_1^{\( 2 \)} & \gtrsim \frac{\lambda_{\min}^\star}{r \sqrt{ \log d}}, \label{eq:gap1-LB}
\end{align}
\end{subequations}
	%where $L$ is the number of random restarts. 
Additionally, for any
fixed vector $\bm{v}\in\mathbb{R}^{r}$, with probability at least
$1-O\left(d^{-10}\right)$, for all $1 \leq \tau \leq L$, one has \begin{subequations} \label{eq:gamma-L2-norm-bound}
\begin{align}
\gamma_{1}^{\star\tau} & \lesssim   \sqrt{\log d} \, \lambda_{\max}^{\star},\label{eq:sigma1-UB} \\
\big\|\bm{\gamma}^{\star\tau }\big\|_2 & \lesssim  \sqrt{r\log d}\,\lambda_{\max}^{\star},\label{eq:gamma-L2-norm-bound-original}\\
\big|\big\langle\bm{v},\bm{\gamma}^{\star\tau }\big\rangle\big| & \lesssim\left\Vert \bm{v}\right\Vert _{2}  \sqrt{\log d}\,\lambda_{\max}^{\star}.\label{eq:gamma-inprod-bound}
\end{align}
\end{subequations}\end{lemma}
\begin{proof} See Appendix~\ref{pf:init_spectra}. \end{proof}
  
%As a remark, Lemma~\ref{lemma:init_spectra} gives the quantification for ``good'' $\bm{g}^\tau$, namely testing if the random generated Gaussian vectors $\bm{g}^\tau$ in the $\tau$-th step is able to help us find a reliable estimate for one tensor factor. Most importantly, although $\bm{\gamma}^\star$ depends on $\bm{U}$ (and thus $\Omega$ and $\bm{E}$), the proof Lemma~\ref{lemma:init_spectra} demonstrates that the quantification for $\bm{g}^\tau$ is actually independent of $\Omega$ and $\bm{E}$.
 %Once we obtain its spectral properties,  \eqref{eq:thm:init_2_loss}, \eqref{eq:thm:init_inf_loss}, \eqref{eq:thm:init_loo_loss}, \eqref{eq:thm:init_loo_entry_loss} and \eqref{eq:thm:init_coeff_loss} will follow.

%With the permutation function $\pi(\cdot)$ (at least its first element) well specified, 
Lemma~\ref{lemma:init_spectra} demonstrates that there exists some $\tau \in \[ L \]$ such that $\gamma_{1}^{\star\tau} -\max_{1<i\leq r}\big| \gamma_{i}^{\star\tau}\big| \gtrsim \lambda_{\min}^\star$. This means that $\overline{\bm{u}}_1^\star$ exhibits the largest correlation with the random projection $\theta$, which further implies that $\overline{\bm{u}}_1^\star$ is the largest singular vector of $\bm{T}^{\star}\times_{3}\bm{\theta}^{\tau}$ with a considerable spectral gap (as we will show shortly).
With the desired spectral gap in place, we are ready to look at the eigenvectors\,/\,singular vectors of interest.  To this end, we find it convenient to introduce another auxiliary vector $\overline{\bm{u}}^{\tau}$, defined as 
 the leading singular vector of $\bm{M}^{\tau}$ (cf.~\eqref{eq:projected-matrix-M-tau}) obeying
%Without loss of generality, we assume that \yxc{I am not sure we can simply assume this...} 
%\yxc{You can asssume $\big\langle \overline{\bm{u}}^1, \bm{u}^\star_{\pi(1)} \big \rangle \geq 0$,  but need to justify why the relation w.r.t.~$\bm{v}_1$ holds. } 
%
\begin{align}
	\label{eq:u1-bar-condition}
	\big\langle \overline{\bm{u}}^{\tau}, \overline{\bm{u}}^\star_{1} \big \rangle \geq 0.
	%\qquad \text{and} \qquad \big \langle \overline{\bm{v}}^1, \bm{u}^\star_{\pi(1)} \big \rangle \geq 0. 
\end{align}
The careful reader would immediately notice the similarity between $ \overline{\bm{u}}^{\tau}$ and $\bm{\nu}^{\tau}$ except for their global signs; namely,  we determine the global sign of  $ \overline{\bm{u}}^{\tau}$ based on the ground truth information \eqref{eq:u1-bar-condition}, but pick the global sign for $\bm{\nu}^{\tau}$ solely based on the observed data (cf.~Algorithm~\ref{alg:localization}). Fortunately, the vectors $ \overline{\bm{u}}^{\tau}$ and $\bm{\nu}^{\tau}$ provably coincide, namely, %
\begin{equation}
	\overline{\bm{u}}^{\tau} = \bm{\nu}^{\tau},
\end{equation}
as we shall demonstrate momentarily in Lemma \ref{lemma:init_coeff}. In a similar way, we also denote by $\overline{\bm{u}}^{\tau, \m}$ the leading singular vector of $\bm{M}^{\tau, \m}$ defined in \eqref{eq:definition-M-tau-m} such that 
\begin{align}
\label{eq:u1-bar-loo-condition} \big\langle \overline{\bm{u}}^{\tau, \m}, \overline{\bm{u}}^\star_{1} \big \rangle \geq 0.
\end{align}
Lemma \ref{lemma:init_coeff} also shows that $\overline{\bm{u}}^{\tau, \m} = \bm{\nu}^{\tau, \m}$.

%
% and  be the leading right singular vector of $\bm{M}^1$ satisfying $\big \langle \overline{\bm{v}}^1, \bm{u}^\star_{\pi(1)} \big \rangle \geq 0$. 
%We will show that claims in Lemma~\ref{lemma:init_2_loss}, Lemma~\ref{lemma:init_loo_entry_loss} and Lemma~\ref{lemma:init_loo_l2_loss} hold for $\overline{\bm{u}}^1$ and $\overline{\bm{v}}^1$. After that, we will show in Lemma~\ref{lemma:init_sign} that the singular vector $\bm{w}_1$ chosen in Algorithm~\ref{alg:localization} such that $\big \langle \bm{T}, \big( \bm{w}^1 \big)^{\ot 3} \big \rangle \geq 0$ is exactly $\overline{\bm{u}}^1$.
%
We shall now take a detour to look at $ \overline{\bm{u}}^{\tau}$, which in turn would help us understand $ \bm{\nu}^{\tau}$. We shall first demonstrate that $ \overline{\bm{u}}^{\tau}$ (and hence $ \bm{\nu}^{\tau}$) is sufficiently close to the corresponding true factor in the $\ell_2$ sense.
\begin{lemma} \label{lemma:init_2_loss} Instate the assumptions
of Theorem~\ref{thm:init}. Let $\overline{\bm{u}}^{\tau}$ and
$\overline{\bm{u}}_{1}^{\star}$ be as defined in \eqref{eq:u1-bar-condition}
and \eqref{eq:defn-ubar-lambda}, respectively. Define $\mathcal{A}$
to be the event such that 
%the conditions \eqref{eq:sigma-order-statistics} and \eqref{eq:gamma-L2-norm-bound} hold. 
$\gamma_{1}^{\star\tau}-\max_{1<i\leq r}\big| \gamma_{i}^{\star\tau}\big| \gtrsim \lambda_{\min}^\star$
and the condition \eqref{eq:gamma-L2-norm-bound} hold. 
Then conditional on this
event $\mathcal{A}$, with probability exceeding $1-O\left(d^{-11}\right)$
one has 
\begin{align}
\big\|\overline{\bm{u}}^{\tau}-\overline{\bm{u}}_{1}^{\star}\big\|_{2} & \lesssim \underbrace{\frac{\mu   r \log d}{d\sqrt{p}}+\frac{\sigma}{\lambda_{\min}^{\star}}\sqrt{\frac{rd\log^{2}d}{p}}+\sqrt{\frac{\mu  r\log d}{d}}}_{=: \, \EE_\proj}. \label{claim:init_2_loss}
\end{align}
\end{lemma}\begin{proof} See Appendix~\ref{pf:init_2_loss}. \end{proof}

Thus far, we have focused on the $\ell_2$ estimation errors.  In order to further quantify the $\ell_{\infty}$ estimation errors, we need to resort to the leave-one-out estimates $\overline{\bm{u}}^{\tau, \m}$ ($1\leq m\leq d$).  Specifically, we shall  justify in the following two lemmas that: (1) the $m$-th leave-one-out estimate $\overline{\bm{u}}^{\tau, \m}$ is close to the truth at least in the $m$-th coordinate, and (2) the vector $\overline{\bm{u}}^{\tau}$ is extremely close to each of the leave-one-out estimates $\overline{\bm{u}}^{\tau, \m}$ ($1\leq m\leq d$). These two observations taken collectively translate to the desired entrywise error control of $\overline{\bm{u}}^{\tau}$. Here, we recall that the global sign of $\overline{\bm{u}}^{\tau, \m}$ (cf.~\eqref{eq:u1-bar-loo-condition})
%(cf.~Algorithm \ref{alg:localization_loo}) 
and the global sign of $\overline{\bm{u}}^{\tau}$ (cf.~\eqref{eq:u1-bar-condition}) are defined in a similar fashion, both using the ground truth information.

\begin{lemma}
\label{lemma:init_loo_entry_loss} Instate the assumptions of Theorem~\ref{thm:init}.
Define $\mathcal{A}$ to be the event such that 
%the conditions \eqref{eq:sigma-order-statistics}
%and \eqref{eq:gamma-L2-norm-bound} hold. 
$\gamma_{1}^{\star\tau}-\max_{1<i\leq r}\big| \gamma_{i}^{\star\tau}\big| \gtrsim \lambda_{\min}^\star$
and the condition \eqref{eq:gamma-L2-norm-bound} hold. 
Then conditional on this
event $\mathcal{A}$, one has, with probability exceeding $1-O\left(d^{-10}\right)$,
that
\begin{align}
\Big|\big[ \overline{\bm{u}}^{\tau,\left(m\right)}-\overline{\bm{u}}_{1}^{\star}\big]_{m}\Big| & \lesssim \EE_\op \sqrt{\frac{\mu r \log d}{d}}, \label{claim:init_loo_entry_loss}
%\Bigg\{\frac{\sqrt{\mu}\,r\log^{7/2}d}{d^{3/2}p}+\frac{\mu r^{3/2}\log^{3}d}{d\sqrt{p}}+\frac{\sigma}{\lambda_{\min}^{\star}}\frac{\sqrt{r}\log^{7/2}d}{p}\nonumber \\
% & \qquad+\frac{\sigma}{\lambda_{\min}^{\star}}\sqrt{\frac{r^{2}d\log^{6}d}{p}}+\sqrt{\frac{\mu r^{2}\log^{2}d}{d}}\Bigg\}\sqrt{\frac{\mu}{d}}.\label{claim:init_loo_entry_loss}
\end{align}
holds for all $m \in \[ d \]$, where $\EE_\op$ is defined as follows:
\begin{align}
\EE_\op := \frac{\sqrt{\mu r} \,  \log^{3}d}{d^{3/2}p}+\frac{\mu  r \log^{5/2}d}{d\sqrt{p}}+ \frac{\sigma}{\lambda_{\min}^\star}\frac{\log^{7/2}d}{p}+\frac{\sigma}{\lambda_{\min}^\star}\sqrt{\frac{rd\log^{5}d}{p}}+\sqrt{\frac{\mu  r\log d}{d}}. \label{def:err-op}
\end{align} 
\end{lemma}
\begin{proof} See Appendix~\ref{pf:init_loo_entry_loss}. \end{proof}

\begin{lemma} \label{lemma:init_loo_l2_loss}
Instate the assumptions of Theorem~\ref{thm:init}. Define $\mathcal{A}$
to be the event such that 
%the conditions \eqref{eq:sigma-order-statistics}
%and \eqref{eq:gamma-L2-norm-bound} hold. 
$\gamma_{1}^{\star\tau}-\max_{1<i\leq r}\big| \gamma_{i}^{\star\tau}\big| \gtrsim \lambda_{\min}^\star$
and the condition \eqref{eq:gamma-L2-norm-bound} hold. 
Then conditional on this
event $\mathcal{A}$, one has, with probability at least $1-O\left(d^{-10}\right)$, for all $m \in \[ d \]$:
\begin{align}
\big\|\overline{\bm{u}}^{\tau}-\overline{\bm{u}}^{\tau, \m}\big\|_{2} & \lesssim \EE_\loo \sqrt{\frac{\mu r \log d}{d}}, \label{claim:init_loo_l2_loss} \\
 \big\|\overline{\bm{u}}^{\tau}-\overline{\bm{u}}_{1}^{\star}\big\|_{\infty} & \lesssim  \( \EE_\op + \EE_\loo \) \sqrt{\frac{\mu r \log d}{d}}, \label{claim:init_inf_loss} \\
\big| \lambda_{\tau} - \lambda_{\tau}^{\m} \big| & \lesssim \EE_\loo \sqrt{\frac{\mu r \log d}{d}} \, \lambda_{\max}^\star, \label{claim:init_loo_coeff_loss}
\end{align}
where $\EE_\op$ and $\EE_\loo$ are defined in (\ref{def:err-op}) and (\ref{def:err-loo}), respectively.
\end{lemma}

\begin{proof} See Appendix~\ref{pf:init_loo_l2_loss}. \end{proof}

	Next, we turn to the estimation accuracy regarding the size of the tensor factors and show that $\lambda_{\tau}$ (produced in Algorithm~\ref{alg:localization}) is close to the truth as well.    As it turns out, a byproduct of this step reveals that $\bm{\nu}^{\tau}= \overline{\bm{u}}^{\tau}$ and $\bm{\nu}^{\tau, \m}= \overline{\bm{u}}^{\tau, \m}$, where  $\overline{\bm{u}}^{\tau}$ and $\overline{\bm{u}}^{\tau, \m}$ are an auxiliary vectors defined in \eqref{eq:u1-bar-condition} and \eqref{eq:u1-bar-loo-condition}, respectively.

\begin{lemma} \label{lemma:init_coeff} Instate the assumptions
	of Theorem~\ref{thm:init}. Assume that the results in
Lemma~\ref{lemma:init_2_loss}, Lemma~\ref{lemma:init_loo_entry_loss}
and Lemma~\ref{lemma:init_loo_l2_loss} hold. Then with probability
at least $1-O\left(d^{-10}\right)$, one has 
\begin{align}
\big|\lambda_{\tau}-\lambda_{1}^{\star}\big|\lesssim
%\Bigg\{\frac{\mu \kappa r\log d}{d\sqrt{p}}+\frac{\sigma}{\lambda_{\min}^{\star}}\sqrt{\frac{rd\log^{2}d}{p}}+\sqrt{\frac{\mu r\log d}{d}}\Bigg\}
\EE_\proj \lambda_{1}^{\star}. \label{claim:init_coeff_loss}
\end{align}
In particular, one has
\begin{align}
\bm{\nu}^{\tau}=\overline{\bm{u}}^{\tau} \qquad \text{and} \qquad \bm{\nu}^{\tau, \m}=\overline{\bm{u}}^{\tau, \m}, \quad 1\leq m \leq d.
\end{align}

\end{lemma}

\begin{proof} See Appendix~\ref{subsec:pf:init_coeff}. \end{proof}

Thus far, we have only proved that one can find a reliable estimate for each tensor factor within  $L$ random trials, provided that $L$ is sufficiently large. To finish up, it remains to show that the pruning procedure $\Call{Prune}$ is capable of returning a rough estimate for each tensor factor without duplication. This is accomplished in the following lemma.

%Lemma~\ref{lemma:init_selection} show that the pruning procedure \Call{Prune} is capable of finding a rough estimate for each tensor factor without duplication.
\begin{lemma} \label{lemma:init_selection}
Instate the assumptions
	of Theorem~\ref{thm:init}. On the event that the results in Lemma~\ref{lemma:init_2_loss}, Lemma~\ref{lemma:init_loo_entry_loss}, Lemma~\ref{lemma:init_loo_l2_loss} and Lemma~\ref{lemma:init_coeff} hold for all $1 \leq i \leq r$, there exists a permutation $\pi(\cdot): [d]\mapsto[d]$ such that: for each $1 \leq i \leq r$, $ \big( \lambda_i, \bm{w}^i \big)$ and $\big(  \lambda_{\pi(i)}^\star, \overline{\bm{u}}^\star_{\pi(i)} \big)$ satisfy \eqref{eq:thm:init_2_loss}, \eqref{eq:thm:init_inf_loss} and \eqref{eq:thm:init_coeff_loss}; $ \big( \lambda_i, \bm{w}^i \big)$ and $\big( \lambda_i^{\m} , \bm{w}^{i, \m} \big)_{i=1}^r$ obey \eqref{eq:thm:init_loo_loss}, \eqref{eq:thm:init_loo_coef_loss} and \eqref{eq:thm:init_loo_entry_loss} for all $1 \leq m \leq d$, where $\big\{ \lambda_i , \bm{w}^i \big\}_{i=1}^r$ and $\big\{ \lambda_i^{\m} , \bm{w}^{i, \m} \big\}_{i=1}^r$ are outputs of Algorithm~\ref{alg:localization} and Algorithm~\ref{alg:localization_loo}, respectively.
\end{lemma}

\begin{proof}
See Appendix~\ref{subsec:pf:init_selection}.
\end{proof}

\section{Discussion}
\label{sec:discussion}

The current paper uncovers the possibility of efficiently and stably completing a low-CP-rank tensor from partial and noisy entries.  Perhaps somewhat unexpectedly, despite the high degree of nonconvexity,  this problem can be solved to  optimal statistical accuracy within nearly linear time, provided that the tensor of interest is well-conditioned, incoherent, and of constant rank.  To the best of our knowledge,  this intriguing message has not been shown in the prior literature.

Moving forward,  one pressing issue is to understand how to improve the algorithmic and theoretical dependency upon the tensor rank $r$ of the proposed method.  Ideally one would desire a fast algorithm whose sample complexity scales as $rd^{1.5}$,  an order that is provably achievable by the sum-of-squares hierarchy.  Additionally, in contrast to the matrix counterpart where the rank is upper bounded by the matrix dimension,  the tensor CP rank is allowed to rise above $d$,  which is commonly referred to as the over-complete case.  Unfortunately, our current initialization scheme (i.e.~the spectral method) fails to work unless $r<d$, and our local analysis for GD falls of accommodating the scenario with $r>d$. It would be of great interest to develop more powerful algorithms --- in addition to more refined analysis --- to tackle such an important  over-complete regime.

Another tantalizing research direction is the exploration of landscape design for tensor completion.  As our heuristic discussions as well as other prior work (e.g.~\cite{richard2014statistical}) suggest,  randomly initialized gradient descent tailored to \eqref{eq:loss-function} seems unlikely to work, unless the sample size is significantly larger than the computational limit.  This might  mean either that there exist spurious local minima in the natural nonconvex least squares formulation \eqref{eq:loss-function}, or that the optimization landscape of \eqref{eq:loss-function} is too flat around some saddle points and hence not amenable to fast computation.   It would be interesting to investigate what families of loss functions allow us to  rule out bad local minima and eliminate the need of careful initialization, which might be  better suited for tensor recovery problems.  

Finally,  in statistical inference and decision making,  one might not be simply satisfied with obtaining a reliable estimate for each missing entry,  but  would also like to report a short confidence interval which is likely to contain the true entry. This boils down to the fundamental task of uncertainty quantification for tensor completion,  which we leave to future investigation.

\section*{Acknowledgements}

Y.~Chen is supported in part by the AFOSR YIP award FA9550-19-1-0030,
by the ONR grant N00014-19-1-2120, by the ARO grants W911NF-20-1-0097 and W911NF-18-1-0303, by the NSF grants CCF-1907661, IIS-1900140 and  DMS-2014279, 
and by the Princeton SEAS innovation award. H.~V.~Poor is supported in part by the NSF grant
DMS-1736417. C.~Cai is supported in part by
Gordon Y.~S.~Wu Fellowships in Engineering. This work was done in part while Y.~Chen
was visiting the Kavli Institute for Theoretical Physics (supported
in part by NSF grant PHY-1748958). We thank Lanqing Yu for many helpful
discussions, and thank Yuling Yan for proofreading the paper.

%\newpage
\appendix

\section{Proofs for local convergence of GD}
\label{sec:proof-local-convergence}

In this section, we establish the key lemmas concerning the convergence properties of GD.  As one can easily see,  treating $\{E_{i,j,k}\}_{1\leq i,j,k\leq d}$ (resp.~$\{\chi_{i,j,k}\}_{1\leq i,j,k\leq d}$) as independent random variables --- which leads to  asymmetric versions of $\bm{E}$ and $\Omega$ --- does not affect the order of our results at all.  In light of this,  we shall adopt such an independent assumption whenever it simplifies our presentation.

%%%%%%%%%%%%%%%%%%%%%%%%%%%%%%%%%%%%%%%%%%%%%%%%%%

\subsection{Proof of Lemma~\ref{lemma:RIC}}
\label{pf:lemma_RIC}

For notational convenience, for any matrix $\bm{M} = \[ \bm{m}_1, \dots, \bm{m}_r \] \in \R^{d \times r}$, let 
\begin{align}
	\widetilde{\bm{M}} := \[ \bm{m}_1 \ot \bm{m}_1, \dots, \bm{m}_r \ot \bm{m}_r \] \in \R^{d^2 \times r},
	\label{defn:widetilde-M}
\end{align}
where for any $\bm{a}, \bm{b}\in \mathbb{R}^d$ we denote $\bm{a}\otimes \bm{b} := {\footnotesize \left[\begin{array}{c}
a_{1}\bm{b}\\
%a_{2}\bm{b}\\
\vdots\\
a_{d}\bm{b}
\end{array}\right]} \in \mathbb{R}^{d^2}$. 

From the Hessian expression \eqref{hessian_qf}, one can decompose 
\begin{align*}
%\label{hessian_qf_decomp}
%\begin{split}
& \vector( \bm{V} )^\top \grad^2 f (\bm{U}) \, \vector ( \bm{V} ) \\
%& \quad = \frac{1}{2p} \Big\| \PP_{\Omega} \Big( \sum_{s \in [r]} ( \bm{v}_s \ot \bm{u}_s^{\ot 2} + \bm{u}_s \ot \bm{v}_s \ot \bm{u}_s + \bm{u}_s^{\ot 2} \ot \bm{v}_s ) \Big) \Big\|_\frob^2 + \frac{1}{p} \sum_{s \in [r]} \Big \langle \PP_{\Omega} \Big( \sum\nolimits_{i \in [r]} \bm{u}_i^{\ot 3} - \bm{T}^\star \Big) \tsym \bm{u}_s, \bm{v}_s \bm{v}_s^\top \Big \rangle \\
& = \underbrace{\frac{1}{3p} \, \Big\| \PP_{\Omega} \Big( \sum_{s \in [r]} \big( \bm{v}_s \ot \bm{u}_s^{\ot 2} + \bm{u}_s \ot \bm{v}_s \ot \bm{u}_s + \bm{u}_s^{\ot 2} \ot \bm{v}_s \big) \Big) \Big\|_\frob^2 
- \frac{1}{3p} \, \Big\| \PP_{\Omega} \Big( \sum_{s \in [r]}  \big( \bm{v}_s \ot \bm{u}_s^{\star \ot 2} + \bm{u}_s^\star \ot \bm{v}_s \ot \bm{u}_s^\star + \bm{u}_s^{\star \ot 2} \ot \bm{v}_s \big) \Big) \Big\|_\frob^2 }_{=: \, \alpha_1}\\
&\quad + \underbrace{\frac{1}{3p} \, \Big\| \PP_{\Omega} \Big( \sum_{s \in [r]} \big( \bm{v}_s \ot \bm{u}_s^{\star \ot 2} + \bm{u}_s^\star \ot \bm{v}_s \ot \bm{u}_s^\star + \bm{u}_s^{\star \ot 2} \ot \bm{v}_s \big)  \Big) \Big\|_\frob^2 
- \frac{1}{3} \, \Big\| \sum_{s \in [r]} \big( \bm{v}_s \ot \bm{u}_s^{\star \ot 2} + \bm{u}_s^\star \ot \bm{v}_s \ot \bm{u}_s^\star + \bm{u}_s^{\star \ot 2} \ot \bm{v}_s \big) \Big\|_\frob^2  }_{=: \, \alpha_2} \\
& \quad + 2 \underbrace{ \Big \langle \frac{1}{p} \PP_{\Omega} \Big( \sum\nolimits_{s \in [r]} \bm{u}_s^{\ot 3} - \bm{T}^\star \Big), \sum\nolimits_{s \in [r]} \bm{v}_s^{\otimes 2} \otimes \bm{u}_s \Big \rangle }_{=: \, \alpha_3} 
+ \underbrace{ \frac{1}{3} \, \Big\| \sum_{s \in [r]} \big( \bm{v}_s \ot \bm{u}_s^{\star \ot 2} + \bm{u}_s^\star \ot \bm{v}_s \ot \bm{u}_s^\star + \bm{u}_s^{\star \ot 2} \ot \bm{v}_s \big) \Big\|_\frob^2 }_{=: \, \alpha_4}.
%\end{split}
\end{align*}
In what follows, we shall bound each of the above terms separately.

\subsubsection{Bounding $\alpha_4$}

%\begin{itemize}
%\item 
With regards to $\alpha_4$, by symmetry we have
\begin{align}
\alpha_4 = \Big\| \sum\nolimits_{s \in [r]} \bm{v}_s \ot \bm{u}_s^{\star \ot 2} \Big\|_\frob^2 + 2  \, \Big \langle \sum\nolimits_{s \in [r]} \bm{v}_s \ot \bm{u}_s^{\star \ot 2} , \sum\nolimits_{s \in [r]} \bm{u}_s^{\star \ot 2} \ot \bm{v}_s \Big \rangle. \label{eq:RIC_alpha4}
\end{align}
In order to control \eqref{eq:RIC_alpha4}, we first see that 
\begin{equation}
\Big\| \sum\nolimits_{s \in [r]} \bm{v}_s \ot \bm{u}_s^{\star \ot 2} \Big\|_\frob = \big\| \bm{V} \widetilde{\bm{U}}^{\star \top} \big\|_\frob,
\end{equation} 
where $\widetilde{\bm{U}}^{\star}$ is as defined in  \eqref{defn:widetilde-M}. Similar to the proof of Lemma~\ref{lemma:incoh}, we can use the fact that $\langle \bm{u}_i^{\star \ot 2}, \bm{u}_j^{\star \ot 2} \rangle = \langle \bm{u}_i^{\star}, \bm{u}_j^{\star} \rangle^2$ and \eqref{asmp_corr} to deduce that
\begin{align}
\label{eq:sigval_tilde_Utrue}
\sigma_{\min} \big(  \widetilde{\bm{U}}^{\star} \big) = \lambda_{\min}^{\star 2 / 3} \(1+o \(1 \) \) \qquad \text{and} \qquad  \mathrm \sigma_{\max} \big( \widetilde{\bm{U}}^{\star} \big) = \lambda_{\max}^{\star 2 / 3} \(1+o\(1\)\), 
\end{align}
provided that $r \ll d / \mu$.
This implies that
\begin{align}
\label{eq:VU-LB-UB}
\frac{19}{20} \lambda_{\min}^{\star 2/3} \norm{\bm{V}}_\frob \leq \sigma_{\min} \big(  \widetilde{\bm{U}}^{\star} \big) \norm{\bm{V}}_\frob \leq \big\| \bm{V} \widetilde{\bm{U}}^{\star \top} \big\|_\frob \leq \sigma_{\max} \big(  \widetilde{\bm{U}}^{\star} \big) \norm{\bm{V}}_\frob \leq \frac{11}{10} \lambda_{\max}^{\star 2 / 3} \norm{\bm{V}}_\frob.
\end{align}
%
%\begin{itemize}
%\item 

(1) Speaking of an upper bound on $\alpha_4$, 
we can invoke the Cauchy-Schwarz inequality followed by \eqref{eq:VU-LB-UB} to reach
\begin{align}
\alpha_4 \leq 3 \, \Big\| \sum\nolimits_{s \in [r]} \bm{v}_s \ot \bm{u}_s^{\star \ot 2} \Big\|_\frob^2 = 3 \, \big\| \bm{V} \widetilde{\bm{U}}^{\star \top} \big\|_\frob^2 \leq \frac{7}{2} \lambda_{\max}^{\star 4 / 3} \norm{\bm{V}}_\frob^2. 
\label{eq:UB-alpha4-RIC-proof}
\end{align}
%

%\item
(2) When it comes to lower bounding $\alpha_4$, the main step boils down to controlling the inner product term in \eqref{eq:RIC_alpha4}. 
Applying the Cauchy-Schwartz inequality gives that
\begin{align*}
\Big \langle \sum\nolimits_{s \in [r]} \bm{v}_s \ot \bm{u}_s^{\star \ot 2}, \sum\nolimits_{s \in [r]} \bm{u}_s^{\star \ot 2} \ot \bm{v}_s \Big \rangle 
& = \sum_{s \in [r]} \left\langle \bm{v}_s, \bm{u}_s^\star \right \rangle^2 \norm{\bm{u}_s^\star}_2^2 + \sum_{s_1 \neq s_2} \left\langle \bm{v}_{s_1}, \bm{u}^\star_{s_2} \right \rangle \left\langle \bm{u}^\star_{s_1}, \bm{v}_{s_2} \right \rangle \left\langle \bm{u}^\star_{s_1}, \bm{u}^\star_{s_2} \right \rangle \\
& \geq  -  \max_{s_1 \neq s_2} \left| \left\langle \bm{u}^\star_{s_1}, \bm{u}^\star_{s_2} \right \rangle \right| \sum_{s_1 \neq s_2} \norm{\bm{v}_{s_1}}_2 \norm{\bm{u}^\star_{s_1}}_2 \norm{\bm{v}_{s_2}}_2  \norm{\bm{u}^\star_{s_2}}_2  \\
& \geq  -  \max_{s_1 \neq s_2} \left| \left\langle \bm{u}^\star_{s_1}, \bm{u}^\star_{s_2} \right \rangle \right|  \Big(\sum\nolimits_{s\in[r]}\left\|\bm{v}_{s} \right\|_{2} \left\|\bm{u}^\star_{s} \right\|_{2}\Big)^{2} \\
& \overset{\text{(i)}}{\geq} -  \max_{s_1 \neq s_2} \left| \left\langle \bm{u}^\star_{s_1}, \bm{u}^\star_{s_2} \right \rangle \right| \norm{\bm{U}^\star}_\frob^2 \norm{\bm{V}}_\frob^2 \\
& \geq - r \sqrt{\frac{\mu}{d}} \, \lammax^{\star 4/3} \norm{\bm{V}}_\frob^2 
 \geq - \frac{1}{40} \lammin^{\star 4/3} \norm{\bm{V}}_\frob^2 ,
\end{align*}
where (i) comes from Cauchy-Schwarz, and the last line follows from \eqref{asmp_corr}, \eqref{asmp_coef} as well as the condition that $r \ll \sqrt{d / \mu}$ and $\kappa \asymp 1$.
Therefore, we can lower bound $\alpha_4$ by (with the assistance of \eqref{eq:VU-LB-UB}) 
\begin{align}
\alpha_4 & = \, \Big\| \sum_{s \in [r]} \bm{v}_s \ot \bm{u}_s^{\star \ot 2} \Big\|_\frob^2 + 2 \, \Big \langle \sum_{s \in [r]} \bm{v}_s \ot \bm{u}_s^{\star \ot 2} , \sum_{s \in [r]} \bm{u}_s^{\star \ot 2} \ot \bm{v}_s \Big \rangle \nonumber\\
& \geq \frac{19}{20} \lambda_{\min}^{\star 4/3} \norm{\bm{V}}_\frob^2 - \frac{1}{20} \lammin^{\star 4/3} \norm{\bm{V}}_\frob^2 
~\geq~ \frac{9}{10} \lambda_{\min}^{\star 4/3} \norm{\bm{V}}_\frob^2.
\end{align}

%\end{itemize}

\subsubsection{Bounding $\alpha_1$}

When it comes to $\alpha_1$, we can expand
\begin{align*}
\Big\| & \PP_{\Omega} \Big( \sum\nolimits_{s \in [r]}  \bm{v}_s \ot \bm{u}_s^{\ot 2} \Big) + \PP_{\Omega} \Big( \sum\nolimits_{s \in [r]} \bm{u}_s \ot \bm{v}_s \ot \bm{u}_s \Big) + \PP_{\Omega} \Big( \sum\nolimits_{s \in [r]} \bm{u}_s^{\ot 2} \ot \bm{v}_s  \Big) \Big\|_\frob^2 \\
& = 3 \, \Big\| \PP_{\Omega} \Big( \sum\nolimits_{s \in [r]}  \bm{v}_s \ot \bm{u}_s^{\ot 2} \Big) \Big\|_\frob^2 + 6 \, \Big \langle \PP_\Omega \Big( \sum\nolimits_{s \in [r]} \bm{v}_s \ot \bm{u}_s^{\ot 2} \Big), \PP_\Omega \Big( \sum\nolimits_{s \in [r]} \bm{u}_s^{\ot 2} \ot \bm{v}_s \Big) \Big \rangle;
\end{align*}
 we can decompose $\big\| \PP_{\Omega} \big( \sum\nolimits_{s \in [r]}  \bm{v}_s \ot \bm{u}_s^{\star \ot 2} \big) + \PP_{\Omega} \big( \sum\nolimits_{s \in [r]} \bm{u}_s^\star \ot \bm{v}_s \ot \bm{u}_s^\star \big) + \PP_{\Omega} \big( \sum\nolimits_{s \in [r]} \bm{u}_s^{\star \ot 2} \ot \bm{v}_s  \big) \big\|_\frob^2$ in a similar way.
As a consequence, 
\begin{align*}
%\begin{split}
\alpha_1 &=  \underbrace{\frac{2}{p} \Bigg( \Big \langle \PP_\Omega \Big( \sum_{s \in [r]} \bm{v}_s \ot \bm{u}_s^{\ot 2} \Big), \PP_\Omega \Big( \sum_{s \in [r]} \bm{u}_s^{\ot 2} \ot \bm{v}_s \Big) \Big \rangle
- \Big \langle \PP_\Omega \Big( \sum_{s \in [r]} \bm{v}_s \ot \bm{u}_s^{\star \ot 2} \Big), \PP_\Omega \Big( \sum_{s \in [r]} \bm{u}_s^{\star \ot 2} \ot \bm{v}_s \Big) \Big \rangle \Bigg)}_{=: \, \beta_1} \\
& \quad + \underbrace{\frac{1}{p} \Bigg( \Big\| \PP_{\Omega} \Big( \sum_{s \in [r]}  \bm{v}_s \ot \bm{u}_s^{\ot 2} \Big) \Big\|_\frob^2 - \Big\| \PP_{\Omega} \Big( \sum_{s \in [r]}  \bm{v}_s \ot \bm{u}_s^{\star \ot 2}\Big) \Big\|_\frob^2 \Bigg)}_{=: \, \beta_2}.
%\end{split}
\end{align*}
We will derive an upper bound on $\beta_1$ in the sequel; the same method immediately applies to $\beta_2$.

For notational convenience, let us define 
\begin{align}
\label{def:RIC_delta}
	\bm{\Delta} := \bm{U} - \bm{U}^\star, \quad 
	\bm{\Delta}_s := \bm{u}_s - \bm{u}_s^\star,  \quad
	\widetilde{\bm{\Delta}} := [\bm{\Delta}_1 \ot \bm{\Delta}_1, \cdots, \bm{\Delta}_r \ot \bm{\Delta}_r] \in \R^{d^2 \times r}. 
\end{align} 
Then one can write 
{\small
\begin{align*}
\frac{1}{2}\beta_1 &= \frac{1}{p} \, \Big \langle \PP_\Omega \Big( \sum_{s \in [r]} \bm{v}_s \ot (\bm{\Delta}_s +\bm{u}_s^\star)^{\ot 2} \Big), \PP_\Omega \Big( \sum_{s \in [r]} (\bm{\Delta}_s +\bm{u}_s^\star)^{\ot 2} \ot \bm{v}_s \Big) \Big \rangle
 - \frac{1}{p} \, \Big \langle \PP_\Omega \Big( \sum_{s \in [r]} \bm{v}_s \ot \bm{u}_s^{\star \ot 2} \Big), \PP_\Omega \Big( \sum_{s \in [r]} \bm{u}_s^{\star \ot 2} \ot \bm{v}_s \Big) \Big \rangle \\
& = \frac{1}{p} \, \Big \langle \PP_\Omega \Big( \sum_{s \in [r]} \bm{v}_s \ot \bm{u}_s^{\star \ot 2} \Big), 
\PP_\Omega \Big( \sum_{s \in [r]}  \bm{u}_s^\star  \ot \bm{\Delta}_s \ot \bm{v}_s \Big) 
+ \PP_\Omega \Big( \sum_{s \in [r]}  \bm{\Delta}_s  \ot \bm{u}_s^\star \ot \bm{v}_s\Big) 
+ \PP_\Omega \Big( \sum_{s \in [r]} \bm{\Delta}_s  \ot \bm{\Delta}_s \ot \bm{v}_s \Big) \Big\rangle \\
& \quad + \frac{1}{p} \, \Big \langle  \PP_\Omega \Big( \sum_{s \in [r]} \bm{v}_s \ot \bm{u}_s^{\star} \ot \bm{\Delta}_s \Big) 
+ \PP_\Omega \Big( \sum_{s \in [r]} \bm{v}_s \ot \bm{\Delta}_s \ot \bm{u}_s^{\star} \Big) 
+ \PP_\Omega \Big( \sum_{s \in [r]} \bm{v}_s \ot \bm{\Delta}^{\ot 2}_s \Big), \\
& \qquad
\PP_\Omega \Big( \sum_{s \in [r]} \bm{u}_s^{\star \ot 2} \ot \bm{v}_s \Big) 
+ \PP_\Omega \Big( \sum_{s \in [r]} \bm{u}_s^{\star}  \ot \bm{\Delta}_s \ot \bm{v}_s \Big) 
+ \PP_\Omega \Big( \sum_{s \in [r]} \bm{\Delta}_s  \ot \bm{u}_s^{\star} \ot \bm{v}_s \Big)  
+  \PP_\Omega \Big( \sum_{s \in [r]} \bm{\Delta}_s  \ot \bm{\Delta}_s \ot \bm{v}_s \Big) \Big\rangle.
\end{align*}}
%\begin{align*}
%\beta_1 &= \frac{2}{p} \, \Big \langle \PP_\Omega \Big( \sum_{s \in [r]} \bm{v}_s \ot (\bm{\Delta}_s +\bm{u}_s^\star)^{\ot 2} \Big), \PP_\Omega \Big( \sum_{s \in [r]} (\bm{\Delta}_s +\bm{u}_s^\star)^{\ot 2} \ot \bm{v}_s \Big) \Big \rangle
% - \frac{2}{p} \, \Big \langle \PP_\Omega \Big( \sum_{s \in [r]} \bm{v}_s \ot \bm{u}_s^{\star \ot 2} \Big), \PP_\Omega \Big( \sum_{s \in [r]} \bm{u}_s^{\star \ot 2} \ot \bm{v}_s \Big) \Big \rangle \\
%& = \frac{2}{p} \, \Big \langle \PP_\Omega \Big( \sum_{s \in [r]} \bm{v}_s \ot \bm{u}_s^{\star \ot 2} \Big), 
% \PP_\Omega \Big( \sum_{s \in [r]}  \bm{u}_s^\star  \ot \bm{\Delta}_s \ot \bm{v}_s \Big) 
%+ \PP_\Omega \Big( \sum_{s \in [r]}  \bm{\Delta}_s  \ot \bm{u}_s^\star \ot \bm{v}_s\Big) 
%+ \PP_\Omega \Big( \sum_{s \in [r]} \bm{\Delta}_s  \ot \bm{\Delta}_s \ot \bm{v}_s \Big) \Big\rangle \\
%& \quad + \frac{2}{p} \, \Big \langle \PP_\Omega \Big( \sum_{s \in [r]} \bm{u}_s^{\star \ot 2} \ot \bm{v}_s \Big) 
%+  \PP_\Omega \Big( \sum_{s \in [r]} \bm{u}_s^{\star}  \ot \bm{\Delta}_s \ot \bm{v}_s \Big) 
%+ \PP_\Omega \Big( \sum_{s \in [r]} \bm{\Delta}_s  \ot \bm{u}_s^{\star} \ot \bm{v}_s \Big)  
%+  \PP_\Omega \Big( \sum_{s \in [r]} \bm{\Delta}_s  \ot \bm{\Delta}_s \ot \bm{v}_s \Big), \\
%& \quad \quad \quad \quad \PP_\Omega \Big( \sum_{s \in [r]} \bm{v}_s \ot \bm{u}_s^{\star} \ot \bm{\Delta}_s \Big) 
%+ \PP_\Omega \Big( \sum_{s \in [r]} \bm{v}_s \ot \bm{\Delta}_s \ot \bm{u}_s^{\star} \Big) 
%+ \PP_\Omega \Big( \sum_{s \in [r]} \bm{v}_s \ot \bm{\Delta}^{\ot 2}_s \Big)
% \Big\rangle.
%\end{align*}

% 
Apply the Cauchy-Schwartz inequality to yield that
%
%{\small
%\begin{align}
%%\begin{split}
%& \left| \beta_1 \right| \lesssim \frac{1}{p} \, \Big\| \PP_\Omega \Big( \sum_{s \in [r]} \bm{v}_s \ot \bm{u}_s^{\star \ot 2} \Big) \Big\|_\frob
% \Big( \Big\| \PP_\Omega \Big( \sum_{s \in [r]}  \bm{u}_s^\star  \ot \bm{\Delta}_s \ot \bm{v}_s \Big) \Big\|_\frob 
% +\Big\| \PP_\Omega \Big( \sum_{s \in [r]}  \bm{\Delta}_s  \ot \bm{u}_s^\star \ot \bm{v}_s \Big)  \Big\|_\frob
% + \Big\|  \PP_\Omega \Big( \sum_{s \in [r]} \bm{\Delta}_s  \ot \bm{\Delta}_s \ot \bm{v}_s \Big)  \Big\|_\frob  \Big) \nonumber\\
%&\quad + \frac{1}{p} \Big( \Big\| \PP_\Omega \Big( \sum_{s \in [r]} \bm{v}_s \ot \bm{u}_s^{\star}\ot \bm{\Delta}_s \Big) \Big\|_\frob 
%+ \Big\| \PP_\Omega \Big( \sum_{s \in [r]} \bm{v}_s \ot \bm{\Delta}_s \ot \bm{u}_s^{\star} \Big)  \Big\|_\frob
%+ \Big\| \PP_\Omega \Big( \sum_{s \in [r]} \bm{v}_s \ot \bm{\Delta}^{\ot 2}_s \Big) \Big\|_\frob \Big) \nonumber\\
%& \qquad \cdot 
%\Big( \Big\|  \PP_\Omega \Big( \sum_{s \in [r]} \bm{u}_s^{\star \ot 2} \ot \bm{v}_s \Big)  \Big\|_\frob 
%+  \Big\| \PP_\Omega \Big( \sum_{s \in [r]}  \bm{u}_s^\star  \ot \bm{\Delta}_s \ot \bm{v}_s \Big) \Big\|_\frob 
%+\Big\| \PP_\Omega \Big( \sum_{s \in [r]}  \bm{\Delta}_s  \ot \bm{u}_s^\star \ot \bm{v}_s \Big)  \Big\|_\frob
%+ \Big\|  \PP_\Omega \Big( \sum_{s \in [r]} \bm{\Delta}_s \ot \bm{\Delta}_s \ot \bm{v}_s \Big)  \Big\|_\frob  \Big).
%\label{eq:RIC_beta}
%%\end{split}
%\end{align}
%}
\begin{align}
%\begin{split}
 \left| \beta_1 \right| & \lesssim \frac{1}{p} \, \Big\| \PP_\Omega \Big( \sum_{s \in [r]} \bm{v}_s \ot \bm{u}_s^{\star \ot 2} \Big) \Big\|_\frob
 \Big( 2\, \Big\| \PP_\Omega \Big( \sum_{s \in [r]}  \bm{u}_s^\star  \ot \bm{\Delta}_s \ot \bm{v}_s \Big) \Big\|_\frob 
 % +\Big\| \PP_\Omega \Big( \sum_{s \in [r]}  \bm{\Delta}_s  \ot \bm{u}_s^\star \ot \bm{v}_s \Big)  \Big\|_\frob
 + \Big\|  \PP_\Omega \Big( \sum_{s \in [r]} \bm{\Delta}_s  \ot \bm{\Delta}_s \ot \bm{v}_s \Big)  \Big\|_\frob  \Big) \nonumber\\
& \quad + \frac{1}{p} \Big( \Big\|  \PP_\Omega \Big( \sum_{s \in [r]} \bm{u}_s^{\star \ot 2} \ot \bm{v}_s \Big)  \Big\|_\frob 
+ 2 \,  \Big\| \PP_\Omega \Big( \sum_{s \in [r]}  \bm{u}_s^\star  \ot \bm{\Delta}_s \ot \bm{v}_s \Big) \Big\|_\frob 
% +\Big\| \PP_\Omega \Big( \sum_{s \in [r]}  \bm{\Delta}_s  \ot \bm{u}_s^\star \ot \bm{v}_s \Big)  \Big\|_\frob
+ \Big\|  \PP_\Omega \Big( \sum_{s \in [r]} \bm{\Delta}_s \ot \bm{\Delta}_s \ot \bm{v}_s \Big)  \Big\|_\frob  \Big) \nonumber
\\
& \quad \quad \cdot \Big( 2\, \Big\| \PP_\Omega \Big( \sum_{s \in [r]} \bm{v}_s \ot \bm{u}_s^{\star}\ot \bm{\Delta}_s \Big) \Big\|_\frob 
% + \Big\| \PP_\Omega \Big( \sum_{s \in [r]} \bm{v}_s \ot \bm{\Delta}_s \ot \bm{u}_s^{\star} \Big)  \Big\|_\frob
+ \Big\| \PP_\Omega \Big( \sum_{s \in [r]} \bm{v}_s \ot \bm{\Delta}^{\ot 2}_s \Big) \Big\|_\frob \Big).
\label{eq:RIC_beta}
%\end{split}
\end{align}

Before we bound the above quantities, we pause to make the following observations. In view of the assumptions of this lemma that $\delta \ll 1/\sqrt{r} \leq 1$, the following holds for all $i \in \[ r \]$:
\begin{subequations}	
\begin{align}
\label{eq:RIC_u_2_loss} 
 \norm{\bm{\Delta}_i}_2 &\leq  \norm{\bm{U} - \bm{U}^\star}_\frob  \leq \delta \norm{\bm{U}^\star}_\frob \leq \delta \sqrt{r} \, \lammax^{\star 1/3} \ll \lammax^{\star 1/3}, \\
 \label{eq:RIC_u_inf_loss} 
 \norm{\bm{\Delta}_i}_\infty &\leq \norm{\bm{U} - \bm{U}^\star}_{2, \infty} \leq \delta \norm{\bm{U}^\star}_{2, \infty} \leq \delta \sqrt{\frac{\mu r}{d}} \,  \lammax^{\star 1/3} \ll \sqrt{\frac{\mu}{d}} \lammax^{\star 1/3}, \\
 \label{eq:RIC_u_2_norm} 
\norm{\bm{u}_i}_2 &\leq \norm{ \bm{u}_i^\star }_2 +  \norm{\bm{\Delta}_i}_2 \leq 2 \lammax^{\star 1/3}, \\
\label{eq:RIC_u_inf_norm} 
 \norm{\bm{u}_i}_\infty &\leq \norm{ \bm{u}_i^\star }_{\infty} +  \norm{\bm{\Delta}_i}_\infty \leq 2 \sqrt{\frac{\mu }{d}} \,  \lammax^{\star 1/3},
\end{align}
\end{subequations}	
Consequently, we also know that
\begin{align}
 \label{eq:RIC_delta_tilde_2inf_norm}
 \big\| \widetilde{\bm{\Delta}} \big\|_{2,\infty} \leq \max_{1\leq i \leq r}  \norm{\bm{\Delta}_i}_\infty   \norm{\bm{\Delta}}_{2, \infty} \leq \delta^2  \norm{\bm{U}^\star}_{2, \infty}^2.
\end{align}
Now, we proceed to prove the claim. Let us define $\mathcal{S}_i := \{ \vec{j} \in [ d ]^2 \mid \chi_{i j_1 j_2} = 1\}$ for each $i \in [d]$. Applying the Chernoff bound and the union bound yields that: with probability at least $1-O(d^{-10})$ one has
\begin{align}
	\max_{i \in \[d \]} | \mathcal{S}_i | \lesssim d^2 p. \label{eq:Si-max}
\end{align} 
provided $p \gg d^{-2} \log d$. It then follows from the Cauchy-Schwarz inequality that
\begin{align*}
\frac{1}{p} \, \Big\| \PP_\Omega \Big( \sum\nolimits_{s \in [r]} \bm{v}_s \ot \bm{\Delta}_s^{ \ot2} \Big) \Big\|_\frob^2
&= \frac{1}{p} \sum_{i \in [ d ], \vec{j} \in [ d ]^2} \chi_{i j_1 j_2} \big \langle \bm{V}_{i, :},  \widetilde{\bm{\Delta}}_{\vec{j}, :} \big\rangle^2 \\
& \leq \frac{1}{p} \sum_{i \in \[ d\]} \big\| \bm{V}_{i, :}  \big\|_2^2 \sum_{\vec{j} \in \mathcal{S}_i}  \big\| \widetilde{\bm{\Delta}}_{\vec{j}, :}  \big\|_2^2 \\
& \leq \frac{1}{p} \max_{i \in [ d ]}  \left| \mathcal{S}_i \right| \big\| \widetilde{\bm{\Delta}} \big\|_{2, \infty}^2 \norm{\bm{V}}_\frob^2 \\
% & \overset{\text{(i)}}{\leq}  2 d^2  \max_{s \in \[ r \]} \norm{ {\bm{\Delta}_s} }_\infty^2 \norm{\bm{\Delta}}_{2, \infty}^2  \norm{\bm{V}}_\frob^2 \\
& \overset{\text{(i)}}{\lesssim}  d^2  \delta^4 \norm{ \bm{U}^\star }_{2, \infty}^4 \norm{\bm{V}}_\frob^2 
%& \leq 2 \delta^4 r d^2 \Big\{ \max_{s \in \[ r \]} \norm{ \bm{u}_s^\star }_\infty^4 \Big\} \norm{\bm{V}}_\frob^2 
 \lesssim \delta^4  \mu^2 r^2 \lambda_{\max}^{\star 4/3} \norm{\bm{V}}_\frob^2, 
\end{align*}
where (i) arises from \eqref{eq:RIC_delta_tilde_2inf_norm} and \eqref{eq:Si-max}.
%, and the last line follows from the assumption.
In a similar manner, we can derive
\begin{align*}
\frac{1}{p} \, \Big\| \PP_\Omega \Big( \sum\nolimits_{s \in [r]}  \bm{u}_s^\star \ot \bm{v}_s \ot \bm{\Delta}_s  \Big) \Big\|_\frob^2 
& \lesssim \frac{1}{p} \max_{1\leq i \leq d} \left| \mathcal{S}_i \right| \max_{1\leq s \leq r} \norm{\bm{u}_s^\star}_\infty^2 \norm{\bm{\Delta}}_{2,\infty}^2 \norm{\bm{V}}_\frob^2 \\
% \lesssim d^2 \max_{1\leq i \leq r} \norm{\bm{u}_s^\star}_\infty^2 \delta^2 \norm{\bm{U}^\star}_{2,\infty}^2  \norm{\bm{V}}_\frob^2
& \lesssim  d^2 \delta^2 \max_{1\leq i \leq r} \norm{\bm{u}_s^\star}_\infty^2 \norm{ \bm{U}^\star }_{2, \infty}^2 \norm{\bm{V}}_\frob^2 \\
& \lesssim \delta^2  \mu^2  r \lambda_{\max}^{\star 4/3} \norm{\bm{V}}_\frob^2 .
 \end{align*}
Regarding $p^{-1} \big\| \PP_\Omega \big( \sum\nolimits_{s \in [r]} \bm{v}_s \ot \bm{u}_s^{\star \ot 2} \big) \big\|_\frob^2$, we apply \cite[Lemma 5]{yuan2016tensor} (with slight modification, which we omit here for brevity) to show that: with probability exceeding $1- O ( d^-{10})$
\begin{align*}
\frac{1}{\sqrt{p}}  \, \Big\| \PP_\Omega \Big( \sum\nolimits_{s \in [r]} \bm{v}_s \ot \bm{u}_s^{\star \ot 2} \Big) \Big\|_\frob \leq \frac{3}{2} \, \Big\|  \sum\nolimits_{s \in [r]} \bm{v}_s \ot \bm{u}_s^{\star \ot 2} \Big\|_\frob
 =  \frac{3}{2}  \big\| \bm{V} \widetilde{\bm{U}}^{\star\top} \big\|_\frob 
\leq \frac{3}{2} \, \big\| \widetilde{\bm{U}}^\star \big\| \norm{\bm{V}}_\frob
 \lesssim \, \lambda_{\max}^{\star 2/3} \norm{\bm{V}}_\frob.
\end{align*}
under the sample size assumptin that $p \gg \mu^2 r^2 d^{-2} \log d$. Here the last inequality makes use of \eqref{eq:sigval_tilde_Utrue}. 
It is self-evident that the above bounds also hold for quantities that appear in \eqref{eq:RIC_beta}. Since $0 < \delta \ll 1/\sqrt{r} < 1$, we obtain
\begin{align*}
|\beta_1| \lesssim \delta \mu \sqrt{r} \lambda_{\max}^{\star 4/3} \norm{\bm{V}}_\frob^2. 
\end{align*}

%The same upper bound  holds for any other $\beta_i$ ($i \neq 1$)
The same upper bound  holds for any other $\beta_2$. Therefore, as long as  $0 < \delta \ll 1 / ( \mu \sqrt{r} )  $ and $\kappa \asymp 1$, we have
\begin{align}
\label{lemma_RIC_alpha_1}
|\alpha_1| \leq \frac{1}{10} \lambda^{\star 4/3}_{\min}  \norm{\bm{V}}_\frob^2.
\end{align}

\subsubsection{Bounding $\alpha_2$} 

Regarding $\alpha_2$, applying  \cite[Lemma 5]{yuan2016tensor} (with slight modification, which we omit here for brevity) implies that: if $p \gg \mu^2 r^2 d^{-2} \log d$, then with probability exceeding $1- O(d^{-10})$,
\begin{align*}
\Bigg| \frac{1}{2p} & \, \Big\| \PP_{\Omega} \Big( \sum\nolimits_{s \in [r]} ( \bm{v}_s \ot \bm{u}_s^{\star \ot 2} + \bm{u}_s^\star \ot \bm{v}_s \ot \bm{u}_s^\star + \bm{u}_s^{\star \ot 2} \ot \bm{v}_s)  \Big) \Big\|_\frob^2 \\
& - \frac{1}{2} \, \Big\| \sum\nolimits_{s \in [r]} ( \bm{v}_s \ot \bm{u}_s^{\star \ot 2} + \bm{u}_s^\star \ot \bm{v}_s \ot \bm{u}_s^\star + \bm{u}_s^{\star \ot 2} \ot \bm{v}_s ) \Big\|_\frob^2 \Bigg|\\
& \qquad \leq \frac{1}{100} \, \Big\| \sum\nolimits_{s \in [r]} ( \bm{v}_s \ot \bm{u}_s^{\star \ot 2} + \bm{u}_s^\star \ot \bm{v}_s \ot \bm{u}_s^\star + \bm{u}_s^{\star \ot 2} \ot \bm{v}_s ) \Big\|_\frob^2 = \frac{1}{100} \alpha_4  \leq \frac{1}{10} \lambda_{\min}^{\star 4/3} \norm{\bm{V}}_\frob^2.
\end{align*}
Here the last inequality arises from \eqref{eq:UB-alpha4-RIC-proof}.

\subsubsection{Bounding $\alpha_3$}

We now move on to bounding $\alpha_3$. The triangle inequality gives
\begin{align*}
\Big | \Big \langle p^{-1} \PP_{\Omega} \Big( \sum\nolimits_{s \in [r]} \bm{u}_s^{\ot 3} - \bm{T}^\star \Big), \sum\nolimits_{s \in [r]} \bm{v}_s^{\otimes 2} \otimes \bm{u}_s \Big \rangle \Big| 
& \leq \sum_{s \in [r]} \Big| \Big \langle p^{-1} \PP_{\Omega} \Big( \sum\nolimits_{i \in [r]} \bm{u}_i^{\ot 3} - \bm{T}^\star \Big) \times_3 \bm{u}_s, \bm{v}_s \bm{v}_s^\top \Big \rangle  \Big| \\
%& \leq \sum_{s \in [r]} \Big\| p^{-1} \PP_{\Omega} \Big( \sum\nolimits_{i \in [r]} \bm{u}_i^{\ot 3} - \bm{T}^\star \Big) \tsym \bm{u}_s \Big\| \norm{\bm{v}_s}_2^2 \\
& \leq \max_{s \in \[ r \]} \Big\| p^{-1} \PP_{\Omega} \Big( \sum\nolimits_{i \in [r]} \bm{u}_i^{\ot 3} - \bm{T}^\star \Big) \times_3 \bm{u}_s \Big\| \sum_{s\in [r]} \norm{ \bm{v}_s }_2^2 \\
& \leq \max_{s \in \[ r \]} \Big\| p^{-1} \PP_{\Omega} \Big( \sum\nolimits_{i \in [r]} \bm{u}_i^{\ot 3} - \bm{T}^\star \Big) \times_3 \bm{u}_s \Big\| \norm{\bm{V}}_\frob^2. 
\end{align*}
Recall the definitions of $\bm{\Delta}$ and $\bm{\Delta}_i$ in \eqref{def:RIC_delta}. Fix an arbitrary $s \in \[ r \]$. From the definition of the operator norm and the triangle inequality, we can derive
\begin{align}
  \Big\| \, p^{-1} \PP_\Omega \Big( \sum\nolimits_{i \in [r]} \bm{u}_i^{\ot 3} - \bm{T}^\star \Big) \times_3 \bm{u}_s \Big\|
  \leq \norm{\bm{u}_s}_2   \Big\| \, p^{-1} \PP_\Omega \Big( \sum\nolimits_{i \in [r]} \bm{u}_i^{\ot 3} - \bm{T}^\star \Big) \Big\|.
\label{eq:UB-operator-norm-1}
\end{align}

 In order to upper bound $\big\| p^{-1} \PP_\Omega \Big( \sum\nolimits_{i \in [r]} \bm{u}_i^{\ot 3} - \bm{T}^\star \Big) \big\| $ as required in \eqref{eq:UB-operator-norm-1}, we invoke the following simple fact, which follows immediately from the definition of the operator norm. Here and throughout, for any tensor $\bm{A}\in \mathbb{R}^{d\times d\times d}$ we denote $$|\bm{A}| := \big[|A_{i, j, k}|\big]_{1\leq i, j, k \leq d} \in \mathbb{R}^{d\times d\times d}.$$
\begin{lemma}
Consider any tensor $\bm{A}, \bm{B} \in \R^{d\times d\times d}$ obeying $|B_{i, j, k}| \geq |A_{i, j, k}|$ for all $1\leq i, j, k \leq d$. One has
\begin{align}
	\norm{\bm{A}} \leq \big\| |\bm{A}| \big\| \leq \big\| |\bm{B}| \big\|.
\end{align}
\end{lemma}

With this lemma in mind, we are ready to derive that
\begin{align*}
\Big\| p^{-1} \PP_\Omega \Big( \sum\nolimits_{i \in [r]} \bm{u}_i^{\ot 3} - \bm{T}^\star \Big) \Big\| & \leq \Big\| p^{-1} \PP_\Omega \Big( \Big| \sum\nolimits_{i \in [r]} \bm{u}_i^{\ot 3} - \bm{T}^\star \Big| \Big) \Big\| \\
& \leq \Big\| \sum\nolimits_{i \in [r]} \bm{u}_i^{\ot 3} - \bm{T}^\star \Big\|_\infty  \norm{p^{-1} \PP_\Omega \big(  \bm{1}^{\ot 3}  \big)}.
\end{align*}
Here, $\bm{1}$ stands for the all-one vector in $\R^d$.
This suggests that we shall upper bound $\norm{p^{-1} \PP_\Omega \big(  \bm{1}^{\ot 3}  \big)}$ and $\big\| \sum\nolimits_{i \in [r]} \bm{u}_i^{\ot 3} - \bm{T}^\star \big\|_\infty$.

Given that $p \gtrsim d^{-3/2} \log^3 d$, applying Lemma~\ref{lemma:T_op_norm} indicates that with probability at least $1-O(d^{-10})$,
\begin{align}
% \norm{ p^{-1} \PP_\Omega \big(  \bm{1}^{\ot 3}  \big) - \bm{1}^{\ot 3}}_2 \leq C d^{3/2} \frac{\log^2 (d)}{(\frac{p}{d^{3/2}})^{1/2}} 
 \norm{ p^{-1} \PP_\Omega \big(  \bm{1}^{\ot 3}  \big) - \bm{1}^{\ot 3}} \lesssim \frac{\log^3 d}{p} + \sqrt{\frac{d \log^5 d}{p}} \lesssim d^{3/2}.
\end{align}
Moreover, it is straightforward to see that 
$\norm{\bm{1}^{\ot 3}} = \norm{\bm{1}}_2^3 = d^{3/2}$. Therefore, one has
\begin{align}
\norm{p^{-1} \PP_\Omega \big(  \bm{1}^{\ot 3}  \big)} \lesssim d^{3/2}.  \label{eq:T-all-one-op_norm}
\end{align}
Next, we turn to $\left\| \cdot \right\|_\infty$.  We first expand
\begin{align*}
\begin{split}
\sum_{i \in [r]} \bm{u}^{\ot 3}_i - \bm{T}^\star &= \sum_{i \in [r]} \big( (\bm{\Delta}_i + \bm{u}^\star_i)^{\ot 3} -\bm{u}_i^{\star \ot 3} \big) =  \sum_{i \in [r]} \bm{\Delta}_i \ot \bm{u}^{\star\ot 2}_i + \sum_{i \in [r]}  \bm{u}^{\star}_i \ot \bm{\Delta}_i \ot \bm{u}^{\star}_i + \sum_{i \in [r]}  \bm{u}^{\star\ot 2}_i \ot \bm{\Delta}_i \\
& \quad + \sum_{i \in [r]}  \bm{\Delta}_i^{\ot 2} \ot \bm{u}^{\star}_i + \sum_{i \in [r]}  \bm{\Delta}_i \ot \bm{u}^{\star}_i \ot \bm{\Delta}_i + \sum_{i \in [r]}  \bm{u}^\star_i \ot \bm{\Delta}_i^{\ot 2} + \sum_{i \in [r]} \bm{\Delta}_i^{\ot 3}.
\end{split}
\end{align*}
By symmetry, it suffices to control
$\big\| \sum_{i \in [r]}  \bm{u}^{\star\ot 2}_i \ot \bm{\Delta}_i \big\| $,
$\big\| \sum_{i \in [r]} \bm{\Delta}_i^{\ot 2} \ot \bm{u}^\star_i  \big\|$
and
$\big\| \sum_{i \in [r]} \bm{\Delta}_i^{\ot 3} \big\|$.
Let us look at the first term. Towards this, for each $\left(i,j,k\right)\in\left[d\right]^{3}$, we can use the Cauchy-Schwartz inequality to control
\begin{align*}
\Big|\Big(\sum\nolimits _{1\leq s\leq r}\bm{u}_{s}^{\star\otimes2}\otimes\bm{\Delta}_{s}\Big)_{i,j,k}\Big| & =\Big|\sum\nolimits _{1\leq s\leq r}\left(\bm{u}_{s}^{\star}\right)_{i}\left(\bm{u}_{s}^{\star}\right)_{j}\left(\bm{\Delta}_{s}\right)_{k}\Big| \\
&	\leq\Big(\sum\nolimits _{1\leq s\leq r}\big[\left(\bm{u}_{s}^{\star}\right)_{i}\big]^{2}\Big)^{1/2}\Big(\sum\nolimits _{1\leq s\leq r}\big[\left(\bm{u}_{s}^{\star}\right)_{j}\big]^{2}\big[\left(\bm{\Delta}_{s}\right)_{k}\big]^{2}\Big)^{1/2} \\
&	\leq \left\Vert \bm{U}^{\star}\right\Vert _{2,\infty}\left\Vert \bm{\Delta}\right\Vert _{2,\infty} \max_{1\leq s\leq r}\left\Vert \bm{u}_{s}^{\star}\right\Vert _{\infty} \\
& \leq \delta \left\Vert \bm{U}^{\star}\right\Vert _{2,\infty}^{2} \max_{1\leq s\leq r}\left\Vert \bm{u}_{s}^{\star}\right\Vert _{\infty},
\end{align*}
which implies that 
\begin{align}
\Big\| \sum\nolimits _{1\leq s\leq r}\bm{u}_{s}^{\star\otimes2}\otimes\bm{\Delta}_{s} \Big\|_\infty \leq \delta \left\Vert \bm{U}^{\star}\right\Vert _{2,\infty}^{2} \max_{1\leq s\leq r}\left\Vert \bm{u}_{s}^{\star}\right\Vert _{\infty} \lesssim \frac{\delta \mu^{3/2} r \lambda_{\max}^\star }{d^{3/2}}. \label{eq:T-loss-inf-UB-term1}
\end{align}
In a similar manner, we can control the remaining two terms by
\begin{align}
\Big\| \sum\nolimits_{i \in [r]} \bm{\Delta}_i^{\ot 2} \ot \bm{u}^\star_i  \Big\| & \leq \left\Vert \bm{\Delta}\right\Vert _{2,\infty}^2 \max_{1\leq s\leq r}\left\Vert \bm{u}_{s}^{\star}\right\Vert _{\infty}
\leq  \delta^2 \norm{\bm{U}^\star}_{2, \infty}^2 \max_{1\leq s\leq r}\left\Vert \bm{u}_{s}^{\star}\right\Vert _{\infty} \leq \frac{\delta^2 \mu^{3/2} r \lambda_{\max}^\star}{d^{3/2}} ; \label{eq:T-loss-inf-UB-term2} \\
\Big\|  \sum\nolimits_{i \in [r]} \bm{\Delta}_i^{\ot 3} \Big\| & \leq \left\Vert \bm{\Delta} \right\Vert _{2,\infty}^2 \max_{1\leq s\leq r}\left\Vert \bm{\Delta}_{s} \right\Vert _{\infty}
\leq \delta^3 \norm{\bm{U}^\star}_{2, \infty}^3  \leq \frac{\delta^3 \mu^{3/2} r^{3/2} \lambda_{\max}^\star}{d^{3/2}} \label{eq:T-loss-inf-UB-term3} . 
\end{align}
%Hence, using the incoherence condition (Assumption \eqref{asmp_incoh}), we see that: with probability exceeding $1-O(d^{-10})$, 
%%
%\begin{align}
%\Big\| p^{-1} \PP_\Omega \Big( \sum_{i \in [r]} \bm{\Delta}_i \ot \bm{u}^{\star\ot 2}_i \Big) \Big\| &\lesssim \delta \mu^{3/2} r^{5/2} \lammax^{\star},   \\
%\Big\| p^{-1} \PP_\Omega \Big( \sum_{i \in [r]} \bm{\Delta}_i^{\ot 2} \ot \bm{u}^\star_i  \Big) \Big\| &\lesssim \delta^2 \mu^{3/2} r^{5/2} \lammax^{\star},  \\
% \Big\| p^{-1} \PP_\Omega \Big( \sum_{i \in [r]} \bm{\Delta}_i^{\ot 3}  \Big) \Big\| &\lesssim \delta^3 \mu^{3/2} r^{5/2} \lammax^{\star}.
%\end{align}
%%
Recall that $0<\delta \ll 1 / r \leq 1$. Putting these together reveals that
\begin{align}
\Big\| \sum\nolimits_{i \in [r]} \bm{u}^{\ot 3}_i - \bm{T}^\star \Big\| \lesssim \frac{\delta \mu^{3/2} r \lambda_{\max}^\star }{d^{3/2}}. \label{eq:gd-T-inf-loss}
\end{align}
This combined with \eqref{eq:T-all-one-op_norm} yields
\begin{align*}
\Big\| \, p^{-1} \PP_\Omega \Big( \sum\nolimits_{i \in [r]} \bm{u}^{\ot 3}_i - \bm{T}^\star \Big) \Big\| \lesssim \delta \mu^{3/2} r \lambda_{\max}^\star, 
\end{align*}
thus indicating that
\begin{align}
 	 \Big\| \, p^{-1}  \PP_\Omega \Big( \sum\nolimits_{i \in [r]} \bm{u}_i^{\ot 3} - \bm{T}^\star \Big) \times_3 \bm{u}_s \Big\|
	 \leq \norm{\bm{u}_s}_2 \Big\| p^{-1} \PP_\Omega \Big( \sum\nolimits_{i \in [r]} \bm{u}^{\ot 3}_i - \bm{T}^\star \Big) \Big\| \lesssim \delta \mu^{3/2} r \lammax^{\star 4/3},
\end{align}
where we use \eqref{eq:RIC_u_2_norm} in the last step.
In view of the condition that $\delta \ll 1 / (\mu^{3/2} r )$ and the assumption $\kappa \asymp 1$, one has with probability greater than $1-O(d^{-10})$,
\begin{align}
\label{lemma_RIC_alpha_3}
|\alpha_3| \leq \frac{1}{10} \lammin^{\star 4/3} \norm{\bm{V}}_\frob^2.
\end{align}
%
%\end{itemize}

\subsubsection{Putting all this together}

Note that the above bounds hold uniformly for all $\bm{V}$. Therefore, 
combining upper bounds for $\alpha_i$ and the union bound, we conclude that with probability exceeding $1-O(d^{-10})$,
\begin{align}
\vector(\bm{V})^\top \grad^2 f(\bm{U}) \vector(\bm{V}) &\geq \alpha_4 - |\alpha_1| - |\alpha_2| - |\alpha_3| \geq \frac{1}{2} \lambda_{\min}^{\star 4/3} \norm{\bm{V}}_\frob^2 \\
  \vector(\bm{V})^\top \grad^2 f(\bm{U}) \vector(\bm{V})  &\leq \alpha_4 + |\alpha_1| + |\alpha_2| + |\alpha_3| \leq 4 \lambda_{\max}^{\star 4/3} \norm{\bm{V}}_\frob^2 \label{eq:smoothness-final}
\end{align}
as claimed. 
%Here, the second inequality \eqref{eq:smoothness-final} immediately leads to the advertised spectral norm upper bound. 

%%%%%%%%%%%%%%%%%%%%%%%%%%%%%%%%%%%%%%%%%%%%%%%%%%%%%%%%%

\subsection{Proof of Lemma~\ref{lemma:hyp_aux}}
\label{pf:hyp_aux}

From \eqref{hyp:U_frob_loss} and \eqref{hyp:U_loo_frob_diff}, we use the triangle inequality to obtain
\begin{align*}
\big\| \bm{U}^{t, \m} - \bm{U}^\star \big\|_{\frob}
 &  \leq \big\| \bm{U}^{t, \m} - \bm{U}^t \big\|_{\frob} + \big\| \bm{U}^t - \bm{U}^\star \big\|_{\frob} \\
 & \leq \( 2 C_1 \rho^t \EE_\local + 2 C_2 \frac{\sigma}{\lammin^{\star}} \sqrt{\frac{ d \log d}{p}}   \) \norm{\bm{U}^\star}_{\frob}
 \end{align*}
Similarly,  we can combine \eqref{hyp:U_2inf_loss} and \eqref{hyp:U_loo_frob_diff} to obtain that
\begin{align*}
\big\| \bm{U}^{t, \m} - \bm{U}^\star \big\|_{2, \infty} & \leq \big\| \bm{U}^{t, \m} - \bm{U}^t \big\|_{2, \infty} + \left\| \bm{U}^t - \bm{U}^\star \right\|_{\frob}
  \leq \big\| \bm{U}^{t, \m} - \bm{U}^t \big\|_{\frob} + \left\| \bm{U}^t - \bm{U}^\star \right\|_{\frob} \\
& \leq \( C_5 \rho^t \EE_\local + C_6 \frac{\sigma}{\lammin^{\star}} \sqrt{\frac{ d \log d}{p}}   \) \norm{\bm{U}^\star}_{2, \infty}
 + \( C_3 \rho^t \EE_\local + C_4 \frac{\sigma}{\lammin^{\star}} \sqrt{\frac{ d \log d}{p}}   \) \norm{\bm{U}^\star}_{2, \infty} \\
& \leq \( \( C_3 + C_5 \) \rho^t \EE_\local + \( C_4 + C_6 \) \frac{\sigma}{\lammin^{\star}} \sqrt{\frac{ d \log d}{p}}   \) \norm{\bm{U}^\star}_{2, \infty} . \\
\end{align*}

%%%%%%%%%%%%%%%%%%%%%%%%%%%%%%%%%%%%%%%%%%%%%%%%%%%%

\subsection{Proof of Lemma~\ref{lemma:U_frob_loss}}
\label{pf:U_frob_loss}

In view of the relation \eqref{eq:gradient-grad-clean}, one has 
\begin{align*}
\left\| \bm{U}^{t+1} - \bm{U}^\star \right\|_\frob & = \left\| \bm{U}^{t} - \eta \( \grad f_{\clean} \( \bm{U}^t \) - p^{-1} \PP_\Omega \(\bm{E} \) \times_1^\seq \bm{U}^t \times_2^\seq \bm{U}^t \) - \bm{U}^\star \right\|_\frob \\
& \leq \underbrace{\left\| \bm{U}^{t} - \eta \grad f_{\clean} \( \bm{U}^t \) - \bm{U}^\star \right\|_\frob }_{=: \, \alpha_1}
+ \underbrace{ \eta \left\| p^{-1} \PP_\Omega \(\bm{E} \) \times_1^\seq \bm{U}^t \times_2^\seq \bm{U}^t \right\|_\frob }_{=: \, \alpha_2},
\end{align*}
which motivates us to bound $\alpha_1$ and $\alpha_2$ separately. 

\bigskip

(1) We start with $\alpha_1$, towards which we find it helpful to define 
\begin{align}
\bm{U}^t(\tau) := \tau \, \bm{U}^t + \( 1-\tau \) \bm{U}^\star.
\end{align}
Given that $ \grad f_\clean \( \bm{U}^\star \) = \bm{0}$ (since $\bm{U}^\star$ is  a global optimizer of $f_\clean$),  
we can use the fundamental theorem of calculus to obtain
\begin{align}
\vector \( \bm{U}^{t} - \eta \grad f_{\clean} \( \bm{U}^t \) - \bm{U}^\star \) 
&= \vector \( \bm{U}^{t} - \eta \grad f_\clean  \( \bm{U}^t \) - \( \bm{U}^\star - \eta \grad f_\clean \( \bm{U}^\star \) \) \)  \\
&= \vector \( \bm{U}^t  - \bm{U}^\star \) - \eta \, \vector  \( \grad f_\clean \(\bm{U}^{t} \) - \grad f_\clean \(\bm{U}^\star \) \) \\
	& = \Bigg( \bm{I}_{dr} - \eta \underset{=: \, \bm{\Gamma}}{\underbrace{ \int_0^1 \grad^2 f_\clean \( \bm{U}^t \(\tau \) \) \,\mathrm{d}\tau }} \Bigg) \, \vector \( \bm{U}^{t} - \bm{U}^\star \).
\end{align}
It then follows that
\begin{align}
	& \left\| \bm{U}^{t} - \eta \grad f_{\clean} \( \bm{U}^t \) - \bm{U}^\star \right\|_\frob^2 
	= \vector \left( \bm{U}^{t} - \bm{U}^\star \right)^\top \( \bm{I}_{dr} - \eta \bm{\Gamma} \)^2 \vector \left( \bm{U}^{t} - \bm{U}^\star \right) \nonumber \\
	& \qquad \leq \left\| \bm{U}^{t} - \bm{U}^\star \right\|_\frob^2 - 2 \eta \, \vector \left( \bm{U}^{t} - \bm{U}^\star \right)^\top \bm{\Gamma} \, \vector \left( \bm{U}^{t} - \bm{U}^\star \right) + \eta^2 \norm{\bm{\Gamma}}^2 \left\| \bm{U}^{t} - \bm{U}^\star \right\|_\frob^2. \label{eq:Ut-UB2-clean}
\end{align}
From the hypothesis \eqref{hyp:U_2inf_loss} as well as our conditions that $\frac{\sigma}{\lambda_{\min}^\star} \sqrt{\frac{d \log d}{p}} + \mathcal{E}_{\mathsf{local}} \ll \frac{1}{\mu^{3/2} r}$, we know that $\bm{U}^t(\tau)$ ($0\leq \tau \leq 1$) satisfies the conditions required in Lemma~\ref{lemma:RIC}. Therefore, applying Lemma~\ref{lemma:RIC} gives that 
\begin{align*}
 \vector \left( \bm{U}^{t} - \bm{U}^\star \right)^\top \bm{\Gamma} \, \vector \left( \bm{U}^{t} - \bm{U}^\star \right) &\geq \frac{1}{2} \lammin^{\star 4/3}  \left\| \bm{U}^{t} - \bm{U}^\star \right\|_\frob^2, \\
 \norm{\bm{\Gamma}} &\leq 4 \lammax^{\star 4/3}.
\end{align*}
Substitution into \eqref{eq:Ut-UB2-clean} indicates that: if $0 < \eta \leq  \lammin^{\star 4/3} / \big(32 \lammax^{\star 8/3} \big)$, then
\begin{align*}
\left\| \bm{U}^{t} - \eta \grad f_{\clean} \( \bm{U}^t \) - \bm{U}^\star \right\|_\frob^2 \leq \big(1 -  \lammin^{\star 4/3} \eta + 16  \lammax^{\star 8/3} \eta^2  \big) \left\| \bm{U}^{t} - \bm{U}^\star \right\|_\frob^2 \leq \big(1 - \frac{1}{2} \lammin^{\star 4/3} \eta  \big) \left\| \bm{U}^{t} - \bm{U}^\star \right\|_\frob^2 ,
\end{align*}
which implies that (since $1-a/2 \geq \sqrt{1-a}$ for $0<a<1$)
\begin{align}
\label{eq:U_frob_loss_alpha1}
	\left\| \bm{U}^{t} - \eta \grad f_{\clean} \( \bm{U}^t \) - \bm{U}^\star \right\|_\frob \leq \big(1 - \frac{1}{4} \lammin^{\star 4/3} \eta  \big) \left\| \bm{U}^{t} - \bm{U}^\star \right\|_\frob . 
\end{align}

\bigskip

(2) We now turn to $\alpha_2$. To simplify presentation, we shall assume that $\{E_{i,j,k}\}_{i,j,k\in[d]}$ (resp.~$\{\chi_{i,j,k}\}_{i,j,k\in[d]}$) are independent random variables. 
Fix an arbitrary $s \in [r]$ and $m \in [d]$. The $m$-th entry of $\PP_\Omega \(\bm{E} \) \times_1 \bm{u}^t_s  \times_2 \bm{u}^t_s $ can be expanded as follows:
\begin{align}
\left| \(  \PP_\Omega \(\bm{E} \) \times_1 \bm{u}^t_s  \times_2 \bm{u}^t_s \)_m \right|
& = \left| \bm{u}^{t \top}_s \(  \PP_\Omega \(\bm{E} \) \)_{:, :, m} \bm{u}^t_s  \right|   \nonumber \\
& \leq \underbrace{ \left| \bm{u}^{t, \m \top}_s \(  \PP_\Omega \(\bm{E} \) \)_{:, :, m} \bm{u}^{t, \m}_s  \right| }_{=: \, \beta_1 }
+  \underbrace{\left| \big( \bm{u}^{t}_s - \bm{u}^{t, \m}_s \big)^\top  \(  \PP_\Omega \(\bm{E} \) \)_{:, :, m} \bm{u}^{t}_s  \right|}_{=: \, \beta_2} \nonumber \\
& \quad + \underbrace{ \left| \bm{u}^{t \top}_s \(  \PP_\Omega \(\bm{E} \) \)_{:, :, m} \big( \bm{u}^{t}_s - \bm{u}^{t, \m}_s \big)  \right| }_{=: \, \beta_3}  
 + \underbrace{\left| \big( \bm{u}^{t}_s - \bm{u}^{t, \m}_s \big)^\top \(  \PP_\Omega \(\bm{E} \) \)_{:, :, m} \big( \bm{u}^{t}_s - \bm{u}^{t, \m}_s \big)  \right| }_{=: \, \beta_4}. \label{defn:beta1-4-all}
\end{align}
Before continuing, we make the following observations: from the hypotheses \eqref{hyp:U_frob_loss}, \eqref{hyp:U_2inf_loss}, \eqref{hyp:U_loo_frob_diff},  as well as our assumption that $\frac{\sigma}{\lambda_{\min}^\star} \sqrt{\frac{d \log d}{p}} + \mathcal{E}_{\mathsf{local}} \ll \frac{1}{\mu^{3/2} r }$, we see that the following holds for all $s \in \[ r \]$:
\begin{subequations}
\label{eq:gd-u-UB}
\begin{align}
\label{eq:gd-u-2-loss}
\big\| \bm{u}^{t}_s - \bm{u}^{\star}_s \big\|_2 & \leq \norm{\bm{U}^t - \bm{U}^\star}_\frob \ll  \frac{1}{r}   \norm{ \bm{U}^\star}_\frob \lesssim \frac{1}{\sqrt{r}}\lambda_{\max}^{\star 1/3},   \\
\label{eq:gd-u-inf-loss}
\big\| \bm{u}^{t}_s - \bm{u}^{\star}_s \big\|_\infty & \leq \norm{\bm{U}^t - \bm{U}^\star}_{2,\infty} \ll  \frac{1}{r}   \norm{ \bm{U}^\star}_{2,\infty} \lesssim \sqrt{\frac{\mu}{r d}} \, \lambda_{\max}^{\star 1/3},   \\
\label{eq:gd-u-loo-2-diff}
\big\| \bm{u}^{t, \m}_s - \bm{u}^{t}_s \big\|_2 & \leq \big\| \bm{U}^{t, \m} - \bm{U}^t \big\|_\frob \ll \frac{1}{r}   \norm{ \bm{U}^\star}_{2,\infty} \lesssim \sqrt{\frac{\mu}{r d}} \, \lambda_{\max}^{\star 1/3}, \\
\label{eq:gd-u-loo-entry-loss} 
\Big| \big( \bm{u}^{t, \m}_s - \bm{u}^{\star}_s \big)_m \Big| & \leq \big\| \big( \bm{U}^{t, \m} - \bm{U}^\star \big)_{m, :} \big\|_2  \ll \frac{1}{r} \left\| \bm{U}^\star \right\|_{2, \infty} \lesssim \sqrt{\frac{\mu}{r d}} \, \lambda_{\max}^\star, \\
\label{eq:gd-u-loo-2-loss} 
\big\| \bm{u}^{t, \m}_s - \bm{u}^{\star}_s \big\|_2 & \leq \big\| \bm{U}^{t, \m} - \bm{U}^\star \big\|_\frob \ll  \frac{1}{r}   \norm{ \bm{U}^\star}_\frob \lesssim \frac{1}{\sqrt{r}}\lambda_{\max}^{\star 1/3}, \\
\label{eq:gd-u-2-norm}
 \big\| \bm{u}^{t}_s \big\|_{2} &\leq \big\| \bm{u}^t_s - \bm{u}^\star_s \big\|_{2} + \big\| \bm{u}^\star_s \big\|_{2} \lesssim \lambda_{\max}^{1/3}, \\
 \label{eq:gd-u-inf-norm}
 \big\| \bm{u}^{t}_s \big\|_{\infty} &\leq \big\| \bm{u}^t_s - \bm{u}^\star_s \big\|_{\infty} + \big\| \bm{u}^\star_s \big\|_{\infty} \lesssim \sqrt{\frac{\mu}{d}} \, \lambda_{\max}^{1/3}, \\
\label{eq:gd-u-loo-2-norm}
 \big\| \bm{u}^{t, \m}_s \big\|_{2} &\leq \big\| \bm{u}^{t, \m}_s - \bm{u}^t_s \big\|_{2} + \big\| \bm{u}^t_s \big\|_{2} \lesssim \lambda_{\max}^{\star 1/3}, \\
\label{eq:gd-u-inf-norm}
 \big\| \bm{u}^{t, \m}_s \big\|_{\infty} &\leq \big\| \bm{u}^{t, \m}_s - \bm{u}^t_s \big\|_{2} + \big\| \bm{u}^{t}_s \big\|_{\infty} \lesssim \sqrt{\frac{\mu}{d}} \, \lambda_{\max}^{1/3}, \\
 \label{eq:gd-U-fro-norm}
 \big\| \bm{U}^t \big\|_\frob & \leq  \big\| \bm{U}^t - \bm{U}^\star \big\|_\frob + \big\| \bm{U}^\star \big\|_\frob \lesssim \big\| \bm{U}^\star \big\|_\frob, \\
  \label{eq:gd-U-2inf-norm}
 \big\| \bm{U}^t \big\|_{2,\infty} & \leq  \big\| \bm{U}^t - \bm{U}^\star \big\|_{2, \infty} + \big\| \bm{U}^\star \big\|_{2,\infty} \lesssim \big\| \bm{U}^\star \big\|_{2,\infty}, \\
   \label{eq:gd-U-loo-fro-norm}
 \big\| \bm{U}^{t, \m} \big\|_{\frob} & \leq  \big\| \bm{U}^{t, \m} - \bm{U}^t \big\|_\frob + \big\| \bm{U}^t \big\|_{\frob} \lesssim \big\| \bm{U}^\star \big\|_{\frob}, \\
  \label{eq:gd-U-loo-2inf-norm}
 \big\| \bm{U}^{t, \m} \big\|_{2, \infty} & \leq  \big\| \bm{U}^{t, \m} - \bm{U}^t \big\|_\frob + \big\| \bm{U}^t \big\|_{2,\infty} \lesssim \big\| \bm{U}^\star \big\|_{2,\infty}.
\end{align}
\end{subequations}
%under our sample size, rank and noise condition.
%

With these estimates in place, we can upper bound the above four terms in \eqref{defn:beta1-4-all} separately. 
\begin{itemize}
\item
For $\beta_1$, we note that, by construction, $\bm{u}^{t, \m}$ is independent of the $m$-th mode-3 slice of $\PP_\Omega \(\bm{E} \)$. This tells us that
\begin{align*}
\bm{u}^{t, \m \top}_s \(  \PP_\Omega \(\bm{E} \) \)_{:, :, m} \bm{u}^{t, \m}_s
= \sum\nolimits_{i, j \in [d]} \big( \bm{u}^{t, \m}_s \big)_i  \big( \bm{u}^{t, \m}_s \big)_j  E_{i, j, m} \chi_{i, j, m}
\end{align*}
can be viewed as a sum of independent zero-mean random variables (conditional on $\PP_{\Omega_{-m}} \( \bm{E} \)$).
It is straightforward to compute that
\begin{align*}
& \max_{i, j \in [d]} \left\| \big( \bm{u}^{t, \m}_s \big)_i  \big( \bm{u}^{t, \m}_s \big)_j  E_{i, j, m} \chi_{i, j, m} \right\|_{\psi_1} \lesssim \sigma \, \big\| \bm{u}^{t, \m}_s \big\|_{\infty}^2 =: L, \\
& \sum_{i, j \in [d]} \E \[ \big( \bm{u}^{t, \m}_s \big)_i^2  \big( \bm{u}^{t, \m}_s \big)_j^2 E_{i, j, m}^2 \chi_{i, j, m}^2   \] \lesssim \sigma^2 p \, \big\| \bm{u}^{t, \m}_s \big\|_{2}^4 \lesssim \sigma^2 \lammax^{\star 2/3} p \, \big\| \bm{u}^{t, \m}_s \big\|_{2}^2 =: V, 
\end{align*}
where $\|\cdot\|_{\psi_1}$ denotes the sub-exponential norm and we use \eqref{eq:gd-u-UB} in the last inequality.
%  \yxc{add explanation for the last inequality (before $V$)}. 
Applying the matrix Bernstein inequality \cite[Corollary 2.1]{koltchinskii2011oracle} yields that: with probability $1 - O(d^{-20})$,
\begin{align}
\left| \bm{u}^{t, \m \top}_s \(  \PP_\Omega \(\bm{E} \) \)_{:, :, m} \bm{u}^{t, \m}_s \right| & \lesssim L \log^2 d + \sqrt{V \log d} \nonumber \\
&  \lesssim \sigma \log^2 d \, \big\| \bm{u}^{t, \m}_s \big\|_{\infty}^2 + \sigma \lammax^{\star 1/3} \sqrt{ p \log d} \, \big\| \bm{u}^{t, \m}_s \big\|_{2} ,
\end{align}

\item For $\beta_2$, we first invoke \cite[lemma 11]{chen2015fast} to demonstrate that: with probability at least $1- O(d^{-11})$,
\begin{align}
\label{E_Omega_slice_op_norm}
\max_{m \in [d]} \big\| \(  \PP_\Omega \(\bm{E} \) \)_{:, :, m} \big\| \lesssim \sigma  \big( \sqrt{dp} + \log d \big).
\end{align}
Next, it is seen that 
\begin{align*}
\Big|  \big( \bm{u}^{t}_s & - \bm{u}^{t, \m}_s \big)^\top  \(  \PP_\Omega \(\bm{E} \) \)_{:, :, m} \bm{u}^{t}_s  \Big|  \leq \big\| \(  \PP_\Omega \(\bm{E} \) \)_{:, :, m} \big\| \big\| \bm{u}^{t}_s - \bm{u}^{t, \m}_s \big\|_2 \big\| \bm{u}^t_s \big\|_2 \\
& \lesssim \sigma \big( \sqrt{dp} + \log d \big) \big\| \bm{U}^{t} - \bm{U}^{t, \m} \big\|_{\frob}  \big\| \bm{u}^t_s \big\|_{2} \\
& \lesssim \sigma \big( \sqrt{dp} + \log d \big)  \( C_5 \rho^t \EE_\local + C_6 \frac{\sigma}{\lammin^{\star}} \sqrt{\frac{ d \log d}{p}}   \) \norm{\bm{U}^\star}_{2, \infty} \big\| \bm{u}^t_s \big\|_2 \\
& \ll \frac{\lammin^{\star 4/3}}{\sqrt{d}}  \( C_5 \rho^t \EE_\local + C_6 \frac{\sigma}{\lammin^{\star}} \sqrt{\frac{ d \log d}{p}}   \) \big\| \bm{u}^t_s \big\|_2,
\end{align*}
where the last inequality follows from the noise condition. Clearly, the above bound holds for $\beta_3$ as well. 
		
\item Regarding $\beta_4$, it is easily seen that $\beta_4 \ll \beta_2$,  given that $\| \bm{U}^{t} - \bm{U}^{t, \m} \|_{\frob} \ll \| \bm{U}^{\star} \|_{2,\infty} / r \leq \lambda_{\min}^\star $ holds according to \eqref{eq:gd-u-UB} and $\kappa \asymp 1$.
%\begin{align*}
% \Big| \big( \bm{u}^{t}_s - \bm{u}^{t, \m}_s \big)^\top \(  \PP_\Omega \(\bm{E} \) \)_{:, :, m} \big( \bm{u}^{t}_s - \bm{u}^{t, \m}_s \big)  \Big|  
% & \leq \left\| \(  \PP_\Omega \(\bm{E} \) \)_{:, :, m} \right\| \big\| \bm{u}^{t}_s - \bm{u}^{t, \m}_s \big\|_2^2 \\
%& \lesssim \sigma  \big( \sqrt{dp} + \log d \big) \big\| \bm{U}^{t} - \bm{U}^{t, \m} \big\|_{\frob}^2 \\
%& \ll \sigma  \big( \sqrt{dp} + \log d \big) \big\| \bm{U}^{t} - \bm{U}^{t, \m} \big\|_{\frob}  \left\| \bm{U}^\star \right\|_{\frob}
%\end{align*}

\item Taking together the above bounds and substituting them into \eqref{defn:beta1-4-all}, we obtain
\begin{align*}
 \left| \(  \PP_\Omega \(\bm{E} \) \times_1 \bm{u}^t_s  \times_2 \bm{u}^t_s \)_m \right| 
& \lesssim \sigma \log^2 d \, \big\| \bm{u}^{t, \m}_s \big\|_{\infty}^2 + \sigma \lammax^{\star 1/3} \sqrt{ p \log d} \, \big\| \bm{u}^{t, \m}_s \big\|_{2} \\
& \quad + o \( 1 \)  \frac{\lammin^{\star 4/3}}{\sqrt{d}}  \( C_5 \rho^t \EE_\local + C_6 \frac{\sigma}{\lammin^{\star}} \sqrt{\frac{ d \log d}{p}}   \) \big\| \bm{u}^t_s \big\|_2.
\end{align*}
Recognizing that this holds for any $m \in [d]$ and $s \in [r]$, one can sum over $m$ and $s$ to deduce that
\begin{align}
 \left\| \PP_\Omega \(\bm{E} \) \times_1^\seq \bm{U}^t  \times_2^\seq \bm{U}^t \right\|_\frob 
 & \overset{(\mathrm{i})}{\lesssim}  \frac{\sigma  \lambda_{\max}^{\star 2/3} \mu \sqrt{r} \log^2 d }{\sqrt{d}} + \sigma \lammax^{\star 1/3}  \sqrt{ d p \log d} \, \big\| \bm{U}^{t, \m}\big\|_{\frob} \nonumber \\
& \quad + o\( 1\) \lammin^{\star 4/3} \( C_5 \rho^t \EE_\local + C_6 \frac{\sigma}{\lammin^{\star}} \sqrt{\frac{ d \log d}{p}}   \) \norm{\bm{U}^t}_{\frob} \nonumber \\
& \leq C \sigma \lammin^{\star 1/3}  \sqrt{ d p \log d} \, \left\| \bm{U}^\star \right\|_{\frob} \nonumber \\
& \quad+ o\( 1\)  \lammin^{\star 4/3} \( C_5 \rho^t \EE_\local + C_6 \frac{\sigma}{\lammin^{\star}} \sqrt{\frac{ d \log d}{p}}   \) \norm{\bm{U}^\star}_{\frob}, \label{eq:U_frob_loss_alpha2}
\end{align}
for some absolute constant $C > 0$, where (i) is true due to \eqref{eq:gd-u-UB}; the last inequality follows from the fact that $\norm{\bm{U}^\star}_\frob \geq \lambda_{\min}^{\star 1/3} \sqrt{r}$ as well as the assumptions that $p \gg \mu d^{-2} \log^3 d$ and $\kappa \asymp 1$.
%
%We thus conclude that 
%\begin{align}
% \left\| p^{-1} \PP_\Omega \(\bm{E} \) \times_1^\seq \bm{U}^t \times_2^\seq \bm{U}^t \right\|_\frob 
% & \leq C \sigma \lammin^{\star 1/3} \sqrt{\frac{ d \log d}{p}} \left\| \bm{U}^\star \right\|_{\frob} \nonumber \\
%& \quad + o\( 1\) \lammin^{\star 4/3} \( C_5 \rho^t \EE_\local + C_6 \frac{\sigma}{\lammin^{\star}} \sqrt{\frac{ d \log d}{p}}   \) \norm{\bm{U}^\star}_{\frob}
%\end{align}

\end{itemize}

\bigskip

(3) Combining \eqref{eq:U_frob_loss_alpha1} and \eqref{eq:U_frob_loss_alpha2} yields that: with probability at least $1-O(d^{-10})$,
\begin{align*}
 \left\| \bm{U}^{t+1} - \bm{U}^\star \right\|_\frob & \leq \big(1 - \frac{1}{4}\lammin^{\star 4/3} \eta  \big) \left\| \bm{U}^{t} - \bm{U}^\star \right\|_\frob + C \eta \sigma \lammin^{\star 1/3} \sqrt{\frac{ d \log d}{p}} \left\| \bm{U}^\star \right\|_{\frob} \\
 & \qquad + o\( 1\)  \lammin^{\star 4/3} \eta \( C_5 \rho^t \EE_\local + C_6 \frac{\sigma}{\lammin^{\star}} \sqrt{\frac{ d \log d}{p}}   \) \norm{\bm{U}^\star}_{\frob} \\
& \leq \big(1 - \frac{1}{4}\lammin^{\star 4/3} \eta  \big) \( C_1 \rho^t \EE_\local + C_2 \frac{\sigma}{\lammin^{\star}} \sqrt{\frac{ d \log d}{p}}  \) \norm{\bm{U}^\star}_{\frob} + C \eta \sigma \lammin^{\star 1/3} \sqrt{\frac{ d \log d}{p}} \left\| \bm{U}^\star \right\|_{\frob} \\
& \qquad + o\( 1\)   \lammin^{\star 4/3} \eta \( C_5 \rho^t \EE_\local + C_6 \frac{\sigma}{\lammin^{\star}} \sqrt{\frac{ d \log d}{p}}   \) \norm{\bm{U}^\star}_{\frob} \\
& \leq \big(1 - \frac{1}{5}\lammin^{\star 4/3} \eta  \big) C_1 \rho^t \EE_\local \norm{\bm{U}^\star}_{\frob}
 + \( \big(1 - \frac{1}{5}\lammin^{\star 4/3} \eta  \big) C_2 + C \eta \lammin^{\star 3/4}  \)  \frac{\sigma}{\lammin^{\star}}  \sqrt{\frac{ d \log d}{p}}   \norm{\bm{U}^\star}_{\frob}  \\
& \leq C_1 \rho^{t+1} \EE_\local \norm{\bm{U}^\star}_{\frob} + C_2 \frac{\sigma}{\lammin^{\star}} \sqrt{\frac{ d \log d}{p}}   \norm{\bm{U}^\star}_{\frob} ,
\end{align*}
 with the proviso that $0 < \eta \leq  \lammin^{\star 4/3} / \big(32 \lammax^{\star 8/3} \big)$, $1 - \big( \lammin^{\star 4/3}  / 5 \big) \eta \leq \rho < 1$ and that $C_2$ is sufficiently large.

%\end{itemize}

%%%%%%%%%%%%%%%%%%%%%%%%%%%%%%%%%%%%%%%%%%%%%%%%%%%%

\subsection{Proof of Lemma~\ref{lemma:U_loo_frob_diff}}
\label{pf:U_loo_frob_diff}

Fix an arbitrary $m \in [r]$. From the definition of $f^{\m}$ in \eqref{eq:defn-fm-loo} and \eqref{grad_f_loo}, we can show that
\begin{align}
& \bm{U}^{t+1}-\bm{U}^{t+1, \m}  = \bm{U}^{t} - \eta \grad f \big( \bm{U}^t \big) - \Big(\bm{U}^{t, \m} - \eta \grad f^{\m}\big( \bm{U}^{t, \m} \big) \Big) \nonumber\\
& \quad = \bm{U}^{t} - \bm{U}^{t, \m} - \eta \(\grad f \big( \bm{U}^t \big) - \grad f \big(\bm{U}^{t, \m} \big) \) - \eta \( \grad f \big(\bm{U}^{t, \m} \big) - \grad f^{\m} \big( \bm{U}^{t, \m} \big) \) \nonumber\\
& \quad = \underbrace{\bm{U}^{t} - \bm{U}^{t, \m} - \eta \(\grad f_\clean \big( \bm{U}^t \big) - \grad f_\clean \big(\bm{U}^{t, \m} \big) \) }_{=: \, \alpha_1} \nonumber\\
	& \quad \quad - \underbrace{ \eta \, \Big( \( p^{-1} \PP_{\Omega_m} - \PP_{m} \) \Big( \sum\nolimits_{s \in [r]} \big( \bm{u}^{t, \m }_s \big)^{\ot 3} - \bm{T}^\star \Big)  \Big) \times_1^\seq \bm{U}^{t, \m} \times_2^\seq \bm{U}^{t, \m}   }_{=: \, \alpha_2} \nonumber\\
	& \qquad  + \underbrace{ \eta \, p^{-1} \PP_{\Omega_m} \( \bm{E} \) \times_1^\seq \bm{U}^{t, \m}  \times_2^\seq \bm{U}^{t, \m}  }_{=: \, \alpha_3} + \underbrace{ \eta  p^{-1}  \, \big\{ \PP_\Omega \( \bm{E} \)  \times_1^\seq  \bm{U}^t  \times_2^\seq  \bm{U}^t -     \PP_\Omega \( \bm{E} \)   \times_1^\seq  \bm{U}^{t, \m} \times_2^\seq   \bm{U}^{t, \m} \big\} }_{=: \, \alpha_4} .
	\label{eq:U-loo-alpha-1234}
\end{align}
%
%Before proceeding to bounding these terms, we pause to collect a few useful properties. 
%From \eqref{hyp:U_2inf_loss} and \eqref{hyp:U_loo_frob_diff}, use the triangle inequality to derive \yxc{Collect these kinds of properties and state them in a lemma right after introducing the key lemmas.}
%%
%\begin{align}
%\label{U_loo_2inf_loss}
%\begin{split}
%\big\| \bm{U}^{t, \m} - \bm{U}^\star \big\|_{2, \infty} & \leq \big\| \bm{U}^{t, \m} - \bm{U}^t \big\|_{\frob} + \big\| \bm{U}^{t} - \bm{U}^\star \big\|_{2, \infty} \\
%& \leq \( 2C_3 \rho^t \left\{ \frac{\mu^{3/2} r^3 \log^3 (d)}{\sqrt{d^{3/2} p}} + \frac{\sqrt{\mu} \, r^2 \log d }{\sqrt{d}} \right\} + 2C_4 \frac{\sigma}{\lammin^{\star}} \sqrt{\frac{\kappa r d \log d}{p}}   \) \norm{\bm{U}^\star}_{2, \infty}
%\end{split}
%\end{align}
%
Before proceeding to bounding these terms, we pause to define
%By the triangle inequality and Cauchy-Schwarz, we know for any $i, j, k \in [d]$ that
%\begin{align*}
%\Big| \sum_{s \in [r]} \big( \bm{u}^{t, \m \ot 3}_s - \bm{u}^{\star \ot 3}_s \big)_{i, j, k}  \Big| 
%& \lesssim \Big| \sum_{s \in [r]} \big( \bm{u}^{t, \m}_s - \bm{u}^{\star}_s \big)_i \big(\bm{u}^{\star}_s \big)_j \big(\bm{u}^{\star}_s \big)_k \Big| \\
%& \leq \big\| \big( \bm{U}^{t, \m} -\bm{U}^\star \big)_{i, :} \big\|_{2}  \big\| \bm{U}^\star_{j, :} \big\|_{2} \max_{s \in [r]} \left| \big(\bm{u}^{\star}_s \big)_k \right| .
%\end{align*}
%This implies that: if we define
%
%
\begin{align}
\label{def:delta_T_loo} & \bm{\Delta}^{t, \m}_{\bm{T}} := \sum_{s \in [r]} \big( \bm{u}^{t, \m }_s \big)^{\ot 3} - \bm{T}^\star.
\end{align}
From \eqref{eq:gd-u-UB} and hypothesis~\eqref{hyp:U_loo_row_loss}, one can applying a similar argument as in \eqref{eq:gd-T-inf-loss} to find that
\begin{align}
\label{delta_T_loo_inf_loss}
\big\| \bm{\Delta}^{t, \m}_{\bm{T}}  \big\|_\infty \lesssim  \max_{s \in [r]} \big\|\bm{u}^{\star}_s \big\|_\infty \big\| \bm{U}^\star \big\|_{2, \infty} \big\| \bm{U}^{t, \m} -\bm{U}^\star \big\|_{2, \infty} 
\lesssim \frac{\mu \sqrt{r} \, \lammax^{\star 2 / 3}}{d} \big\| \bm{U}^{t, \m} -\bm{U}^\star \big\|_{2, \infty}. 
\end{align}
%
%where for notational convenience we define
Now we begin to bound the terms in \eqref{eq:U-loo-alpha-1234} separately. 

\bigskip

(1) We start with $\alpha_1$. For any $0\leq \tau \leq 1$, define 
\begin{align*}
\bm{U}^{t, \m}(\tau) := \tau \, \bm{U}^t + \( 1-\tau \) \bm{U}^{t, \m}.
\end{align*}
The fundamental theorem of calculus yields
\begin{align*}
	\bm{U}^{t} - \bm{U}^{t, \m} - \eta \(\grad f_\clean \big( \bm{U}^t \big) - \grad f_\clean \big(\bm{U}^{t, \m} \big) \) 
	= \underbrace{\Big(\bm{I}_{dr} - \eta \int_0^1 \grad^2 f_\clean \big( \bm{U}^{t, \m} \( \tau \) \big) \, \mathrm{d}\tau \Big)}_{=: \bm{\Gamma}}  \big(\bm{U}^{t} - \bm{U}^{t, \m} \big).
\end{align*}
By Lemma~\ref{lemma:hyp_aux} and our assumptions on the noise, we know that $\bm{U}^{t, \m}(\tau)$ ($ 0 \leq \tau \leq 1$) satisfies the conditions in Lemma~\ref{lemma:RIC} for any $\tau \in [0, 1]$. 
%Moreover, we know $\bm{U}^{t} - \bm{U}^{t, \m}$ is a valid direction from the hypothesis \eqref{hyp:U_loo_frob_diff}. 
Applying the same argument as the one used to bound $\left\| \bm{U}^{t} - \eta \grad f_{\clean} \( \bm{U}^t \) - \bm{U}^\star \right\|_\frob $ in Lemma~\ref{lemma:U_frob_loss}, we show that
\begin{align}
\label{U_loo_frob_diff_term1}
\Big\| \bm{U}^{t} - \bm{U}^{t, \m} - \eta \Big(\grad f \big( \bm{U}^t \big) - \grad f \big( \bm{U}^{t, \m} \big) \Big) \Big\|_2 &\leq \big(1 - \frac{1}{4} \lammin^{\star 4/3} \eta  \big) \big\| \bm{U}^{t} - \bm{U}^{t, \m} \big\|_\frob ,
\end{align}
provided that $0 < \eta \leq  \lammin^{\star 4/3} / \big(32 \lammax^{\star 8/3} \big)$.

In what follows, we shall assume that $\{E_{i,j,k}\}_{i,j,k\in[d]}$ (resp.~$\{\chi_{i,j,k}\}_{i,j,k\in[d]}$) are independent random variables in order to simplify presentation. 

\bigskip
(2) The next step is to bound $\alpha_2$. 
%Towards this, we first see that
%%
%\begin{align}
%\begin{split}
%\label{beta2_ub}
%& \Big\| \Big( \big( p^{-1} \PP_{\Omega_m} - \PP_{m} \big)  \big( \bm{\Delta}^{t, \m}_{\bm{T}} \big) \Big)  \tsym \bm{U}^{t, \m} \tsym \bm{U}^{t, \m} \Big\|_\frob \\
%& \qquad \qquad \leq \Big\| \Big( \big( p^{-1} \PP_{\Omega_m} - \PP_{m} \big) \big( \bm{\Delta}^{t, \m}_{\bm{T}} \big) \Big) \times_1^\seq \bm{U}^{t, \m} \times_2^\seq \bm{U}^{t, \m} \Big\|_\frob \\
%& \qquad \qquad \quad + \Big\| \Big( \big( p^{-1} \PP_{\Omega_m} - \PP_{m} \big)\big( \bm{\Delta}^{t, \m}_{\bm{T}} \big) \Big) \times_1^\seq \bm{U}^{t, \m} \times_3^\seq \bm{U}^{t, \m} \Big\|_\frob \\
%& \qquad \qquad \quad + \Big\| \Big( \big( p^{-1} \PP_{\Omega_m} - \PP_{m} \big)\big( \bm{\Delta}^{t, \m}_{\bm{T}} \big) \Big) \times_2^\seq \bm{U}^{t, \m} \times_3^\seq \bm{U}^{t, \m} \Big\|_\frob.
%\end{split}
%\end{align}
%%
%We shall focus on upper bounding the first term; the other two terms can be controlled in the same way.
%
For notational simplicity, define
\begin{align*}
	\bm{V}^{t, \m} &:= \Big( \( p^{-1} \PP_{\Omega_m} - \PP_{m} \) \big( \bm{\Delta}^{t, \m}_{\bm{T}} \big) \Big) \times_1^\seq \bm{U}^{t, \m} \times_2^\seq \bm{U}^{t, \m} ; \\
	\bm{v}^{t, \m}_s &:= \Big( \( p^{-1} \PP_{\Omega_m}  - \PP_m \) \big( \bm{\Delta}^{t, \m}_{\bm{T}} \big) \Big) \times_1 \bm{u}^{t, \m}_s \times_2 \bm{u}^{t, \m}_s , \qquad s \in [r].
\end{align*}

In order to control the Frobenius norm of $\bm{V}^{t, \m}$, we shall start by  considering the $m$-th row of $\bm{V}^{t, \m}$. In view of the definitions of $\PP_{\Omega_m}$ and $\PP_m$ (cf.~Appendix \ref{sec:LOO-sequence}), we can expand
\begin{align*}
\bm{V}^{t, \m}_{m, :} &= \sum_{i, j \in [d]} \sum_{s \in [r]} \( p^{-1} \chi_{i, j, m}-1 \) \big( \bm{\Delta}^{t, \m}_{\bm{T}} \big)_{i, j, m} \big( \bm{u}^{t, \m}_s \big)_i \big( \bm{u}^{t, \m}_s \big)_j   \bm{e}_s^\top.
\end{align*}
We recognize that $\big\{ \bm{u}_s^{t, \m} \big\}_{s=1}^r$ is independent of $\Omega_m$, making it convenient for us to upper bound $\big\| \bm{V}^{t, \m}_{m, :} \big\|_2$. Specifically, for any $i, j \in [d]$, from \eqref{eq:gd-u-UB} and \eqref{delta_T_loo_inf_loss} we have
\begin{align*}
\Big\| \sum_{s \in [r]} \big( \bm{\Delta}^{t, \m}_{\bm{T}} \big)_{i, j, m} \big( \bm{u}^{t, \m}_s \big)_i \big( \bm{u}^{t, \m}_s \big)_j \( p^{-1} \chi_{i, j, m}-1 \)  \bm{e}_s^\top \Big\|_2 
&\leq \frac{1}{p} \, \big\| \bm{\Delta}^{t, \m}_{\bm{T}}  \big\|_\infty \max_{s \in [r]} \big\| \bm{u}^{t, \m}_s \big\|_\infty \big\| \bm{U}^{t, \m} \big\|_{2, \infty} \\
& \lesssim \frac{\mu^2 r \lammax^{\star 4/3}}{d^2 p}   \big\| \bm{U}^{t, \m} - \bm{U}^\star \big\|_{2, \infty} =: L_1.
% & \lesssim p^{-1} \, \Big\| \bm{\Delta}^{t, \m} \big\|_\infty \left\| \bm{U}^\star \right\|_{2, \infty}^4,
\end{align*}
%
%\yxc{Explain the last inequality above.} 
In addition, it is easy to verify that $\E \big[ \bm{V}^{t, \m}_{m, :} \big] = \bm{0}$ and 
\begin{align*}
\sum_{i, j \in [d]} \E & \Bigg[ \Big\| \sum\nolimits_{s \in [r]}  \big( \bm{\Delta}^{t, \m}_{\bm{T}} \big)_{i, j, m} \big( \bm{u}^{t, \m}_s \big)_i \big( \bm{u}^{t, \m}_s \big)_j ( p^{-1} \chi_{i, j, m}-1 ) \, \bm{e}_s^\top \Big\|_2^2 \Bigg] \\
&=   \sum_{i, j \in [d]} \sum_{s \in [r]}\big( \bm{\Delta}^{t, \m}_{\bm{T}} \big)^2_{i, j, m} \big( \bm{u}^{t, \m}_s \big)_i^2 \big( \bm{u}^{t, \m}_s \big)_j^2 \, \E \big[ ( p^{-1}{\chi}_{i, j, m}-1 )^2 \big] \\
& \leq \frac{1}{p} \, \big\| \bm{\Delta}^{t, \m}_{\bm{T}} \big\|_\infty^2 \max_{s \in [r]} \big\| \bm{u}^{t, \m}_s \big\|_2^2 \big\| \bm{U}^{t, \m} \big\|_\frob^2 \lesssim \frac{\mu^2 r^2 \lammax^{\star 8/3}}{d^2 p} \big\| \bm{U}^{t, \m} - \bm{U}^\star \big\|_{2, \infty}^2 =: V_1,
\end{align*}
where the last inequality holds due to \eqref{eq:gd-u-UB} and \eqref{delta_T_loo_inf_loss}. We then apply the matrix Bernstein inequality to yield that: with probability exceeding $1-  O(d^{-20})$,
\begin{align}
\big\| \bm{V}^{t, \m}_{m, :} \big\|_2 & \lesssim L_1 \log d + \sqrt{V_1 \log d} \lesssim \Bigg\{ \frac{\mu^2 r \lammax^{\star 4/3} \log d}{d^2 p} + \frac{\mu r \lammax^{\star 4/3} \sqrt{ \log d}}{d \sqrt{p}} \Bigg\} \big\| \bm{U}^{t, \m} - \bm{U}^\star \big\|_\infty \nonumber \\
& \asymp \frac{\mu r \lammax^{\star 4/3} \sqrt{ \log d}}{d \sqrt{p}} \big\| \bm{U}^{t, \m} - \bm{U}^\star \big\|_\infty , \label{V_loo_row_2_norm}
\end{align}
where the last step holds as long as $p \gg \mu^2 d^{-2} \log d$.

Next, we turn to the $k$-th row of $\bm{V}^{t, \m}$ for any $k \neq m$. For each $s \in [r]$, we have
\begin{align*}
\big( \bm{v}^{t, \m}_s \big)_k &= \big( \bm{u}^{t, \m}_s \big)_m \sum_{j \in [d]} \big( \bm{\Delta}^{t, \m}_{\bm{T}} \big)_{m, j, k}   \big( \bm{u}^{t, \m}_s \big)_j ( p^{-1} \chi_{m, j, k}-1 ) \\
& \quad + \big( \bm{u}^{t, \m}_s \big)_m \sum_{i: i \neq m} \big( \bm{\Delta}^{t, \m}_{\bm{T}} \big)_{i, m, k} \big( \bm{u}^{t, \m}_s \big)_i (p^{-1}\chi_{i, m, k}-1 ).
%& \leq \Big| \sum_{j \in [d]} \big( \bm{\Delta}^{t, \m}_{\bm{T}} \big)_{m, j, k}  \big( \bm{u}^{t, \m}_s \big)_j \big( \bm{u}^{t, \m}_s \big)_k (p^{-1} \chi_{m, j, k}-1 ) \Big| 
%+ \Big| \sum_{i \neq m} \big( \bm{\Delta}^{t, \m}_{\bm{T}} \big)_{i, m, k}  \big( \bm{u}^{t, \m}_s \big)_i \big( \bm{u}^{t, \m}_s \big)_k (p^{-1} \chi_{i, m, k}-1) \Big|
\end{align*}
Similar to the proof of Lemma~\ref{lemma:sum_jk_square_cont_ineq}, we can show that with probability at least $1-  O(d^{-20})$, 
%\yxc{This needs more explanation}
%
\begin{align*}
 \sum_{s \in [r]} \sum_{k: k \neq m } & \Big( \big( \bm{u}^{t, \m}_s \big)_m \sum\nolimits_{j \in [d]} \big( \bm{\Delta}^{t, \m}_{\bm{T}} \big)_{m, j, k}  \big( \bm{u}^{t, \m}_s \big)_j (p^{-1} \chi_{m, j, k}-1 ) \Big)^2 \\
& \lesssim  \sum_{s \in [r]} \sum_{k: k \neq m } \E \Big[ \Big( \big( \bm{u}^{t, \m}_s \big)_m \sum\nolimits_{j \in [d]} \big( \bm{\Delta}^{t, \m}_{\bm{T}} \big)_{m, j, k}  \big( \bm{u}^{t, \m}_s \big)_j (p^{-1} \chi_{m, j, k}-1 ) \Big)^2 \Big] \\
& \lesssim \frac{1}{p} \sum_{s \in [r]} \big( \bm{u}^{t, \m}_s \big)_m^2 \sum_{j, k \in [d]} \big( \bm{\Delta}^{t, \m}_{\bm{T}} \big)_{m, j, k}^2 \big( \bm{u}^{t, \m}_s \big)_j^2  \\
& \leq \frac{d}{p} \max_{s \in [r]} \big\| \bm{u}^{t, \m}_s \big\|_\infty^2 \big\| \bm{\Delta}^{t, \m}_{\bm{T}} \big\|_\infty^2 \big\| \bm{U}^{t, \m} \big\|_\frob^2
 \lesssim \frac{\mu^3 r^2 \lammax^{\star 8/3}}{d^2 p} \big\| \bm{U}^{t, \m} - \bm{U}^\star \big\|_{2, \infty}^2 , 
%& \sum_{k \neq m } \Big( \sum_{i \neq m} \big( \bm{\Delta}^{t, \m}_{\bm{T}} \big)_{i, m, k}  \big( \bm{u}^{t, \m}_s \big)_i \big( \bm{u}^{t, \m}_s \big)_k (p^{-1} \chi_{i, m, k}-1) \Big)^2 \lesssim d^2 p^{-1} \big\| \bm{\Delta}^{t, \m}_{\bm{T}} \big\|_\infty^2 \big\| \bm{u}^{t, \m}_s \big\|_\infty^4
\end{align*}
where the last inequality follows from \eqref{eq:gd-u-UB}, and \eqref{delta_T_loo_inf_loss}. It is easily seen that the bound also holds for the summation over $i \neq m$. We can then use Cauchy-Schwartz  to arrive at
\begin{align}
\label{V_loo_k_row_2_norm_appendix}
\sum_{k: k \neq m} \big\| \bm{V}^{t, \m}_{k, :} \big\|_2^2 = \sum_{s \in [r]} \sum_{k: k \neq m} \big( \bm{v}^{t, \m}_s \big)_k^2 \lesssim \frac{\mu^3 r^2 \lammax^{\star 8/3}}{d^2 p} \big\| \bm{U}^{t, \m} - \bm{U}^\star \big\|_{2, \infty}^2. 
\end{align}

Combining \eqref{V_loo_k_row_2_norm_appendix} with \eqref{V_loo_row_2_norm} and invoking the union bound, we conclude  that: with probability exceeding $1- O(d^{-20})$,
\begin{align}
\begin{split}
\label{U_loo_frob_diff_term2}
\Big\| \Big( \big( p^{-1} \PP_{\Omega_m}  - \PP_m \big) \big( \bm{\Delta}^{t, \m}_{\bm{T}} \big) \Big) \times_1^\seq \bm{U}^{t, \m} \times_2^\seq \bm{U}^{t, \m} \Big\|_\frob =  \big\| \bm{V}^{t, \m} \big\|_\frob \leq C  \frac{\mu^{3/2} r \lammin^{\star 4/3} \sqrt{\log d} }{d \sqrt{p}} \, \big\| \bm{U}^{t, \m} - \bm{U}^\star \big\|_{2, \infty}.
\end{split}
\end{align}
%%
%We can apply  similar arguments to control the remaining two terms in \eqref{beta2_ub}. To summarize, with probability exceeding $1-  O(d^{-20})$ one has
%%
%\begin{align}
%\Big\| \Big( \big( p^{-1} \PP_{\Omega_m}  - \PP_m \big) \big( \bm{\Delta}^{t, \m}_{\bm{T}} \big) \Big) \tsym \bm{U}^{t, \m} \tsym \bm{U}^{t, \m} \Big\|_\frob
%  \leq C \lammax^{\star 4/3} \frac{\mu^{3/2} r^{3/2} \sqrt{ \log d} }{d \sqrt{p}} \, \big\| \bm{U}^{t, \m} - \bm{U}^\star \big\|_{2, \infty}
%\end{align}
%%
for some absolute constant $C > 0$, where we use the assumption that $\kappa \asymp 1$.

\bigskip

(3) For $\alpha_3$, following a similar argument for $\alpha_2$, we define
\begin{align*}
	 \bm{W}^{t, \m} &:= \PP_{\Omega_m} \( \bm{E} \) \times_1^\seq \bm{U}^{t, \m}_s \times_2^\seq \bm{U}^{t, \m}_s \\
	 \bm{w}^{t, \m}_s &:= \PP_{\Omega_m} \( \bm{E} \) \times_1 \bm{u}^{t, \m}_s \times_2 \bm{u}^{t, \m}_s
\end{align*}
for each $s \in [r]$. The $m$-th row of $ \bm{W}^{t, \m}$ is a sum of independent zero-mean random vectors:
\begin{align*}
\bm{W}^{t, \m}_{m, :} &= \sum_{i, j \in [d]} \sum_{s \in [r]}  \big( \bm{u}^{t, \m}_s \big)_i \big( \bm{u}^{t, \m}_s \big)_j E_{i, j, m} \chi_{i, j, m} \bm{e}_s^\top.
\end{align*}
With \eqref{eq:gd-u-UB} in place, it is easy to verify that
\begin{align*}
\max_{i, j \in [d]} \Big\| \sum\nolimits_{s \in [r]} \big( \bm{u}^{t, \m}_s \big)_i \big( \bm{u}^{t, \m}_s \big)_j E_{i, j, m} \chi_{i, j, m} \bm{e}_s^\top \Big\|_{\psi_1} & \leq \sigma \max_{s \in [r]} \big\| \bm{u}^{t, \m}_s \big\|_{\infty}  \big\| \bm{U}^{t, \m} \big\|_{2, \infty}  \\
& \lesssim \sigma \sqrt{\frac{\mu }{d}} \, \lammax^{\star 1/3} \norm{\bm{U}^\star}_{2, \infty} =: L_2
\end{align*}
and
\begin{align*}
\sum_{i, j \in [d]} \E \Bigg[ \Big\| \sum\nolimits_{s \in [r]}  \big( \bm{u}^{t, \m}_s \big)_i \big( \bm{u}^{t, \m}_s \big)_j E_{i, j, m} \chi_{i, j, m} \bm{e}_s^\top  \Big\|_2^2   \Bigg]
& = \sum_{i, j \in [d]} \sum_{s \in [r]} \big( \bm{u}^{t, \m}_s \big)_i^2 \big( \bm{u}^{t, \m}_s \big)_j^2 \E \big[ \( E_{i, j, m} \chi_{i, j, m} \)^2 \big] \\
& \leq \sigma^2 p \max_{s \in [r]} \big\| \bm{u}^{t, \m}_s \big\|_2^2 \big\| \bm{U}^{t, \m} \big\|_\frob^2 \\
& \lesssim \sigma^2 d p \lammax^{\star 2/3}  \norm{\bm{U}^\star}_{2, \infty}^2 =: V_2,
\end{align*}
where we have used \eqref{eq:gd-u-UB} and the fact that $\norm{\bm{U}^\star}_{\frob} \leq \sqrt{d} \norm{\bm{U}^\star}_{2, \infty}$.
Apply the matrix Bernstein inequality to reveal that: with probability at least $1- O(d^{-20})$,
\begin{align}
\left\| \bm{W}^{t, \m}_{m, :} \right\|_2 &= \Big\| \sum_{i, j \in [d]} \sum_{s \in [r]}  \big( \bm{u}^{t, \m}_s \big)_i \big( \bm{u}^{t, \m}_s \big)_j E_{i, j, m} \chi_{i, j, m} \bm{e}_s^\top \Big\|_2 \nonumber \\
& \lesssim L_2 \log^2 d + \sqrt{V_2 \log d} \asymp \sigma \lammax^{\star 1/3} \sqrt{ d p \log d}  \norm{\bm{U}^\star}_{2, \infty} , \label{eq:U_loo_2_loss_W}
\end{align}
where the last inequality holds as long as $p \gg \mu d^{-2} \log^3 d$. 

As for the other rows, we have
\begin{align*}
\big( \bm{w}^{t, \m}_s \big)_k &= \big( \bm{u}^{t, \m}_s \big)_m \sum_{j \in [d]}   \big( \bm{u}^{t, \m}_s \big)_j E_{m, j, k} \chi_{m, j, k}
 + \big( \bm{u}^{t, \m}_s \big)_m \sum_{i: i \neq m}   \big( \bm{u}^{t, \m}_s \big)_i E_{i, m, k} \chi_{i, m, k}
\end{align*}
for each $s \in [r]$. Arguing similarly as in the proof of Lemma~\ref{lemma:sum_jk_square_cont_ineq_subg},  
we have with probability at least $1- O(d^{-20})$,
\begin{align*}
\sum_{s \in [r]} \sum_{k: k \neq m} & \Big( \big( \bm{u}^{t, \m}_s \big)_m \sum\nolimits_{j \in [d]}   \big( \bm{u}^{t, \m}_s \big)_j E_{m, j, k} \chi_{m, j, k} \Big)^2 \\
& \lesssim  \sum_{s \in [r]} \sum_{k: k \neq m } \E \Big[ \Big( \big( \bm{u}^{t, \m}_s \big)_m \sum_{j \in [d]}   \big( \bm{u}^{t, \m}_s \big)_j \( E_{m, j, k} \chi_{m, j, k} \) \Big)^2 \Big] \\
& \leq \sigma^2 d p \sum_{s \in [r]} \big( \bm{u}^{t, \m}_s \big)_m^2 \sum_{j \in [d]}  \big( \bm{u}^{t, \m}_s \big)_j^2 \\
& \leq \sigma^2 d p \max_{s \in [r]} \big\| \bm{u}^{t, \m}_s \big\|_2^2 \big\| \bm{U}^{t, \m} \big\|_{2, \infty}^2
 \lesssim \sigma^2 \lammax^{\star 2/3} d p  \norm{\bm{U}^\star}_{2, \infty}^2 ,
\end{align*}
where the last inequality follows from \eqref{eq:gd-u-UB}. Additionally, the summation over $\{i: i \neq m\}$ can be controlled using the same argument. Therefore, we use Cauchy-Schwarz to find that
\begin{align}
\sum_{k: k \neq m} \left\| \bm{W}^{t, \m}_{k, :} \right\|_2^2 \lesssim \sigma^2 \lammax^{\star 2/3} d p  \norm{\bm{U}^\star}_{2, \infty}^2.
\end{align}
Combined with \eqref{eq:U_loo_2_loss_W} and the assumption that $\kappa \asymp 1$, we obtain that
\begin{align}
\label{U_loo_frob_diff_term3}
 \left\| p^{-1} \PP_{\Omega_m} \( \bm{E} \) \times_1^\seq \bm{U}^{t, \m} \times_2^\seq \bm{U}^{t, \m} \right\|_\frob = \big\| \bm{W}^{t, \m} \big\|_\frob \leq \widetilde{C} \sigma \lammin^{\star 1/3} \sqrt{ d p \log d} \norm{\bm{U}^\star}_{2, \infty},
\end{align}
%
%which implies that
%%
%\begin{align}
%\label{U_loo_frob_diff_term3}
% \left\| p^{-1} \PP_{\Omega_m} \( \bm{E} \) \tsym \bm{U}^{t, \m} \tsym \bm{U}^{t, \m} \right\|_\frob \leq \widetilde{C} \sigma \lammax^{\star 1/3}\sqrt{\frac{\mu r d \log d}{p}} \norm{\bm{U}^\star}_{2, \infty}
%\end{align}
%%
for some absolute constant $\widetilde{C} > 0$.

\bigskip

(4) Regarding $\alpha_4$, we use the triangle inequality to show that for each $s \in [r]$,
\begin{align*}
 \big\|  \PP_\Omega \(\bm{E} \) \times_1 \bm{u}^t_s \times_2 \bm{u}^t_s & - \PP_\Omega \(\bm{E} \) \times_1 \bm{u}^{t, \m}_s \times_2 \bm{u}^{t, \m}_s \big\|_2  \\
&  \leq \big\|  \PP_\Omega \(\bm{E} \) \times_1 \big( \bm{u}^t_s -  \bm{u}^{t, \m}_s \big) \times_2 \bm{u}^{t, \m}_s \big\|_2  
 + \big\|  \PP_\Omega \(\bm{E} \) \times_1  \bm{u}^{t, \m}_s \times_2 \big( \bm{u}^t_s -  \bm{u}^{t, \m}_s \big) \big\|_2   \\
 & \quad + \big\|  \PP_\Omega \(\bm{E} \) \times_1 \big( \bm{u}^t_s - \bm{u}^{t, \m}_s \big) \times_2 \big( \bm{u}^t_s - \bm{u}^{t, \m}_s \big) \big\|_2  \\
 & \leq 2 \norm{ \PP_\Omega \(\bm{E} \) } \big\| \bm{u}^t_s - \bm{u}^{t, \m}_s \big\|_2 \big\| \bm{u}^{t, \m}_s \big\|_2 + \norm{ \PP_\Omega \(\bm{E} \) } \big\| \bm{u}^t_s - \bm{u}^{t, \m}_s \big\|_2^2 \\
 & \lesssim \norm{ \PP_\Omega \(\bm{E} \) } \big\| \bm{u}^t_s - \bm{u}^{t, \m}_s \big\|_2 \big\| \bm{u}^{t, \m}_s \big\|_2 ,
\end{align*}
where the last line follows from \eqref{eq:gd-u-UB}. From Corollary~\ref{cor:bound-PE}, we can further upper bound
\begin{align*}
\big\|  \PP_\Omega \(\bm{E} \) \times_1 \bm{u}^t_s \times_2 \bm{u}^t_s - \PP_\Omega \(\bm{E} \) \times_1 \bm{u}^{t, \m}_s \times_2 \bm{u}^{t, \m}_s \big\|_2
& \leq \sigma \big( \sqrt{dp} + \log d \big) \log^{5/2} d \, \big\| \bm{u}^t_s - \bm{u}^{t, \m}_s \big\|_2  \big\| \bm{u}^{t, \m}_s \big\|_2 .
\end{align*}
As a result, we sum over $s \in [r]$ and use the Cauchy-Schwartz inequality to derive
\begin{align}
\frac{1}{p} \big\| \PP_\Omega \(\bm{E} \) \times_1^\seq \bm{U}^t  \times_2^\seq \bm{U}^t & -  \PP_\Omega \(\bm{E} \) \times_1^\seq \bm{U}^{t, \m} \times_2^\seq  \bm{U}^{t, \m} \big\|_\frob \nonumber \\
&  \lesssim \frac{\sigma}{p} \big( \sqrt{dp} + \log d \big) \log^{5/2} d \,  \big\| \bm{U}^t - \bm{U}^{t, \m} \big\|_\frob \big\| \bm{U}^{t, \m} \big\|_\frob \nonumber \\
&  \lesssim \frac{\sigma}{p}  \big( \sqrt{dp} + \log d \big) \log^{5/2} d \( C_5 \rho^t \EE_\local + C_6 \frac{\sigma}{\lammin^{\star}} \sqrt{\frac{ d \log d}{p}}   \)  \norm{\bm{U}^\star}_{2, \infty}^2 \nonumber \\
%& \quad \ll \lammin^{\star 4/3} \( C_5 \rho^t \left\{ \frac{\mu^{3/2} r^3 \log^3 (d)}{\sqrt{d^{3/2} p}} + \frac{\sqrt{\mu} \, r^2 \log d }{\sqrt{d}} \right\} + C_6 \frac{\sigma}{\lammin^{\star}} \sqrt{\frac{ \mu r d \log d}{p}}  \) \norm{\bm{U}^\star}_{2, \infty}
& \ll  \lammin^{\star 4/3} \( C_5 \rho^t \EE_\local + C_6 \frac{\sigma}{\lammin^{\star}} \sqrt{\frac{ d \log d}{p}}  \) \norm{\bm{U}^\star}_{2, \infty}, \label{U_loo_frob_diff_term4}
\end{align}
where the last step arises from conditions that $\frac{\sigma}{\lambda_{\min}^\star} \sqrt{\frac{d \log d}{p}} \ll \frac{1}{\mu^{3/2} r}$ and $\kappa \asymp 1$.

\bigskip
(5) Taking \eqref{U_loo_frob_diff_term1}, \eqref{U_loo_frob_diff_term2}, \eqref{U_loo_frob_diff_term3} and \eqref{U_loo_frob_diff_term4} together, we can invoke the sample size assumption that $p \gg \mu^3 r^2 d^{-2} \log^3 d$ and the union bound to show that: with probability greater than $1-O(d^{-10})$ one has 
\begin{align*}
 \big \| \bm{U}^{t+1}  &-\bm{U}^{t+1, \m} \big\|_\frob \\
& \leq \big(1 - \frac{1}{4}  \lammin^{\star 4/3} \eta  \big) \big\| \bm{U}^{t} - \bm{U}^{t, \m} \big\|_\frob + C \frac{\mu^{3/2} r  \lammin^{\star 4/3} \sqrt{ \log d} }{d \sqrt{p}} \, \big\| \bm{U}^{t, \m} - \bm{U}^\star \big\|_{2, \infty} \\
& \quad + \widetilde{C} \eta \sigma \lammin^{\star 1/3}\sqrt{\frac{\mu r d \log d}{p}} \norm{\bm{U}^\star}_{2, \infty} + o \( 1 \)  \lammin^{\star 4/3} \eta \( C_5 \rho^t \EE_\local + C_6 \frac{\sigma}{\lammin^{\star}} \sqrt{\frac{ d \log d}{p}}  \) \norm{\bm{U}^\star}_{2, \infty} \\
& \leq \big(1 - \frac{1}{4}  \lammin^{\star 4/3} \eta  \big)   \( C_5 \rho^t \EE_\local + C_6 \frac{\sigma}{\lammin^{\star}} \sqrt{\frac{ d \log d}{p}}   \) \norm{\bm{U}^\star}_{2, \infty} + \widetilde{C} \eta \sigma \lammin^{\star 1/3}\sqrt{\frac{ d \log d}{p}} \norm{\bm{U}^\star}_{2, \infty} \\
& \quad + C \eta \lammax^{\star 4/3}  \frac{\mu^{3/2} r \sqrt{ \log d} }{d \sqrt{p}} \( 2 C_3 \rho^t  \EE_\local + 2 C_4 \frac{\sigma}{\lammin^{\star}} \sqrt{\frac{ d \log d}{p}}   \) \norm{\bm{U}^\star}_{2, \infty} \\
& \quad + o \( 1 \)  \lammin^{\star 4/3} \eta \( C_5 \rho^t \EE_\local + C_6 \frac{\sigma}{\lammin^{\star}} \sqrt{\frac{ \mu r d \log d}{p}}  \) \norm{\bm{U}^\star}_{2, \infty} \\
& \leq \big(1 - \frac{1}{5}  \lammin^{\star 4/3} \eta  \big) C_5 \rho^t \EE_\local \norm{\bm{U}^\star}_{2, \infty}
 + \( \big( 1- \frac{1}{5} \lammin^{\star 4/3} \eta \big) C_6 + \widetilde{C} \lammin^{\star 4/3} \eta \) \frac{\sigma}{\lammin^{\star}} \sqrt{\frac{ d \log d}{p}} \norm{\bm{U}^\star}_{2, \infty} \\
& \leq \big(1 - \frac{1}{5}  \lammin^{\star 4/3} \eta  \big)  C_5 \rho^t \EE_\local  \norm{\bm{U}^\star}_{2, \infty} + C_6 \frac{\sigma}{\lammin^{\star}} \sqrt{\frac{ d \log d}{p}} \norm{\bm{U}^\star}_{2, \infty} \\
& \leq  \( C_5 \rho^{t+1} \EE_\local + C_6 \frac{\sigma}{\lammin^{\star}} \sqrt{\frac{ d \log d}{p}}   \) \norm{\bm{U}^\star}_{2, \infty} ,
\end{align*}
provided $0 < \eta \leq  \lammin^{\star 4/3} / \big(32 \lammax^{\star 8/3} \big)$, $1 -  \big( \lambda_{\min}^{\star 4/3} / 5 \big) \eta \leq \rho < 1$ and $C_6$ is sufficiently large.

%\end{itemize}

%%%%%%%%%%%%%%%%%%%%%%%%%%%%%%%%%%%%%%%%%%%%%%%%%%

\subsection{Proof of Lemma~\ref{lemma:U_loo_row_loss}}
\label{pf:U_loo_row_loss}

Fix an arbitrary $m \in [d]$. Recall our notation of $\bm{\Delta}^{t, \m}_{\bm{T}}$ in \eqref{def:delta_T_loo}.
%
%\begin{align}
%	\label{def:delta_T_loo_2} & \bm{\Delta}^{t, \m}_{\bm{T}} := \sum_{s \in [r]} \big( \bm{u}^{t, \m }_s \big)^{\ot 3} - \bm{T}^\star.
%\end{align}
%
To simplify presentation, we further define 
\begin{align}
\label{u_t+1_loo_tilde}
	\widehat{\bm{U}}^{t+1, \m} &:= \bm{U}^{t, \m} - \eta \( p^{-1} \PP_{\Omega_{-m}} \big( \bm{\Delta}^{t, \m}_{\bm{T}} - \bm{E} \big) \times_1^\seq \bm{U}^\star \times_2^\seq \bm{U}^\star + \PP_{m} \big( \bm{\Delta}^{t, \m}_{\bm{T}} \big) \times_1^\seq \bm{U}^\star \times_2^\seq \bm{U}^\star \),  \\
	\bm{\Delta}^{t, \m}_s &:= \bm{u}^{t, \m}_s - \bm{u}^\star_s,
\end{align}
for each $s \in [r]$. 
%From Lemma~\ref{lemma:hyp_aux} and our condition that $\frac{\sigma}{\lambda_{\min}^\star} \sqrt{\frac{d \log d}{p}} + \mathcal{E}_{\mathsf{local}} \ll \frac{1}{\mu^{3/2}\sqrt{r}}$, we have 
%%
%\begin{align}
%&  \max_{s \in [r]} \big\| \bm{u}^{t, \m}_s - \bm{u}^\star_s \big\|_2 \leq \big\| \bm{U}^{t, \m} - \bm{U}^\star \big\|_\frob   \ll  \norm{\bm{U}^\star}_{\frob} / \sqrt{r} \label{u_loo_2_loss}.
%%  \max_{s \in [r]} \big\| \bm{u}^{t, \m}_s \big\|_2 \leq  \max_{s \in [r]} \big\| \bm{u}^{t, \m}_s - \bm{u}^\star_s \big\|_2 +  \max_{s \in [r]} \big\| \bm{u}^{\star}_s \big\|_2  \lesssim \lammax^{\star 1/3} \label{u_loo_2_norm}
%\end{align}
%In addition, the hypothesis~\eqref{hyp:U_loo_row_loss} implies that
%\begin{align}
%\label{u_loo_entry_loss} 
%\begin{split}
% \max_{s \in [r]} \big| \big( \bm{\Delta}^{t, \m}_s \big)_m \big| & \leq \big\| \big( \bm{U}^{t, \m} - \bm{U}^\star \big)_{m, :} \big\|_2  \ll \frac{1}{ \sqrt{r}} \left\| \bm{U}^\star \right\|_{2, \infty} \lesssim \sqrt{\frac{\mu}{d}} \, \lambda_{\max}^\star.
%\end{split}
%\end{align}

Apply the triangle inequality to yield
\begin{align*}
\big\| \big(\bm{U}^{t+1, \m} - \bm{U}^\star \big)_{m, :} \big\|_2 &\leq  \underbrace{\big\| \big(\widehat{\bm{U}}^{t+1, \m} -\bm{U}^\star \big)_{m, :} \big\|_2}_{=: \, \alpha_1}
+ \underbrace{ \big\|\big(\bm{U}^{t+1, \m} - \widehat{\bm{U}}^{t+1, \m} \big)_{m, :} \big\|_2}_{=: \, \alpha_2}, 
\end{align*}
leaving us with two terms to deal with. As it turns out, we will show that $\alpha_1$ is the dominant term and $\alpha_2$ is negligible. To simplify presentation, we shall assume that $\{E_{i,j,k}\}_{i,j,k\in[d]}$ (resp.~$\{\chi_{i,j,k}\}_{i,j,k\in[d]}$) are independent random variables.

\begin{itemize}
\item
Regarding $\alpha_1$, the definition of $\PP_{\Omega_{-m}}$ allows us to derive 
\begin{align*}
	\big( \widehat{\bm{U}}^{t+1, \m} - \bm{U}^\star \big)_{m, :}  &= \big( \bm{U}^{t, \m} - \bm{U}^\star -\eta \, \bm{\Delta}^{t, \m}_{\bm{T}} \times_1^\seq \bm{U}^\star \times_2^\seq \bm{U}^\star \big)_{m, :}  \\
	&= \big( \bm{U}^{t, \m} - \bm{U}^\star \big)_{m, :} - \eta \, \big( \bm{\Delta}^{t, \m}_{\bm{T}} \times_1^\seq \bm{U}^\star \times_2^\seq \bm{U}^\star \big)_{m, :} .
\end{align*}
We can express $\bm{\Delta}^{t, \m}_{\bm{T}} = \sum_{s \in [r]} \big( \bm{\Delta}_s^{t, \m} + \bm{u}_s^\star \big)^{\ot 3} - \bm{u}_s^{\star \ot 3}$ and compute that
\begin{align}
%\begin{split}
&  \big( \bm{\Delta}^{t, \m}_{\bm{T}}  \times_1 \bm{u}^{\star}_s \times_2 \bm{u}^{\star}_s \big)_m \nonumber\\ 
& \qquad =  
 \Big( \big\| \bm{u}^{\star}_s \big\|_2^2 + \big\langle \bm{\Delta}^{t, (m)}_s, \bm{u}^\star_s \big\rangle \Big)^2 \big( \bm{\Delta}_s^{t, \m} \big)_m
 + \Big( 2\, \big\| \bm{u}^{\star}_s \big\|_2^2 \big\langle \bm{\Delta}^{t, (m)}_s, \bm{u}^\star_s \big\rangle  + \big\langle \bm{\Delta}^{t, (m)}_s, \bm{u}^{\star}_s \big\rangle^2 \Big) \big( \bm{u}^{\star}_s \big)_m \nonumber\\
& \qquad\qquad + \sum_{i:i \neq s} \Big( \big\langle \bm{u}^\star_i, \bm{u}^\star_s \big\rangle + \big\langle \bm{\Delta}^{t, (m)}_i, \bm{u}^\star_s \big\rangle \Big)^2 \big( \bm{\Delta}_i^{t, \m} \big)_m \nonumber\\
& \qquad\qquad + \sum_{i:i \neq s}  \Big( 2 \, \big\langle \bm{u}^{\star}_i, \bm{u}^{\star}_s \big\rangle \big\langle \bm{\Delta}^{t, (m)}_i, \bm{u}^\star_s \big\rangle  + \big\langle \bm{\Delta}^{t, (m)}_i, \bm{u}^{\star}_s \big\rangle^2 \Big) \big( \bm{u}^{\star}_i \big)_m 
	\label{Delta_T_tsym_utrue}
%\end{split}
\end{align}
for each $s \in [r]$. This further indicates that
\begin{align*}
\left\| \big( \widehat{\bm{U}}^{t+1, \m} - \bm{U}^\star \big)_{m, :} \right\|_2 
& \leq \underbrace{\Big\| \sum\nolimits_{s \in [r]} \Big( 1 - \eta  \( \big\| \bm{u}^{\star}_s \big\|_2^2 + \big\langle \bm{\Delta}^{t, (m)}_s, \bm{u}^\star_s \big\rangle \)^2  \Big)  \big( \bm{\Delta}_s^{t, \m} \big)_m \bm{e}_s^\top \Big\|_2 }_{=: \, \beta_1} \\
&  \quad + \underbrace{ \eta \, \Big\| \sum\nolimits_{s \in [r]}  \Big( 2 \, \big\| \bm{u}^{\star}_s \big\|_2^2 \big\langle \bm{\Delta}^{t, (m)}_s, \bm{u}^\star_s \big\rangle  + \big\langle \bm{\Delta}^{t, (m)}_s, \bm{u}^{\star}_s \big\rangle^2 \Big) \big( \bm{u}^{\star}_s \big)_m \bm{e}_s^\top \Big\|_2 }_{=: \, \beta_2} \\
& \quad + \underbrace{ \eta \, \Big\| \sum\nolimits_{s \in [r]}  \sum\nolimits_{i:i \neq s} \Big( \big\langle \bm{u}^\star_i, \bm{u}^\star_s \big\rangle + \big\langle \bm{\Delta}^{t, (m)}_i, \bm{u}^\star_s \big\rangle \Big)^2 \big( \bm{\Delta}_i^{t, \m} \big)_m \bm{e}_s^\top \Big\|_2 }_{=: \, \beta_3} \\
& \quad + \underbrace{ \eta \, \Big\| \sum\nolimits_{s \in [r]} \sum\nolimits_{i:i \neq s}  \Big( 2 \, \big\langle \bm{u}^{\star}_i, \bm{u}^{\star}_s \big\rangle \big\langle \bm{\Delta}^{t, (m)}_i, \bm{u}^\star_s \big\rangle  + \big\langle \bm{\Delta}^{t, (m)}_i, \bm{u}^{\star}_s \big\rangle^2 \Big) \left( \bm{u}^{\star}_i \right)_m \bm{e}_s^\top \Big\|_2 }_{=: \, \beta_4}. 
\end{align*}
In what follows, we will control the four terms separately.
\begin{itemize}
\item For $\beta_1$, by \eqref{eq:gd-u-UB}, we use Cauchy-Schwarz to show that
\begin{align*}
 \Big( \big\| \bm{u}^\star_s \big\|_2^2 + \big\langle \bm{\Delta}^{t, (m)}_s, \bm{u}^\star_s \big\rangle \Big)^2
 \geq \( \big\| \bm{u}^\star_s \big\|_2^2 - \big\| \bm{\Delta}^{t, (m)}_s \big\|_2 \big\| \bm{u}^{\star}_s \big\|_2 \)^2
 \geq \frac{2}{3} \big\| \bm{u}^\star_s \big\|_2^4 \geq \frac{2}{3} \lammin^{\star 4/3}
\end{align*}
for each $s \in [r]$. It follows that
\begin{align}
\beta_1 \leq \big( 1- \frac{2}{3}  \lammin^{\star 4/3} \eta \big) \left\| \big( \bm{U}^{t, \m} - \bm{U}^\star \big)_{m, :} \right\|_2.
\end{align}

\item Regarding $\beta_2$, by \eqref{eq:gd-u-UB}, we apply the Cauchy-Schwarz inequality again to get that: for each $s \in [r]$,
\begin{align*}
 \Big| \Big( 2 \, \big\| \bm{u}^{\star}_s \big\|_2^2 \big\langle \bm{\Delta}^{t, (m)}_s, \bm{u}^\star_s \big\rangle  + \big\langle \bm{\Delta}^{t, (m)}_s, \bm{u}^{\star}_s \big\rangle^2 \Big) \big( \bm{u}^{\star}_s \big)_m \Big|
 \leq 3 \, \big\| \bm{u}^{\star}_s \big\|_2^3 \big\| \bm{u}^{\star}_s \big\|_\infty \big\| \bm{\Delta}^{t, (m)}_s \big\|_2
 \leq 3 \sqrt{\frac{\mu}{d}} \, \lammax^{\star 4/3} \big\| \bm{\Delta}^{t, (m)}_s \big\|_2,
\end{align*}
Together with the assumption that $\kappa \asymp 1$, this implies that
\begin{align}
\beta_2 \leq 3 \eta \lammin^{\star 4/3} \sqrt{\frac{\mu}{d}} \,  \big\| \bm{U}^{t, \m} - \bm{U}^\star \big\|_\frob.
\end{align}

\item With regards to $\beta_3$, we show that for each $s \in [r]$:
\begin{align*}
\Big| & \sum\nolimits_{i:i \neq s}  \Big(  \big\langle \bm{u}^\star_i, \bm{u}^\star_s \big\rangle + \big\langle \bm{\Delta}^{t, (m)}_i, \bm{u}^\star_s \big\rangle \Big)^2 \big( \bm{\Delta}_i^{t, \m} \big)_m \Big| \\
& \;\; \lesssim \max_{i:i \neq s} \big| \big\langle \bm{u}^\star_i, \bm{u}^\star_s \big\rangle \big|^2  \sum_{i:i \neq s} \left| \big( \bm{\Delta}_i^{t, \m} \big)_m \right| 
+ \big\| \bm{u}^\star_s \big\|_2^2 \sum_{i:i \neq s} \big\| \bm{\Delta}^{t, (m)}_i \big\|_2^2   \left| \big( \bm{\Delta}_i^{t, \m} \big)_m \right| \\
& \;\; \leq  \lammax^{\star 4/3} \frac{\mu \sqrt{r} }{d} \big\| \big( \bm{U}^{t, \m} - \bm{U}^\star \big)_{m, :} \big\|_2  
+ \lammax^{\star 2 / 3} \max_{i:i \neq s} \big\| \bm{\Delta}^{t, (m)}_i \big\|_2 \big\| \bm{U}^{t, \m} - \bm{U}^\star \big\|_\frob \big\| \big( \bm{U}^{t, \m} - \bm{U}^\star \big)_{m, :} \big\|_2   \\
& \;\; \leq \lammax^{\star 4/3} \Big(  \frac{\mu \sqrt{r} }{d}  + o\( 1/r \) \Big) \big\| \big( \bm{U}^{t, \m} - \bm{U}^\star \big)_{m, :} \big\|_2 \ll \frac{\lammax^{\star 4/3}}{\sqrt{r}} \big\| \big( \bm{U}^{t, \m} - \bm{U}^\star \big)_{m, :} \big\|_2, 
\end{align*}
where the last line follows from \eqref{eq:gd-u-UB} that $\max_{i:i \neq s} \big\| \bm{\Delta}^{t, (m)}_i \big\|_2 \leq \big\| \bm{U}^{t, \m} - \bm{U}^\star \big\|_\frob \ll \lambda_{\max}^{\star 1/3} / \sqrt{r}$ and the low rank condition $r \ll \sqrt{d / \mu}$.
Summing over $s \in [r]$, we get
\begin{align}
\beta_3 \ll \eta  \lammin^{\star 4/3} \left\| \big( \bm{U}^{t, \m} - \bm{U}^\star \big)_{m, :} \right\|_2
\end{align}
under the condition $\kappa \asymp 1$.

\item Turning attention to $\beta_4$, we observe that for each $s \in [r]$,
\begin{align*}
& \Big| \sum\nolimits_{i:i \neq s}  \Big(  2 \, \big\langle \bm{u}^{\star}_i, \bm{u}^{\star}_s \big\rangle  \big\langle \bm{\Delta}^{t, (m)}_i, \bm{u}^\star_s \big\rangle  + \big\langle \bm{\Delta}^{t, (m)}_i, \bm{u}^{\star}_s \big\rangle^2 \Big) \left( \bm{u}^{\star}_i \right)_m \Big| \\
& \quad \lesssim \max_{i:i \neq s} \big| \big\langle \bm{u}^\star_i, \bm{u}^\star_s \big\rangle \big| \big\| \bm{u}^\star_s \big\|_2 \sum\nolimits_{i:i \neq s} \big\| \bm{\Delta}^{t, (m)}_i \big\|_2   \left| \big( \bm{u}_i^\star \big)_m \right|
+  \big\| \bm{u}^\star_s \big\|_2^2 \sum\nolimits_{i:i \neq s} \big\| \bm{\Delta}^{t, (m)}_i \big\|_2^2   \left| \big( \bm{u}_i^\star \big)_m \right| \\
& \quad\leq \lammax^{\star}  \sqrt{\frac{\mu }{d}} \,  \big\| \bm{U}^{\star} \big\|_{2, \infty}   \big\| \bm{U}^{t, \m} - \bm{U}^\star \big\|_{\frob}   
+ \lammax^{\star 2/3} \max_{i:i \neq s} \big\| \bm{\Delta}^{t, (m)}_i \big\|_2 \,  \big\| \bm{U}^\star \big\|_{2, \infty} \big\| \bm{U}^{t, \m} - \bm{U}^\star \big\|_\frob \\
& \quad\lesssim \lammax^{\star}\sqrt{\frac{\mu r}{d}} \, \Big( \sqrt{\frac{\mu}{d}} \, \lammax^{\star 1/3}  + o ( 1 / \sqrt{r}  )  \Big) \big\| \bm{U}^{t, \m} - \bm{U}^\star \big\|_{\frob} \ll \lammax^{\star 4/3} \sqrt{\frac{\mu}{d}}\,  \big\| \bm{U}^{t, \m} - \bm{U}^\star \big\|_\frob,
\end{align*}  
where the last line follows from \eqref{eq:gd-u-UB} and the rank assumption $r \ll \sqrt{d / \mu}$. As a consequence, 
\begin{align}
\beta_4 \ll \eta \lammax^{\star 4/3} \sqrt{\frac{\mu r}{d}} \big\| \bm{U}^{t, \m} - \bm{U}^\star \big\|_\frob .
\end{align}

\item Therefore, we have
\begin{align}
\label{eq:u_loo_utrue_row_term2}
\left\| \big( \widehat{\bm{U}}^{t+1, \m} - \bm{U}^\star \big)_{m, :} \right\|_2 
&\leq \big( 1-  \frac{1}{3}  \lammin^{\star 4/3} \eta \big) \left\| \big( \bm{U}^{t, \m} - \bm{U}^\star \big)_{m, :} \right\|_2
+ 4 \eta \lammax^{\star 4/3} \sqrt{\frac{\mu r}{d}}  \, \big\| \bm{U}^{t, \m} - \bm{U}^\star \big\|_\frob.
\end{align}

\end{itemize}

\item
With regards to $\alpha_2$, it follows from the definition \eqref{grad_f_loo} and \eqref{u_t+1_loo_tilde} that 
\begin{align}
& \bm{U}^{t+1, \m} - \widehat{\bm{U}}^{t+1, \m} \nonumber \\
& = - \eta \(  \( p^{-1} \PP_{\Omega_{-m}} +\PP_{m} \) \big( \bm{\Delta}^{t, \m}_{\bm{T}} \big) \times_1^\seq \bm{U}^{t, \m} \times_2^\seq \bm{U}^{t, \m} - \( p^{-1} \PP_{\Omega_{-m}} +\PP_{m} \) \big( \bm{\Delta}^{t, \m}_{\bm{T}} \big) \times_1^\seq \bm{U}^\star \times_2^\seq  \bm{U}^\star \) \nonumber \\
& \quad + \frac{\eta}{p} \(  \PP_{\Omega_{-m}} \( \bm{E} \) \times_1^\seq \bm{U}^{t, \m} \times_2^\seq \bm{U}^{t, \m} - \PP_{\Omega_{-m}} \( \bm{E} \) \times_1^\seq \bm{U}^\star \times_2^\seq \bm{U}^\star \) .
\end{align}
Recall the definition of $\PP_{\Omega_{-m}}$ and $\PP_{m}$. For the $m$-th row, we have
\begin{align*}
\big( \bm{U}^{t+1, \m} - \widehat{\bm{U}}^{t+1, \m} \big)_{m, :} = -\eta  \big( \bm{\Delta}^{t, \m}_{\bm{T}} \times_1^\seq \bm{U}^{t, \m} \times_2^\seq \bm{U}^{t, \m} - \bm{\Delta}^{t, \m}_{\bm{T}} \times_1^\seq \bm{U}^\star \times_2^\seq \bm{U}^\star \big)_{m, :}
\end{align*}
From the triangle inequality, we can further decompose
\begin{align*}
& \Big\|\big( \bm{\Delta}^{t, \m}_{\bm{T}} \times_1^\seq \bm{U}^{t, \m} \times_2^\seq \bm{U}^{t, \m} - \bm{\Delta}^{t, \m}_{\bm{T}} \times_1^\seq \bm{U}^\star \times_2^\seq \bm{U}^\star \big)_{m, :} \Big\|_2 \\
&\quad \leq \underbrace{\norm{ \big( \bm{\Delta}^{t, \m}_{\bm{T}} \times_1^\seq \big( \bm{U}^{t, \m} - \bm{U}^\star \big) \times_2^\seq \bm{U}^\star \big)_{m, :} }_2}_{=: \gamma_1}
 +  \underbrace{\norm{ \big( \bm{\Delta}^{t, \m}_{\bm{T}}  \times_1^\seq \bm{U}^\star \times_2^\seq \big( \bm{U}^{t, \m} - \bm{U}^\star \big) \big)_{m, :} }_2}_{=: \gamma_2} \\
& \quad \quad + \underbrace{ \norm{ \big( \bm{\Delta}^{t, \m}_{\bm{T}}  \times_1^\seq \big( \bm{U}^{t, \m} - \bm{U}^\star \big) \times_2^\seq \big( \bm{U}^{t, \m} - \bm{U}^\star \big) \big)_{m, :} }_2 }_{=: \gamma_3}.
\end{align*}

Let us consider $\gamma_1$ first. It is straightforward to calculate that
\begin{align*}
& \big( \bm{\Delta}^{t, \m}_{\bm{T}} \times_1  \bm{\Delta}^{t, \m}_s \times_2 \bm{u}^\star_s \big)_m = \sum_{i \in [r]} \big\langle \bm{u}_i^{t, \m}, \bm{\Delta}_s^{t, \m} \big \rangle \big \langle \bm{u}_i^{t, \m}, \bm{u}_s^\star \big \rangle  \big( \bm{\Delta}_i^{t, \m} \big)_m \\
& \quad + \sum_{i \in [r]} \Big( \big \langle \bm{\Delta}_i^{t, \m}, \bm{\Delta}_s^{t, \m} \big \rangle \big \langle \bm{u}_i^{t, \m}, \bm{u}_s^{t, \m} \big \rangle + \big \langle \bm{u}_i^\star, \bm{\Delta}_s^{t, \m} \big \rangle \big \langle \bm{\Delta}_i^{t, \m}, \bm{u}_s^\star \big \rangle + \big \langle \bm{\Delta}_i^{t, \m}, \bm{\Delta}_s^{t, \m} \big \rangle \big \langle \bm{\Delta}_i^{t, \m}, \bm{u}_s^{t, \m} \big \rangle  \Big) \big( \bm{u}_i^\star \big)_m.
\end{align*}
for each $s \in [r]$.
From \eqref{eq:gd-u-UB}, we use the triangle inequality and the Cauchy-Schwarz inequality to upper bound
\begin{align*}
 \Big| \big( & \bm{\Delta}^{t, \m}_{\bm{T}}  \times_1 \bm{\Delta}^{t, \m}_s \times_2 \bm{u}^\star_s \big)_m \Big| \lesssim \big\| \bm{\Delta}_s^{t, \m}  \big\|_2  \norm{\bm{u}_s^\star}_2 \sum_{i \in [r]} \big\| \bm{u}_i^{t, \m} \big\|_2^2   \left| \big( \bm{\Delta}_i^{t, \m} \big)_m \right| \\
& \quad + \big\| \bm{\Delta}_s^{t, \m} \big\|_2 \big( \norm{\bm{u}_s^\star}_2 + \big\| \bm{u}_s^{t, \m} \big\|_2 \big) \sum_{i \in [r]}  \big\| \bm{\Delta}_i^{t, \m} \big\|_2 \big( \big\| \bm{u}_i^{t, \m} \big\|_2 + \norm{\bm{u}_i^\star}_2 + \big\| \bm{\Delta}_i^{t, \m} \big\|_2 \big)  \left| \left( \bm{u}^{\star}_i \right)_m \right| \\
& \lesssim \lammax^{\star 2/3}  \big\| \bm{\Delta}_s^{t, \m}  \big\|_2 \sum_{i \in [r]} \big\| \bm{u}_i^{t, \m} \big\|_2   \left| \big( \bm{\Delta}_i^{t, \m} \big)_m \right| + \lammax^{\star 2/3} \sqrt{\frac{\mu}{d}} \, \big\| \bm{\Delta}_s^{t, \m}  \big\|_2 \sum_{i \in [r]}  \big\| \bm{\Delta}_i^{t, \m} \big\|_2  \big\| \bm{u}_i^{\star} \big\|_2 \\
& \lesssim  \lammax^{\star 2/3} \big\| \bm{\Delta}_s^{t, \m}  \big\|_2 \big\| \bm{U}^{t, \m} \big\|_\frob  \big\| \big( \bm{U}^{t, \m} - \bm{U}^\star \big)_{m, :} \big\|_2 +  \lammax^{\star 2/3} \sqrt{\frac{\mu}{d}} \, \big\| \bm{\Delta}_s^{t, \m} \big\|_2 \big\| \bm{U}^{\star} \big\|_\frob  \big\| \bm{U}^{t, \m} - \bm{U}^\star \big\|_\frob \\
& \lesssim \lammax^{\star} \sqrt{r} \, \big\| \bm{\Delta}_s^{t, \m}  \big\|_2  \big\| \big( \bm{U}^{t, \m} - \bm{U}^\star \big)_{m, :} \big\|_2 +  \lammax^{\star} \sqrt{\frac{\mu r}{d}} \, \big\| \bm{\Delta}_s^{t, \m} \big\|_2  \big\| \bm{U}^{t, \m} - \bm{U}^\star \big\|_\frob
\end{align*}
We then sum over $s \in [r]$ to find that
\begin{align}
\norm{ \( \bm{\Delta}^{t, \m}_{\bm{T}}  \times_1^\seq \big( \bm{U}^{t, \m} - \bm{U}^\star \big) \times_2^\seq \bm{U}^\star \)_{m, :} }_2 & \lesssim \lammax^{\star} \sqrt{r} \, \big\| \bm{U}^{t, \m} - \bm{U}^\star \big\|_\frob \big\| \big( \bm{U}^{t, \m} - \bm{U}^\star \big)_{m, :} \big\|_2 \nonumber  \\
& \quad + \lammax^{\star} \sqrt{\frac{\mu r}{d}} \big\| \bm{U}^{t, \m} - \bm{U}^\star \big\|_\frob^2, \label{eq:Delta_T_Delta_U_U}
\end{align}
%which further implies that
%\begin{align}
%\norm{ \( \bm{\Delta}^{t, \m}_{\bm{T}}  \boxtimes \big( \bm{U}^{t, \m} - \bm{U}^\star \big) \boxtimes \bm{U}^\star \)_{m, :} }_2 & \lesssim \lammax^{\star} \sqrt{r} \, \big\| \bm{U}^{t, \m} - \bm{U}^\star \big\|_\frob \big\| \big( \bm{U}^{t, \m} - \bm{U}^\star \big)_{m, :} \big\|_2 \nonumber  \\
%& \quad + \lammax^{\star } \sqrt{\frac{\mu r}{d}} \big\| \bm{U}^{t, \m} - \bm{U}^\star \big\|_\frob^2. \label{eq:Delta_T_Delta_U_U}
%\end{align}
%
Moreover, it is easy to see that the upper bound also holds for $\gamma_2$. As for $\gamma_3$, we can express
\begin{align}
 \big( \bm{\Delta}^{t, \m}_{\bm{T}} \times_1 \bm{\Delta}^{t, \m}_s \times_2 \bm{\Delta}^{t, \m}_s \big)_m 
&= \sum_{i \in [r]}  \Big( 2 \, \big\langle \bm{u}^{\star}_i, \bm{\Delta}^{t, (m)}_s \big\rangle \big\langle \bm{\Delta}^{t, (m)}_i, \bm{\Delta}^{t, (m)}_s \big\rangle  + \big\langle \bm{\Delta}^{t, (m)}_i, \bm{\Delta}^{t, (m)}_s \big\rangle^2 \Big) \left( \bm{u}^{\star}_i \right)_m  \nonumber \\
& \quad + \sum_{i \in [r]} \big\langle \bm{u}^{t, (m)}_i, \bm{\Delta}^{t, (m)}_s \big\rangle^2 \big( \bm{\Delta}_i^{t, \m} \big)_m. \label{Delta_T_tsym_delta}
\end{align}
Similarly, we combine \eqref{eq:gd-u-UB} wtih the triangle inequality and the Cauchy-Schwarz inequality to bound
\begin{align*}
 & \Big| \Big(  \bm{\Delta}^{t, \m}_{\bm{T}} \times_1 \bm{\Delta}^{t, \m}_s \times_2 \bm{\Delta}^{t, \m}_s \Big)_m \Big|  \\
 & \quad\lesssim \big\| \bm{\Delta}^{t, (m)}_s \big\|_2^2 \sum_{i \in [r]} \big\| \bm{\Delta}^{t, (m)}_i \big\|_2 \left\| \bm{u}^{\star}_i \right\|_2  \left| \left( \bm{u}^{\star}_i \right)_m \right|
 + \big\| \bm{\Delta}^{t, (m)}_s \big\|_2^2 \sum_{i \in [r]}  \big\| \bm{u}^{t, (m)}_i \big\|_2^2 \left| \big( \bm{\Delta}_i^{t, \m} \big)_m \right| \\
&\quad \lesssim \lammax^{\star1/3} \big\| \bm{\Delta}^{t, (m)}_s \big\|_2^2   \sum_{i \in [r]}  \big\| \bm{\Delta}^{t, (m)}_i \big\|_2  \left| \left( \bm{u}^{\star}_i \right)_m \right|
 + \lammax^{\star1/3} \big\| \bm{\Delta}^{t, (m)}_s \big\|_2^2  \sum_{i \in [r]} \big\| \bm{u}^{t, (m)}_i \big\|_2 \left| \big( \bm{\Delta}_i^{t, \m} \big)_m \right| \\
&\quad \leq \lammax^{\star1/3} \big\| \bm{\Delta}^{t, (m)}_s \big\|_2^2  \left\| \bm{U}^\star \right\|_{2, \infty} \big\| \bm{U}^{t, \m} - \bm{U}^\star \big\|_\frob  
+ \lammax^{\star1/3} \big\| \bm{\Delta}^{t, (m)}_s \big\|_2^2 \big\| \bm{U}^{t, \m} \big\|_\frob \big\| \big( \bm{U}^{t, \m} - \bm{U}^\star \big)_{m, :} \big\|_2   \\
&\quad \leq \lammax^{\star1/3}  \left\| \bm{U}^\star \right\|_{2, \infty} \big\| \bm{U}^{t, \m} - \bm{U}^\star \big\|_\frob^2  \big\| \bm{\Delta}^{t, (m)}_s \big\|_2  \\
&\quad \quad + \lammax^{\star1/3}  \left\| \bm{U}^\star \right\|_\frob  \big\| \big( \bm{U}^{t, \m} - \bm{U}^\star \big)_{m, :} \big\|_2   \big\| \bm{U}^{t, \m} - \bm{U}^\star \big\|_\frob \big\| \bm{\Delta}^{t, (m)}_s \big\|_2  .
%& \leq \lammax^{\star 2}  \left\{ \sqrt{\frac{\mu r}{d}} \, \big\| \bm{U}^{t, \m} - \bm{U}^\star \big\|_\frob + \sqrt{r} \, \big\| \bm{U}^{t, \m} - \bm{U}^\star \big\|_{2, \infty} \right\} \big\| \bm{U}^{t, \m} - \bm{U}^\star \big\|_\frob \big\| \bm{\Delta}^{t, (m)}_s \big\|_2 \\
%& \ll \lammax^{\star 3} \sqrt{\mu r} \, \big\| \bm{U}^{t, \m} - \bm{U}^\star \big\|_{2, \infty} \big\| \bm{\Delta}^{t, (m)}_s \big\|_2
\end{align*}
Sum over $s \in [r]$ to obtain 
\begin{align*}
%\begin{split}
& \norm{ \( \bm{\Delta}^{t, \m}_{\bm{T}}  \times_1^\seq \big( \bm{U}^{t, \m} - \bm{U}^\star \big) \times_2^\seq \big( \bm{U}^{t, \m} - \bm{U}^\star \big) \)_{m, :} }_2 \\
%&= \eta \( \sum_{s \in [r]} \( \bm{\Delta}^{t, \m}_{\bm{T}} \tsymsqr  \bm{\Delta}^{t, \m}_s \)_m^2 \)^{1/2} \\
&  \qquad \leq   \lammax^{\star 1/3}  \sqrt{r} \, \big\| \bm{U}^{t, \m} - \bm{U}^\star \big\|_\frob^2  \big\| \big( \bm{U}^{t, \m} - \bm{U}^\star \big)_{m, :} \big\|_2 + \lammax^{\star 1/3} \sqrt{\frac{\mu r}{d}} \, \big\| \bm{U}^{t, \m} - \bm{U}^\star \big\|_\frob^3.
%\end{split}
\end{align*}
Since $\big\| \bm{U}^{t, \m} - \bm{U}^\star \big\|_\frob \ll \lambda_{\max}^{\star 1/3} / \sqrt{r} \leq 1$ by \eqref{eq:gd-u-UB}, combined with \eqref{eq:Delta_T_Delta_U_U}, we find that
\begin{align}
& \norm{\( \bm{\Delta}^{t, \m}_{\bm{T}} \times_1^\seq \bm{U}^{t, \m} \times_2^\seq \bm{U}^{t, \m} - \bm{\Delta}^{t, \m}_{\bm{T}} \times_1^\seq \bm{U}^\star \times_2^\seq \bm{U}^\star \)_{m, :}}_2  \nonumber \\
&\qquad \lesssim \lammax^{\star 1/3} \sqrt{r} \, \big\| \bm{U}^{t, \m} - \bm{U}^\star \big\|_\frob \big\| \big( \bm{U}^{t, \m} - \bm{U}^\star \big)_{m, :} \big\|_2 + \lammax^{\star 1/3} \sqrt{\frac{\mu r}{d}} \big\| \bm{U}^{t, \m} - \bm{U}^\star \big\|_\frob^2 \nonumber \\
& \qquad \ll  \lammin^{\star 4/3} \big\| \big( \bm{U}^{t, \m} - \bm{U}^\star \big)_{m, :} \big\|_2 + \lammin^{\star 4/3} \sqrt{\frac{\mu r}{d}} \, \big\| \bm{U}^{t, \m} - \bm{U}^\star \big\|_\frob. \label{eq:u_loo_utrue_row_term1}
\end{align}
where the last inequality follows from the assumption that $\kappa \asymp 1$.

\item
Putting \eqref{eq:u_loo_utrue_row_term2} and \eqref{eq:u_loo_utrue_row_term1} together, we reach the conclusion from \eqref{U_loo_frob_loss} and the condition $r \ll \sqrt{d / \mu}$ that,
\begin{align*}
 \Big\| \big(\bm{U}^{t+1, \m}  - \bm{U}^\star \big)_{m, :} \Big\|_2 & \leq  \big( 1-  \frac{1}{4} \lammin^{\star 4/3} \eta \big) \, \big\| \big(\bm{U}^{t+1}-\bm{U}^\star \big)_{m, :} \big\|_{2}
+ 5 \eta \lammin^{\star 4/3} \sqrt{\frac{\mu r}{d}} \, \big\| \bm{U}^{t, \m} - \bm{U}^\star \big\|_\frob \\
& \leq \big( 1- \frac{1}{4} \lammin^{\star 4/3} \eta \big) \( C_7 \rho^t \EE_\local + C_8 \frac{\sigma}{\lammin^{\star}} \sqrt{\frac{ d \log d}{p}}   \) \norm{\bm{U}^\star}_{2, \infty} \\
&\quad + 5 \eta \lammin^{\star 4/3} r \sqrt{\frac{\mu}{d}}  \( 2 C_1 \rho^t \EE_\local + 2 C_2 \frac{\sigma}{\lammin^{\star}} \sqrt{\frac{ d \log d}{p}}   \) \norm{\bm{U}^\star}_{2,\infty}  \\
%& \quad\leq \big( 1-  \lammin^{\star 4/3} \eta \big) \( C_7 \rho^t \EE_\local + C_8 \frac{\sigma}{\lammin^{\star}} \sqrt{\frac{ d \log d}{p}}   \) \norm{\bm{U}^\star}_{2, \infty} \\
%& \quad\quad\quad + 10 \eta \lammax^{\star 4/3} \sqrt{\frac{\mu r}{d}}  \( 2 C_1 \rho^t \EE_\local + 2 C_2 \frac{\sigma}{\lammin^{\star}} \sqrt{\frac{ r d \log d}{p}}   \) \norm{\bm{U}^\star}_{2, \infty}  \\
& \leq \( \big( 1- \frac{1}{4} \lammin^{\star 4/3} \eta \big) C_7 + o \( 1 \)  C_1  \lammin^{\star 4/3} \eta  \) \rho^t \EE_\local  \norm{\bm{U}^\star}_{2, \infty} \\
& \quad + \( \big( 1- \lammin^{\star 4/3} \eta \big) C_8 + o \( 1 \)  C_2  \lammin^{\star 4/3} \eta \) \frac{\sigma}{\lammin^{\star}} \sqrt{\frac{ d \log d}{p}}  \norm{\bm{U}^\star}_{2, \infty} \\
& \leq \( C_7 \rho^{t+1} \EE_\local + C_8 \frac{\sigma}{\lammin^{\star}} \sqrt{\frac{ d \log d}{p}} \) \norm{\bm{U}^\star}_{2, \infty} ,
\end{align*}
provided that $0 < \eta \leq  \lammin^{\star 4/3} / \big(32 \lammax^{\star 8/3} \big)$, $1 - \lambda_{\min}^{\star 4/3} \eta /5  \leq \rho < 1$, $C_7, C_8$ are sufficiently large.
\end{itemize}

\noindent Recognizing that the above bound holds for any $1\leq m \leq d$,  we conclude the proof.

%\end{proof}

\subsection{Proof of Lemma~\ref{lemma:U_2inf_loss}}
\label{pf:U_2inf_loss}

It is easy to see that
\begin{align}
\begin{split}
\big\| \big(\bm{U}^{t+1}-\bm{U}^\star \big)_{m, :} \big\|_2 
	%&\leq \big\| \big(\bm{U}^{t+1}-\bm{U}^{t+1, \m} \big)_{m, :} \big\|_2 + \big\| \big(\bm{U}^{t+1, \m}-\bm{U}^\star \big)_{m, :} \big\|_2 \\
& \leq \big\| \bm{U}^{t+1}-\bm{U}^{t+1, \m} \big\|_\frob + \big\| \big(\bm{U}^{t+1, \m}-\bm{U}^\star \big)_{m, :} \big\|_2.
\end{split}
\end{align}
Combining Lemma~\ref{lemma:U_loo_frob_diff} and Lemma~\ref{lemma:U_loo_row_loss}, we conclude that with probability at least $1-O(d^{-10})$,
\begin{align*}
%\begin{split}
\big\| \big(\bm{U}^{t+1}-\bm{U}^\star \big)_{m, :} \big\|_2
& \leq  \( C_5 \rho^t \EE_\local + C_6 \frac{\sigma}{\lammin^{\star}} \sqrt{\frac{ d \log d}{p}}   \) \norm{\bm{U}^\star}_{2, \infty} \\
& \quad\quad +  \( C_7 \rho^t \EE_\local + C_8 \frac{\sigma}{\lammin^{\star}} \sqrt{\frac{ d \log d}{p}}   \) \norm{\bm{U}^\star}_{2, \infty} \\
& \leq \( C_3 \rho^t \EE_\local + C_4 \frac{\sigma}{\lammin^{\star}} \sqrt{\frac{ d \log d}{p}}   \) \norm{\bm{U}^\star}_{2, \infty} ,
%\end{split}
\end{align*}
with the proviso that $C_3 / \( C_5 + C_7 \)$ and $C_4 / \( C_6 + C_8 \)$ are both sufficiently large.

\section{Proofs for retrieving tensor components}

\subsection{Proof of Lemma~\ref{lemma:init_spectra}}

\label{pf:init_spectra}

We shall often operate upon the event where the claims in Lemma~\ref{lemma:eigsp_dist_op}
hold, which happens with very high probability (i.e.~at least $1-O\left(d^{-10}\right)$).
Recall the definition of $\bm{\gamma}^{\star\tau}$ in (\ref{def:init_sigval}).
Since $\bm{\theta}^{\tau}=\bm{U}\bm{U}^{\top}\bm{g}^{\tau}$, this
allows us to write that: for each $1\leq i\leq r$,
\begin{align*}
\gamma_{i}^{\star\tau} & =\lambda_{i}^{\star}\left\langle \bm{U}\bm{U}^{\top}\overline{\bm{u}}_{i}^{\star},\bm{g}^{\tau}\right\rangle ,
\end{align*}
where we recall that $\lambda_{i}^{\star}=\|\bm{u}_{i}^{\star}\|_{2}^3$
and $\overline{\bm{u}}_{i}^{\star}=\bm{u}_{i}^{\star}/\|\bm{u}_{i}^{\star}\|_{2}$.
Given that $\bm{g}^{\tau}$ is a Gaussian vector independent of $\bm{U}$,
we observe that $\bm{\gamma}^{\star\tau}$ is zero-mean Gaussian conditional
on $\Omega$ and $\bm{E}$. In order to understand the order statistics
associated with this vector, we first look at its covariance matrix.

Denote by $\bm{\Sigma}^{\tau}$ the covariance matrix of $\bm{\gamma}^{\star\tau}$
(conditional on $\bm{U}$). Then we have \begin{subequations} \label{eq:defn-Sigma-ii-ij}
\begin{align}
\Sigma_{i,i}^{\tau} & =\lambda_{i}^{\star2}\left\Vert \bm{U}\bm{U}^{\top}\overline{\bm{u}}_{i}^{\star}\right\Vert _{2}^{2}=\lambda_{i}^{\star2}\left\Vert \mathcal{P}_{\bm{U}}\big(\overline{\bm{u}}_{i}^{\star}\big)\right\Vert _{2}^{2},\label{eq:defn-Sigma-ii-ij-1}\\
\Sigma_{i,j}^{\tau} & =\lambda_{i}^{\star}\lambda_{j}^{\star}\left\langle \bm{U}\bm{U}^{\top}\overline{\bm{u}}_{i}^{\star},\bm{U}\bm{U}^{\top}\overline{\bm{u}}_{j}^{\star}\right\rangle =\lambda_{i}^{\star}\lambda_{j}^{\star}\left\langle \overline{\bm{u}}_{i}^{\star},\mathcal{P}_{\bm{U}}\big(\overline{\bm{u}}_{j}^{\star}\big)\right\rangle \nonumber \\
 & =\lambda_{i}^{\star}\lambda_{j}^{\star}\Big(\big\langle \overline{\bm{u}}_{i}^{\star},\overline{\bm{u}}_{j}^{\star} \big\rangle - \big\langle \overline{\bm{u}}_{i}^{\star},\mathcal{P}_{\bm{U}^{\perp}}\big(\overline{\bm{u}}_{j}^{\star}\big) \big\rangle \Big),\label{eq:defn-Sigma-ii-ij-2}
\end{align}
\end{subequations} where we denote by $\mathcal{P}_{\bm{U}}(\bm{z})=\bm{U}\bm{U}^{\top}\bm{z}$
and $\mathcal{P}_{\bm{U}^{\perp}}(\bm{z})=\left(\bm{I}-\bm{U}\bm{U}^{\top}\right)\bm{z}$.
In addition, since the unit vector $\overline{\bm{u}}_{i}^{\star}$
lies in the span of the columns of $\bm{U}_{\mathsf{orth}}^{\star}$
(cf.~(\ref{defn:U-orth})), it follows from Lemma~\ref{lemma:subspace_basis_dist}
that \begin{subequations} 
\begin{align}
\left\Vert \mathcal{P}_{\bm{U}}\big(\overline{\bm{u}}_{i}^{\star}\big)\right\Vert _{2} & =\big\|\bm{U}\bm{U}^{\top}\overline{\bm{u}}_{i}^{\star}\big\|_{2}=\big\|\bm{U}\bm{R}\left(\bm{U}\bm{R}\right)^{\top}\overline{\bm{u}}_{i}^{\star}\big\|_{2}\geq\sqrt{1-\left\Vert \bm{U}\bm{R}-\bm{U}_{\mathsf{orth}}^{\star}\right\Vert ^{2}},\label{ui_proj_2norm}\\
\left\Vert \mathcal{P}_{\bm{U}^{\perp}}\big(\overline{\bm{u}}_{i}^{\star}\big)\right\Vert _{2} & =\big\|\big(\bm{I}-\bm{U}\bm{U}^{\top}\big)\overline{\bm{u}}_{i}^{\star}\big\|_{2}=\big\|\big(\bm{I}-\bm{U}\bm{R}\left(\bm{U}\bm{R}\right)^{\top}\big)\overline{\bm{u}}_{i}^{\star}\big\|_{2}\leq\left\Vert \bm{U}\bm{R}-\bm{U}_{\mathsf{orth}}^{\star}\right\Vert ,\label{ui_proj_perpnorm}
\end{align}
\end{subequations} where $\bm{R}$ is a rotation matrix defined in
\eqref{defn:R-rotation}. This together with \eqref{eq:defn-Sigma-ii-ij-1}
gives 
\begin{align}
\lambda_{i}^{\star2}\big(1-\left\Vert \bm{U}\bm{R}-\bm{U}_{\mathsf{orth}}^{\star}\right\Vert ^{2}\big)\leq\Sigma_{i,i}^{\tau}\leq\lambda_{i}^{\star2},\label{eq:cov_ii}
\end{align}
where we have also used the fact that $\left\Vert \mathcal{P}_{\bm{U}}\big(\overline{\bm{u}}_{i}^{\star}\big)\right\Vert _{2}\leq\left\Vert \overline{\bm{u}}_{i}^{\star}\right\Vert _{2}=1$.
Moreover, taking together \eqref{eq:defn-Sigma-ii-ij-1}, \eqref{ui_proj_perpnorm}
and the incoherence condition, we see that for any $1 \leq i\neq j \leq r$, 
\begin{align}
\left|\Sigma_{i,j}^{\tau}\right| & \leq\lambda_{i}^{\star}\lambda_{j}^{\star} \left\{ \left|\left\langle \overline{\bm{u}}_{i}^{\star},\overline{\bm{u}}_{j}^{\star}\right\rangle \right|+\left\Vert \overline{\bm{u}}_{i}^{\star}\right\Vert _{2}\left\Vert \mathcal{P}_{\bm{U}^{\perp}}\big(\overline{\bm{u}}_{j}^{\star}\big)\right\Vert _{2} \right\}
\leq\lambda_{i}^{\star}\lambda_{j}^{\star}\underbrace{\big(\sqrt{\mu / d}+\left\Vert \bm{U}\bm{R}-\bm{U}_{\mathsf{orth}}^{\star}\right\Vert \big)}_{=:\delta_{1}},\label{cov_ij_ub}
\end{align}
which is expected to be small if $\delta_{1}$ is small.

From our assumptions on the sample size, the rank and the condition
number, we can invoke Lemma~\ref{lemma:eigsp_dist_op} to see that
$\kappa r\delta_{1}\ll1$, where $\delta_{1}$ is defined in \eqref{cov_ij_ub}
and $\kappa=\lambda_{\max}^{\star}/\lambda_{\min}^{\star}$. Thus,
we can decompose $\bm{\Sigma}^{\tau}$ into two components as follows
\begin{align*}
\bm{\Sigma}^{\tau}=\underbrace{\left(1-\kappa r\delta_{1}\right)\bm{D}^{\star2}}_{=:\,\widehat{\bm{\Sigma}}^{\tau}}+\underbrace{\bm{\Sigma}^{\tau}-\left(1-\kappa r\delta_{1}\right)\bm{D}^{\star2}}_{=:\,\breve{\bm{\Sigma}}^{\tau}},
\end{align*}
where 
\begin{align*}
\bm{D}^{\star} := \mathsf{diag}\left( \lambda_1^\star,\cdots, \lambda_r^\star \right)\in\mathbb{R}^{r\times r}.
\end{align*}
As it turns out, both $\widehat{\bm{\Sigma}}^{\tau}$ and $\breve{\bm{\Sigma}}^{\tau}$
are positive definite. Indeed, we first learn from \eqref{eq:cov_ii}
and \eqref{cov_ij_ub} that: the $i$-th digaonal entry of $\breve{\bm{\Sigma}}^{\tau}$
obeys
\begin{align*}
\breve{\Sigma}_{i,i}^{\tau} & \geq\lambda_{i}^{\star2}\big(1-\left\Vert \bm{U}\bm{R}-\bm{U}_{\mathsf{orth}}^{\star}\right\Vert ^{2}\big)-\left(1-\kappa r\delta_{1}\right)\lambda_{i}^{\star2}\\
 & =\lambda_{i}^{\star2}\big(\kappa r\delta_{1}-\left\Vert \bm{U}\bm{R}-\bm{U}_{\mathsf{orth}}^{\star}\right\Vert ^{2}\big)\geq\lambda_{i}^{\star2}\left(\kappa r\delta_{1}-\delta_{1}^{2}\right)\\
 & \overset{(\text{i})}{>}\kappa\lambda_{i}^{\star2}\left(r-1\right)\delta_{1}\overset{\text{(ii)}}{\geq}\lambda_{i}^{\star}\lambda_{\max}^{\star}\left(r-1\right)\delta_{1}\\
 & \geq\sum_{j:j\neq i}\big|\breve{\Sigma}_{i,j}^{\tau}\big|,
\end{align*}
where (i) holds since $\delta_{1}<1$ under our assumptions, (ii)
follows since $\kappa\lambda_{i}^{\star}\geq\kappa\lambda_{\min}^{\star}=\lambda_{\max}^{\star}$,
and the last line makes use of (\ref{cov_ij_ub}). This implies that
$\breve{\bm{\Sigma}}^{\tau}$ is diagonally dominant, and hence $\breve{\bm{\Sigma}}^{\tau}\succeq\bm{0}$.
In conclusion, both $\widehat{\bm{\Sigma}}^{\tau}$ and $\breve{\bm{\Sigma}}^{\tau}$
are positive definite.

Let $\widehat{\bm{\gamma}}^{\star\tau}$ and $\breve{\bm{\gamma}}^{\star\tau}$
be independent zero-mean Gaussian random vectors with covariance matrices
$\widehat{\bm{\Sigma}}^{\tau}$ and $\breve{\bm{\Sigma}}^{\tau}$,
respectively. Clearly, the distribution of $\bm{\gamma}^{\star\tau}$
is identical to that of $\widehat{\bm{\gamma}}^{\star\tau}+\breve{\bm{\gamma}}^{\star\tau}$.
Consequently, it allows us to look at the distributions of these two
random vectors separately.
\begin{itemize}
\item In view of (\ref{eq:cov_ii}) and the fact $\kappa r\delta_{1}<1$, one
has 
\begin{align}
\breve{\Sigma}_{i,i}^{\tau}\leq\lambda_{i}^{\star2}-\left(1-\kappa r\delta_{1}\right)\lambda_{i}^{\star2}=\kappa r\delta_{1}\lambda_{i}^{\star2}.
\end{align}
Thus, with probability at least
$1-O\left( d^{-10} \right)$, we have 
\begin{align}
\left\Vert \breve{\bm{\gamma}}^{\star\tau}\right\Vert _{\infty} & \lesssim\mathbb{E}\Big[\max_{1\leq i\leq r}\left|\breve{\gamma}_{i}^{\star\tau}\right|\Big]+\max_{1\leq i\leq r}\sqrt{\mathsf{Var}\big(\breve{\gamma}_{i}^{\star\tau}\big) \log d} \lesssim\lambda_{\max}^{\star}\sqrt{\kappa r\delta_{1} \log d}  \ll\lambda_{\min}^{\star} /  \big( r \sqrt{\log d} \big) ,\label{eq:breve-gamma-inf-bound-UB}
\end{align}
where the last step arises from $\kappa^3 r^3 \delta_{1}\log^{2}d\ll1$ under our sample size, noise and rank conditions. By the condition that $L\asymp r^{2\kappa^2} \log^{3/2} r$, $r \ll d$ and $\kappa \asymp 1$, we take a union bound over $\tau \in \[ L\]$ to find that each entry of $\breve{\bm{\gamma}}^{\star\tau}$ is fairly small for all $\tau \in \[ L \]$. Another immediate consequence
of (\ref{eq:breve-gamma-inf-bound-UB}) is that: for any fixed vector $\bm{v}\in\mathbb{R}^{r}$, with probability
at least $1-O(d^{-10})$, for all $\tau \in \[ L \]$,
\begin{subequations}
\begin{align}
\left\Vert \breve{\bm{\gamma}}^{\star\tau}\right\Vert _{2} & \leq\sqrt{r}\left\Vert \breve{\bm{\gamma}}^{\star\tau}\right\Vert _{\infty}\ll \lambda_{\min}^{\star},\label{eq:l2-breve-gamma-1}\\
\left|\left\langle \bm{v},\breve{\bm{\gamma}}^{\star\tau}\right\rangle \right| & \leq\left\Vert \bm{v}\right\Vert _{2}\left\Vert \breve{\bm{\gamma}}^{\star\tau}\right\Vert _{2}\ll\left\Vert \bm{v}\right\Vert _{2} \lambda_{\min}^{\star}.\label{eq:w_inprod-breve-gamma-1}
\end{align}
\end{subequations}
\item We then turn attention to $\widehat{\bm{\gamma}}^{\star\tau}$, which
is composed of independent Gaussian random variables. Let us define $\widehat{\Delta}^\tau := \widehat{\gamma}_{1}^{\star\tau}-\max_{1<i\leq r}\big|\widehat{\gamma}_{i}^{\star\tau}\big|$ for each $1 \leq \tau \leq L$ and let $\widehat{\Delta}_1^{\( 1 \)} \geq \widehat{\Delta}_1^{\( 2 \)} \geq \dots \geq \widehat{\Delta}_1^{\( L \)}$ denote the order statistics of $\big\{ \widehat{\Delta}_1^\tau \big\}_{\tau = 1}^L$ in descending order. Fix any small
constant $\delta>0$. Invoke Lemma~\ref{lemma:indep_gaussian_gap}
to demonstrate that: with probability greater than $1-\delta/r$,
\begin{subequations}
\label{eq:3claim-hat-gamma-1} 
\begin{align}
%\widehat{\gamma}_{1}^{\star\tau}-\max_{1<i\leq r}\big|\widehat{\gamma}_{i}^{\star\tau}\big| 
\widehat{\Delta}_1^{\( 1 \)} & \gtrsim\lambda_{\min}^{\star}, \\
\widehat{\Delta}_1^{\( 1 \)} - \widehat{\Delta}_1^{\( 2 \)} & \gtrsim \frac{\lambda_{\min}^\star}{r \sqrt{ \log d}},
\end{align}
\end{subequations}
where we use the conditions that $L\asymp r^{2\kappa^2} \log^{3/2} r$, $r \ll d$ and $\kappa \asymp 1$. 
In addition, let $\widehat{\bm{\gamma}}_{\smallsetminus 1}^{\star\tau}:=\big[\widehat{\gamma}_{2}^{\star\tau},\cdots,\widehat{\gamma}_{r}^{\star\tau}\big]^\top \in\mathbb{R}^{r-1}$.
We know from standard Gaussian concentration inequalities and union bounds that for any fixed vector $\bm{v}\in\mathbb{R}^{r}$, with
probability $1-O(d^{-20})$, for all $\tau \in \[L \]$, 
\begin{subequations}
\begin{align}
\widehat{\gamma}_{1}^{\star\tau} & \lesssim \big( \sqrt{\log L} + \sqrt{\log d} \big) \,\lambda_{\max}^{\star}\asymp \sqrt{\log d}\,\lambda_{\max}^{\star}, \label{eq:gamma-hat-UB} \\
\big\|\widehat{\bm{\gamma}}^{\star\tau}\big\|_{2} & \leq\widehat{\gamma}_{1}^{\star\tau}+\big\|\widehat{\bm{\gamma}}_{\smallsetminus 1}^{\star\tau}\big\|_{2}\lesssim\big(  \sqrt{\log d}+\sqrt{r\log d}\big)\,\lambda_{\max}^{\star}\lesssim\sqrt{r\log d}\,\lambda_{\max}^{\star},\label{theta_2_norm-11}\\
\big|\big\langle\bm{v},\widehat{\bm{\gamma}}^{\star\tau}\big\rangle\big| & \leq\widehat{\gamma}_{1}^{\star\tau}\left\Vert \bm{v}\right\Vert _{2}+\big|\big\langle\bm{v}_{\smallsetminus 1},\widehat{\bm{\gamma}}_{\smallsetminus 1}^{\star\tau}\big\rangle\big|\lesssim\left\Vert \bm{v}\right\Vert _{2}\sqrt{\log d} \, \lambda_{\max}^{\star},\label{eq:w_inprod-hat-gamma-1}
\end{align}
\end{subequations}
where $\bm{v}_{\smallsetminus 1} := \left[v_{2},\cdots,\cdots,v_{r}\right]\in\mathbb{R}^{r-1}$.
\item Putting (\ref{eq:breve-gamma-inf-bound-UB}) and (\ref{eq:3claim-hat-gamma-1})
together and invoking the triangle inequality immediately establish (\ref{eq:sigma1-sigma2-gap-LB}) and \eqref{eq:gap1-LB}.
On the other hand, combining (\ref{eq:breve-gamma-inf-bound-UB}) with (\ref{eq:gamma-hat-UB}) proves \eqref{eq:sigma1-UB}; (\ref{eq:l2-breve-gamma-1}) and (\ref{theta_2_norm-11})
taken collectively establish (\ref{eq:gamma-L2-norm-bound-original}),
whereas (\ref{eq:w_inprod-breve-gamma-1}) and (\ref{eq:w_inprod-hat-gamma-1})
prove (\ref{eq:gamma-inprod-bound}).
\end{itemize}

\subsection{Proof of Lemma~\ref{lemma:init_2_loss}}

\label{pf:init_2_loss}

%Without loss of generality, we assume $\big|\gamma_{i}^{\star\tau}\big|=\max_{1\leq s\leq r}\left|\gamma_{s}^{\star\tau}\right|$, in which case one has $\pi\left(1\right)=1$ (cf.~\eqref{eq:defn-pi-1}).
Recall that the vector of interest $\overline{\bm{u}}^{\tau}$ is the
leading singular vector of $\bm{M}^{\tau}$ (as constructed in \eqref{eq:projected-matrix-M-tau}),
where $\bm{M}^{\tau}$ satisfies
\begin{align}
\bm{M}^{\tau} & =p^{-1}\bm{T}\times_{3}\bm{\theta}^{\tau}=\bm{T}^{\star}\times_{3}\bm{\theta}^{\tau}+\left(p^{-1}\bm{T}-\bm{T}^{\star}\right)\times_{3}\bm{\theta}^{\tau}\nonumber \\
 & =\gamma_{1}^{\star\tau}\overline{\bm{u}}_{1}^{\star}\overline{\bm{u}}_{1}^{\star\top}+\sum_{s:s\neq 1}\gamma_{s}^{\star\tau}\overline{\bm{u}}_{s}^{\star}\overline{\bm{u}}_{s}^{\star\top}+\left(p^{-1}\bm{T}-\bm{T}^{\star}\right)\times_{3}\bm{\theta}^{\tau}\nonumber \\
 & =\underbrace{\gamma_{1}^{\star\tau}\overline{\bm{u}}_{1}^{\star}\overline{\bm{u}}_{1}^{\star\top}+\sum_{s:s\neq 1}\gamma_{s}^{\star\tau}\left(\bm{I}-\overline{\bm{u}}_{1}^{\star}\overline{\bm{u}}_{1}^{\star\top}\right)\overline{\bm{u}}_{s}^{\star}\overline{\bm{u}}_{s}^{\star\top}\left(\bm{I}-\overline{\bm{u}}_{1}^{\star}\overline{\bm{u}}_{1}^{\star\top}\right)}_{=:\,\bm{M}^{\star\tau}}\nonumber \\
 & \qquad+\underbrace{\sum_{s:s\neq 1}\gamma_{s}^{\star\tau}\left\langle \overline{\bm{u}}_{s}^{\star},\overline{\bm{u}}_{1}^{\star}\right\rangle \left(\overline{\bm{u}}_{1}^{\star}\overline{\bm{u}}_{s}^{\star\top}+\overline{\bm{u}}_{s}^{\star}\overline{\bm{u}}_{1}^{\star\top}\right)-\sum_{s:s\neq 1}\gamma_{s}^{\star\tau}\left\langle \overline{\bm{u}}_{s}^{\star},\overline{\bm{u}}_{1}^{\star}\right\rangle ^{2}\overline{\bm{u}}_{s}^{\star}\overline{\bm{u}}_{s}^{\star\top}}_{=:\,\bm{C}^{\tau}} +\underbrace{\left(p^{-1}\bm{T}-\bm{T}^{\star}\right)\times_{3}\bm{\theta}^{\tau}}_{=:\,\bm{F}^{\tau}},\label{eq:M}
\end{align}
and $ \gamma_{i}^{\star\tau}~(1\leq i \leq r)$ is defined in (\ref{def:init_sigval}).

In what follows, we shall view $\bm{C}^{\tau}$ and $\bm{F}^{\tau}$ as perturbation terms
superimposed on $\bm{M}^{\star\tau}$. Lemma~\ref{lemma:init_pertur_op} below proves that their operator
norms are all small under our sample size, noise and rank conditions, which enables to apply
Wedin's theorem to justify the $\ell_{2}$ proximity between $\overline{\bm{u}}^{\tau}$
and $\overline{\bm{u}}_{1}^{\star}$. 

\begin{lemma}\label{lemma:init_pertur_op} 
%Assume the notation and setting of Lemma~\ref{lemma:init_2_loss}. 
Instate the assumptions of Lemma~\ref{lemma:init_2_loss}. With probability at least $1-O\left(d^{-10}\right)$, one has
\begin{align}
\big\Vert \bm{F}^{\tau}\big\Vert  & \lesssim\frac{\sqrt{\mu r}\,\lambda_{\max}^{\star}\log^{3}d}{d^{3/2}p}+\frac{\mu r\lambda_{\max}^{\star}\log^{5/2}d}{d\sqrt{p}}+\frac{\sigma\log^{7/2}d}{p}+\sigma\sqrt{\frac{rd\log^{5}d}{p}}, \label{eq:F_op_norm_UB}\\
\big\Vert \bm{C}^{\tau}\big\Vert  & \lesssim\sqrt{\frac{\mu  r\log d}{d}}\,\lambda_{\max}^{\star}, \label{eq:C_op_norm_UB}\\
 \big\Vert \bm{F}^{\tau}\overline{\bm{u}}_{1}^{\star}\big\Vert _{2}  & \lesssim\frac{\mu r\lambda_{\max}^{\star}\log d}{d\sqrt{p}}+\sigma\sqrt{\frac{rd\log^{2}d}{p}}, \label{eq:F_u_2_norm_UB}\\
 \big\Vert \bm{C}^{\tau}\overline{\bm{u}}_{1}^{\star}\big\Vert _{2}  & \lesssim\sqrt{\frac{\mu  r\log d}{d}}\,\lambda_{\max}^{\star}. \label{eq:C_u_2_norm_UB}
\end{align}
\end{lemma}
\begin{proof}See Appendix~\ref{subsec:pf:init_pertur_op_UB}.\end{proof}
As a consequence, recalling the definition of $\EE_\op$ and $\EE_\proj$ in (\ref{def:err-op}) and (\ref{claim:init_2_loss}) respectively, one has
\begin{align}
\big\Vert \bm{M}^{\tau}-\bm{M}^{\star\tau}\big\Vert & \leq \big\Vert \bm{F}^{\tau}\big\Vert + \big\Vert \bm{C}^{\tau}\big\Vert \nonumber \\
& \lesssim \underbrace{\frac{\sqrt{\mu r}\,\lambda_{\max}^{\star}\log^{3}d}{d^{3/2}p}+\frac{\mu r\lambda_{\max}^{\star}\log^{5/2}d}{d\sqrt{p}}+\frac{\sigma\log^{7/2}d}{p}+\sigma\sqrt{\frac{rd\log^{5}d}{p}}+\sqrt{\frac{\mu  r\log d}{d}}\,\lambda_{\max}^{\star} }_{= \, \EE_\op \cdot \lambda_{\min}^\star},\label{eq:M_hat_M_op_loss}
\end{align}
and
\begin{align}
 \big\Vert \big(\bm{M}^{\tau}-\bm{M}^{\star\tau}\big)\overline{\bm{u}}_{1}^{\star}\big\Vert _{2}
 \leq  \big\Vert \bm{F}^{\tau}\overline{\bm{u}}_{1}^{\star}\big\Vert _{2} +  \big\Vert \bm{C}^{\tau}\overline{\bm{u}}_{1}^{\star}\big\Vert _{2}  & \lesssim \underbrace{\frac{\mu r\lambda_{\max}^{\star}\log d}{d\sqrt{p}}+\sigma\sqrt{\frac{rd\log^{2}d}{p}}+\sqrt{\frac{\mu r  \log d}{d}}\,\lambda_{\max}^{\star} }_{= \, \EE_\proj \cdot \lambda_{\min}^\star} .
\label{eq:M_hat_M_u_2_norm}
\end{align}
It then follows from Weyl's inequality that 
\begin{align}
\max_{i\in\left[d\right]} & \left|\sigma_{i}\big(\bm{M}^{\tau}\big)-\sigma_{i}\big(\bm{M}^{\star\tau}\big)\right| \leq\big\|\bm{M}^{\tau}-\bm{M}^{\star\tau}\big\| \lesssim \EE_\op \cdot \lambda_{\min}^\star \ll \lambda_{\min}^\star \label{eq:sigval_M_hat}
% &\quad \lesssim\frac{\sqrt{\mu r}\,\lambda_{\max}^{\star}\log^{3}d}{d^{3/2}p}+\frac{\mu r\lambda_{\max}^{\star}\log^{5/2}d}{d\sqrt{p}}+\frac{\sigma\log^{7/2}d}{p}+\sigma\sqrt{\frac{rd\log^{5}d}{p}}+\sqrt{\frac{\mu r\log d}{d}}\,\lambda_{\max}^{\star}
\end{align}
where $\sigma_{i}(\bm{Z})$ denotes the $i$-th largest singular value
of a matrix $\bm{Z}$ and we use the condition that $\EE_\op \ll 1$. All in all, these arguments justify that the
spectrum of $\bm{M}^{\tau}$ is fairly close to that of $\bm{M}^{\star\tau}$.

Next, we look at the gap between the two leading singular
values of $\bm{M}^{\star\tau}$. To begin with, it is self-evident from
the definition of $\bm{M}^{\star\tau}$ that: $\overline{\bm{u}}_{1}^{\star}$
is the singular vector of $\bm{M}^{\star\tau}$. In fact, we claim
one further result, that is, $\overline{\bm{u}}_{1}^{\star}$ is indeed
the leading singular vector of $\bm{M}^{\star\tau}$ whose singular value
is given by $\sigma_{1}\big(\bm{M}^{\star\tau}\big)=\gamma_{1}^{\star\tau}$.
Towards this end, let us define 
\begin{align*}
\bm{U}_{\smallsetminus 1}^{\tau}  := \left(\bm{I}-\overline{\bm{u}}_{1}^{\star}\overline{\bm{u}}_{1}^{\star\top}\right)\overline{\bm{U}}^{\star} \in\mathbb{R}^{d\times\left(r-1\right)} \qquad \text{and} \qquad \bm{\gamma}_{\smallsetminus 1}^{\star\tau}  := \big[\gamma_{2}^{\star\tau}, \cdots,\gamma_{r}^{\star\tau}\big]^\top \in\mathbb{R}^{r-1},
\end{align*}
allowing us to write 
\begin{align*}
\sum_{s:s\neq 1}\gamma_{s}^{\star\tau}\left(\bm{I}-\overline{\bm{u}}_{1}^{\star}\overline{\bm{u}}_{1}^{\star\top}\right)\overline{\bm{u}}_{s}^{\star}\overline{\bm{u}}_{s}^{\star\top}\left(\bm{I}-\overline{\bm{u}}_{1}^{\star}\overline{\bm{u}}_{1}^{\star\top}\right)=\bm{U}_{\smallsetminus 1}^{\tau}\mathsf{diag}\big(\bm{\gamma}_{\smallsetminus 1}^{\star\tau}\big)\bm{U}_{\smallsetminus 1}^{\tau\top} =: \,\bm{M}_{\smallsetminus 1}^{\star\tau}.
\end{align*}
We note that from Lemma~\ref{lemma:incoh}, one has
\begin{align*}
\big\|\bm{U}_{\smallsetminus 1}^{\tau}\big\|=\big\|\left(\bm{I}-\overline{\bm{u}}_{1}^{\star}\overline{\bm{u}}_{1}^{\star\top}\right)\overline{\bm{U}}^{\star}\big\|\leq\big\|\overline{\bm{U}}^{\star}\big\|\leq\sqrt{1+r\sqrt{\mu/d}}.
\end{align*}
Let $| \gamma^{\star \tau} |_{(1)} \geq \dots \geq | \gamma^{\star \tau} |_{(r)} $ denote the absolute values of $\{ \gamma_i^{\star \tau} \}_{i=1}^r$ in descending order. This together with Lemma~\ref{lemma:init_spectra} implies that 
\begin{align*}
\big\Vert \bm{M}_{\smallsetminus 1}^{\star\tau}\big\Vert  & \leq\big|\gamma^{\star\tau}\big|_{(2)}\big\|\bm{U}_{\smallsetminus 1}^{\tau}\big\|^{2}\leq \big|\gamma^{\star\tau}\big|_{(2)} \big(1+r\sqrt{\mu/d}\big)\\
 & \leq \gamma_{1}^{\star\tau}-\big(\gamma_{1}^{\star\tau}-\big|\gamma^{\star\tau}\big|_{(2)}\big)+r\sqrt{\mu/d} \, \gamma_{1}^{\star\tau} < \gamma_{1}^{\star\tau},
\end{align*}
as long as $\kappa r \sqrt{(\mu \log d)/d}\ll 1$. Given that $\overline{\bm{u}}_{1}^{\star}$ is the 
singular vector of $\bm{M}^{\star\tau}$ with singular value $\gamma_{1}^{\star\tau}$,
we can conclude that $\sigma_{1}\big(\bm{M}^{\star\tau}\big)=\gamma_{1}^{\star\tau}$.
% which is the singular value associated with the singular vector $\overline{\bm{u}}_{1}^{\star}$.
This also allows us to lower bound the gap between the two largest
singular values $\bm{M}^{\star\tau}$ as follows 
\begin{align}
\sigma_{1}\big(\bm{M}^{\star\tau}\big)-\sigma_{2}\big(\bm{M}^{\star\tau}\big) & \geq \gamma_{1}^{\star\tau}-\big\|\bm{M}_{\smallsetminus 1}^{\star\tau}\big\|\geq\gamma_{1}^{\star\tau}-\big|\gamma^{\star\tau}\big|_{(2)}\big(1+r\sqrt{\mu/d}\big)\nonumber \\
 & \gtrsim \gamma_{1}^{\star\tau} -\big|\gamma^{\star\tau}\big|_{(2)}\gtrsim\lambda_{\min}^{\star},\label{eq:sigval_M_star_gap}
\end{align}
provided that $\kappa r \sqrt{(\mu \log d)/d}\ll 1$. We also know from (\ref{eq:sigval_M_hat}) and (\ref{eq:sigval_M_star_gap}) that
\begin{align*}
\sigma_{1}\big(\bm{M}^{\tau}\big)-\sigma_{2}\big(\bm{M}^{\tau}\big) \geq \sigma_{1}\big(\bm{M}^{\star\tau}\big)-\sigma_{2}\big(\bm{M}^{\star\tau}\big) - 2 \, \big\Vert \bm{M}^{\tau}-\bm{M}^{\star\tau}\big\Vert \gtrsim \lambda_{\min}^{\star}.
\end{align*}
Combined with \eqref{eq:M_hat_M_u_2_norm} and Wedin's theorem, we conclude that
\begin{align}
 \big\Vert \overline{\bm{u}}^{\tau}-\overline{\bm{u}}_{1}^{\star}\big\Vert _{2}  & \leq\frac{ \big\Vert \big(\bm{M}^{\tau}-\bm{M}^{\star\tau}\big)\overline{\bm{u}}_{1}^{\star}\big\Vert _{2} }{\sigma_{1}\big(\bm{M}^{\tau}\big)-\sigma_{2}\big(\bm{M}^{\tau}\big)-\big\Vert \bm{M}^{\tau}-\bm{M}^{\star\tau}\big\Vert } \lesssim \frac{\mu  r\log d}{d\sqrt{p}}+\frac{\sigma}{\lambda_{\min}^{\star}}\sqrt{\frac{ rd\log^{2}d}{p}}+\sqrt{\frac{\mu  r\log d}{d}}. \label{eq:init_2_loss}
\end{align}
Here, we have made use of the fact that $\overline{\bm{u}}^{\tau}$
is the leading singular vector of $\bm{M}^{\tau}$ 
obeying $\big\langle \overline{\bm{u}}^{\tau},\overline{\bm{u}}_{1}^{\star}\big\rangle \geq0$.

%%%%%%%%%%%%%%%%%%%%%%%%%%%%%%%%%%%%%%%%%%%%%%%%%%%%%%%%%%%

\subsection{Proof of Lemma~\ref{lemma:init_pertur_op}}

\label{subsec:pf:init_pertur_op_UB}

\subsubsection{\texorpdfstring{Controlling $\bm{F}^{\tau}$}{the truth}}
\begin{itemize}
\item We first consider the spectral norm of $\bm{F}^{\tau}$. Recall the definition
that $\bm{\theta}^{\tau}=\bm{U}\bm{U}^{\top}\bm{g}^{\tau}$. Let us define
$$\bm{\theta}^{\star\tau}:=\bm{U}_{\mathsf{orth}}^{\star}\bm{U}_{\mathsf{orth}}^{\star\top}\,\bm{g}^{\tau}$$
and decompose 
\begin{align*}
\bm{F}^{\tau} =\left(p^{-1}\bm{T}-\bm{T}^{\star}\right)\times_{3}\bm{\theta}^{\tau} =\underbrace{\left(p^{-1}\bm{T}-\bm{T}^{\star}\right)\times_{3}\bm{\theta}^{\star\tau}}_{=:\,\bm{X}}+\underbrace{\left(p^{-1}\bm{T}-\bm{T}^{\star}\right)\times_{3}\big(\bm{\theta}^{\tau}-\bm{\theta}^{\star\tau}\big)}_{=:\,\bm{Y}}.
\end{align*}
In the sequel, we shall control these two terms separately. 
\begin{itemize}
\item
To bound
$\left\Vert \bm{X}\right\Vert $, observe that $\bm{\theta}^{\star\tau}$
is independent of $p^{-1}\bm{T}-\bm{T}^{\star}$. By Lemma~\ref{lemma:T_loss_times3_op_norm},
one has
\begin{equation}
\big\Vert \left(p^{-1}\bm{T}-\bm{T}^{\star}\right)\times_{3}\bm{\theta}^{\star\tau}\big\Vert \lesssim\big\Vert \bm{\theta}^{\star\tau}\big\Vert _{\infty}\sqrt{\frac{\mu r\log d}{dp}}\,\lambda_{\max}^{\star} + \big\Vert \bm{\theta}^{\star\tau}\big\Vert _{\infty} \frac{\sigma \log^{5/2} d}{p}+\big\Vert \bm{\theta}^{\star\tau}\big\Vert _{2}\sigma\sqrt{\frac{d\log d}{p}}.\label{eq:X_op_UB}
\end{equation}
This suggests that we need to control the $\ell_{\infty}$
and $\ell_{2}$ norms of $\bm{\theta}^{\star\tau}$.
Using standard results on Gaussian random vectors and Lemma~\ref{lemma:incoh}, we know that with
probability at least $1-O\left(d^{-20}\right)$,
\begin{align}
\big\Vert \bm{\theta}^{\star\tau}\big\Vert _{\infty} & =\big\Vert \bm{U}_{\mathsf{orth}}^{\star}\bm{U}_{\mathsf{orth}}^{\star\top}\,\bm{g}^{\tau}\big\Vert _{\infty}\lesssim\left\Vert \bm{U}_{\mathsf{orth}}^{\star}\right\Vert _{2,\infty}\sqrt{\log d}\leq\sqrt{\frac{\mu r\log d}{d}},\label{eq:theta_star_inf_UB}\\
\big\Vert \bm{\theta}^{\star\tau}\big\Vert _{2} & =\big\Vert \bm{U}_{\mathsf{orth}}^{\star}\bm{U}_{\mathsf{orth}}^{\star\top}\,\bm{g}^{\tau}\big\Vert _{2}\lesssim\left\Vert \bm{U}_{\mathsf{orth}}^{\star}\right\Vert _{\mathrm{F}}\sqrt{\log d}=\sqrt{r\log d}.\label{eq:theta_star_2_UB}
\end{align}
Combining (\ref{eq:X_op_UB}) with (\ref{eq:theta_star_inf_UB}) and
(\ref{eq:theta_star_2_UB}) reveals that with
probability exceeding $1-O\left(d^{-20}\right)$,
\begin{align}
\big\Vert \left(p^{-1}\bm{T}-\bm{T}^{\star}\right)\times_{3}\bm{\theta}^{\star\tau}\big\Vert & \lesssim\frac{\mu r\lambda_{\max}^{\star}\log d}{d\sqrt{p}}+ \frac{\sigma}{p} \sqrt{\frac{\mu r \log^6 d}{d}}+ \sigma\sqrt{\frac{rd\log^{2}d}{p}} \nonumber \\
& \asymp \frac{\mu r\lambda_{\max}^{\star}\log d}{d\sqrt{p}}+ \sigma\sqrt{\frac{rd\log^{2}d}{p}}, \label{eq:T_theta_star_op_UB}
\end{align}
where the last inequality holds as long as $p \gtrsim \mu d^{-2} \log^4 d$.

\item Turning to $\bm{Y}$, we can simply upper bound
\begin{align*}
\big\Vert \left(p^{-1}\bm{T}-\bm{T}^{\star}\right)\times_{3}\big(\bm{\theta}^{\tau}-\bm{\theta}^{\star\tau}\big)\big\Vert  & \leq\left\Vert p^{-1}\bm{T}-\bm{T}^{\star}\right\Vert \big\Vert \bm{\theta}^{\tau}-\bm{\theta}^{\star\tau}\big\Vert _{2}.
\end{align*}
Since $\mathsf{rank}\left(\bm{U}\bm{U}^{\top}-\bm{U}_{\mathsf{orth}}^{\star}\bm{U}_{\mathsf{orth}}^{\star\top}\right)\leq2r$,
Lemma~\ref{lemma:eigsp_dist_op} and the standard result of Gaussian random vectors yields
that: with probability at least $1-O\left(d^{-12}\right)$,
\begin{align}
\big\Vert \bm{\theta}^{\tau}-\bm{\theta}^{\star\tau}\big\Vert _{2} & =\left\Vert \left(\bm{U}\bm{U}^{\top}-\bm{U}_{\mathsf{orth}}^{\star}\bm{U}_{\mathsf{orth}}^{\star\top}\right)\bm{g}^{\tau}\right\Vert _{2}\lesssim\left\Vert \bm{U}\bm{U}^{\top}-\bm{U}_{\mathsf{orth}}^{\star}\bm{U}_{\mathsf{orth}}^{\star\top}\right\Vert _{\mathrm{F}}\sqrt{\log d}\nonumber \\
 & \leq\left\Vert \bm{U}\bm{U}^{\top}-\bm{U}_{\mathsf{orth}}^{\star}\bm{U}_{\mathsf{orth}}^{\star\top}\right\Vert \sqrt{2r\log d} \lesssim \EE_{\mathsf{se}} \sqrt{r \log d} \ll1,\label{eq:theta_tau_2_loss}
\end{align}
where we recall the definition of $\EE_{\mathsf{se}}$ in (\ref{def:err-sp}) and that $\EE_{\mathsf{se}} \ll 1/\sqrt{r \log d}$ by our conditions. Moreover, by Lemma~\ref{lemma:T_op_norm}, we know
that with probability exceeding $1-O \(d^{-10} \)$,
\begin{align}
\left\Vert p^{-1}\bm{T}-\bm{T}^{\star}\right\Vert  & \leq\left\Vert p^{-1}\mathcal{P}_{\Omega}\left(\bm{T}^{\star}\right)-\bm{T}^{\star}\right\Vert +\left\Vert p^{-1}\mathcal{P}_{\Omega}\left(\bm{E}\right)\right\Vert \nonumber \\
 & \lesssim\frac{\sqrt{\mu r}\,\lambda_{\max}^{\star}\log^{3}d}{d^{3/2}p}+\frac{\mu \sqrt{r}\,\lambda_{\max}^{\star}\log^{5/2}d}{d\sqrt{p}}+\frac{\sigma\log^{7/2}d}{p}+\sigma\sqrt{\frac{d\log^{5}d}{p}}.\label{eq:T_op_loss}
\end{align}
Combining (\ref{eq:theta_tau_2_loss}) and (\ref{eq:T_op_loss}), we find that
\begin{align}
\big\Vert \left(p^{-1}\bm{T}-\bm{T}^{\star}\right)\times_{3}\big(\bm{\theta}^{\tau}-\bm{\theta}^{\star\tau}\big)\big\Vert  & \lesssim\frac{\sqrt{\mu r}\,\lambda_{\max}^{\star}\log^{3}d}{d^{3/2}p}+\frac{\mu \sqrt{r}\,\lambda_{\max}^{\star}\log^{5/2}d}{d\sqrt{p}}+\frac{\sigma\log^{7/2}d}{p}+\sigma\sqrt{\frac{d\log^{5}d}{p}}.\label{eq:T_theta_theta_star_op_UB}
\end{align}
Putting (\ref{eq:T_theta_star_op_UB}) and (\ref{eq:T_theta_theta_star_op_UB})
together shows that
\begin{align*}
\big\Vert \bm{F}^{\tau}\big\Vert \lesssim\frac{\sqrt{\mu r}\,\lambda_{\max}^{\star}\log^{3}d}{d^{3/2}p}+\frac{\mu r\lambda_{\max}^{\star}\log^{5/2}d}{d\sqrt{p}}+\frac{\sigma\log^{7/2}d}{p}+\sigma\sqrt{\frac{rd\log^{5}d}{p}}.
\end{align*}
\end{itemize}

\item Next, we turn to $\big\Vert \bm{F}^{\tau}\overline{\bm{u}}_{1}^{\star}\big\Vert _{2}$.
By the definition of the operator norm, we know that
\begin{align*}
\big\Vert \bm{F}^{\tau}\overline{\bm{u}}_{1}^{\star}\big\Vert _{2} & =\left\Vert \left(p^{-1}\bm{T}-\bm{T}^{\star}\right)\times_{2}\overline{\bm{u}}_{1}^{\star}\times_{3}\bm{\theta}^{\tau}\right\Vert _{2} \leq\left\Vert \left(p^{-1}\bm{T}-\bm{T}^{\star}\right)\times_{2}\overline{\bm{u}}_{1}^{\star}\right\Vert \big\Vert \bm{\theta}^{\tau}\big\Vert _{2}.
\end{align*}
Applying Lemma~\ref{lemma:T_loss_times3_op_norm} again reveals that
with probability at least $1-O\left(d^{-12}\right)$,
\begin{align}
\left\Vert \left(p^{-1}\bm{T}-\bm{T}^{\star}\right)\times_{2}\overline{\bm{u}}_{1}^{\star}\right\Vert  & \lesssim\left\Vert \overline{\bm{u}}_{1}^{\star}\right\Vert _{\infty}\sqrt{\frac{\mu r\log d}{dp}}\,\lambda_{\max}^{\star}+ \left\Vert \overline{\bm{u}}_{1}^{\star}\right\Vert _{\infty} \frac{\sigma \log^{5/2} d}{p} + \left\Vert \overline{\bm{u}}_{1}^{\star}\right\Vert _{2}\sigma\sqrt{\frac{d\log d}{p}}\nonumber \\
& \lesssim \frac{\mu\sqrt{r}\,\lambda_{\max}^{\star}\sqrt{\log d}}{d\sqrt{p}} + \frac{\sigma \sqrt{\mu \log^{5} d}}{\sqrt{d} \, p}+\sigma\sqrt{\frac{d\log d}{p}}  \nonumber \\
 & \asymp \frac{\mu\sqrt{r}\,\lambda_{\max}^{\star}\sqrt{\log d}}{d\sqrt{p}}+\sigma\sqrt{\frac{d\log d}{p}}, \label{eq:T_loss_utrue_op_norm}
\end{align}
where the last step arises from the condition that $p \gtrsim \mu d^{-2} \log^4 d$.
In addition, realizing that $\bm{U}$ consists of eigenvectors, standard Gaussian random vectors results
give that with probability at least $1-O\left(d^{-12}\right)$,
\begin{equation}
\left\Vert \bm{\theta}^\tau \right\Vert _{2}=\big\Vert \bm{U}\bm{U}^{\top}\bm{g}^{\tau}\big\Vert _{2}\lesssim\left\Vert \bm{U}\right\Vert _{\mathrm{F}}\sqrt{\log d}=\sqrt{r\log d}.\label{eq:theta_2_norm}
\end{equation}
Combining (\ref{eq:T_loss_utrue_op_norm}) and (\ref{eq:theta_2_norm}) shows that with probability exceeding $1-O \(d^{-10} \)$,
\begin{align*}
\big\Vert \bm{F}^{\tau}\overline{\bm{u}}_{1}^{\star}\big\Vert _{2}\lesssim\frac{\mu r\lambda_{\max}^{\star}\log d}{d\sqrt{p}}+\sigma\sqrt{\frac{rd\log^{2}d}{p}}.
\end{align*}
%Clearly, applying the similar argument indicates that the above bound also holds for $\big\Vert \overline{\bm{u}}_{1}^{\star\top}\bm{F}^{\tau}\big\Vert _{2}$.
\end{itemize}

\subsubsection{\texorpdfstring{Controlling $\bm{C}^{\tau}$}{the truth}}
Recall the definition of $\bm{C}^{\tau}$ in (\ref{eq:M}). We first consider
the spectral norm of $\bm{C}^{\tau}$. It is straightforward to compute
that
\begin{align}
\big\Vert \bm{C}^{\tau}\big\Vert  & \leq2\,\Big\|\sum\nolimits_{s:s\neq 1}\gamma_{s}^{\star\tau}\left\langle \overline{\bm{u}}_{s}^{\star},\overline{\bm{u}}_{1}^{\star}\right\rangle \overline{\bm{u}}_{s}^{\star}\Big\|_{2}+\Big|\sum\nolimits_{s:s\neq 1}\gamma_{s}^{\star\tau}\left\langle \overline{\bm{u}}_{s}^{\star},\overline{\bm{u}}_{1}^{\star}\right\rangle ^{2}\Big|\nonumber \\
	& \lesssim\max_{s:s\neq 1}\left|\left\langle \overline{\bm{u}}_{s}^{\star},\overline{\bm{u}}_{1}^{\star}\right\rangle \right|\big\|\overline{\bm{U}}^{\star}\big\|\big\Vert \bm{\gamma}^{\star\tau}\big\Vert _{2}+\max_{s:s\neq 1}\left\langle \overline{\bm{u}}_{s}^{\star},\overline{\bm{u}}_{1}^{\star}\right\rangle ^{2} \sqrt{r} \big\Vert \bm{\gamma}^{\star\tau}\big\Vert _{2}\nonumber \\
	& \lesssim\Big( \sqrt{\frac{\mu}{d}}+\frac{\mu \sqrt{r}}{d}\Big)  \sqrt{r\log d}\,\lambda_{\max}^{\star}\asymp\sqrt{\frac{\mu  r\log d}{d}}\,\lambda_{\max}^{\star}\label{C_norm_ub}
\end{align}
if $\mu r / d \lesssim 1$, where we recall that $\overline{\bm{U}}^{\star}=\left[\overline{\bm{u}}_{1}^{\star},\cdots,\overline{\bm{u}}_{r}^{\star}\right]$.
Here, the last line holds owing to (\ref{asmp_corr}),
Lemma~\ref{lemma:incoh} (which justifies that $\big\|\overline{\bm{U}}^{\star}\big\|\lesssim1$
if $r\sqrt{\mu/d}\leq1$) and Lemma~\ref{lemma:init_spectra}. 

The claim~(\ref{eq:C_u_2_norm_UB}) arises from the definition of the spectral norm that $\big\Vert \bm{C}^{\tau}\overline{\bm{u}}_{1}^{\star}\big\Vert _{2} \leq \big\Vert \bm{C}^{\tau}\big\Vert$.

%As for the claim~(\ref{eq:C_u_2_norm_UB}),  simple
%calculation gives that
%\begin{align*}
%\bm{C}^{\tau}\overline{\bm{u}}_{1}^{\star}=\sum_{s:s\neq 1}\gamma_{1}^{\star\tau}\left\langle \overline{\bm{u}}_{s}^{\star},\overline{\bm{u}}_{1}^{\star}\right\rangle \overline{\bm{u}}_{s}^{\star}
%\end{align*}
%and, consequently, 
%\begin{align*}
%\big\Vert \bm{C}^{\tau}\overline{\bm{u}}_{1}^{\star}\big\Vert _{2} \leq\max_{s:s\neq 1}\left|\left\langle \overline{\bm{u}}_{s}^{\star},\overline{\bm{u}}_{1}^{\star}\right\rangle \right|\big\|\overline{\bm{U}}^{\star}\big\|\big\Vert \bm{\gamma}^{\star\tau}\big\Vert _{2}\lesssim\sqrt{\frac{\mu  r\log d}{d}}\,\lambda_{\max}^{\star}.
%\end{align*}

\subsection{Proof of Lemma~\ref{lemma:init_loo_entry_loss}}
\label{pf:init_loo_entry_loss}

Let us fix an arbitrary $m \in \[d\]$. We remind the readers of several definitions:
(1) $\gamma_{1}^{\star\tau}$: see (\ref{def:init_sigval}); 
%(2) $\gamma_{\tau}^{(m)}$: see Algorithm \ref{alg:localization_loo}; 
(2) $\bm{M}^{\star\tau}$: see (\ref{eq:M}); and (3) $\bm{M}^{\tau,\left(m\right)}$: see \eqref{eq:definition-M-tau-m}.

Before continuing, we state two immediate facts. First, it has already
been observed in Appendix~\ref{pf:init_2_loss} that $\overline{\bm{u}}_{1}^{\star}$
is a singular vector of $\bm{M}^{\star\tau}$ with singular value $\gamma_{1}^{\star\tau}$,
and hence 
\begin{align}
\big(\overline{\bm{u}}_{1}^{\star}\big)_{m} & =\big(\gamma_{1}^{\star\tau}\big)^{-1}\bm{M}_{m,:}^{\star\tau}\overline{\bm{u}}_{1}^{\star}.\label{eq:u1-star-m}
\end{align}
Here and throughout, $\bm{A}_{m,:}$ denotes the $m$-th row of a
matrix $\bm{A}$. Second, $\overline{\bm{u}}^{\tau, \left(m\right)}$
is the top  singular vector of $\bm{M}^{\tau, \left(m\right)}$
such that $\langle\overline{\bm{u}}^{\tau, \left(m\right)},\overline{\bm{u}}_{1}^{\star}\rangle\geq0$,
%(which is precisely $\bm{w}^{\tau, \left(m\right)}$ in Algorithm \ref{alg:localization_loo}),
and we denote by $\gamma_{\tau}^{(m)}$ the associated singular value.
Recall our definition of  $\bm{\nu}^{\tau, \m}$ in Algorithm~\ref{alg:localization_loo}. Similar to the case of $\bm{\nu}^\tau$, we will show shortly in Lemma~\ref{lemma:init_coeff} that the global signs of $\bm{\nu}^{\tau, \m}$ and  $\overline{\bm{u}}^{\tau, \m}$ coincide, and hence
\begin{align}
	\bm{\nu}^{\tau, \m} = \overline{\bm{u}}^{\tau, \m}. 
\end{align}

As a result, the proof of this lemma boils down to showing that  $\overline{\bm{u}}^{\tau, \left(m\right)}$ (and hence $\bm{\nu}^{\tau, \m}$)
 is sufficiently close to
$\overline{\bm{u}}_{1}^{\star}$ in the $m$-th entry. 
Towards this end, observe that
\begin{align}
\big(\overline{\bm{u}}^{\tau, \left(m\right)}\big)_{m} & =\big(\gamma_{\tau}^{(m)}\big)^{-1}\bm{M}_{m,:}^{\tau, \left(m\right)}\overline{\bm{u}}^{\tau, \left(m\right)}.\label{eq:u1-star-m-m}
\end{align}
%The goal of this lemma is to show that $\overline{\bm{u}}^{\tau, \left(m\right)}$
%(and hence $\bm{w}^{\tau, \left(m\right)}$)
%
The above two facts (\ref{eq:u1-star-m}) and (\ref{eq:u1-star-m-m})
together with the triangle inequality lead to
\begin{align}
	%\big(\gamma_{1}^{\star\tau}\big)^{-1}
	\Big|\big(\overline{\bm{u}}^{\tau, \left(m\right)}-\overline{\bm{u}}_{1}^{\star}\big)_{m}\Big| 
	& \leq\left|\Big\{\big(\gamma_{\tau}^{(m)}\big)^{-1}-\big(\gamma_{1}^{\star\tau}\big)^{-1}\Big\}\bm{M}_{m,:}^{\tau, \left(m\right)}\overline{\bm{u}}^{\tau, \left(m\right)}\right| \nonumber \\
	& \quad +\big(\gamma_{1}^{\star\tau}\big)^{-1}\left|\big(\bm{M}^{\tau, \left(m\right)}-\bm{M}^{\star\tau}\big)_{m,:}\overline{\bm{u}}^{\tau, \left(m\right)}\right|\nonumber \\
 & \quad+\big(\gamma_{1}^{\star\tau}\big)^{-1}\left|\bm{M}_{m,:}^{\star\tau}\big(\overline{\bm{u}}^{\tau, \left(m\right)}-\overline{\bm{u}}_{1}^{\star}\big)\right|\nonumber \\
 & \leq\underbrace{\left|\big(\gamma_{\tau}^{(m)}\big)^{-1}-\big(\gamma_{1}^{\star\tau}\big)^{-1}\right|\big\|\bm{M}_{m,:}^{\tau, \left(m\right)}\big\|_{2}\big\|\overline{\bm{u}}^{\tau, \left(m\right)}\big\|_{2}}_{=: \, \alpha_{1}} \nonumber \\
 & \quad+\underbrace{\big(\gamma_{1}^{\star\tau}\big)^{-1}\big\|\big(\bm{M}^{\tau, \left(m\right)}-\bm{M}^{\star\tau}\big)_{m,:}\big\|_{2}\big\|\overline{\bm{u}}^{\tau, \left(m\right)}\big\|_{2}}_{=: \, \alpha_{2}}\nonumber \\
 & \quad+\underbrace{\big(\gamma_{1}^{\star\tau}\big)^{-1}\big\Vert \bm{M}_{m,:}^{\star\tau}\big\Vert _{2}\big\|\overline{\bm{u}}^{\tau, \left(m\right)}-\overline{\bm{u}}_{1}^{\star}\big\|_{2}}_{=: \, \alpha_{3}}.\label{eq:defn-alpha123-u}
\end{align}
Therefore, it suffices to upper bound the above three quantities separately.

\subsubsection{Controlling $\alpha_{3}$}

%We start by  looking at $\alpha_3$ (cf.~\eqref{eq:defn-alpha123-u}). 

The first step to bound $\alpha_{3}$ (cf.~(\ref{eq:defn-alpha123-u}))
is to control $\big\|\bm{M}_{m,:}^{\star\tau}\big\|_{2}$. Towards this
end, we first observe from the incoherence conditions that
\begin{align}
\max_{s:s\neq 1}\big|\left(\overline{\bm{u}}_{s}^{\star}-\left\langle \overline{\bm{u}}_{s}^{\star},\overline{\bm{u}}_{1}^{\star}\right\rangle \overline{\bm{u}}_{1}^{\star}\right)_{m}\big| & \leq\max_{s:s\neq 1}\left\Vert \overline{\bm{u}}_{s}^{\star}\right\Vert _{\infty}+\max_{s:s\neq 1}\big|\left\langle \overline{\bm{u}}_{s}^{\star},\overline{\bm{u}}_{1}^{\star}\right\rangle \big|\left\Vert \overline{\bm{u}}_{1}^{\star}\right\Vert _{\infty} \nonumber \\
&  \leq \max_{s:s\neq 1}\left\Vert \overline{\bm{u}}_{s}^{\star}\right\Vert _{\infty}+\max_{s:s\neq 1}\left\| \overline{\bm{u}}_{s}^{\star} \right\|_2 \left\| \overline{\bm{u}}_{1}^{\star} \right\|_2 \left\Vert \overline{\bm{u}}_{1}^{\star}\right\Vert _{\infty} \lesssim\sqrt{\frac{\mu}{d}}. \label{eq:UB101}
\end{align}
 When combined with the definition
(\ref{eq:M}), this gives
\begin{align}
\big\|\bm{M}_{m,:}^{\star\tau}\big\|_{2} & =\Big\|\gamma_{1}^{\star\tau}\left(\overline{\bm{u}}_{1}^{\star}\right)_{m}\overline{\bm{u}}_{1}^{\star\top}+\sum\nolimits_{s:s\neq 1}\gamma_{s}^{\star\tau}\left(\overline{\bm{u}}_{s}^{\star}-\left\langle \overline{\bm{u}}_{s}^{\star},\overline{\bm{u}}_{1}^{\star}\right\rangle \overline{\bm{u}}_{1}^{\star}\right)_{m}\overline{\bm{u}}_{s}^{\star\top}\left(\bm{I}-\overline{\bm{u}}_{1}^{\star}\overline{\bm{u}}_{1}^{\star\top}\right)\Big\|_{2}\nonumber \\
 & \leq \gamma_{1}^{\star\tau} \left\Vert \overline{\bm{u}}_{1}^{\star}\right\Vert _{\infty}\left\Vert \overline{\bm{u}}_{1}^{\star}\right\Vert _{2}+\Big\|\sum\nolimits_{s:s\neq 1}\gamma_{s}^{\star\tau}\left(\overline{\bm{u}}_{s}^{\star}-\left\langle \overline{\bm{u}}_{s}^{\star},\overline{\bm{u}}_{1}^{\star}\right\rangle \overline{\bm{u}}_{1}^{\star}\right)_{m}\overline{\bm{u}}_{s}^{\star\top}\Big\|_{2}\nonumber \\
 & \leq \gamma_{1}^{\star\tau} \left\Vert \overline{\bm{u}}_{1}^{\star}\right\Vert _{\infty}+\max_{s:s\neq 1}\big|\left(\overline{\bm{u}}_{s}^{\star}-\left\langle \overline{\bm{u}}_{s}^{\star},\overline{\bm{u}}_{1}^{\star}\right\rangle \overline{\bm{u}}_{1}^{\star}\right)_{m}\big|\big\Vert \bm{\gamma}^{\star\tau}\big\Vert _{2}\big\|\overline{\bm{U}}^{\star}\big\|\nonumber \\
 & \overset{(\text{i})}{\lesssim} \gamma_{1}^{\star\tau} \sqrt{\frac{\mu}{d}}+\big\Vert \bm{\gamma}^{\star\tau}\big\Vert _{2}\sqrt{\frac{\mu}{d}}\asymp\big\Vert \bm{\gamma}^{\star\tau}\big\Vert _{2}\sqrt{\frac{\mu}{d}}
 \lesssim\sqrt{\frac{\mu  r\log d}{d}}\,\lambda_{\max}^{\star},\label{eq:M_row_2_norm}
\end{align}
where (i) arises from (\ref{eq:UB101}) and $\big\|\overline{\bm{U}}^{\star}\big\|\lesssim1$
if $r\sqrt{\mu/d}\ll1$, and the last step comes from Lemma~\ref{lemma:init_spectra}.

The second step is to upper bound $\big\|\overline{\bm{u}}^{\tau, \left(m\right)}-\overline{\bm{u}}_{1}^{\star}\big\|_{2}$.
Towards this, we resort to Wedin's theorem as follows
\begin{align}
 \big\|\overline{\bm{u}}^{\tau, \left(m\right)} -\overline{\bm{u}}_{1}^{\star}\big\|_{2}  \leq\frac{ \big\|\big(\bm{M}^{\tau, \left(m\right)}-\bm{M}^{\star\tau}\big)\overline{\bm{u}}_{1}^{\star}\big\|_{2} }{\sigma_{1}\big(\bm{M}^{\star\tau}\big)-\sigma_{2}\big(\bm{M}^{\star\tau}\big)-\big\|\bm{M}^{\tau, \left(m\right)}-\bm{M}^{\star\tau}\big\|},\label{eq:v-1-u-1-Wedin}
\end{align}
where we rely on the fact that $\big\langle\overline{\bm{u}}^{\tau, \left(m\right)},\overline{\bm{u}}_{1}^{\star}\big\rangle\geq0$. To complete this bound, we need
to control $\bm{M}^{\tau, \left(m\right)}-\bm{M}^{\star\tau}$. Before we
move on, we find it helpful to introduce
\begin{align}
 & \widehat{\bm{M}}^{\tau, \left(m\right)}:= p^{-1}\bm{T}\times_{3}\bm{\theta}^{\tau, \left(m\right)}.\label{eq:projected-matrix-M-tau_g}
\end{align}
Let $\widehat{\bm{u}}^{\tau, \left(m\right)}$ denote the top left singular
vector of $\widehat{\bm{M}}^{\tau, \left(m\right)}$ such that
\begin{align}
\big\langle\widehat{\bm{u}}^{\tau, \left(m\right)},\overline{\bm{u}}_{1}^{\star}\big\rangle\geq0.
\end{align}
Since we have already bounded $\big\Vert \bm{M}^{\tau}-\bm{M}^{\star\tau}\big\Vert $
in Lemma~\ref{lemma:init_pertur_op}, we can decompose
\begin{align*}
\bm{M}^{\tau, \left(m\right)}-\bm{M}^{\star\tau} =\bm{M}^{\tau, \left(m\right)}-\bm{M}^{\tau}  +\bm{M}^{\tau}-\bm{M}^{\star\tau} =\bm{M}^{\tau, \left(m\right)}-\widehat{\bm{M}}^{\tau, \left(m\right)}+\widehat{\bm{M}}^{\tau, \left(m\right)}-\bm{M}^{\tau} + \bm{M}^{\tau}-\bm{M}^{\star\tau}.
\end{align*}
With these definitions in place, Lemma~\ref{lemma:M_loo_op_loss}
below provides the desired bounds.
\begin{lemma}\label{lemma:M_loo_op_loss}Instate the assumptions
of Lemma~\ref{lemma:init_loo_entry_loss}. With probability at least $1-O\left(d^{-10}\right)$, the following
holds simultaneously for all $1\leq m\leq d$:
\begin{align}
\big\|\widehat{\bm{M}}^{\tau, \left(m\right)}-\bm{M}^{\tau, \left(m\right)}\big\| & \lesssim\frac{\mu r\lambda_{\max}^{\star}\sqrt{\log d}}{d\sqrt{p}}+\sigma\sqrt{\frac{rd\log d}{p}} \label{eq:M_hat_loo_op_norm},\\
\big\|\widehat{\bm{M}}^{\tau, \left(m\right)}-\bm{M}^{\tau}\big\| & \lesssim \EE_\loo \sqrt{\frac{\mu r \log d}{d}} \lambda_{\max}^\star \label{eq:M_hat_M_hat_loo_op_norm}, 
%\lesssim\frac{\left(\mu r\right)^{3/2}\lambda_{\max}^{\star}\log^{3/2}d}{d^{2}p}+\frac{\mu r\lambda_{\max}^{\star}\log d}{d^{3/2}\sqrt{p}}+\frac{\sigma^{2}}{\lambda_{\min}^{\star}}\frac{\sqrt{\mu r}\,d\log^{3/2}d}{p}+\sigma\sqrt{\frac{\mu r\log^{2}d}{p}}.
\end{align}
where $\EE_\loo$ is defined in \eqref{def:err-loo}.
As a result, one has
\begin{align}
 \big\|\big(\bm{M}^{\tau, \left(m\right)} -\bm{M}^{\star\tau}\big)\overline{\bm{u}}_{1}^{\star}\big\|_{2} \leq \big\|\bm{M}^{\tau}-\bm{M}^{\tau, \left(m\right)}\big\|  \lesssim  \frac{\mu r\lambda_{\max}^{\star}\sqrt{\log d}}{d\sqrt{p}}+\sigma\sqrt{\frac{rd\log d}{p}}.
\end{align}
\end{lemma}\begin{proof}See Appendix~\ref{subsec:pf:M_loo_op_loss}.\end{proof}

We then can further combine (\ref{eq:M_hat_M_op_loss}) and (\ref{eq:M_hat_M_u_2_norm}) to deduce that 
\begin{align}
\begin{split}\big\|\bm{M}^{\tau, \left(m\right)}-\bm{M}^{\star\tau}\big\| & \leq\big\|\bm{M}^{\tau}-\bm{M}^{\tau, \left(m\right)}\big\|+\big\|\bm{M}^{\tau}-\bm{M}^{\star\tau}\big\|\\
 & \lesssim \underbrace{\frac{\sqrt{\mu r}\,\lambda_{\max}^{\star}\log^{3}d}{d^{3/2}p}+\frac{\mu r\lambda_{\max}^{\star}\log^{5/2}d}{d\sqrt{p}}+\frac{\sigma\log^{7/2}d}{p}+\sigma\sqrt{\frac{rd\log^{5}d}{p}}+\sqrt{\frac{\mu  r\log d}{d}}\,\lambda_{\max}^{\star}}_{= \, \EE_\op \cdot \lambda_{\min}^\star},
\end{split}
\label{M_loo_op_loss}
\end{align}
and
\begin{align}
 \big\|\big(\bm{M}^{\tau, \left(m\right)} -\bm{M}^{\star\tau}\big)\overline{\bm{u}}_{1}^{\star}\big\|_{2} \lesssim \underbrace{\frac{\mu r \lambda_{\max}^\star \log d}{d\sqrt{p}}+\sigma \sqrt{\frac{rd\log^{2}d}{p}}+\sqrt{\frac{\mu  r\log d}{d}} \, \lambda_{\max}^\star}_{= \, \EE_\proj \cdot \lambda_{\min}^\star}.
\end{align}
In particular,
we have $\big\|\bm{M}^{\tau, \left(m\right)}-\bm{M}^{\star\tau}\big\|\ll\lambda_{\min}^{\star}$ under our conditions,
and it follows from \eqref{eq:sigval_M_star_gap} that
\begin{align*}
\sigma_{1}\big(\bm{M}^{\star\tau}\big)-\sigma_{2}\big(\bm{M}^{\star\tau}\big)-\big\|\bm{M}^{\tau, \left(m\right)}-\bm{M}^{\star\tau}\big\|\gtrsim\lambda_{\min}^{\star}.
\end{align*}
Invoke the bound (\ref{eq:v-1-u-1-Wedin}) to obtain 
\begin{align}
 \big\|\overline{\bm{u}}^{\tau, \left(m\right)} -\overline{\bm{u}}_{1}^{\star}\big\|_{2}  \lesssim\frac{1}{\lambda_{\min}^{\star}} \big\|\big(\bm{M}^{\tau, \left(m\right)}-\bm{M}^{\star\tau}\big)\overline{\bm{u}}_{1}^{\star}\big\|_{2} \lesssim \underbrace{\frac{\mu  r\log d}{d\sqrt{p}}+\frac{\sigma}{\lambda_{\min}^{\star}}\sqrt{\frac{rd\log^{2}d}{p}}+\sqrt{\frac{\mu  r\log d}{d}}}_{= \, \EE_\proj}.\label{u_tilde_loo_2_loss}
\end{align}

To finish up, combine this with (\ref{eq:M_row_2_norm}) and the spectral
condition to arrive at 
\begin{align}
\alpha_{3}\lesssim  \EE_\proj  \sqrt{\frac{\mu r  \log d}{d}} \lesssim \EE_\op \sqrt{\frac{\mu r \log d}{d}},
\end{align}
which results from the fact that $\EE_\proj \leq \EE_\op$ (cf.~(\ref{def:err-op})).

\subsubsection{Controlling $\alpha_{2}$}

We then turn to $\alpha_{2}$ (cf.~(\ref{eq:defn-alpha123-u})).
Recall the definition of $\bm{M}^{\tau, \left(m\right)}$ in (\ref{eq:projected-matrix-M-tau}).
It is straightforward to verify that 
\begin{align*}
\big(\bm{M}^{\tau, \left(m\right)}-\bm{M}^{\star\tau}\big)_{m,:} & =\big(p^{-1}\bm{T}^{\tau, \left(m\right)}\times_{3}\bm{\theta}^{\tau, \left(m\right)}-\bm{T}^{\star}\times_{3}\bm{\theta}^{\tau}\big)_{m,:}+\big(\bm{T}^{\star}\times_{3}\bm{\theta}^{\tau}-\bm{M}^{\star\tau}\big)_{m,:}\\
 & =\big(p^{-1}\bm{T}^{\tau, \left(m\right)}\times_{3}\bm{\theta}^{\tau, \left(m\right)}-\bm{T}^{\star}\times_{3}\bm{\theta}^{\tau}\big)_{m,:}+\bm{C}_{m,:}^{\tau}\\
 & =\bm{T}_{m,:,:}^{\star}\times_{3}\big(\bm{\theta}^{\tau, \left(m\right)}-\bm{\theta}^{\tau}\big)+\bm{C}_{m,:}^{\tau},
\end{align*}
where $\bm{C}^{\tau}$ is defined in (\ref{eq:M}).

From the incoherence conditions, we can upper bound 
\begin{align}
\big\Vert \bm{C}_{m,:}^{\tau}\big\Vert _{2} & =\Big\|\sum\nolimits_{s:s\neq 1}\gamma_{s}^{\star\tau}\left\langle \overline{\bm{u}}_{s}^{\star},\overline{\bm{u}}_{1}^{\star}\right\rangle \left(\left(\overline{\bm{u}}_{1}^{\star}\right)_{m}\overline{\bm{u}}_{s}^{\star\top}+\left(\overline{\bm{u}}_{s}^{\star}\right)_{m}\overline{\bm{u}}_{1}^{\star\top}\right)-\sum\nolimits_{s:s\neq 1}\gamma_{s}^{\star\tau}\left\langle \overline{\bm{u}}_{s}^{\star},\overline{\bm{u}}_{1}^{\star}\right\rangle ^{2}\left(\overline{\bm{u}}_{1}^{\star}\right)_{m}\overline{\bm{u}}_{1}^{\star\top}\Big\|_{2}\nonumber \\
 & \leq\Big\|\sum\nolimits_{s:s\neq 1}\gamma_{s}^{\star\tau}\left\langle \overline{\bm{u}}_{s}^{\star},\overline{\bm{u}}_{1}^{\star}\right\rangle \left(\overline{\bm{u}}_{1}^{\star}\right)_{m}\overline{\bm{u}}_{s}^{\star\top}\Big\|_{2}+\Big|\sum\nolimits_{s:s\neq 1}\gamma_{s}^{\star\tau}\left\langle \overline{\bm{u}}_{s}^{\star},\overline{\bm{u}}_{1}^{\star}\right\rangle \left(\overline{\bm{u}}_{s}^{\star}\right)_{m}\Big| +\Big|\sum\nolimits_{s:s\neq 1}\gamma_{s}^{\star\tau}\left\langle \overline{\bm{u}}_{s}^{\star},\overline{\bm{u}}_{1}^{\star}\right\rangle ^{2}\left(\overline{\bm{u}}_{1}^{\star}\right)_{m}\Big|\nonumber \\
 & \lesssim\big\Vert \bm{\gamma}^{\star\tau}\big\Vert _{2} \Big( \max_{s:s\neq 1}\left|\langle\overline{\bm{u}}_{s}^{\star},\overline{\bm{u}}_{1}^{\star}\rangle\right|\left\Vert \overline{\bm{u}}_{1}^{\star}\right\Vert _{\infty}\big\|\overline{\bm{U}}^{\star}\big\|
 +\max_{s:s\neq 1}\left|\langle\overline{\bm{u}}_{s}^{\star},\overline{\bm{u}}_{1}^{\star}\rangle\right|\big\Vert \overline{\bm{U}}^{\star}\big\Vert _{2, \infty}
 +\max_{s:s\neq 1}\left|\langle\overline{\bm{u}}_{s}^{\star},\overline{\bm{u}}_{1}^{\star}\rangle\right|^{2}\left\Vert \overline{\bm{u}}_{1}^{\star}\right\Vert _{\infty} \sqrt{r} \Big) \nonumber \\
 & \lesssim\big\Vert \bm{\gamma}^{\star\tau}\big\Vert _{2} \( \sqrt{\frac{\mu }{d}} \sqrt{\frac{\mu}{d}} + \sqrt{\frac{\mu }{d}} \sqrt{\frac{\mu r}{d}} + \sqrt{\frac{\mu}{d}} \sqrt{\frac{\mu r}{d}} \) \lesssim \sqrt{\frac{\mu r}{d}} \sqrt{\frac{\mu r \log d}{d}}  \, \lambda_{\max}^\star \label{eq:C_2inf_norm_UB},
\end{align}
where we use the fact that $\big\Vert \bm{\gamma}^{\star\tau}\big\Vert _{2} \lesssim \sqrt{r \log d} \, \lambda_{\max}^\star$ from Lemma~\ref{lemma:init_spectra} and $\big\Vert \overline{\bm{U}}^{\star}\big\Vert \lesssim 1$ from Lemma~\ref{lemma:incoh}.

In addition, we can express
\begin{align*}
\bm{T}_{m,:,:}^{\star}\times_{3}\big(\bm{\theta}^{\tau, \left(m\right)}-\bm{\theta}^{\tau}\big)=\sum_{s\in\left[r\right]}\lambda_{s}^{\star}\big(\overline{\bm{u}}_{s}^{\star}\big)_{m}\big\langle\overline{\bm{u}}_{s}^{\star},\bm{\theta}^{\tau, \left(m\right)}-\bm{\theta}^{\tau}\big\rangle\overline{\bm{u}}_{s}^{\star}.
\end{align*}
By construction, we know that
\begin{align*}
\bm{\theta}^{\tau}-\bm{\theta}^{\tau, \left(m\right)} & =\big(\bm{U}\bm{U}^{\top}-\bm{U}^{(m)}\bm{U}^{\left(m\right)\top}\big)\,\bm{g}^{\tau}
\end{align*}
is a zero-mean Gaussian random vector conditional on $\mathcal{P}_{\Omega}\left(\bm{E}\right)$.
Using standard results on Gaussian random vectors, one has: with
probability at least $1-O\left(d^{-11}\right)$, for each $s\in\left[r\right]$ and $m \in \[ d \]$,
\begin{align}
\left|\big\langle\overline{\bm{u}}_{s}^{\star},\bm{\theta}^{\tau}-\bm{\theta}^{\tau, \left(m\right)}\big\rangle\right| & =\left|\big\langle\bm{g}^{\tau},\big(\bm{U}\bm{U}^{\top}-\bm{U}^{(m)}\bm{U}^{\left(m\right)\top}\big) \, \overline{\bm{u}}_{s}^{\star}\big\rangle\right| \lesssim\big\|\big(\bm{U}\bm{U}^{\top}-\bm{U}^{(m)}\bm{U}^{\left(m\right)\top}\big) \, \overline{\bm{u}}_{s}^{\star}\big\|_{2}\sqrt{\log d} \nonumber\\
& \leq\big\|\bm{U}\bm{U}^{\top}-\bm{U}^{(m)}\bm{U}^{\left(m\right)\top}\big\|\sqrt{\log d} \label{eq:utrue_theta_diff}
\end{align}
and
\begin{equation}
\big\|\bm{\theta}^{\tau}-\bm{\theta}^{\tau, \left(m\right)}\big\|_{2}\lesssim\big\|\bm{U}\bm{U}^{\top}-\bm{U}^{(m)}\bm{U}^{\left(m\right)\top}\big\|_{\mathrm{F}}\sqrt{\log d}. \label{eq:g_loo_l2_loss}
\end{equation}
Therefore, we have
\begin{align*}
\big\|\bm{T}_{m,:,:}^{\star}\times_{3}\big(\bm{\theta}^{\tau, \left(m\right)}-\bm{\theta}^{\tau}\big)\big\|_{2} & \leq\max_{s\in\left[r\right]}\big|\lambda_{s}^{\star}\big\langle\overline{\bm{u}}_{s}^{\star},\bm{\theta}^{\tau, \left(m\right)}-\bm{\theta}^{\tau}\big\rangle\big| \big\|\overline{\bm{U}}^{\star}\big\|_{2, \infty} \big\|\overline{\bm{U}}^{\star}\big\|\\
 & \lesssim \lambda_{\max}^{\star} \big\|\overline{\bm{U}}^{\star}\big\|_{2, \infty}\big\|\bm{U}\bm{U}^{\top}-\bm{U}^{(m)}\bm{U}^{\left(m\right)\top}\big\|\sqrt{\log d}\\
 & \lesssim \EE_\loo \sqrt{\frac{\mu r}{d}} \sqrt{\frac{\mu r \log d}{d}}  \, \lambda_{\max}^\star \ll \sqrt{\frac{\mu r}{d}} \sqrt{\frac{\mu r \log d}{d}}  \, \lambda_{\max}^\star
\end{align*}
where we have used Lemma~\ref{lemma:U_U_loo_dist_op} as well as the conditions that $\EE_\loo \ll 1$, $\big\|\overline{\bm{U}}^{\star}\big\| \lesssim 1$ and $\big\|\overline{\bm{U}}^{\star}\big\|_{2,\infty} \lesssim \sqrt{\mu r / d}$.

Putting the above bounds together, we arrive at 
\begin{align}
\left\Vert \big(\bm{M}^{\tau, \left(m\right)}-\bm{M}^{\star\tau}\big)_{m,:}\right\Vert _{2} & \leq\big\|\bm{T}_{m,:,:}^{\star}\times_{3}\big(\bm{\theta}^{\tau, \left(m\right)}-\bm{\theta}^{\tau}\big)\big\|_{2}+\big\Vert \bm{C}_{m,:}^{\tau}\big\Vert _{2} \lesssim \sqrt{\frac{\mu r}{d}} \sqrt{\frac{\mu r \log d}{d}}  \, \lambda_{\max}^\star.\label{eq:M_loo_M_true_row_2_loss}
\end{align}
We therefore conclude that
\begin{align*}
\alpha_{2}\lesssim \sqrt{\frac{\mu r}{d}} \sqrt{\frac{\mu r \log d}{d}}  \, \lambda_{\max}^\star \lesssim \EE_\op \sqrt{\frac{\mu r \log d}{d}}  \, \lambda_{\max}^\star
\end{align*}
where we remind the reader of the definition of $\EE_\op$ in (\ref{def:err-op}).

\subsubsection{Controlling $\alpha_{1}$}

The remaining quantity to control is $\alpha_{1}$ (see (\ref{eq:defn-alpha123-u})).
Invoke Weyl's inequality to show that 
\begin{align*}
\big|\gamma_{1}^{\star\tau}-\gamma_{\tau}^{(m)}\big| & \leq\big\|\bm{M}^{\tau, \left(m\right)}-\bm{M}^{\star\tau}\big\| \lesssim \EE_\op \lambda_{\min}^\star \ll\gamma_{1}^{\star\tau},
\end{align*}
where the last inequality arises from (\ref{M_loo_op_loss}) and Lemma~\ref{lemma:init_spectra}.
Under our sample size, rank and noise conditions, we have 
\begin{align*}
\frac{1}{2}\gamma_{1}^{\star\tau}\leq\gamma_{1}^{\star\tau}-\big|\gamma_{1}^{\star\tau}-\gamma_{\tau}^{(m)}\big|\leq\gamma_{\tau}^{(m)}\leq\big|\gamma_{1}^{\star\tau}-\gamma_{\tau}^{(m)}\big|+\gamma_{1}^{\star\tau}\leq2\gamma_{1}^{\star\tau}.
\end{align*}
This indicates that 
\begin{align*}
\frac{\big|\gamma_{1}^{\star\tau}-\gamma_{\tau}^{(m)}\big|}{\gamma_{1}^{\star\tau}\gamma_{\tau}^{(m)}}\lesssim\frac{1}{\big(\gamma_{1}^{\star\tau}\big)^{2}} & \big\|\bm{M}^{\tau, \left(m\right)}-\bm{M}^{\star\tau}\big\| \lesssim\frac{1}{\lambda_{\min}^{\star}} \EE_\op.
\end{align*}
Moreover, we learn from (\ref{eq:M_loo_M_true_row_2_loss}) and (\ref{eq:M_row_2_norm}) that 
\begin{align*}
\big\|\bm{M}_{m,:}^{\tau, \left(m\right)}\big\|_{2} & \leq\big\|\big(\bm{M}^{\tau, \left(m\right)}-\bm{M}^{\star\tau}\big)_{m,:}\big\|_{2}+\big\|\bm{M}_{m,:}^{\star\tau}\big\|_{2} \\
& \lesssim \frac{\mu r\sqrt{ \log d}}{d} \, \lambda_{\max}^\star +  \sqrt{\frac{\mu r\log d}{d}}\,\lambda_{\max}^{\star} \asymp \sqrt{\frac{\mu r\log d}{d}}\,\lambda_{\min}^{\star},
\end{align*}
where the last step follows from the fact that $\mu r \leq d$ and $\kappa \asymp 1$. Hence,
we reach the conclusion that 
\begin{align*}
\alpha_{1} & \lesssim\frac{\big|\gamma_{1}^{\star\tau}-\gamma_{\tau}^{(m)}\big|}{\gamma_{1}^{\star\tau}\gamma_{\tau}^{(m)}}\big\|\bm{M}_{m,:}^{\tau, \left(m\right)}\big\|_{2} \lesssim \EE_\op \sqrt{\frac{\mu r \log d}{d}} .
\end{align*}

\subsubsection{Combining $\alpha_1$, $\alpha_2$ and $\alpha_3$}

Putting together all of the preceding bounds on $\alpha_{1}$, $\alpha_{2}$
and $\alpha_{3}$ immediately establishes the lemma.

\subsection{Proof of Lemma~\ref{lemma:M_loo_op_loss}}

\label{subsec:pf:M_loo_op_loss}

First of all, if the claims (\ref{eq:M_hat_loo_op_norm}) and (\ref{eq:M_hat_M_hat_loo_op_norm}) can be established, then putting them together yields
\begin{align}
\big\|\bm{M}^{\tau}-\bm{M}^{\tau, \left(m\right)}\big\| & \leq\big\|\bm{M}^{\tau}-\widehat{\bm{M}}^{\tau, \left(m\right)}\big\|+\big\|\widehat{\bm{M}}^{\tau, \left(m\right)}-\bm{M}^{\tau, \left(m\right)}\big\|\lesssim\frac{\mu r\lambda_{\max}^{\star}\sqrt{\log d}}{d\sqrt{p}}+\sigma\sqrt{\frac{ rd\log d}{p}},\label{M_hat_loo_M_hat_op_loss}
\end{align}
where we recall the definition of $\EE_\loo$ in \eqref{def:err-loo} and use the sample size, noise and rank conditions. The rest of the proof is thus dedicated to establishing (\ref{eq:M_hat_loo_op_norm}) and (\ref{eq:M_hat_M_hat_loo_op_norm}). In what follows, we shall assume $\{E_{i,j,k}\}_{i,j,k\in[d]}$ (resp.~$\{\chi_{i,j,k}\}_{i,j,k\in[d]}$) are independent random variables to simplify presentation.

\subsubsection{\texorpdfstring{Proximity of $\bm{M}^{\tau, \left(m\right)}$ and $\widehat{\bm{M}}^{\tau, \left(m\right)}$}{the truth}}

Recall the definition of $\bm{M}^{\tau, \left(m\right)}=p^{-1}\bm{T}^{(m)}\times_{3}\bm{\theta}^{\tau, \left(m\right)}$
in (\ref{eq:projected-matrix-M-tau}). Comparing this with the definition of $\widehat{\bm{M}}^{\tau, \left(m\right)}$
in (\ref{eq:projected-matrix-M-tau_g}), we see that \begin{subequations} 
\begin{align}
\big(\widehat{\bm{M}}^{\tau, \left(m\right)}-\bm{M}^{\tau, \left(m\right)}\big)_{i,j} & =\theta_{m}^{\tau, \left(m\right)} \big(T_{i,j,m}^{\star}(p^{-1}\chi_{i,j,m}-1 )+p^{-1}E_{i,j,m}\chi_{i,j,m}\big), \; i\neq m,j\neq m, \label{eq:M1m-ij-formula}\\
\big(\widehat{\bm{M}}^{\tau, \left(m\right)}-\bm{M}^{\tau, \left(m\right)}\big)_{i,m} & = \theta_{k}^{\tau, \left(m\right)} \sum\nolimits_{k\in[d]} \big( T_{i,m,k}^{\star}(p^{-1}\chi_{i,m,k}-1 )+p^{-1}E_{i,m,k}\chi_{i,m,k} \big) , \; i\neq m, \label{eq:M1m-im-formula}\\
\big(\widehat{\bm{M}}^{\tau, \left(m\right)}-\bm{M}^{\tau, \left(m\right)}\big)_{m,j} & = \theta_{k}^{\tau, \left(m\right)} \sum\nolimits_{k\in[d]} \big( T_{m,j,k}^{\star}(p^{-1}\chi_{m,j,k}-1 )+p^{-1} E_{m,j,k}\chi_{m,j,k} \big) .\label{eq:M1m-mj-formula}
\end{align}
\end{subequations}

Note that $\bm{\theta}^{\tau, \m}\sim\mathcal{N}\left(\bm{0},\bm{U}^{\m}\bm{U}^{\m \top}\right)$
conditional on $\mathcal{P}_{\Omega}\left(\bm{E}\right)$. Standard Gaussian concentration inequalities reveal that with probability exceeding $1-O(d^{-10})$,
\begin{align}
\big\Vert \bm{\theta}^{\tau, \m}\big\Vert _{2}\lesssim\sqrt{r\log d}.\label{eq:theta_loo_2_norm}
\end{align}
From Lemmas~\ref{lemma:eigsp_dist_op}-\ref{lemma:U_U_loo_dist_op} and the fact that $\max\{ \EE_{\mathsf{se}}, \EE_\loo \} \ll 1$,
we have
\begin{align*}
\max_{i\in[d]}\mathsf{Var}\big(\theta_{i}^{\tau}\big) & =\left\Vert \bm{U}\right\Vert _{2,\infty}^{2}\lesssim \frac{\mu r}{d} + \left\Vert \bm{U}_{\mathsf{orth}}^{\star}\right\Vert _{2,\infty}^{2}\asymp\frac{\mu r}{d},\\
\max_{i\in[d]}\mathsf{Var}\big(\theta_{i}^{\tau, \left(m\right)}\big) & =\big\|\bm{U}^{(m)}\big\|_{2,\infty}^{2}\lesssim \left\Vert \bm{U}\right\Vert _{2,\infty}^{2}\lesssim\frac{\mu r}{d}, \quad 1 \leq m \leq d.
\end{align*}
As a consequence, standard concentration results assert that with
probability $1-O(d^{-10})$,
\begin{align}
\big\Vert \bm{\theta}^{\tau}\big\Vert _{\infty} & \leq\sqrt{\max_{i\in[d]} \mathsf{Var}\big(\theta_{i}^{\tau}\big)  \log d}\lesssim\sqrt{\frac{\mu r\log d}{d}};\label{eq:theta_inf_norm}\\
\big\|\bm{\theta}^{\tau, \left(m\right)}\big\|_{\infty} & \leq\sqrt{\max_{i\in[d]} \mathsf{Var}\big(\theta_{i}^{\tau, \m}\big)  \log d} \lesssim\sqrt{\frac{\mu r\log d}{d}}, \quad 1 \leq m \leq d.\label{eq:theta_loo_inf_norm}
\end{align}

\begin{itemize}
\item Regarding the $m$-th row of $\widehat{\bm{M}}^{\tau, \left(m\right)}-\bm{M}^{\tau, \left(m\right)}$,
apply Lemma~\ref{lemma:sum_jk_square_cont_ineq} to show that with
probability $1-O\left(d^{-11}\right)$,
\begin{align*}
\sum\nolimits_{j\in[d]}\Big(\sum\nolimits_{k\in[d]}T_{m,j,k}^{\star}\theta_{k}^{\tau, \left(m\right)}(p^{-1}\chi_{m,j,k}-1)\Big)^{2}\lesssim\frac{\mu r\lambda_{\max}^{\star2}}{dp}\big\|\bm{\theta}^{\tau, \left(m\right)}\big\|_{\infty}^{2}\lesssim\frac{\mu^{2}r^{2}\lambda_{\max}^{\star2}\log d}{d^{2}p},
\end{align*}
where the last inequality comes from (\ref{eq:theta_loo_inf_norm}). In addition,
Lemma~\ref{lemma:sum_jk_square_cont_ineq_subg} indicates that with
probability exceeding $1-O\left(d^{-11}\right)$,
\begin{align*}
\sum\nolimits_{j\in[d]}\Big(\sum\nolimits_{k\in[d]}\theta_{k}^{\tau, \left(m\right)}E_{m,j,k}\chi_{m,j,k}\Big)^{2} & \lesssim\sigma^{2}dp\,\big\|\bm{\theta}^{\tau, \left(m\right)}\big\|_{2}^{2}+\sigma^{2}\big\|\bm{\theta}^{\tau, \left(m\right)}\big\|_{\infty}^{2}\log^{5}d\\
 & \lesssim\sigma^{2}rdp\log d+\frac{\sigma^{2}\mu r\log^{6}d}{d} \asymp\sigma^{2}rdp\log d,
\end{align*}
where the second line comes from (\ref{eq:theta_loo_inf_norm}) and \eqref{eq:theta_loo_2_norm},
and the last inequality holds as long as $p\gg\mu d^{-2}\log^{5}d$.
These together with (\ref{eq:M1m-mj-formula}) allow us to obtain
\begin{align}
\left\Vert \big(\widehat{\bm{M}}^{\tau, \left(m\right)}-\bm{M}^{\tau, \left(m\right)}\big)_{m,:}\right\Vert _{2}^{2}\lesssim\frac{\mu^{2}r^{2}\lambda_{\max}^{\star2}\log d}{d^{2}p}+\frac{\sigma^{2}rd\log d}{p}.\label{M_loo_row_l2_loss}
\end{align}
Clearly, this bound is also valid for $\sum_{i:i\neq m}\big\{\big(\widehat{\bm{M}}^{\tau, \left(m\right)}-\bm{M}^{\tau, \left(m\right)}\big)_{i,m}\big\}^{2}$,
namely,
\begin{align}
\sum_{i:i\neq m}\left\{ \big(\widehat{\bm{M}}^{\tau, \left(m\right)}-\bm{M}^{\tau, \left(m\right)}\big)_{i,m}\right\} ^{2}\lesssim\frac{\mu^{2}r^{2}\lambda_{\max}^{\star2}\log d}{d^{2}p}+\frac{\sigma^{2}rd\log d}{p}.\label{M_loo_col_l2_loss}
\end{align}
\item When it comes to the remaining entries of $\widehat{\bm{M}}^{\tau, \left(m\right)}-\bm{M}^{\tau, \left(m\right)}$, by the fact that the spectral norm of a submatrix is always less than or equal to that of the whole matrix,
%we can invoke Lemma~\ref{lemma:sum_ij_square_cont_ineq}, Lemma~\ref{lemma:sum_ij_square_cont_ineq_subg}
%and (\ref{eq:theta_loo_inf_norm}) to demonstrate that: with probability
%$1-O\left(d^{-11}\right)$,
%\begin{align*}
% & \sum_{i,j:i\neq m,j\neq m}\left\{ \big(\widehat{\bm{M}}^{\tau, \left(m\right)}-\bm{M}^{\tau, \left(m\right)}\big)_{i,j}\right\} ^{2}\\
% & \qquad\lesssim\big(\theta_{m}^{\tau, \left(m\right)}\big)^{2}\sum_{i,j:\,i\neq m,j\neq m}\Big(T_{i,j,m}^{\star2}(p^{-1}\chi_{i,j,m}-1 )^{2}+(p^{-1}E_{i,j,m}\chi_{i,j,m})^{2}\Big)\\
% & \qquad\lesssim\left\{ \frac{\mu r\lambda_{\max}^{\star2}}{dp}+\frac{\sigma^{2}d^{2}}{p}\right\} \big\|\bm{\theta}^{\tau, \left(m\right)}\big\|_{\infty}^{2}\lesssim\frac{\mu^{2}r^{2}\lambda_{\max}^{\star2}\log d}{d^{2}p}+\frac{\sigma^{2}\mu rd\log d}{p}.
%\end{align*}
applying the matrix Bernstein inequality gives that with probability
$1-O\left(d^{-11}\right)$,
\begin{align*}
\Big\| \Big[ \big(\widehat{\bm{M}}^{\tau, \left(m\right)}-\bm{M}^{\tau, \left(m\right)}\big)_{i,j}  \Big]_{i, j\neq m}  \Big\| & \lesssim  \big\|\bm{\theta}^{\tau, \left(m\right)}\big\|_{\infty} \( \frac{\log d}{p} \norm{\bm{A}^\star}_\infty + \sqrt{\frac{\log d}{p}} \norm{\bm{A}^{\star \top}}_{2,\infty} + \frac{\sigma \log^2 d}{p} + \sigma \sqrt{\frac{d \log d}{p}}  \) \\
& \lesssim \sqrt{\frac{\mu r \log d}{d}} \( \frac{\sqrt{\mu r} \, \lambda_{\max}^\star \log d}{d^{3/2} p } + \frac{\mu \sqrt{r} \, \lambda_{\max}^\star \sqrt{\log d}}{d \sqrt{p}} + \frac{\sigma \log^2 d}{p} + \sigma \sqrt{\frac{d \log d}{p}} \)  \\
& \lesssim \frac{\mu r\lambda_{\max}^{\star}\sqrt{\log d}}{d\sqrt{p}}+\sigma\sqrt{\frac{rd\log d}{p}},
\end{align*}
as long as our sample size and rank condition holds.

\item Putting the preceding bounds together yields
\begin{equation}
\big\|\widehat{\bm{M}}^{\tau, \left(m\right)}-\bm{M}^{\tau, \left(m\right)}\big\| \lesssim\frac{\mu r\lambda_{\max}^{\star}\sqrt{\log d}}{d\sqrt{p}}+\sigma\sqrt{\frac{rd\log d}{p}}.
\end{equation}
\end{itemize}

\subsubsection{\texorpdfstring{Proximity of $\bm{M}^{\tau}$ and $\widehat{\bm{M}}^{\tau, \left(m\right)}$}{the truth}}

Recall the definitions of $\bm{M}^{\tau}$ and $\widehat{\bm{M}}^{\tau, \left(m\right)}$
in (\ref{eq:defn-Mtau}) and (\ref{eq:projected-matrix-M-tau_g}), respectively. From
the definition of the operator norm and the triangle inequality, we
have
\begin{align}
\big\|\bm{M}^{\tau}-\widehat{\bm{M}}^{\tau, \left(m\right)}\big\| & \leq \underbrace{\big\|\bm{T}^{\star}\times_{3}\big(\bm{\theta}^{\tau}-\bm{\theta}^{\tau, \left(m\right)}\big)\big\|}_{=: \, \alpha_1}
+ \underbrace{\big\|\left(p^{-1}\bm{T}-\bm{T}^{\star}\right)\times_{3}\big(\bm{\theta}^{\tau}-\bm{\theta}^{\tau, \left(m\right)}\big)\big\|}_{=: \, \alpha_2}.\label{eq:M-theta-relation}
\end{align}
\begin{itemize}
\item To control $\alpha_1$, we can express
\begin{align}
\bm{T}^{\star}\times_{3}\big(\bm{\theta}^{\tau}-\bm{\theta}^{\tau, \left(m\right)}\big) 
& =\sum_{s\in\left[r\right]}\lambda_{s}^{\star}\big\langle\overline{\bm{u}}_{s}^{\star},\bm{\theta}^{\tau}-\bm{\theta}^{\tau, \left(m\right)}\big\rangle\,\overline{\bm{u}}_{s}^{\star}\overline{\bm{u}}_{s}^{\star\top}.
\end{align}

As shown in (\ref{eq:utrue_theta_diff}),
with probability at least $1-O\left(d^{-12}\right)$,
\begin{align*}
\big|\big\langle\overline{\bm{u}}_{s}^{\star},\bm{\theta}^{\tau, \left(m\right)}-\bm{\theta}^{\tau}\big\rangle\big|\lesssim\big\|\bm{U}\bm{U}^{\top}-\bm{U}^{(m)}\bm{U}^{\left(m\right)\top}\big\|\sqrt{\log d}.
\end{align*}

Consequently, we know from Lemma~\ref{lemma:U_U_loo_dist_op} that with probability at least $1-O\( d^{-10} \)$,
\begin{align}
\big\|\bm{T}^{\star}\times_{3}\big(\bm{\theta}^{\tau}-\bm{\theta}^{\tau, \left(m\right)}\big)\big\| & \leq\max_{s\in[r]}\big|\lambda_{s}^{\star}\big\langle\overline{\bm{u}}_{s}^{\star},\bm{\theta}^{\tau}-\bm{\theta}^{\tau, \left(m\right)}\big\rangle\big|\big\|\overline{\bm{U}}^{\star}\big\|^{2} \lesssim\big\|\bm{U}\bm{U}^{\top}-\bm{U}^{(m)}\bm{U}^{\left(m\right)\top}\big\|\sqrt{\log d}\,\lambda_{\max}^{\star} \nonumber \\
&  \lesssim \EE_\loo \sqrt{\frac{\mu r \log d}{d}}\,\lambda_{\max}^{\star} , \label{eq:Ttrue-theta-loo-diff-2norm}
\end{align}
where we use the fact that $\big\|\overline{\bm{U}}^{\star}\big\|\lesssim1$
if $r\sqrt{\mu/d}\ll1$. 
\item When it comes to $\alpha_2$, combining (\ref{eq:T_op_loss})
and (\ref{eq:g_loo_l2_loss}) with our sample size, noise and rank conditions, one has 
\begin{align}
\big\|\left(p^{-1}\bm{T}-\bm{T}^{\star}\right)\times_{3}\big(\bm{\theta}^{\tau}-\bm{\theta}^{\tau, \left(m\right)}\big)\big\| & \leq\left\Vert p^{-1}\bm{T}-\bm{T}^{\star}\right\Vert \big\|\bm{\theta}^{\tau}-\bm{\theta}^{\tau, \left(m\right)}\big\|_{2}\ll\lambda_{\max}^{\star}\big\|\bm{\theta}^{\tau}-\bm{\theta}^{\tau, \left(m\right)}\big\|_{2} \nonumber \\
 & \lesssim\big\|\bm{U}\bm{U}^{\top}-\bm{U}^{(m)}\bm{U}^{\left(m\right)\top}\big\|_{\mathrm{F}}\sqrt{\log d}\,\lambda_{\max}^{\star} \nonumber \\
 & \lesssim \EE_\loo \sqrt{\frac{\mu r \log d}{d}} \,\lambda_{\max}^{\star}. \label{eq:T-Ttrue-theta-loo-diff-2norm}
\end{align}
\item Combining (\ref{eq:M-theta-relation}), (\ref{eq:Ttrue-theta-loo-diff-2norm}) and (\ref{eq:T-Ttrue-theta-loo-diff-2norm}), we conclude that
\begin{align}
\big\|\bm{M}^{\tau}-\widehat{\bm{M}}^{\tau, \left(m\right)}\big\| & \lesssim \EE_\loo \sqrt{\frac{\mu r \log d}{d}}\,\lambda_{\max}^{\star} . \nonumber
% & \lesssim\frac{\left(\mu r\right)^{3/2}\lambda_{\max}^{\star3}\log^{3/2}d}{d^{2}p}+\frac{\mu r\lambda_{\max}^{\star3}\log d}{d^{3/2}\sqrt{p}}+\frac{\sigma^{2}}{\lambda_{\min}^{\star3}}\frac{\sqrt{\mu r}\,d\log^{3/2}d}{p}+\sigma\sqrt{\frac{\mu r\log^{2}d}{p}}.
\end{align}
\end{itemize}

\subsection{Proof of Lemma~\ref{lemma:init_loo_l2_loss}}

\label{pf:init_loo_l2_loss}

\begin{itemize}

\item

We start with the first claim regarding $\|\overline{\bm{u}}^{\tau}-\bm{\nu}^{\tau, \left(m\right)}\|_{2}$,
or equivalently, $\|\overline{\bm{u}}^{\tau}-\overline{\bm{u}}^{\tau, \left(m\right)}\|_{2}$ (as argued in the proof of Lemma~\ref{lemma:init_loo_entry_loss}).
By the triangle inequality, we can upper bound the following two terms
separately: 
\begin{align}
\big\|\overline{\bm{u}}^{\tau}-\overline{\bm{u}}^{\tau, \left(m\right)}\big\|_{2}\leq\underbrace{\big\|\overline{\bm{u}}^{\tau}-\widehat{\bm{u}}^{\tau, \left(m\right)}\big\|_{2}}_{=: \, \beta_{1}}+\underbrace{\big\|\widehat{\bm{u}}^{\tau, \left(m\right)}-\overline{\bm{u}}^{\tau, \left(m\right)}\big\|_{2}}_{=: \, \beta_{2}}.\label{eq:defn-beta1-beta2-300}
\end{align}
Here, we remind the reader that $\widehat{\bm{u}}^{\tau, \left(m\right)}$
is the top left singular vector of $\widehat{\bm{M}}^{\tau, \left(m\right)}$
(see \eqref{eq:projected-matrix-M-tau_g}) obeying $\big\langle\widehat{\bm{u}}^{\tau, \left(m\right)}, \overline{\bm{u}}_{1}^{\star}\big\rangle\geq0$.

\begin{itemize}

\item The first term $\beta_{1}$ shall be bounded via Wedin's theorem.
From \eqref{eq:sigval_M_hat} and \eqref{eq:sigval_M_star_gap}, we have 
\begin{align}
\sigma_{1}\big(\bm{M}^{\tau}\big)-\sigma_{2}(\bm{M}^{\tau}) & \geq\sigma_{1}(\bm{M}^{\star\tau})-\sigma_{2}(\bm{M}^{\star\tau})-2 \, \big\Vert \bm{M}^{\tau}-\bm{M}^{\star\tau}\big\Vert \gtrsim \lambda_{\min}^{\star}.
\label{eq:sigval_M_hat_gap}
\end{align}
Combined with Lemma~\ref{lemma:M_loo_op_loss}, one has
\begin{align*}
\sigma_{1}\big(\bm{M}^{\tau}\big)-\sigma_{2}(\bm{M}^{\tau})-\big\|\bm{M}^{\tau}-\widehat{\bm{M}}^{\tau, \left(m\right)}\big\| \gtrsim \lambda_{\min}^\star.
\end{align*}
Note that we have already shown in the proof of Lemma~\ref{pf:init_2_loss}
and Lemma~\ref{pf:init_loo_entry_loss} that $\big\|\overline{\bm{u}}^{\tau}-\overline{\bm{u}}_{1}^{\star}\big\|_{2}=o(1)$
and $\big\|\widehat{\bm{u}}^{\tau, \left(m\right)}-\overline{\bm{u}}_{1}^{\star}\big\|_{2}=o(1)$,
which implies that $\bm{u}^{\tau}$ and $\widehat{\bm{u}}^{\tau, \left(m\right)}$
are positively correlated. 
%Meanwhile, the same argument reveals that$\bm{v}^{\tau}$ and $\widehat{\bm{v}}^{\tau, \left(m\right)}$ are also
%positively correlated. 
Thus, one can invoke Wedin's theorem and use
the bound \eqref{eq:M_hat_M_hat_loo_op_norm} to reach
\begin{align}
 \big\|\overline{\bm{u}}^{\tau}-\widehat{\bm{u}}^{\tau, \left(m\right)}\big\|_{2} & \leq\frac{\big\|\bm{M}^{\tau}-\widehat{\bm{M}}^{\tau, \left(m\right)}\big\|}{\sigma_{1}\big(\bm{M}^{\tau}\big)-\sigma_{2}(\bm{M}^{\tau})-\big\|\bm{M}^{\tau}-\widehat{\bm{M}}^{\tau, \left(m\right)}\big\|} \lesssim \frac{1}{\lambda_{\min}^\star} \big\|\bm{M}^{\tau}-\widehat{\bm{M}}^{\tau, \left(m\right)}\big\| \nonumber \\
& \lesssim \EE_\loo \sqrt{\frac{\mu r \log d}{d}} \leq  \EE_\loo \sqrt{r \log d} \, \max \big\{ \sqrt{\mu / d}, \, \big\Vert \overline{\bm{u}}^{\tau}\big\Vert _{\infty} \big\} .\label{u_u_hat_loo_2_norm}
% & \lesssim\Bigg\{\frac{\mu r^{3/2}\log^{3/2}d}{d^{3/2}p}+\frac{\sqrt{\mu}\,r\log d}{d\sqrt{p}}+\frac{\sigma^{2}}{\lambda_{\min}^{\star2}.}\frac{\sqrt{r}\,d^{3/2}\log^{3/2}d}{p}+\frac{\sigma}{\lambda_{\min}}\sqrt{\frac{rd\log^{2}d}{p}}\Bigg\}\sqrt{\frac{\mu}{d}}\nonumber \\
% & \leq\Bigg\{\frac{\mu r^{3/2}\log^{3/2}d}{d^{3/2}p}+\frac{\sqrt{\mu}\,r\log d}{d\sqrt{p}}+\frac{\sigma^{2}}{\lambda_{\min}^{\star2}.}\frac{\sqrt{r}\,d^{3/2}\log^{3/2}d}{p}+\frac{\sigma}{\lambda_{\min}}\sqrt{\frac{rd\log^{2}d}{p}}\Bigg\}\Big(\sqrt{\frac{\mu}{d}}+\big\Vert \overline{\bm{u}}^{\tau}\big\Vert _{\infty}+\big\Vert \overline{\bm{v}}^{\tau}\big\Vert _{\infty}\Big).\label{u_u_hat_loo_2_norm}
\end{align}
In addition to this bound on $\beta_{1}$, we also make note of the
following simple bound
\begin{align}
 \big\|\widehat{\bm{u}}^{\tau, \left(m\right)}\big\|_{\infty} & \leq\big\|\overline{\bm{u}}^{\tau}-\widehat{\bm{u}}^{\tau, \left(m\right)}\big\|_{\infty}+\big\Vert \overline{\bm{u}}^{\tau}\big\Vert _{\infty}\leq\big\|\overline{\bm{u}}^{\tau}-\widehat{\bm{u}}^{\tau, \left(m\right)}\big\|_{2}+\big\Vert \overline{\bm{u}}^{\tau}\big\Vert _{\infty} \lesssim \max \big\{ \sqrt{\mu / d}, \, \big\Vert \overline{\bm{u}}^{\tau}\big\Vert _{\infty} \big\},\label{u_hat_loo_inf_norm}
\end{align}
where the last inequality follows from our sample size, noise and rank condition
that $\EE_\loo \sqrt{r \log d} \ll 1$.

\item The second term $\beta_{2}$ is also controlled via Wedin's theorem:
\begin{align}
 \big\|\widehat{\bm{u}}^{\tau, \left(m\right)}-\overline{\bm{u}}^{\tau, \left(m\right)}\big\|_{2}
 \leq\frac{  \big\|\big(\widehat{\bm{M}}^{\tau, \left(m\right)}-\bm{M}^{\tau, \left(m\right)}\big)\overline{\bm{u}}^{\tau, \left(m\right)}\big\|_{2} }{\sigma_{1}\big(\bm{M}^{\tau, \left(m\right)}\big)-\sigma_{2}\big(\bm{M}^{\tau, \left(m\right)}\big)-\big\|\widehat{\bm{M}}^{\tau, \left(m\right)}-\bm{M}^{\tau, \left(m\right)}\big\|}.\label{eq:Wedin-200}
\end{align}
The denominator term is easy to handle. With \eqref{M_hat_loo_M_hat_op_loss} and \eqref{eq:sigval_M_hat_gap}
in mind, we can apply Weyl's inequality to obtain
\begin{align}
\sigma_{1}\big(\bm{M}^{\tau, \left(m\right)}\big)-\sigma_{2}\big(\bm{M}^{\tau, \left(m\right)}\big) & \geq\sigma_{1}\big(\bm{M}^{\tau}\big)-\sigma_{2}(\bm{M}^{\tau})-2\,\big\|\bm{M}^{\tau}-\bm{M}^{\tau, \left(m\right)}\big\|\gtrsim\lambda_{\min}^{\star}.
\label{eq:sigval_M_hat_loo_gap}
\end{align}
From Lemma~\ref{lemma:M_loo_op_loss}, one has $\big\|\bm{M}^{\tau}-\bm{M}^{\tau, \left(m\right)}\big\|\ll\lambda_{\min}^{\star}$.
Therefore, we know that
\begin{align*}
\sigma_{1}\big(\bm{M}^{\tau, \left(m\right)}\big)-\sigma_{2}\big(\bm{M}^{\tau, \left(m\right)}\big)-\big\|\widehat{\bm{M}}^{\tau, \left(m\right)}-\bm{M}^{\tau, \left(m\right)}\big\|\gtrsim\lambda_{\min}^{\star}.
\end{align*}
In addition, Lemma~\ref{lemma:M_hat_M_loo_u_2_norm} below develops an upper bound on the numerator term:
\begin{lemma}
\label{lemma:M_hat_M_loo_u_2_norm}
Instate the assumptions of Lemma~\ref{lemma:init_loo_l2_loss}. With probability at least $1-O\( d^{-10} \)$, one has
\begin{align}
\big\|\big(\widehat{\bm{M}}^{\tau, \left(m\right)}  -\bm{M}^{\tau, \left(m\right)}\big)\overline{\bm{u}}^{\tau, \left(m\right)}\big\|_{2} \lesssim \Bigg\{\frac{\mu r\lambda_{\max}^{\star}\log d}{d\sqrt{p}}+\sigma\sqrt{\frac{rd\log^{2}d}{p}}\Bigg\} \big\|\overline{\bm{u}}^{\tau, \left(m\right)}\big\|_{\infty}.
\end{align}
\end{lemma}
\begin{proof}
See Appendix~\ref{subsection:pf:M_hat_M_loo_u_2_norm}.
\end{proof}

\item Substitution of the above bounds into \eqref{eq:Wedin-200} yields
\begin{align*}
 \big\|\widehat{\bm{u}}^{\tau, \left(m\right)}-\overline{\bm{u}}^{\tau, \left(m\right)}\big\|_{2} \lesssim\Bigg\{\frac{\mu r \log d}{d\sqrt{p}}+\frac{\sigma}{\lambda_{\min}^\star}\sqrt{\frac{ rd\log^{2}d}{p}}\Bigg\} \big\|\overline{\bm{u}}^{\tau, \left(m\right)}\big\|_{\infty} \ll \big\|\overline{\bm{u}}^{\tau, \left(m\right)}\big\|_{\infty}, 
\end{align*}
where the last step holds as long as $p\gg\mu^2 r^{2}d^{- 2}\log^{2}d$
and $\sigma/\lambda_{\min}^{\star}\ll\sqrt{p /( rd \log^{2}d)}$.
In addition, from (\ref{u_hat_loo_inf_norm}), we observe that
\begin{align*}
\begin{split}\big\|\overline{\bm{u}}^{\tau, \left(m\right)}\big\|_{\infty} & \leq\big\|\widehat{\bm{u}}^{\tau, \left(m\right)}-\overline{\bm{u}}^{\tau, \left(m\right)}\big\|_{\infty}+\big\|\widehat{\bm{u}}^{\tau, \left(m\right)}\big\|_{\infty}\leq\big\|\widehat{\bm{u}}^{\tau, \left(m\right)}-\overline{\bm{u}}^{\tau, \left(m\right)}\big\|_{2}+\big\|\widehat{\bm{u}}^{\tau, \left(m\right)}\big\|_{\infty}\\
 & \lesssim o \(1\) \big\|\overline{\bm{u}}^{\tau, \left(m\right)}\big\|_{\infty} + \big\|\widehat{\bm{u}}^{\tau, \left(m\right)}\big\|_{\infty}
   \leq o\(1\) \big\|\overline{\bm{u}}^{\tau, \left(m\right)}\big\|_{\infty} + \max \big\{ \sqrt{\mu / d}, \, \big\Vert \overline{\bm{u}}^{\tau}\big\Vert _{\infty} \big\},
\end{split}
\end{align*}
from which we can deduce that
\begin{align*}
\big\|\overline{\bm{u}}^{\tau, \left(m\right)}\big\|_{\infty} \lesssim \max \big\{ \sqrt{\mu / d}, \, \big\Vert \overline{\bm{u}}^{\tau}\big\Vert _{\infty} \big\}.
\end{align*}
 
As a consequence, one immediately obtains
\begin{align}
 \big\|\widehat{\bm{u}}^{\tau, \left(m\right)}-\overline{\bm{u}}^{\tau, \left(m\right)}\big\|_{2}\lesssim\Bigg\{\frac{\mu r \log d}{d\sqrt{p}}+\frac{\sigma}{\lambda_{\min}^\star }\sqrt{\frac{ rd\log^{2}d}{p}}\Bigg\}\max \big\{ \sqrt{\mu / d}, \, \big\Vert \overline{\bm{u}}^{\tau}\big\Vert _{\infty} \big\} .\label{u_hat_loo_u_tilde_loo_2_norm}
\end{align}

\item Combining \eqref{eq:defn-beta1-beta2-300}, \eqref{u_u_hat_loo_2_norm}
and \eqref{u_hat_loo_u_tilde_loo_2_norm} and the definition of $\EE_\loo$, we arrive at
\begin{align}
\big\|\overline{\bm{u}}^{\tau} & -\overline{\bm{u}}^{\tau, \left(m\right)}\big\|_{2} \lesssim \EE_\loo \sqrt{r \log d} \, \max \big\{ \sqrt{\mu / d}, \, \big\Vert \overline{\bm{u}}^{\tau}\big\Vert _{\infty} \big\}.
\label{u_loo_l2_loss}
\end{align}
Comparing this bound with the first claim of the lemma, we see that
the claim can be established as long as we can show that 
\begin{align}
  \big\Vert \overline{\bm{u}}^{\tau}\big\Vert _{\infty}\lesssim \sqrt{\mu / d}.\label{init_inf_norm_ub}
\end{align}
To justify this bound \eqref{init_inf_norm_ub}, we make use of Lemma~\ref{lemma:init_loo_entry_loss}
to derive that 
\begin{align*}
\left|\big(\overline{\bm{u}}^{\tau}-\overline{\bm{u}}_{1}^{\star}\big)_{m}\right| & \leq\left|\big(\overline{\bm{u}}^{\tau}-\overline{\bm{u}}^{\tau, \left(m\right)}\big)_{m}\right|+\left|\big(\overline{\bm{u}}^{\tau, \left(m\right)}-\overline{\bm{u}}_{1}^{\star}\big)_{m}\right| \\ 
& \leq\big\|\overline{\bm{u}}^{\tau}-\overline{\bm{u}}^{\tau, \left(m\right)}\big\|_{2}+\left|\big(\overline{\bm{u}}^{\tau, \left(m\right)}-\overline{\bm{u}}_{1}^{\star}\big)_{m}\right|\\
& \lesssim (\EE_\loo + \EE_\op ) \sqrt{r \log d} \, \max \big\{ \sqrt{\mu / d}, \, \big\Vert \overline{\bm{u}}^{\tau}\big\Vert _{\infty} \big\}
% & \lesssim\Bigg\{\frac{\mu r^{3/2}\log^{7/2}d}{d^{3/2}p}+\frac{\mu r^{3/2}\log^{3}d}{d\sqrt{p}}+\frac{\sigma^{2}}{\lambda_{\min}^{\star}}\frac{\sqrt{r}\,d^{3/2}\log^{3/2}d}{p}\\
% & \quad+\frac{\sigma}{\lambda_{\min}^{\star}}\sqrt{\frac{r^{2}d\log^{6}d}{p}}+\sqrt{\frac{\mu r^{2}\log^{2}d}{d}}\Bigg\}\Big(\sqrt{\frac{\mu}{d}}+\big\Vert \overline{\bm{u}}^{\tau}\big\Vert _{\infty}+\big\Vert \overline{\bm{v}}^{\tau}\big\Vert _{\infty}\Big)
\end{align*}
for each $m\in[d]$. Maximizing over $m \in [d]$ gives that 
\begin{align}
\big\|\overline{\bm{u}}^{\tau}-\overline{\bm{u}}_{1}^{\star}\big\|_{\infty} & \lesssim \( \EE_\loo + \EE_\op \) \sqrt{r \log d} \, \max \big\{ \sqrt{\mu / d}, \, \big\Vert \overline{\bm{u}}^{\tau}\big\Vert _{\infty} \big\}  \label{init_inf_loss_temp}\\
& \ll  \max \big\{ \sqrt{\mu / d}, \, \big\Vert \overline{\bm{u}}^{\tau}\big\Vert _{\infty} \big\},
%& \lesssim\Bigg\{\frac{\mu r^{3/2}\log^{7/2}d}{d^{3/2}p}+\frac{\mu r^{3/2}\log^{3}d}{d\sqrt{p}}+\frac{\sigma^{2}}{\lambda_{\min}^{\star}}\frac{\sqrt{r}\,d^{3/2}\log^{3/2}d}{p}\nonumber \\
% & \quad+\frac{\sigma}{\lambda_{\min}^{\star}}\sqrt{\frac{r^{2}d\log^{6}d}{p}}+\sqrt{\frac{\mu r^{2}\log^{2}d}{d}}\Bigg\}\Big(\sqrt{\frac{\mu}{d}}+\big\Vert \overline{\bm{u}}^{\tau}\big\Vert _{\infty}+\big\Vert \overline{\bm{v}}^{\tau}\big\Vert _{\infty}\Big),
\end{align}
where we use the condition that $\( \EE_\loo + \EE_\op \) \sqrt{r \log d} \ll 1$. Apply the triangle inequality to yield
\begin{align*}
\big\|\overline{\bm{u}}^{\tau}\big\|_{\infty} & \leq\big\|\overline{\bm{u}}^{\tau}-\overline{\bm{u}}_{1}^{\star}\big\|_{\infty}+\big\|\overline{\bm{u}}_{1}^{\star}\big\|_{\infty}\leq o \(1\)  \big\Vert \overline{\bm{u}}^{\tau}\big\Vert _{\infty} +\sqrt{\mu / d}.
\end{align*}
These allow us to establish the claim \eqref{init_inf_norm_ub}, which
in turn finishes the proof for the first claim of this lemma.

\end{itemize}

\item
The second claim \eqref{claim:init_inf_loss} of this lemma follows immediately from \eqref{init_inf_norm_ub}
and \eqref{init_inf_loss_temp}.

\item 
It remains to prove the last claim \eqref{claim:init_loo_coeff_loss}. Recall the definition of $\lambda_\tau$ and $\lambda_\tau^{\m}$ in \eqref{def:lambda_tau}. We can decompose
\begin{align}
 \big\langle p^{-1}\bm{T}^{\m},(\overline{\bm{u}}^{\tau, \m})^{\otimes3}\big\rangle - \big\langle p^{-1}\bm{T},(\overline{\bm{u}}^{\tau})^{\otimes3}\big\rangle & =\underbrace{\big\langle p^{-1} \bm{T}^{\m} - p^{-1} \bm{T}, (\overline{\bm{u}}^{\tau, \m})^{\otimes3} \big\rangle}_{=:\,\beta_{1}} \nonumber \\
 & \quad+\underbrace{\big\langle p^{-1} \bm{T},(\overline{\bm{u}}^{\tau, \m})^{\otimes3} - (\overline{\bm{u}}^{\tau})^{\otimes3}\big\rangle}_{=:\,\beta_{2}} \label{eq:coeff-loo-decomp}
\end{align}
In what follows, we will control $\beta_1$ and $\beta_2$ seperately.

 For $\beta_1$, we note that all non-zero entries of $ \bm{T}^{\m} - \bm{T}$ are located in the $m$th slices, and are independent of $\overline{\bm{u}}^{\tau, \m}$. This type of quantities have appeared many times and we omit the detailed proof for conciseness. By the Bernstein inequality, one can show that with probability at least $1-O\( d^{-10} \)$,
\begin{align}
\Big| \big\langle p^{-1} \bm{T}^{\m} - p^{-1} \bm{T}, (\overline{\bm{u}}^{\tau, \m})^{\otimes3} \big\rangle \Big| \lesssim \sqrt{\frac{\mu r \log d}{d^2 p}} \sqrt{\frac{\mu}{d}} \, \lambda_{\max}^\star. \label{eq:coeff-loo-decomp-beta1}
\end{align}

Next, we turn to $\beta_2$. From our sample size and noise condition, Lemma~\ref{lemma:incoh} and Corollary~\ref{cor:bound-PE} demonstrates that with probability at least $1-O \( d^{-10} \)$,
\begin{align*}
\left\| p^{-1} \bm{T}\right\| \leq \left\| p^{-1} \bm{T} - \bm{T}^\star\right\| + \left\| \bm{T}^\star\right\| \leq \left\| p^{-1} \bm{T} - \bm{T}^\star\right\| + \left\| \bm{A}^\star\right\| \lesssim \lambda_{\max}^\star.
\end{align*}
where we use the fact that the tensor spectral norm is always less
than or equal to that of its matricization. By the definition of the operator norm, one has
\begin{align}
\left| \beta_2 \right| & \leq 3 \, \big|  \big\langle p^{-1} \bm{T}, (\overline{\bm{u}}^{\tau, \m})^{\otimes2} \otimes \big(\overline{\bm{u}}^{\tau, \m} - \overline{\bm{u}}^{\tau} \big) \big\rangle \big| + 3 \, \big| \big\langle p^{-1} \bm{T}, (\overline{\bm{u}}^{\tau, \m}) \otimes \big(\overline{\bm{u}}^{\tau, \m} - \overline{\bm{u}}^{\tau} \big)^{\otimes 2} \big\rangle \big| \nonumber \\
& \quad + \big| \big\langle p^{-1} \bm{T}, \big(\overline{\bm{u}}^{\tau, \m} - \overline{\bm{u}}^{\tau} \big)^{\otimes 3} \big\rangle \big|  \nonumber \\
& \lesssim \left\| p^{-1} \bm{T}\right\| \Big( \big\|\overline{\bm{u}}^{\tau, \m} - \overline{\bm{u}}^{\tau}\big\|_2 + \big\|\overline{\bm{u}}^{\tau, \m} - \overline{\bm{u}}^{\tau}\big\|_2^2 + \big\|\overline{\bm{u}}^{\tau, \m} - \overline{\bm{u}}^{\tau}\big\|_2^3 \Big) \nonumber \\
& \lesssim \EE_\loo \sqrt{\frac{\mu r \log d}{d}} \, \lambda_{\max}^\star \label{eq:coeff-loo-decomp-beta2}
\end{align}
where we use \eqref{u_loo_l2_loss} and \eqref{init_inf_norm_ub} in the last step and the fact that $\EE_\loo \sqrt{\mu r \log d / d} \ll 1$.

Combining \eqref{eq:coeff-loo-decomp-beta1} and \eqref{eq:coeff-loo-decomp-beta2} immediately establishes \eqref{claim:init_loo_coeff_loss}.

\end{itemize}

\subsection{Proof of Lemma~\ref{lemma:M_hat_M_loo_u_2_norm}}
\label{subsection:pf:M_hat_M_loo_u_2_norm}

Recalling the definitions of $\widehat{\bm{M}}^{\tau, \left(m\right)}$ and $\bm{M}^{\tau, \left(m\right)}$ in \eqref{eq:projected-matrix-M-tau_g} and \eqref{eq:projected-matrix-M-tau}, respectively, we observe that $\widehat{\bm{M}}^{\tau, \left(m\right)}-\bm{M}^{\tau, \left(m\right)}$
is independent of $\overline{\bm{u}}^{\tau, \left(m\right)}$ conditional
on $\mathcal{P}_{\Omega_{-m}}\left(\bm{E}\right)$ and $\bm{g}$. 
%In what follows, we will upper bound the size of $\big(\widehat{\bm{M}}^{\tau, \left(m\right)}-\bm{M}^{\tau, \left(m\right)}\big)\overline{\bm{u}}^{\tau, \left(m\right)}$.
% and the same analysis can be applied to $\overline{\bm{u}}^{\tau, \left(m\right) \top} \big(\widehat{\bm{M}}^{\tau, \left(m\right)}-\bm{M}^{\tau, \left(m\right)}\big)$
\begin{itemize}
\item The $m$-th entry of $\big(\widehat{\bm{M}}^{\tau, \left(m\right)}-\bm{M}^{\tau, \left(m\right)}\big)\overline{\bm{u}}^{\tau, \left(m\right)}$
can be written as
\begin{align}
 & \big(\widehat{\bm{M}}^{\tau, \left(m\right)}-\bm{M}^{\tau, \left(m\right)}\big)_{m,:}\overline{\bm{u}}^{\tau, \left(m\right)}=\big(p^{-1}\big(\bm{T}-\bm{T}^{(m)}\big)\times_{3}\bm{\theta}^{\tau,(m)}\big)_{m,:}\overline{\bm{u}}^{\tau, \left(m\right)}\nonumber \\
 & \qquad=\underbrace{\sum_{j,k\in[d]}T_{m,j,k}^{\star}\overline{u}_{j}^{\tau, \left(m\right)}\theta_{k}^{\tau, \left(m\right)}(p^{-1}\chi_{m,j,k}-1)}_{=: \alpha_1}+\underbrace{\sum_{j,k\in[d]} \overline{u}_{j}^{\tau, \left(m\right)}\theta_{k}^{\tau, \left(m\right)} p^{-1} E_{m,j,k}\chi_{m,j,k}}_{=: \alpha_2}.\label{eq:M1m-split-100}
\end{align}

\begin{itemize}
\item For the first term $\alpha_1$, it is easily seen from (\ref{eq:theta_loo_inf_norm}) and incoherence conditions
that
\begin{align*}
L_{1}:=\max_{j,k\in[d]}\left|T_{m,j,k}^{\star}\overline{u}_{j}^{\tau, \left(m\right)}\theta_{k}^{\tau, \left(m\right)}(p^{-1}\chi_{m,j,k}-1)\right| & \leq p^{-1}\left\Vert \bm{T}^{\star}\right\Vert _{\infty}\big\|\overline{\bm{u}}^{\tau, \left(m\right)}\big\|_{\infty}\big\|\bm{\theta}^{\tau, \left(m\right)}\big\|_{\infty}\\
 & \lesssim\frac{\mu r\lambda_{\max}^{\star}\sqrt{\log d}}{d^{2}p}\big\|\overline{\bm{u}}^{\tau, \left(m\right)}\big\|_{\infty};
\end{align*}
and
\begin{align*}
V_{1}:=\sum_{j,k\in[d]}T_{m,j,k}^{\star2}\big(\overline{u}_{j}^{\tau, \left(m\right)}\big)^{2}\big(\theta_{k}^{\tau, \left(m\right)}\big)^{2}\mathbb{E}\big[(p^{-1}\chi_{m,j,k}-1)^{2}\big] & \leq p^{-1}\left\Vert \bm{A}^{\star}\right\Vert _{2,\infty}^{2}\big\|\overline{\bm{u}}^{\tau, \left(m\right)}\big\|_{\infty}^{2}\big\|\bm{\theta}^{\tau, \left(m\right)}\big\|_{\infty}^{2}\\
 & \lesssim\frac{\mu^{2}r^{2}\lambda_{\max}^{\star2}\log d}{d^{2}p}\big\|\overline{\bm{u}}^{\tau, \left(m\right)}\big\|_{\infty}^{2}.
\end{align*}
Apply the Bernstein inequality to yield that with probability at least
$1-O\left(d^{-11}\right)$,
\begin{align*}
\sum_{j,k\in[d]}T_{m,j,k}^{\star}\overline{u}_{j}^{\tau, \left(m\right)}\theta_{k}^{\tau, \left(m\right)}(p^{-1}\chi_{m,j,k}-1)\lesssim\sqrt{V_{1}\log d}+L_{1}\log d\lesssim\frac{\mu r\lambda_{\max}^{\star}\log d}{d\sqrt{p}}\big\|\overline{\bm{u}}^{\tau, \left(m\right)}\big\|_{\infty},
\end{align*}
where the last inequality holds as long as
$p\gg d^{-2}\log^{2}d$. 
\item Regarding $\alpha_2$ (cf.~\eqref{eq:M1m-split-100}),
it is straightforward to compute that
\begin{align*}
L_{2} & :=\big\Vert p^{-1}\overline{u}_{j}^{\tau, \left(m\right)}\theta_{k}^{\tau, \left(m\right)}E_{m,j,k}\chi_{m,j,k}\big\Vert _{\psi_{1}}\leq\frac{\sigma}{p}\big\|\overline{\bm{u}}^{\tau, \left(m\right)}\big\|_{\infty}\big\|\bm{\theta}^{\tau, \left(m\right)}\big\|_{\infty} \\
& \, \lesssim\frac{\sigma}{p}\sqrt{\frac{\mu r\log d}{d}} \, \big\|\overline{\bm{u}}^{\tau, \left(m\right)}\big\|_{\infty}, 
\end{align*}
with $\|\cdot\|_{\psi_{1}}$ denoting the sub-exponential norm, and
\begin{align*}
V_{2} & :=\mathbb{E}\Big[\big( \sum\nolimits_{j,k\in[d]}\overline{u}_{j}^{\tau, \left(m\right)}\theta_{k}^{\tau, \left(m\right)} p^{-1}E_{m,j,k}\chi_{m,j,k}\big)^{2}\Big] \\
& \, =\sum_{j,k\in[d]}\big(\overline{u}_{j}^{\tau, \left(m\right)}\big)^{2}\big(\theta_{k}^{\tau, \left(m\right)}\big)^{2}\mathbb{E}\big[p^{-2}E_{m,j,k}^{2}\chi_{m,j,k}\big] \\
 & \, \leq\frac{\sigma^{2}}{p}\,\big\|\overline{\bm{u}}^{\tau, \left(m\right)}\big\|_{2}^{2}\big\|\bm{\theta}^{\tau, \left(m\right)}\big\|_{2}^{2}\lesssim\frac{\sigma^{2}rd\log d}{p} \, \big\|\overline{\bm{u}}^{\tau, \left(m\right)}\big\|_{\infty}^{2}.
\end{align*}
Then the Bernstein inequality reveals that with probability at least
$1-O\left(d^{-11}\right)$,
\begin{align*}
\Big|p^{-1}\sum_{j,k\in[d]}\overline{u}_{j}^{\tau, \left(m\right)}\theta_{k}^{\tau, \left(m\right)}E_{m,j,k}\chi_{m,j,k}\Big|\lesssim L_{2}\log^{2}d+\sqrt{V_{2}\log d}\leq\sigma\sqrt{\frac{rd\log^{2}d}{p}}\big\|\overline{\bm{u}}^{\tau, \left(m\right)}\big\|_{\infty},
\end{align*}
where the last inequality follows from our sample size condition.
\item Substituting these into \eqref{eq:M1m-split-100}, we arrive at
\begin{align}
\Big|\big(\widehat{\bm{M}}^{\tau, \left(m\right)}-\bm{M}^{\tau, \left(m\right)}\big)_{m,:}\overline{\bm{u}}^{\tau, \left(m\right)}\Big|\lesssim\Bigg\{\frac{\mu r\lambda_{\max}^{\star}\log d}{d\sqrt{p}}+\sigma\sqrt{\frac{rd\log^{2}d}{p}}\Bigg\}\big\|\overline{\bm{u}}^{\tau, \left(m\right)}\big\|_{\infty}.\label{M_loo_tilde_v_term1}
\end{align}
\end{itemize}
\item For the remaining entries of $\big(\widehat{\bm{M}}^{\tau, \left(m\right)}-\bm{M}^{\tau, \left(m\right)}\big)\overline{\bm{u}}^{\tau, \left(m\right)}$,
we have
\begin{align*}
 & \big(\widehat{\bm{M}}^{\tau, \left(m\right)}-\bm{M}^{\tau, \left(m\right)}\big)_{i,:}\overline{\bm{u}}^{\tau, \left(m\right)}\\
 & \qquad=\theta_{m}^{\tau, \left(m\right)} \Big( \sum\nolimits_{j:j\neq m} \overline{u}_{j}^{\tau, \left(m\right)} \big( T_{i,j,m}^{\star}(p^{-1}\chi_{i,j,m}-1)+  p^{-1}E_{i,j,m}\chi_{i,j,m} \big) \Big) \\
 & \qquad\qquad+\overline{v}_{m}^{\tau, \left(m\right)} \Big( \sum\nolimits_{k\in[d]}\theta_{k}^{\tau, \left(m\right)} \big( T_{i,m,k}^{\star}(p^{-1}\chi_{i,m,k}-1)+p^{-1}E_{i,m,k}\chi_{i,m,k} \big) \Big)
\end{align*}
for any $i\neq m$. From Lemma~\ref{lemma:sum_jk_square_cont_ineq}
and \eqref{eq:theta_loo_inf_norm}, we have, with probability at least $1-O\left(d^{-11}\right)$,
that
\begin{align*}
\big(\theta_{m}^{\tau, \left(m\right)}\big)^{2} & \sum_{i:i\neq m}\Big(\sum_{j:j\neq m}T_{i,j,m}^{\star}\overline{u}_{j}^{\tau, \left(m\right)}(p^{-1}\chi_{i,j,m}-1)\Big)^{2}+\big(\overline{v}_{m}^{\tau, \left(m\right)}\big)^{2}\sum_{i:i\neq m}\Big(\sum_{k\in[d]}T_{i,m,k}^{\star}\theta_{k}^{\tau, \left(m\right)}(p^{-1}\chi_{i,m,k}-1)\Big)^{2}\\
 & \lesssim\frac{\mu r\lambda_{\max}^{\star2}}{dp}\big\|\bm{\theta}^{\tau, \left(m\right)}\big\|_{\infty}^{2}\big\|\overline{\bm{u}}^{\tau, \left(m\right)}\big\|_{\infty}^{2}\\
 & \lesssim\frac{\mu^{2}r^{2}\lambda_{\max}^{\star2}\log d}{d^{2}p}\big\|\overline{\bm{u}}^{\tau, \left(m\right)}\big\|_{\infty}^{2}.
\end{align*}
Combined with \eqref{eq:theta_loo_2_norm} and \eqref{eq:theta_loo_inf_norm}, Lemma~\ref{lemma:sum_jk_square_cont_ineq_subg} reveals that with
probability at least $1-O\left(d^{-11}\right)$,
\begin{align*}
\big(\theta_{m}^{\tau, \left(m\right)}\big)^{2} & \sum_{i:i\neq m}\Big(p^{-1}\sum_{j:j\neq m}\overline{u}_{j}^{\tau, \left(m\right)}E_{i,j,m}\chi_{i,j,m}\Big)^{2}+\big(\overline{v}_{m}^{\tau, \left(m\right)}\big)^{2}\sum_{i:i\neq m}\Big(p^{-1}\sum_{k\in[d]}\theta_{k}^{\tau, \left(m\right)}E_{i,m,k}\chi_{i,m,k}\Big)^{2}\\
 & \lesssim\frac{\sigma^{2} d}{p}\big\|\overline{\bm{u}}^{\tau, \left(m\right)}\big\|_{2}^{2}\big\|\bm{\theta}^{\tau, \left(m\right)}\big\|_{\infty}^{2}
 +\frac{\sigma^{2} d}{p}\big\|\bm{\theta}^{\tau, \left(m\right)}\big\|_{2}^{2}\big\|\overline{\bm{u}}^{\tau, \left(m\right)}\big\|_{\infty}^{2}
 +\frac{\sigma^{2} \log^5 d}{p^2} \big\|\overline{\bm{u}}^{\tau, \left(m\right)}\big\|_{\infty}^{2}\big\|\bm{\theta}^{\tau, \left(m\right)}\big\|_{\infty}^{2}\\
 & \lesssim\frac{\sigma^{2} rd\log d}{p}\big\|\overline{\bm{u}}^{\tau, \left(m\right)}\big\|_{\infty}^{2},
\end{align*}
which implies that
\begin{align}
\sum_{i:i\neq m}\Big(\big(\widehat{\bm{M}}^{\tau, \left(m\right)}-\bm{M}^{\tau, \left(m\right)}\big)_{i,:}\overline{\bm{u}}^{\tau, \left(m\right)}\Big)^{2} & \lesssim\left\{ \frac{\mu^{2}r^{2}\lambda_{\max}^{\star2}\log d}{d^{2}p}+\frac{\sigma^{2} rd\log d}{p}\right\} \big\|\overline{\bm{u}}^{\tau, \left(m\right)}\big\|_{\infty}^{2}.\label{M_loo_tilde_v_term2}
\end{align}
\item Therefore, combine \eqref{M_loo_tilde_v_term1} and \eqref{M_loo_tilde_v_term2}
to obtain that
\begin{align*}
\big\|\big(\widehat{\bm{M}}^{\tau, \left(m\right)}-\bm{M}^{\tau, \left(m\right)}\big)\overline{\bm{u}}^{\tau, \left(m\right)}\big\|_{2}\lesssim\Bigg\{\frac{\mu r\lambda_{\max}^{\star}\log d}{d\sqrt{p}}+\sigma\sqrt{\frac{ rd\log^{2}d}{p}}\Bigg\}\big\|\overline{\bm{u}}^{\tau, \left(m\right)}\big\|_{\infty}.
\end{align*}
\end{itemize}

\subsection{Proof of Lemma~\ref{lemma:init_coeff}}

\label{subsec:pf:init_coeff}

By definition, the only possible difference between $\overline{\bm{u}}^{\tau}$ and
$\bm{\nu}^{\tau}$ lies in how their global signs are chosen. To show that $\overline{\bm{u}}^{\tau}=\bm{\nu}^{\tau}$, 
we first claim for the moment that
\begin{align}
\label{claim:coeff}
\left| \big\langle p^{-1}\bm{T},(\overline{\bm{u}}^{\tau})^{\otimes3}\big\rangle-\lambda_{1}^{\star} \right| \lesssim \EE_\proj  \cdot \lambda_{1}^{\star},
\end{align}
where $\EE_\proj$ is defined in (\ref{claim:init_2_loss}). Given that $\EE_\proj \ll 1$ under our sample size, noise and rank condition, this immediately implies that $\big\langle p^{-1}\bm{T},(\overline{\bm{u}}^{\tau})^{\otimes3} \big\rangle > 0$. Consequently, by construction, the global signs of $\overline{\bm{u}}^{\tau}$ and
$\bm{\nu}^{\tau}$ coincide. Moreover, from \eqref{claim:init_loo_coeff_loss} and the condition that $\EE_\loo \sqrt{\mu r \log d / d} \ll 1$, one also knows that $\big\langle p^{-1}\bm{T}^{\m},(\overline{\bm{u}}^{\tau, \m})^{\otimes3} \big\rangle > 0$ and hence the global signs of $\overline{\bm{u}}^{\tau, \m}$ and
$\bm{\nu}^{\tau, \m}$ also coincide. 

In addition, recall that $\lambda_{\tau}=\big\langle p^{-1}\bm{T},({\bm{\nu}}^{\tau})^{\otimes3}\big\rangle$. One thus has
\begin{align}
	\lambda_{\tau}=\big\langle p^{-1}\bm{T},({\bm{\nu}}^{\tau})^{\otimes3}\big\rangle = \big\langle p^{-1}\bm{T},(\overline{\bm{u}}^{\tau})^{\otimes3}\big\rangle,
\end{align}
which taken collectively with \eqref{claim:coeff} justifies \eqref{claim:init_coeff_loss}.

The rest of the proof then comes down to establishing the claim (\ref{claim:coeff}). Towards this, we first decompose 
\begin{align}
\big\langle p^{-1}\bm{T},(\overline{\bm{u}}^{\tau})^{\otimes3}\big\rangle-\lambda_{1}^{\star} & =\underbrace{\big\langle p^{-1}\bm{T},(\overline{\bm{u}}^{\tau})^{\otimes3}\big\rangle-\big\langle\bm{T}^{\star},(\overline{\bm{u}}^{\tau})^{\otimes3}\big\rangle}_{=:\,\beta_{1}} \nonumber \\
 & \quad+\underbrace{\big\langle\bm{T}^{\star},(\overline{\bm{u}}^{\tau})^{\otimes3}\big\rangle-\big\langle\bm{T}^{\star},(\overline{\bm{u}}_{1}^{\star})^{\otimes3}\big\rangle}_{=:\,\beta_{2}}+\underbrace{\big\langle\bm{T}^{\star},(\overline{\bm{u}}_{1}^{\star})^{\otimes3}\big\rangle-\lambda_{1}^{\star}}_{=:\,\beta_{3}}. \label{eq:coeff-decomp}
\end{align}

In what follows, we shall upper bound these three terms separately.

\subsubsection{Controlling $\beta_{1}$}

Let us start with $\beta_{1}$, For simplicity of notation, let us
define $\bm{\Delta}_{1}:=\overline{\bm{u}}^{\tau}-\overline{\bm{u}}_{1}^{\star}$. By construction, $\bm{T}$ and $\bm{T}^\star$ are symmetric.
We then can expand
\begin{align}
  \big\langle p^{-1}\bm{T}-\bm{T}^{\star},(\overline{\bm{u}}^{\tau})^{\otimes3}\big\rangle & =\big\langle p^{-1}\bm{T}-\bm{T}^{\star},(\overline{\bm{u}}_{1}^{\star}+\bm{\Delta}_{1})^{\otimes3}\big\rangle \nonumber \\
 &= \big\langle p^{-1}\bm{T}-\bm{T}^{\star},(\overline{\bm{u}}_{1}^{\star})^{\otimes3}\big\rangle
 + 3 \, \big\langle p^{-1}\bm{T}-\bm{T}^{\star},\bm{\Delta}_{1}\otimes(\overline{\bm{u}}_{1}^{\star})^{\otimes2}\big\rangle
  + 3\, \big\langle p^{-1}\bm{T}-\bm{T}^{\star},\bm{\Delta}_{1}^{\otimes2}\otimes\overline{\bm{u}}_{1}^{\star}\big\rangle \nonumber \\
 & \quad +\big\langle p^{-1}\bm{T}-\bm{T}^{\star},\bm{\Delta}_{1}^{\otimes3}\big\rangle. \label{eq:T-Ttrue-u_inprod}
\end{align}
%\begin{align}
% & \big\langle p^{-1}\bm{T}-\bm{T}^{\star},(\overline{\bm{u}}^{\tau})^{\otimes3}\big\rangle=\big\langle p^{-1}\bm{T}-\bm{T}^{\star},(\overline{\bm{u}}_{1}^{\star}+\bm{\Delta}_{1})^{\otimes3}\big\rangle \nonumber \\
% & \qquad=\big\langle p^{-1}\bm{T}-\bm{T}^{\star},(\overline{\bm{u}}_{1}^{\star})^{\otimes3}\big\rangle+\big\langle p^{-1}\bm{T}-\bm{T}^{\star},\bm{\Delta}_{1}\otimes(\overline{\bm{u}}_{1}^{\star})^{\otimes2}\big\rangle+\big\langle p^{-1}\bm{T}-\bm{T}^{\star},\overline{\bm{u}}_{1}^{\star}\otimes\bm{\Delta}_{1}\otimes\overline{\bm{u}}_{1}^{\star}\big\rangle \nonumber \\
% & \qquad\quad+\big\langle p^{-1}\bm{T}-\bm{T}^{\star},(\overline{\bm{u}}_{1}^{\star})^{\otimes2}\otimes\bm{\Delta}_{1}\big\rangle+\big\langle p^{-1}\bm{T}-\bm{T}^{\star},\bm{\Delta}_{1}^{\otimes2}\otimes\overline{\bm{u}}_{1}^{\star}\big\rangle+\big\langle p^{-1}\bm{T}-\bm{T}^{\star},\bm{\Delta}_{1}\otimes\overline{\bm{u}}_{1}^{\star}\otimes\bm{\Delta}_{1}\big\rangle \nonumber \\
% & \qquad\quad+\big\langle p^{-1}\bm{T}-\bm{T}^{\star},\overline{\bm{u}}_{1}^{\star}\otimes\bm{\Delta}_{1}^{\otimes2}\big\rangle+\big\langle p^{-1}\bm{T}-\bm{T}^{\star},\bm{\Delta}_{1}^{\otimes3}\big\rangle. \label{eq:T-Ttrue-u_inprod}
%\end{align}

 We first look at the first term of (\ref{eq:T-Ttrue-u_inprod}) which only consists of $\overline{\bm{u}}_{1}^{\star}$. 
%By Lemma~\ref{lemma:T_loss_times3_op_norm}, one has with probability
%at least $1-O\left(d^{-20}\right)$,
%\begin{align*}
%\left\Vert \left(p^{-1}\bm{T}-\bm{T}^{\star}\right)\times_{3}\overline{\bm{u}}_{1}^{\star}\right\Vert  & \lesssim\left\Vert \overline{\bm{u}}_{1}^{\star}\right\Vert _{\infty}\sqrt{\frac{\mu r\log d}{dp}}\,\lambda_{\max}^{\star}+\left\Vert \overline{\bm{u}}_{1}^{\star}\right\Vert _{2}\sigma\sqrt{\frac{d\log d}{p}}\\
% & \lesssim\frac{\mu\sqrt{r}\,\lambda_{\max}^{\star}\sqrt{\log d}}{d\sqrt{p}}+\sigma\sqrt{\frac{d\log d}{p}}.
%\end{align*}
As shown in (\ref{eq:T_loss_utrue_op_norm}), with probability at least $1-O\left(d^{-11}\right)$, one has  
\begin{align*}
\left\Vert \left(p^{-1}\bm{T}-\bm{T}^{\star}\right)\times_{3}\overline{\bm{u}}_{1}^{\star}\right\Vert \lesssim\frac{\mu\sqrt{r}\,\lambda_{\max}^{\star}\sqrt{\log d}}{d\sqrt{p}}+\sigma\sqrt{\frac{d\log d}{p}} \leq \EE_\proj \lambda_{\min}^\star.
\end{align*}
It follows that
\begin{align*}
\left|\big\langle p^{-1}\bm{T}-\bm{T}^{\star},(\overline{\bm{u}}_{1}^{\star})^{\otimes3}\big\rangle\right|\leq\left\Vert \left(p^{-1}\bm{T}-\bm{T}^{\star}\right)\times_{3}\overline{\bm{u}}_{1}^{\star}\right\Vert \left\Vert \overline{\bm{u}}_{1}^{\star}\right\Vert _{2}^{2}\lesssim \EE_\proj \lambda_{\min}^\star,
\end{align*}
 where we recall the definition of $\EE_\proj$ in \eqref{claim:init_2_loss}.
 As for the term linear in $\bm{\Delta}_1$, by Lemma~\ref{lemma:init_2_loss}, we know that $\left\Vert \bm{\Delta}_{1}\right\Vert _{2} \lesssim \EE_\proj$. As a result, one
has
\begin{align*}
\left|\big\langle p^{-1}\bm{T}-\bm{T}^{\star},\bm{\Delta}_{1}\otimes(\overline{\bm{u}}_{1}^{\star})^{\otimes2}\big\rangle\right| & \leq\left\Vert \left(p^{-1}\bm{T}-\bm{T}^{\star}\right)\times_{3}\overline{\bm{u}}_{1}^{\star}\right\Vert \left\Vert \overline{\bm{u}}_{1}^{\star}\right\Vert _{2}\left\Vert \bm{\Delta}_{1}\right\Vert _{2} \lesssim \EE_\proj^2 \lambda_{\min}^\star.
\end{align*}
%Clearly, the bound also holds for $\big\langle p^{-1}\bm{T}-\bm{T}^{\star},\overline{\bm{u}}_{1}^{\star}\otimes\bm{\Delta}_{1}\otimes\overline{\bm{u}}_{1}^{\star}\big\rangle$ and $\big\langle p^{-1}\bm{T}-\bm{T}^{\star},(\overline{\bm{u}}_{1}^{\star})^{\otimes2}\otimes\bm{\Delta}_{1}\big\rangle$.
 
 We then turn to the quadratic terms in $\bm{\Delta}_1$. Similar to the above arguments, one can deduce that
\begin{align*}
\left|\big\langle p^{-1}\bm{T}-\bm{T}^{\star},\bm{\Delta}_{1}^{\otimes2}\otimes\overline{\bm{u}}_{1}^{\star}\big\rangle\right|\leq\left\Vert \left(p^{-1}\bm{T}-\bm{T}^{\star}\right)\times_{3}\overline{\bm{u}}_{1}^{\star}\right\Vert \left\Vert \bm{\Delta}_{1}\right\Vert _{2}^{2} \lesssim \EE_\proj^3 \lambda_{\min}^\star.
\end{align*}
%Apparently, the same bound holds for $\big\langle p^{-1}\bm{T}-\bm{T}^{\star},\bm{\Delta}_{1}\otimes\overline{\bm{u}}_{1}^{\star}\otimes\bm{\Delta}_{1}\big\rangle$ and $\big\langle p^{-1}\bm{T}-\bm{T}^{\star},\overline{\bm{u}}_{1}^{\star}\otimes\bm{\Delta}_{1}^{\otimes2}\big\rangle$.
 
Finally, we can simply upper bound the last term in (\ref{eq:T_loss_utrue_op_norm}) by
\begin{align*}
\left|\big\langle p^{-1}\bm{T}-\bm{T}^{\star},\bm{\Delta}_{1}^{\otimes3}\big\rangle\right|\leq\left\Vert p^{-1}\bm{T}-\bm{T}^{\star}\right\Vert \left\Vert \bm{\Delta}_{1}\right\Vert _{2}^{3} \lesssim \EE_\op \EE_\proj^3 \lambda_{\min}^\star \ll \EE_\proj^3 \lambda_{\min}^\star,
\end{align*}
where the last step is due to the fact that $\EE_\op \ll 1$. By our sample size, noise and rank conditions, one has $\EE_\proj \ll 1$. Putting these bounds together reveals that
\begin{equation}
\left|\big\langle p^{-1}\bm{T}-\bm{T}^{\star},(\overline{\bm{u}}^{\tau})^{\otimes3}\big\rangle\right|\lesssim \EE_\proj \lambda_{\min}^\star. \label{eq:u_coef_alpha1}
\end{equation}

\subsubsection{Controlling $\beta_{2}$}

Recall the definition of $\beta_2$ in (\ref{eq:coeff-decomp}) and $\bm{\Delta}_1 = \overline{\bm{u}}^{\tau} - \overline{\bm{u}}_{1}^{\star}$. We can further decompose
%\begin{align*}
%\big\langle\bm{T}^{\star},(\overline{\bm{u}}^{\tau})^{\otimes3}-(\overline{\bm{u}}_{1}^{\star})^{\otimes3}\big\rangle & =\big\langle\bm{T}^{\star},(\overline{\bm{u}}_{1}^{\star})^{\otimes2}\otimes\bm{\Delta}_{1}\big\rangle+\big\langle\bm{T}^{\star},\overline{\bm{u}}_{1}^{\star}\otimes\bm{\Delta}_{1}\otimes\overline{\bm{u}}_{1}^{\star}\big\rangle+\big\langle\bm{T}^{\star},\bm{\Delta}_{1}\otimes(\overline{\bm{u}}_{1}^{\star})^{\otimes2}\big\rangle\\
% & \quad+\big\langle\bm{T}^{\star},\bm{\Delta}_{1}^{\otimes2}\otimes\overline{\bm{u}}_{1}^{\star}\big\rangle+\big\langle\bm{T}^{\star},\bm{\Delta}_{1}\otimes\overline{\bm{u}}_{1}^{\star}\otimes\bm{\Delta}_{1}\big\rangle+\big\langle\bm{T}^{\star},\overline{\bm{u}}_{1}^{\star}\otimes\bm{\Delta}_{1}^{\otimes2}\big\rangle +\big\langle\bm{T}^{\star},\bm{\Delta}_{1}^{\otimes3}\big\rangle.
%\end{align*}
\begin{align*}
\big\langle\bm{T}^{\star},(\overline{\bm{u}}^{\tau})^{\otimes3}-(\overline{\bm{u}}_{1}^{\star})^{\otimes3}\big\rangle & = 3\, \big\langle\bm{T}^{\star},(\overline{\bm{u}}_{1}^{\star})^{\otimes2}\otimes\bm{\Delta}_{1}\big\rangle
+ 3\, \big\langle\bm{T}^{\star},\bm{\Delta}_{1}^{\otimes2}\otimes\overline{\bm{u}}_{1}^{\star}\big\rangle +\big\langle\bm{T}^{\star},\bm{\Delta}_{1}^{\otimes3}\big\rangle.
\end{align*}

We first consider the first term which is linear in $\bm{\Delta}_{1}$. Since $\bm{T}^{\star}$ is a symmetric tensor, we have
\begin{align*}
\bm{T}^{\star}\times_{1}\overline{\bm{u}}_{1}^{\star}=\bm{T}^{\star}\times_{2}\overline{\bm{u}}_{1}^{\star}=\bm{T}^{\star}\times_{3}\overline{\bm{u}}_{1}^{\star}=\sum\nolimits_{s\in\left[r\right]}\lambda_{1}^{\star}\left\langle\overline{\bm{u}}_{s}^{\star},\overline{\bm{u}}_{1}^{\star}\right\rangle\overline{\bm{u}}_{s}^{\star}\overline{\bm{u}}_{s}^{\star\top}.
\end{align*}
By Lemma~\ref{lemma:incoh}, one has
\begin{align*}
\left\Vert \bm{T}^{\star}\times_{1}\overline{\bm{u}}_{1}^{\star}\right\Vert \leq\max_{1\leq s\leq r}\left|\lambda_{1}^{\star}\big\langle\overline{\bm{u}}_{s}^{\star},\overline{\bm{u}}_{1}^{\star}\big\rangle\right|\big\|\overline{\bm{U}}^{\star}\big\|^{2} \leq (\lambda_{1}^\star + \lambda_{\max}^\star \sqrt{\mu / d})  \lesssim \lambda_{1}^{\star}, 
\end{align*}
which arises from $\max_{s\neq i}\left|\big\langle\overline{\bm{u}}_{s}^{\star},\overline{\bm{u}}_{1}^{\star}\big\rangle\right|\leq \sqrt{\mu / d} $
and $\big\|\overline{\bm{U}}^{\star}\big\|\lesssim1$ as long as $r\sqrt{\mu/d}\ll1$.
As a result, one has
\begin{align*}
\left|\big\langle\bm{T}^{\star},(\overline{\bm{u}}_{1}^{\star})^{\otimes2}\otimes\bm{\Delta}_{1}\big\rangle\right|\leq\left\Vert \bm{T}^{\star}\times_{1}\overline{\bm{u}}_{1}^{\star}\right\Vert \left\Vert \overline{\bm{u}}_{1}^{\star}\right\Vert _{2}\left\Vert \bm{\Delta}_{1}\right\Vert _{2}\lesssim\lambda_{1}^{\star}\left\Vert \bm{\Delta}_{1}\right\Vert _{2} \lesssim \EE_\proj \lambda_{1}^{\star}.
\end{align*}
%Clearly, the bound is also applicable to $\big\langle\bm{T}^{\star},\overline{\bm{u}}_{1}^{\star}\otimes\bm{\Delta}_{1}\otimes\overline{\bm{u}}_{1}^{\star}\big\rangle$ and $\big\langle\bm{T}^{\star},\bm{\Delta}_{1}\otimes(\overline{\bm{u}}_{1}^{\star})^{\otimes2}\big\rangle$.

In a similar manner, we also know that
\begin{align*}
\left|\big\langle\bm{T}^{\star},\bm{\Delta}_{1}^{\otimes2}\otimes\overline{\bm{u}}_{1}^{\star}\big\rangle\right|\leq\left\Vert \bm{T}^{\star}\times_{1}\overline{\bm{u}}_{1}^{\star}\right\Vert \left\Vert \bm{\Delta}_{1}\right\Vert _{2}^{2}\lesssim \EE_\proj^2 \lambda_{1}^{\star}. 
\end{align*}
%The same upper bounds hold for $\big\langle\bm{T}^{\star},\bm{\Delta}_{1}\otimes\overline{\bm{u}}_{1}^{\star}\otimes\bm{\Delta}_{1}\big\rangle$ as well as $\big\langle\bm{T}^{\star},\overline{\bm{u}}_{1}^{\star}\otimes\bm{\Delta}_{1}^{\otimes2}\big\rangle$.

Finally, using the fact that the tensor spectral norm is always less
than or equal to that of its matricization, we find that
\begin{align*}
\left|\big\langle\bm{T}^{\star},\bm{\Delta}_{1}^{\otimes3}\big\rangle\right|\leq\left\Vert \bm{T}^{\star}\right\Vert \left\Vert \bm{\Delta}_{1}\right\Vert _{2}^{3}\leq\left\Vert \bm{A}^{\star}\right\Vert \left\Vert \bm{\Delta}_{1}\right\Vert _{2}^{3} \lesssim \EE_\proj^3 \lambda_{1}^\star.
\end{align*}

Combining this with the fact that $\EE_\proj \ll 1$, we conclude that
\begin{equation}
\left|\big\langle\bm{T}^{\star},(\overline{\bm{u}}^{\tau})^{\otimes3}-(\overline{\bm{u}}_{1}^{\star})^{\otimes3}\big\rangle\right|\lesssim \EE_\proj \lambda_{1}^\star. \label{eq:u_coef_alpha2}
\end{equation}

\subsubsection{Controlling $\beta_{3}$}

It remains to control $\beta_3$. Straightforward calculation reveals that
\begin{align*}
\big\langle\bm{T}^{\star},(\overline{\bm{u}}_{1}^{\star})^{\otimes3}\big\rangle & =\sum_{s\in\left[r\right]}\lambda_{s}^{\star}\big\langle(\overline{\bm{u}}_{s}^{\star})^{\otimes3},(\overline{\bm{u}}_{1}^{\star})^{\otimes3}\big\rangle =\lambda_{1}^{\star}\left\Vert \overline{\bm{u}}_{1}^{\star}\right\Vert _{2}^{6}
 +\sum_{s:s\neq 1}\lambda_{s}^{\star}\big\langle\overline{\bm{u}}_{s}^{\star},\overline{\bm{u}}_{1}^{\star}\big\rangle^{6}.
\end{align*}
By the incoherence conditions, we can upper bound
\begin{align}
\left|\big\langle\bm{T}^{\star},(\overline{\bm{u}}_{1}^{\star})^{\otimes3}\big\rangle-\lambda_{1}^{\star}\right| & =\sum_{s:s\neq 1}\lambda_{1}^{\star}\big\langle\overline{\bm{u}}_{s}^{\star},\overline{\bm{u}}_{1}^{\star}\big\rangle^{6}\leq r \, \max_{s:s\neq 1}\left|\big\langle\overline{\bm{u}}_{s}^{\star},\overline{\bm{u}}_{1}^{\star}\big\rangle\right|^{6} \lambda_{\max}^{\star} \leq\frac{\mu^{3}r\lambda_{\max}^{\star}}{d^{3}}.\label{eq:u_coef_alpha3}
\end{align}

\subsubsection{Combining $\beta_1$, $\beta_2$ and $\beta_3$}

Putting (\ref{eq:u_coef_alpha1}), (\ref{eq:u_coef_alpha2}) and (\ref{eq:u_coef_alpha3})
together, we find that
\begin{align*}
\left|\big\langle p^{-1}\bm{T},(\overline{\bm{u}}^{\tau})^{\otimes3}\big\rangle-\lambda_{1}^{\star}\right| & \lesssim \EE_\proj \lambda_{1}^\star +\frac{\mu r\lambda_{\max}^{\star}}{d^{3}} \leq \( \EE_\proj + \frac{\mu^3 r }{d^3}  \) \lambda_{1}^\star \asymp \EE_\proj \lambda_{1}^\star ,
\end{align*}
where the last step follows from the condition  $\mu \leq d$, $r\ll\sqrt{d/\mu}$ and the definition $\EE_\proj \geq  \sqrt{\mu r / d}$ (cf.~(\ref{claim:init_2_loss})).

\subsection{Proof of Lemma~\ref{lemma:init_selection}}
\label{subsec:pf:init_selection}

We first show that for each $i \in \[ r \]$, $\big( \bm{w}_i, \lambda_{i} \big)$ and $\big( \bm{w}_i^{\m}, \lambda_{i}^{\m} \big)$ (returned by $\Call{Prune}$ in Algorithm~\ref{alg:localization} and Algorithm~\ref{alg:localization_loo} respectively) satisfy \eqref{claim:init_loo_entry_loss}, \eqref{claim:init_loo_l2_loss} and  \eqref{claim:init_loo_coeff_loss}; in other words, we want to show that they correspond to the same index $\tau \in \[ L \]$ (and are hence produced using  the same Gaussian random vector $\bm{g}^\tau$).
The proof idea is this: given that the proposed algorithms select the pair with the largest spectral gap in each round of $\Call{Prune}$, it suffices to ensure that 
there is sufficient separation between the largest and the second largest spectral gaps (so that both algorithms can identify the same $\tau$).

By Lemma~\ref{lemma:init_spectra} and union bounds, we know that with probability at least $1-\delta$, for each $i \in \[ r \]$, 
\begin{align}
	\Delta^{\( 1 \)}_i - \Delta^{\( 2 \)}_i \gtrsim \lambda_{\min}^\star / \big( r \sqrt{\log d} \big),
\end{align}
where we recall that $\Delta_i^\tau := \gamma_{i}^{\star\tau}-\max_{j:j\neq i}\big| \gamma_{j}^{\star\tau}\big|$, and $\Delta_i^{\( 1 \)} \geq \Delta_i^{\( 2 \)} \geq \dots \geq \Delta_i^{\( L \)}$ denote the order statistics of $\big\{ \Delta_i^\tau \big\}_{\tau = 1}^L$ in descending order. As shown in the proof of Lemma~\ref{lemma:init_2_loss}, the spectral gap of $\bm{M}^{\tau}$ is well approximated by $\max_i \Delta_i^\tau$, namely, 
\begin{align*}
\Big| \max_{1\leq i \leq r} \Delta_i^\tau - \big( \sigma_1 \( \bm{M}^\tau \) - \sigma_2 \( \bm{M}^{\tau} \) \big)  \Big| \lesssim \big\| \bm{M}^\tau - \bm{M}^{\tau \star} \big\| \ll \frac{\lambda_{\min}^\star}{r \sqrt{\log d}}
\end{align*}
under our sample size, noise and rank conditions \eqref{eq:init_cond}. Moreover, from Lemma~\ref{lemma:M_loo_op_loss}, we see that $\bm{M}^{\tau}$ and $\bm{M}^{\tau, \m}$ are extremely close in terms of the spectral norm, i.e.
\begin{align*}
 \big\| \bm{M}^\tau - \bm{M}^{\tau, \m } \big\| \ll \frac{\lambda_{\min}^\star}{r \sqrt{\log d}}.
\end{align*}
This implies that the perturbation incurred by the leave-out-one procedure is relatively small compared to the difference between the largest and the second largest spectral gaps of $\bm{M}^\tau$. Consequently, the leave-one-out estimates $\big\{ \big( \bm{w}_i^{\m}, \lambda_{i}^{\m} \big) \big\}_{i=1}^r$ returned by Algorithm~\ref{alg:localization_loo} and the true estimates $\big\{ \big( \bm{w}_i, \lambda_{i} \big) \big\}_{i=1}^r$ should correspond to the same trials and should be generated by the same set of Gaussian random vectors. As a result, they obey \eqref{eq:thm:init_loo_loss}, \eqref{eq:thm:init_loo_coef_loss} and \eqref{eq:thm:init_loo_entry_loss} for all $1 \leq m \leq d$.

From the discussion above, we also know that as long as $\sigma_1 \( \bm{M}^\tau \) - \sigma_2 \( \bm{M}^{\tau} \) \gtrsim \lambda_{\min}^\star$, one has $\left\| \bm{\nu}^{\tau} - \overline{\bm{u}}_i^\star \right\|_2 \lesssim \EE_\proj$ for some $i \in \[ r \]$. This is an immediate consequence of Lemma~\ref{lemma:init_2_loss} and the fact that the spectral gap of $\bm{M}^{\tau}$ and $\max_i \Delta_i^\tau$ are extremely close.

It remains to show that our pruning procedure can return estimates of tensor factors without duplicates.
Suppose that there exist $1\leq \tau_1 \neq \tau_2 \leq L$ such that $\left\| \bm{\nu}^{\tau_1} - \overline{\bm{u}}_i^\star \right\|_2 \lesssim \EE_\proj$ and $\left\| \bm{\nu}^{\tau_2} - \overline{\bm{u}}_i^\star \right\|_2 \lesssim \EE_\proj$ for some $i \in \[ r\]$. By the triangle inequality, one has
\begin{align*}
\left| \left\langle \bm{\nu}^{\tau_1}, \bm{\nu}^{\tau_2} \right\rangle \right| &= \left| \norm{ \overline{\bm{u}}_i^\star }_2^2 +  \left\langle \bm{\nu}^{\tau_1} - \overline{\bm{u}}_i^\star , \overline{\bm{u}}_i^\star \right\rangle + \left\langle \overline{\bm{u}}_i^\star, \bm{\nu}^{\tau_2} - \overline{\bm{u}}_i^\star \right\rangle + \left\langle \bm{\nu}^{\tau_1} - \overline{\bm{u}}_i^\star, \bm{\nu}^{\tau_2} - \overline{\bm{u}}_i^\star \right\rangle \right|  \\
& \geq 1- \norm{ \bm{\nu}^{\tau_1} - \overline{\bm{u}}_i^\star}_2 - \norm{ \bm{\nu}^{\tau_2} - \overline{\bm{u}}_i^\star}_2 - \norm{ \bm{\nu}^{\tau_1} - \overline{\bm{u}}_i^\star}_2  \norm{ \bm{\nu}^{\tau_2} - \overline{\bm{u}}_i^\star}_2 \\
& \geq 1- 2 \EE_\proj - \EE_\proj^2 \geq 1- 3 \EE_\proj,
\end{align*}
provided that $\EE_\proj \ll 1$. In addition, for any $j \neq i, j \in \[ r \]$, we know that exists some $1 \leq \tau_3 \leq L$ such that
\begin{align*}
\left\| \bm{\nu}^{\tau_3} - \overline{\bm{u}}_j^\star \right\|_2 \lesssim \EE_\proj.
\end{align*}
Recall our incoherence condition in (\ref{asmp_corr}). It is easy to see that
\begin{align*}
\left| \left\langle \bm{\nu}^{\tau_1}, \bm{\nu}^{\tau_3} \right\rangle \right| &= \left| \left\langle \overline{\bm{u}}_i^\star , \overline{\bm{u}}_j^\star \right\rangle +  \left\langle \bm{\nu}^{\tau_1} - \overline{\bm{u}}_i^\star , \overline{\bm{u}}_i^\star \right\rangle + \left\langle \overline{\bm{u}}_i^\star, \bm{\nu}^{\tau_3} - \overline{\bm{u}}_i^\star \right\rangle + \left\langle \bm{\nu}^{\tau_1} - \overline{\bm{u}}_i^\star, \bm{\nu}^{\tau_3} - \overline{\bm{u}}_j^\star \right\rangle \right|  \\
& \leq  \left| \left\langle \overline{\bm{u}}_i^\star , \overline{\bm{u}}_j^\star \right\rangle \right| 
+ \norm{ \bm{\nu}^{\tau_1} - \overline{\bm{u}}_i^\star}_2 
+ \norm{ \bm{\nu}^{\tau_3} - \overline{\bm{u}}_j^\star}_2 
+ \norm{ \bm{\nu}^{\tau_1} - \overline{\bm{u}}_i^\star}_2  \norm{ \bm{\nu}^{\tau_3} - \overline{\bm{u}}_j^\star}_2 \\
& \leq \sqrt{\mu / d} + 2 \EE_\proj + \EE_\proj^2 \leq \sqrt{\mu / d} + 3 \EE_\proj \ll 1 - 3\EE_\proj, 
\end{align*}
with the proviso that $\mu \ll d$ and $\EE_\proj \ll 1$. 

The above argument reveals a clear separation between $\left| \left\langle \bm{\nu}^{\tau_1}, \bm{\nu}^{\tau_3} \right\rangle \right|$ and $\left| \left\langle \bm{\nu}^{\tau_1}, \bm{\nu}^{\tau_2} \right\rangle \right|$. As an immediate consequence, the proposed pruning procedure successfully removes all duplication while securing an estimate for each tensor factor. 

\subsection{Proof of Corollary~\ref{cor:init}}
\label{pf:cor:init}

Fix any arbitrary small constant $\delta > 0$. From Theorems~\ref{thm:init}-\ref{thm:init-loo} and the assumptions of Theorem~\ref{thm:main}, one knows that with probability exceeding $1-\delta$, there exists a permutation $\pi(\cdot): [d]\mapsto[d]$ such that for all $1 \leq i \leq r$,
\begin{align*}
& \big\| \bm{w}^i - \overline{\bm{u}}^\star_{\pi(i)} \big\|_2  \leq \delta, \qquad \big\| \bm{w}^i - \overline{\bm{u}}^\star_{\pi(i)} \big\|_\infty  \leq \delta \sqrt{\frac{1}{d}}, \qquad \big|  \lambda_{i} - \lambda_{\pi(i)}^\star \big|  \leq \delta \lambda_{\max}^\star, \\
& \big\| \bm{w}^i - \bm{w}^{i, \m} \big\|_2  \leq \delta \sqrt{\frac{1}{d}}, \qquad \big| \lambda_{i}  - \lambda_{i}^{\m} \big| \leq \delta \sqrt{\frac{1}{d}} \, \lambda_{\max}^\star, \qquad \big| \big( \bm{w}^i - \overline{\bm{u}}^\star_{\pi(i)} \big)_m \big\|_2  \leq \delta \sqrt{\frac{1}{d}}
\end{align*}
for some $0 < \delta \ll 1/ (\mu^{3/2} r) < 1$. To prove the corollary, we shall just combine the above results. 

Without loss of generality, assume that $\pi ( i ) = i$ for each $i \in \[ r \]$. Given that $\delta \ll 1$ and $\kappa \asymp 1$, by the triangle inequality, one has $\lambda_i \asymp \lambda_i^\star$ for all $i \in \[ r\]$, which further implies that
\begin{align*}
\big| \lambda^{1/3}_i -  \lambda_{i}^{\star 1/3} \big| 
	%= \frac{\big| \lambda_i - \lambda_{i}^\star \big|}{ \lambda_i^{2/3} +  \lambda_i^{1/3} \lambda_{i}^{\star 1/3} + \lambda_{i}^{\star 2/3}} 
	\lesssim \frac{\big| \lambda_i - \lambda_{i}^\star \big|}{\lambda_{i}^{\star 2/3}} \lesssim \delta \lambda_{\max}^{\star 1/3}.
\end{align*}
Consequently, we can apply the triangle inequality to demonstrate that: for each $1 \leq i \leq r$,
\begin{align*}
\big\| \lambda^{1/3}_i\bm{w}^i - \bm{u}_i^\star \big\|_2 \leq \big| \lambda^{1/3}_i -  \lambda_{i}^{\star 1/3} \big| \big\| \overline{\bm{u}}_i^\star \big\|_2 + \lambda^{\star 1/3}_i \big\| \bm{w}^i - \overline{\bm{u}}^\star_{i} \big\|_2  \lesssim \delta \lambda_{\max}^{\star 1/3}.
\end{align*}
Arguing similarly, we also see that 
\begin{align*}
\big\| \lambda^{1/3}_i\bm{w}^i - \bm{u}_i^\star \big\|_\infty & \lesssim \delta \sqrt{\frac{1}{d}} \, \lambda_{\max}^{\star 1/3}, \\
\big\| \lambda^{1/3}_i\bm{w}^i - \big( \lambda_{i}^{(m)} \big)^{1/3} \bm{w}^{i, (m)} \big\|_2 & \lesssim \delta \sqrt{\frac{1}{d}} \, \lambda_{\max}^{\star 1/3},  \\
\Big| \Big( \big( \lambda_{i}^{(m)} \big)^{1/3} \bm{w}^{i, (m)} -  \bm{u}_i^\star \Big)_m \Big| & \lesssim \delta \sqrt{\frac{1}{d}} \, \lambda_{\max}^{\star 1/3}
\end{align*}
hold for all $i \in \[ r \]$ and $m \in \[ d \]$.
Recall that $\bm{U}^0 = \big[ \lambda_i^{1/3} \bm{w}^i \big]_{1 \leq i \leq r}$. One can deduce that
\begin{align*}
\big\| \bm{U}^0 - \bm{U}^\star \big\|_\frob & \lesssim \delta \sqrt{r} \, \lambda_{\max}^{\star 1/3} \lesssim \delta \left\| \bm{U}^\star \right\|_\frob, \\
\big\| \bm{U}^0 - \bm{U}^\star \big\|_{2,\infty} & \lesssim \delta \sqrt{\frac{r}{d}} \, \lambda_{\max}^{\star 1/3} \lesssim \delta \left\| \bm{U}^\star \right\|_{2,\infty}, \\
\big\| \bm{U}^0 - \bm{U}^{0, (m)} \big\|_{2,\infty} & \lesssim \delta \sqrt{\frac{r}{d}} \, \lambda_{\max}^{\star 1/3} \lesssim \delta \left\| \bm{U}^\star \right\|_{2,\infty}, \\
\big\| \big( \bm{U}^{0, (m)} - \bm{U}^\star \big)_{m,:} \big\|_{2} & \lesssim \delta \sqrt{\frac{r}{d}} \, \lambda_{\max}^{\star 1/3} \lesssim \delta \left\| \bm{U}^\star \right\|_{2,\infty},
\end{align*}
where we have used the condition that $\kappa \asymp 1$ and the fact that $\norm{\bm{U}^\star}_{\frob} \geq \sqrt{r} \, \lambda_{\min}^{\star 1/3}$ and $\norm{\bm{U}^\star}_{2,\infty} \geq \norm{\bm{U}^\star}_{\frob} / \sqrt{d} $.

%keeping potential candidates in each step. Together with maximum spectral gap criterion, we are guaranteed to find a new component in each step.

\section{Proof of Corollary~\ref{corollary:tensor-entries}}
\label{sec:corollary}

This section establishes Corollary~\ref{corollary:tensor-entries}. First of all, it is easy to see that: given the estimation accuracy established in Theorem \ref{thm:main}, the permutation matrices that best match $\bm{U}^t$ to  $\bm{U}^{\star}$ remain unchanged as $t$ increases. Therefore, we assume without loss of generality that $\bm{I}_{r}=\arg\min_{\bm{\Pi}\in\mathsf{perm}_{r}}\left\Vert \bm{U}^t\bm{\Pi}-\bm{U}^{\star}\right\Vert $ for all $t\geq 0$.

Suppose that $r \sqrt{\mu / d} \ll 1$. We claim for the moment that: if a matrix $\bm{U}=[\bm{u}_1,\cdots,\bm{u}_r]\in \mathbb{R}^{d\times r}$ satisfies
\begin{align*}
\left\| \bm{U} - \bm{U}^\star \right\|_{\frob} \leq \delta \left\| \bm{U}^\star \right\|_{\frob} \qquad \text{and} \qquad \left\| \bm{U} - \bm{U}^\star \right\|_{2,\infty} \leq \delta \left\| \bm{U}^\star \right\|_{2,\infty}
\end{align*}
for any $0 \leq \delta \ll 1 / (\mu^{3/2} r) \leq 1$, then one has
\begin{align}
\left\| \bm{T} - \bm{T}^\star \right\|_\frob \lesssim \delta  \left\| \bm{T}^\star \right\|_\frob \qquad \text{and} \qquad \left\| \bm{T} - \bm{T}^\star \right\|_\infty \lesssim \delta \sqrt{\mu^3 r} \,  \left\| \bm{T}^\star \right\|_\infty, \label{eq:U-loss-T-loss}
\end{align}
where $\bm{T}:= \sum_{i=1}^r \bm{u}_i \otimes \bm{u}_i \otimes \bm{u}_i$. As already shown in the analysis of Theorem~\ref{thm:main}, one has
\begin{align*}
	\left\| \bm{U}^t - \bm{U}^\star \right\|_{\frob} &\lesssim  \( C_1 \frac{\rho^{t+1}}{\mu^{3/2} r}  + C_2 \frac{\sigma}{\lammin^{\star}} \sqrt{\frac{ d \log d}{p}}   \) \norm{\bm{U}^\star}_{\frob}, \\
	\left\| \bm{U}^t - \bm{U}^\star \right\|_{2, \infty} &\lesssim \( C_3 \rho^{t+1} \frac{\rho^{t+1}}{\mu^{3/2} r} + C_4 \frac{\sigma}{\lammin^{\star}} \sqrt{\frac{ d \log d}{p}}   \) \norm{\bm{U}^\star}_{2, \infty},
\end{align*}
from which Corollary~\ref{corollary:tensor-entries} follows immediately.

It remains to prove the claim \eqref{eq:U-loss-T-loss}. For notational convenience, let us define $\bm{\Delta}:=\bm{U}-\bm{U}^{\star}$ and $\bm{\Delta}_{s}:=\bm{u}_{s}-\bm{u}_{s}^{\star}$
for each $1\leq s\leq r$. Then we can expand
\begin{align}
\bm{T}-\bm{T}^{\star} & =\sum_{1\leq s\leq r}\bm{u}_{s}^{\otimes3}-\bm{u}_{s}^{\star\otimes3}
 = \sum_{1\leq s\leq r}\bm{u}_{s}^{\star\otimes2}\otimes\bm{\Delta}_{s}+\sum_{1\leq s\leq r}\bm{u}_{s}^{\star}\otimes\bm{\Delta}_{s}\otimes\bm{u}_{s}^{\star}+\sum_{1\leq s\leq r}\bm{\Delta}_{s}\otimes\bm{u}_{s}^{\star\otimes2}\nonumber \\
 & \qquad+\sum_{1\leq s\leq r}\bm{u}_{s}^{\star}\otimes\bm{\Delta}_{s}^{\otimes2}+\sum_{1\leq s\leq r}\bm{\Delta}_{s}\otimes\bm{u}_{s}^{\star}\otimes\bm{\Delta}_{s}+\sum_{1\leq s\leq r}\bm{\Delta}_{s}^{\otimes2}\otimes\bm{u}_{s}^{\star}+\sum_{1\leq s\leq r}\bm{\Delta}_{s}^{\otimes3}.\label{eq:T-Ttrue-decomp}
\end{align}

\paragraph{(1) Euclidean loss.}
We first look at the loss measured by $\|\cdot\|_{\mathrm{F}}$. In view of the symmetric structure of tensors, it suffices to control $\sum_{1\leq s\leq r}\bm{u}_{s}^{\star\otimes2}\otimes\bm{\Delta}_{s}$, $\sum_{1\leq s\leq r}\bm{u}_{s}^{\star}\otimes\bm{\Delta}_{s}^{\otimes2}$ and $\sum_{1\leq s\leq r}\bm{\Delta}_{s}^{\otimes3}$. 

Let us define $\bm{W}_1 := \[ \bm{u}_s^\star \otimes \bm{\Delta}_s  \]_{1\leq s \leq r} \in \R^{d^2 \times r}$ and $\bm{W}_2 := \[  \bm{\Delta}_s^{\otimes 2}  \]_{1\leq s \leq r} \in \R^{d^2 \times r}$. Recalling the fact that $\left\Vert \bm{U}^{\star}\right\Vert \leq \big\| \overline{\bm{U}}^\star \big\| \lambda_{\max}^{\star1/3} \lesssim \lambda_{\max}^{\star1/3}$
(established in Lemma~\ref{lemma:incoh}), we have
\begin{align*}
\Big\|\sum\nolimits _{1\leq s\leq r}\bm{u}_{s}^{\star\otimes2}\otimes\bm{\Delta}_{s}\Big\|_{\mathrm{F}} & =\left\Vert \bm{U}^{\star}\bm{W}_{1}^{\top}\right\Vert _{\mathrm{F}}\leq\left\Vert \bm{U}^{\star}\right\Vert \left\Vert \bm{W}_{1}\right\Vert _{\mathrm{F}} \lesssim \lambda_{\max}^{\star1/3} \left\Vert \bm{W}_{1}\right\Vert _{\mathrm{F}}, \\
\Big\|\sum\nolimits _{1\leq s\leq r}\bm{u}_{s}^{\star}\otimes\bm{\Delta}_{s}^{\otimes2}\Big\|_{\mathrm{F}} & =\left\Vert \bm{U}^{\star}\bm{W}_{2}^{\top}\right\Vert _{\mathrm{F}} \leq \left\Vert \bm{U}^{\star}\right\Vert \left\Vert \bm{W}_{2}\right\Vert _{\mathrm{F}} \lesssim \lambda_{\max}^{\star1/3} \left\Vert \bm{W}_{2}\right\Vert _{\mathrm{F}}, \\
\Big\|\sum\nolimits _{1\leq s\leq r}\bm{\Delta}_{s}^{\otimes3}\Big\|_{\mathrm{F}} & = \left\Vert \bm{\Delta}\bm{W}_{2}^{\top}\right\Vert _{\mathrm{F}} \leq \left\Vert \bm{\Delta} \right\Vert \left\Vert \bm{W}_{2}\right\Vert _{\mathrm{F}} \leq \left\Vert \bm{\Delta} \right\Vert_\frob \left\Vert \bm{W}_{2}\right\Vert _{\mathrm{F}}.
\end{align*}

It then suffices to control $\left\Vert \bm{W}_{1}\right\Vert _{\mathrm{F}}$ and $\left\Vert \bm{W}_{2}\right\Vert _{\mathrm{F}}$. If $\norm{\bm{\Delta}}_{\frob} \leq \delta \norm{\bm{U}^\star}_{\frob} \leq \delta \sqrt{r} \, \lambda_{\max}^{\star 1/3}$, then it is easy to bound
\begin{align*}
\left\Vert \bm{W}_{1}\right\Vert _{\mathrm{F}}^{2} & =\sum_{1\leq s\leq r}\left\Vert \bm{u}_{s}^{\star}\otimes\bm{\Delta}_{s}\right\Vert _{2}^{2} = \sum_{1\leq s\leq r}\left\Vert \bm{u}_{s}^{\star}\right\Vert _{2}^{2}\left\Vert \bm{\Delta}_{s}\right\Vert _{2}^{2}\leq\max_{1\leq s\leq r}\left\Vert \bm{u}_{s}^{\star}\right\Vert _{2}^{2}\left\Vert \bm{\Delta}\right\Vert _{\mathrm{F}}^{2}\leq \delta^2 r \lambda_{\max}^{\star4/3}, \\
\left\Vert \bm{W}_{2}\right\Vert _{\mathrm{F}}^{2} & =\sum_{1\leq s\leq r}\left\Vert \bm{\Delta}_{s}^{\otimes 2}\right\Vert _{2}^{2} = \sum_{1\leq s\leq r}\left\Vert \bm{\Delta}_{s}\right\Vert _{2}^{4}\leq\max_{1\leq s\leq r}\left\Vert \bm{\Delta}_{s} \right\Vert _{2}^{2}\left\Vert \bm{\Delta}\right\Vert _{\mathrm{F}}^{2} \leq \left\Vert \bm{\Delta}\right\Vert _{\mathrm{F}}^{4}\leq \delta^4 r^2 \lambda_{\max}^{\star4/3}.
\end{align*}
Therefore, one has
\begin{align}
\Big\|\sum\nolimits _{1\leq s\leq r}\bm{u}_{s}^{\star\otimes2}\otimes\bm{\Delta}_{s}\Big\|_{\mathrm{F}} & \lesssim \delta \sqrt{r} \lambda_{\max}^\star,  \label{eq:T-fro-loss-first-term}\\
\Big\|\sum\nolimits _{1\leq s\leq r}\bm{u}_{s}^{\star}\otimes\bm{\Delta}_{s}^{\otimes2}\Big\|_{\mathrm{F}} & \lesssim \delta^2 r \lambda_{\max}^\star, \label{eq:T-fro-loss-second-term} \\
\Big\|\sum\nolimits _{1\leq s\leq r}\bm{\Delta}_{s}^{\otimes3}\Big\|_{\mathrm{F}} & \lesssim \delta^3 r^{3/2}\lambda_{\max}^\star. \label{eq:T-fro-loss-third-term}
\end{align} 
Since $0 \leq \delta \ll  r^{-1} \leq 1$, combining \eqref{eq:T-fro-loss-first-term}, \eqref{eq:T-fro-loss-second-term} and \eqref{eq:T-fro-loss-third-term} with the fact that $\norm{\bm{T^\star}}_\frob \geq \sqrt{r} \lambda_{\min}^\star / 2$ (established in Lemma~\ref{lemma:incoh}), we conclude that
\begin{align}
\norm{\bm{T} - \bm{T}^\star} \lesssim \delta \sqrt{r} \lambda_{\max}^\star \lesssim \delta \norm{\bm{T^\star}}_\frob.
\end{align}

\paragraph{(2) $\ell_{\infty}$ loss.} Next, we turn to the $\left\Vert \cdot\right\Vert _{\infty}$ loss. Again, it suffices to focus on $\sum_{1\leq s\leq r}\bm{u}_{s}^{\star\otimes2}\otimes\bm{\Delta}_{s}$, $\sum_{1\leq s\leq r}\bm{u}_{s}^{\star}\otimes\bm{\Delta}_{s}^{\otimes2}$ and $\sum_{1\leq s\leq r}\bm{\Delta}_{s}^{\otimes3}$.
% Let us start with the first term which is linear in $\bm{\Delta}$. For each $\left(i,j,k\right)\in\left[d\right]^{3}$, we can use the
%Cauchy-Schwartz inequality to upper bound
%\begin{align*}
%\Big|\Big(\sum\nolimits _{1\leq s\leq r}\bm{u}_{s}^{\star\otimes2}\otimes\bm{\Delta}_{s}\Big)_{i,j,k}\Big| & =\Big|\sum\nolimits _{1\leq s\leq r}\left(\bm{u}_{s}^{\star}\right)_{i}\left(\bm{u}_{s}^{\star}\right)_{j}\left(\bm{\Delta}_{s}\right)_{k}\Big|\\
% & \leq\Big(\sum\nolimits _{1\leq s\leq r}\big[\left(\bm{u}_{s}^{\star}\right)_{i}\big]^{2}\Big)^{1/2}\Big(\sum\nolimits _{1\leq s\leq r}\big[\left(\bm{u}_{s}^{\star}\right)_{j}\big]^{2}\big[\left(\bm{\Delta}_{s}\right)_{k}\big]^{2}\Big)^{1/2}\\
% & \leq\max_{1\leq s\leq r}\left\Vert \bm{u}_{s}^{\star}\right\Vert _{\infty}\left\Vert \bm{U}^{\star}\right\Vert _{2,\infty}\left\Vert \bm{\Delta}\right\Vert _{2,\infty}\\
% & \leq \delta \max_{1\leq s\leq r}\left\Vert \bm{u}_{s}^{\star}\right\Vert _{\infty} \left\Vert \bm{U}^{\star}\right\Vert _{2,\infty}^{2},
%\end{align*}
%where we use the condition of the lemma in the last line. Maximizing over $\left(i,j,k\right)$ implies
%that
%\begin{equation}
%\Big\|\sum\nolimits _{1\leq s\leq r}\left(\bm{u}_{s}^{\star}\right)^{\otimes2}\otimes\bm{\Delta}_{s}\Big\|_{\infty}\leq  \delta \frac{\mu^{3/2}r\lambda_{\max}^{\star}}{d^{3/2}}.\label{eq:T-inf-loss-first-term}
%\end{equation}
From \eqref{eq:T-loss-inf-UB-term1}, \eqref{eq:T-loss-inf-UB-term2} and \eqref{eq:T-loss-inf-UB-term3} shown in the proof of Lemma~\ref{lemma:RIC}, one has
\begin{align}
\Big\|\sum\nolimits _{1\leq s\leq r}\left(\bm{u}_{s}^{\star}\right)^{\otimes2}\otimes\bm{\Delta}_{s}\Big\|_{\infty}& \leq \delta \max_{1\leq s\leq r}\left\Vert \bm{u}_{s}^{\star}\right\Vert _{\infty} \left\Vert \bm{U}^{\star}\right\Vert _{2,\infty}^{2} \leq  \delta \frac{\mu^{3/2}r\lambda_{\max}^{\star}}{d^{3/2}},\label{eq:T-inf-loss-first-term} \\
\Big\|\sum\nolimits _{1\leq s\leq r}\bm{u}_{s}^{\star}\otimes\bm{\Delta}_{s}^{\otimes2}\Big\|_{\infty}
 & \leq \delta^2 \max_{1\leq s\leq r}\left\Vert \bm{u}_{s}^{\star}\right\Vert _{\infty}\left\Vert \bm{U}^{\star}\right\Vert _{2,\infty}^{2}
 \leq   \frac{\delta^2 \mu^{3/2}r\lambda_{\max}^{\star}}{d^{3/2}}, \label{eq:T-inf-loss-second-term} \\
\Big\|\sum\nolimits _{1\leq s\leq r}\bm{\Delta}_{s}^{\otimes3}\Big\|_{\infty} & \leq\delta^{3}\left\Vert \bm{U}^{\star}\right\Vert _{2,\infty}^{3}\leq  \frac{\delta^3 \mu^{3/2}r^{3/2}\lambda_{\max}^{\star}}{d^{3/2}}.\label{eq:T-inf-loss-third-term}
\end{align}
 Putting (\ref{eq:T-inf-loss-first-term}), (\ref{eq:T-inf-loss-second-term})
and (\ref{eq:T-inf-loss-third-term}) together with the condition that $0 < \delta \ll r^{-1} \leq 1$, we arrive at
\begin{align*}
\left\Vert \bm{T}-\bm{T}^{\star}\right\Vert _{\infty}\lesssim\frac{\delta \mu^{3/2}r\lambda_{\max}^{\star}}{d^{3/2}}.
\end{align*}
In addition, from the lower bound on $\left\Vert \bm{T}^{\star}\right\Vert _{\mathrm{F}}$, one has
\begin{align*}
\left\Vert \bm{T}^{\star}\right\Vert _{\infty}\geq\frac{1}{d^{3/2}}\left\Vert \bm{T}^{\star}\right\Vert _{\mathrm{F}}\gtrsim\sqrt{\frac{r}{d^{3}}}\,\lambda_{\min}^{\star},
\end{align*}
which allows us to conclude that
\begin{align*}
\left\Vert \bm{T}-\bm{T}^{\star}\right\Vert _{\infty}\lesssim\sqrt{\mu^{3}r}\,\delta\left\Vert \bm{T}^{\star}\right\Vert _{\infty}.
\end{align*}

\section{Auxiliary lemmas}

This section gathers several auxiliary lemmas that prove useful when establishing our main results. 

\subsection{Statements of auxiliary lemmas} 

We begin by stating all auxiliary lemmas formally, with the proofs postponed to subsequent subsections.  We shall define
\begin{align}
	\label{eq:defn-Ubar-matrix-1}
	\overline{\bm{U}}^{\star} := [\overline{\bm{u}}_{1}^{\star},\cdots,\overline{\bm{u}}_{r}^{\star}],\qquad\text{with} \qquad \overline{\bm{u}}_{i}^{\star} := \bm{u}_{i}^{\star}/ \left\|\bm{u}_{i}^{\star} \right\|_{2}.
\end{align}	

\begin{lemma}
\label{lemma:incoh}
	Suppose that Assumption \ref{asmp_U} holds, and assume that $r  \sqrt{ \mu / d } \leq c_1$ for some sufficiently small universal constant $c_2>0$. Then for  $d$ sufficiently large, the matrices $\bm{A}^{\star}$, $\bm{B}^{\star}$ and $\bm{U}_{\mathsf{orth}}^{\star}$ (defined respectively in \eqref{definition:Astar}, \eqref{eq:defn-Bstar-EB} and \eqref{defn:U-orth}) obey
\begin{align*}
& \frac{1}{2} \lambda_{\min}^\star \leq \norm{\bm{A}^{\star}}_\frob \leq 2 \lambda_{\max}^\star, \quad \norm{\bm{A}^{\star}}_\infty  \leq  \frac{\sqrt{2\mu r} \lambda_{\max}^{\star }}{d^{3/2}}, \quad \norm{\bm{A}^{\star}}_{2, \infty} \leq  \sqrt{\frac{2\mu r}{d}} \, \lambda_{\max}^{\star}, \quad \norm{\bm{A}^{\star \top}}_{2, \infty} \leq  \frac{ \mu \sqrt{2r} \, \lambda_{\max}^{\star }}{d}, \\
& \norm{ \bm{A}^\star } =  \lambda_{\max}^{\star} \( 1 + O\left( r \sqrt{ \tfrac{\mu}{d} } \right) \), \qquad
 \lambda_i ( \bm{B}^{\star} ) = \lambda_{(i)}^{\star 2} \( 1 + O\left( r \sqrt{ \tfrac{\mu}{d} } \right) \), ~~ i \in \[r\], \\
& \norm{\bm{B}^{\star}}_{2, \infty}  \leq 2 \sqrt{\frac{\mu r}{d}} \, \lambda_{\max}^{\star 2}, \qquad \left\| \bm{U}^\star_\orth \right\|_{2, \infty} \leq \sqrt{\frac{2\mu r}{d}}, \qquad \| \overline{\bm{U}}^{\star \top}  \overline{\bm{U}}^{\star} - \bm{I} \| \leq r\sqrt{\frac{\mu}{d}}. 
\end{align*}
Here, $\|\bm{A}\|_{2,\infty} := \max_{i} \|\bm{A}_{i,:}\|_2$,  $\lambda_{(i)}^{\star}$ stands for the $i$-th largest value in $\{\lambda_i^{\star}\}_{1\leq i\leq r}$ (or equivalently $\{\|\bm{u}_i^{\star}\|_2^3 \}_{1\leq i\leq r}$), and $\lambda_i(\bm{B}^{\star})$ represents the $i$-th largest eigenvalue of $\bm{B}^{\star}$.  
\end{lemma}

\begin{proof}
See Appendix~\ref{pf:incoh}.
\end{proof}

\begin{lemma}
\label{lemma:T_op_norm}
Let $\bm{R} \in \mathbb{R}^{d \times d \times d}$ be a random order-3 tensor  with independent entries $\left\{ R_{i, j, k} \right\}_{i,j,k\in \[ d\]}$ obeying
\begin{align*}
\mathbb{E} \[ R_{i,j,k} \] = 0, \quad \max_{i, j, k \in \[ d \]} \left| R_{i, j, k} \right| \leq B.
\end{align*}
Define
\begin{align}
\label{def:sigma_mode}
\sigma_\mode^2 :=  \max_{j,k\in[d]} \sum\nolimits_{i\in[d]} \E [ R^2_{i,j,k} ] +  \max_{i,k\in[d]} \sum\nolimits_{j\in[d]} \E [ R^2_{i,j,k} ] + \max_{i,j\in[d]} \sum\nolimits_{k\in[d]}\E [ R^2_{i,j,k} ]  .
\end{align}
Then with probability exceeding $1- O \( d^{-10} \)$, one has
\begin{align}
\left\|  \bm{R} \right\| \lesssim B \log^3 d + \sigma_\mode \log^{5/2} d.
\end{align}
\end{lemma}
\begin{proof}
See Appendix~\ref{pf:T_op_norm}.
\end{proof}

An immediate consequence of this lemma is the following: 
\begin{corollary}
\label{cor:bound-PE}
With probability at least $1- O \( d^{-10} \)$, one has
\begin{align}
\big\| p^{-1} \PP_\Omega ( \bm{T}^\star ) - \bm{T}^\star \big\| & \lesssim \frac{\sqrt{\mu r} \, \lambda_{\max}^\star \log^3 d}{d^{3/2} p } + \frac{ \mu \sqrt{ r} \, \lambda_{\max}^\star \log^{5/2} d}{d \sqrt{p}} ;  \label{claim:T_missing_op_norm}\\
	\big\| \PP_\Omega ( \bm{E} ) \big\| & \lesssim \sigma \big( \log^{7/2}d + \sqrt{dp} \, \log^{5/2} d \big) \label{claim:T_noise_op}. 
%\sigma \sqrt{ \left\{ dp \vee \log d  \right\} \log^5 d} .
\end{align}
\end{corollary}
\begin{proof}
See Appendix~\ref{pf:T_op_norm}. 
%\yxc{Read this later.}
\end{proof}

%\begin{lemma}
%\label{lemma:T_op_norm}
%Suppose that $p \gg \mu^3 r^2 d^{-3/2} \log^4 d$ and $\sigma \ll \lammax^{\star 3} \min \left\{ \sqrt{\frac{p}{d \log^5 d}}, \frac{p}{\log^3 d}  \right\}$. Then with probability exceeding $1- O \( d^{-10} \)$, we have
%\begin{align}
%	\left\| p^{-1} \bm{T} \right\| \lesssim \lammax^{\star 3}.
%%\sigma \sqrt{ \left\{ dp \vee \log d  \right\} \log^5 d} .
%\end{align}
%\end{lemma}
%
%\begin{proof}
%See Appendix~\ref{pf:T_op_norm}.
%\end{proof}

\begin{lemma}
\label{lemma:T_loss_times3_op_norm}
Suppose that $p\gtrsim d^{-2}\log^{3}d$ and that $\mu \log^2 d \lesssim d$.
Then for any fixed vector $\bm{w}\in\mathbb{R}^{d}$, with probability
$1-O\left(d^{-10}\right)$, one has
\begin{align*}
\left\Vert \left(p^{-1}\bm{T}-\bm{T}^{\star}\right)\times_{3}\bm{w}\right\Vert \lesssim\left\Vert \bm{w}\right\Vert _{\infty}\sqrt{\frac{\mu r\log d}{dp}}\,\lambda_{\max}^{\star} + \left\Vert \bm{w}\right\Vert _{\infty} \frac{\sigma \log^{5/2} d}{ p} +\left\Vert \bm{w}\right\Vert _{2}\sigma\sqrt{\frac{d\log d}{p}},
\end{align*}
where $\times_{3}$ is defined in Section \ref{sec:Notation-1}. The results also holds if we replace $\times_{3}$ with $\times_{1}$ or $\times_{2}$.
\end{lemma}
\begin{proof} See Appendix \ref{sec:proof-lemma-T_loss_times3_op_norm}. \end{proof}

%\begin{lemma}
%\label{lemma:indep_gaussian_gap}
%	Let $ X_i \sim \mathcal{N}(0, \sigma_i^2) $ ($1\leq i\leq r$) be a collection of independent zero-mean Gaussian random variables. 
%	Let $\sigma_{\max } \triangleq \max_{i \in [r]} |\sigma_i|$ and $\sigma_{\min} \triangleq \min_{i \in [r]} |\sigma_i|$, and suppose that $\sigma_{\max} \asymp \sigma_{\min}$. Then for any $0< \delta < 1/2$, with probability greater than $1-2\delta$ we have 
%	%\yxc{The first claim is not meaningful when $r=1$...}
%%
%\begin{subequations}
%\begin{align}
%	| X |_{(1)} &\gtrsim \sigma_{\min} \sqrt{ \max \left\{ \log \( r/2 \) - \log \log \( 2/\delta \), \delta^2 \right\}}, \\
%	| X |_{(1)} &\lesssim \sigma_{\max} \( \sqrt{\log \( 2 r \)} + \sqrt{\log \(2/\delta \)} \), \\
%	| X |_{(1)} - | X |_{(2)} &\gtrsim \frac{\delta \sigma_{\min}}{\sqrt{\log(2 r)} + \sqrt{\log(2  /\delta)}},
%\end{align}
%\end{subequations}
%%
%	where $| X |_{(1)}$ and $| X |_{(2)}$ denote respectively the largest and the second largest values in $\{|X_i|\}_{1\leq i\leq r}$. 
%\end{lemma}

\begin{lemma}
\label{lemma:indep_gaussian_gap}
Let $\{ X_{i, j} \}_{1\leq i \leq r, 1\leq j \leq L}$ be a sequence of i.i.d.~standard Gaussian random variables, where $r  \geq 2$ and $L \geq 1$. Consider some quantities $\kappa \geq 1, \Delta > 0, 0<\delta <1/2$. There exists some universal constant $C>0$ such that if 
\begin{align*}
	L \geq C r^{2\kappa^2} \big( \kappa \sqrt{\log r} + \Delta \big) \exp (\Delta^2 ) \log\frac{1}{\delta},
\end{align*}
then with probability at least $1-\delta$, there exists some $1\leq j_0 \leq L$ such that 
\begin{align*}
X_{1, j_0} > \kappa\max_{i : 1 < i \leq r} \left| X_{i, j_0} \right| + \Delta.
\end{align*}
In addition, define $\Delta_j := X_{1, j} - \kappa\max_{i : 1 < i \leq r} \left| X_{i, j} \right|$ for each $1 \leq j \leq L$. Then with probability at least $1- 2\delta$,
\begin{align*}
\Delta_{(1)} - \Delta_{(2)} \gtrsim \frac{\delta}{\sqrt{\log L} + \sqrt{\log \( 1 / \delta \)}}.
\end{align*}
where $\Delta_{(1)} \geq \Delta_{(2)} \geq \dots \geq \Delta_{(L)}$ denote the order statistics of $\left\{ \Delta_j\right\}_{j=1}^L$ in descending order.
\end{lemma}

\begin{proof}
	See Appendix~\ref{pf:indep_gaussian_gap}. 
\end{proof}

\begin{lemma}
\label{lemma:subspace_basis_dist}
	Let $\bm{U}$ (resp.~$\bm{V}$) be a $d\times r$ matrix with orthonormal columns.  Suppose that $\| \bm{U}\bm{U}^{\top} -  \bm{V}\bm{V}^{\top} \| \leq \delta$. Then for any unit vector $\bm{u}_0\in \mathbb{R}^d$ lying in $\mathsf{span}(\bm{U})$, we have
\begin{align}
	\big\| \PP_{\bm{V}} (\bm{u}_0 ) \big\|_2 \geq \sqrt{1-\delta^2} \qquad \text{and} \qquad \big\| \PP_{\bm{V}^\perp} (\bm{u}_0 ) \big\|_2 \leq \delta,
\end{align}
where we denote by $\PP_{\bm{V}} (\bm{u}_0 ) := \bm{V}\bm{V}^\top \bm{u}_0$ and $\PP_{\bm{V}^\perp} (\bm{u}_0 ) = \big( \bm{I}_d - \bm{V}\bm{V}^\top \big) \bm{u}_0$.
\end{lemma}

\begin{proof} 
See Appendix \ref{pf:subspace_basis_dist}. 
\end{proof}

Additionally, we record several facts concerning the set of Bernoulli random variables $\{\chi_{i, j, k} \}_{1\leq i,j,k\leq d}$. We recall that 
\begin{align}
	\chi_{i, j, k} := \ind \{(i,j,k)\in \Omega\},
\end{align}
which is a Bernoulli random variable with mean $p$.  

%\begin{lemma}
%\label{lemma:sum_i_square_cont_ineq}
%Suppose $p \gg d^{-3/2} \log^4 d$, with probability exceeding $1- O \( d^{-10} \)$, for any $1\leq j, k \leq d$,
%\begin{align}
%\sum_{i \in [d]} T^{\star 2}_{i,j,k} ( p^{-1} \chi_{i,j,k} - 1 )^2   
% \lesssim \frac{\mu^3 r^2 \lambda_{\max}^{\star 2} }{d^3 p^2} \left\{  \log d + \frac{dp}{\mu r} \right\} .
%\end{align}
%\end{lemma}
%
%\begin{proof}
%See Appendix~\ref{pf:sum_i_square_cont_ineq}.
%\end{proof}

\begin{lemma}
\label{lemma:sum_ij_square_cont_ineq}
Suppose that $p \gtrsim d^{-2}\log d $. With probability exceeding $1- O \( d^{-10} \)$, one has
\begin{align}
\sum_{i, j \in [d]} T^{\star 2}_{i,j,k} ( p^{-1} \chi_{i, j, k} - 1 )^2  
 \lesssim \frac{\mu r \lambda_{\max}^{\star 2} }{d p} , \qquad 1\leq k\leq d. 
\end{align}
\end{lemma}
\begin{proof}
See Appendix~\ref{pf:sum_ij_square_cont_ineq}.
\end{proof}

%\begin{lemma}
%\label{lemma:sum_i_square_cont_ineq_subg}
%With probability exceeding $1- O \( d^{-10} \)$, one has
%%
%\begin{align}
%\sum_{i \in [d]}  \( p^{-1} E_{i, j, k} \chi_{i, j, k}  \)^2  
%	\lesssim \sigma^2 p^{-2} \left\{ \log^2 d + dp \right\} ,
%\qquad 1\leq j, k \leq d. 
%\end{align}
%%
%\end{lemma}
%
%\begin{proof}
%See Appendix~\ref{pf:sum_i_square_cont_ineq_subg}. 
%\end{proof}

\begin{lemma}
\label{lemma:sum_ij_square_cont_ineq_subg}
Suppose that $p\gtrsim d^{-2}\log^2 d$.  With probability exceeding $1- O \( d^{-10} \)$, one has
\begin{align}
\sum_{i, j \in [d]}  \( p^{-1} E_{i, j, k} \chi_{i, j, k}  \)^2  
	\lesssim  \sigma^{2}d^{2} / p,
\qquad 1\leq k\leq d. 
\end{align}
\end{lemma}
\begin{proof}
See Appendix~\ref{pf:sum_ij_square_cont_ineq_subg}. 
\end{proof}

\begin{lemma}
\label{lemma:sum_jk_square_cont_ineq}
Suppose $p \gtrsim d^{-2} \log d$ and $\mu \log^2 d \lesssim d$. Consider any fixed vector $\bm{w}\in \R^d$ and any $1\leq i\leq d$. With probability exceeding $1- O ( d^{-10} )$, one has
\begin{align}
\sum_{j \in [d]} \Big(\sum_{k \in [d]} T^\star_{i,j,k} w_{k} ( p^{-1} \chi_{i, j, k} - 1) \Big)^2 
 \lesssim \frac{\mu r \lammax^{\star 2}}{dp} \norm{\bm{w}}_\infty^2.
\end{align}
\end{lemma}

\begin{proof}
See Appendix~\ref{pf:sum_jk_square_cont_ineq}.
\end{proof}

\begin{lemma}
\label{lemma:sum_jk_square_cont_ineq_subg}
Consider any fixed vector $\bm{w} \in \R^d$.  With probability  $1- O ( d^{-10} )$ one has
\begin{align}
 \sum_{j \in [d]} \Big(\sum_{k \in [d]} w_{k}  E_{i, j, k} \chi_{i, j, k} \Big)^2 \lesssim \sigma^2 d p \norm{\bm{w}}_2^2 + \sigma^{2}\left\|\bm{w} \right\|_{\infty}^{2}\log^5 d, \qquad 1\leq i \leq d. 
\end{align}
\end{lemma}

\begin{proof}
See Appendix~\ref{pf:sum_jk_square_cont_ineq_subg}.
\end{proof}

\subsection{Proof of Lemma~\ref{lemma:incoh}}
\label{pf:incoh}

\begin{enumerate}
	\item To begin with, the incoherence condition \eqref{asmp_incoh} gives
%
%\begin{align}
%\|\bm{A}^{\star}\|_{\infty}=\|\bm{T}^{\star}\|_{\infty}=\Big\|\sum_{i=1}^{r}\bm{u}_{i}^{\star\otimes3}\Big\|_{\infty}
%	\leq r\max_{i\in[r]}\|\bm{u}_{i}^{\star}\|_{\infty}^{3}
%	\leq r\max_{i\in[r]} \Big( \sqrt{\frac{\mu}{d}}  \|\bm{u}_{i}^{\star}\|_2 \Big) ^{3}
%	\leq \frac{\mu_{1}^{3/2}r\lambda_{\max}^{\star3}}{d^{3/2}}.
%\end{align}
%
\begin{align*}
\|\bm{T}^{\star}\|_{\mathrm{F}}^2 &= \Big\langle \sum\nolimits_{i\in[r]} \bm{u}_{i}^{\star\otimes3}, \sum\nolimits_{i\in[r]} \bm{u}_{i}^{\star\otimes3} \Big\rangle \\
& = \sum_{1 \leq i \leq r} \norm{\bm{u}_{i}^{\star \ot 3}}_{\mathrm{F}}^2 + \sum_{1 \leq i \neq j \leq r} \left\langle \bm{u}_{i}^{\star\otimes3}, \bm{u}_{j}^{\star\otimes3} \right\rangle \\
& \overset{(\mathrm{i})}{=} \sum_{1 \leq i \leq r} \norm{\bm{u}_{i}^{\star}}_2^6 + \sum_{1 \leq i \neq j \leq r}\left\langle \bm{u}_{i}^{\star}, \bm{u}_{j}^{\star} \right\rangle^3 \\
& \leq r \max_{1 \leq i \leq r} \norm{\bm{u}_{i}^{\star}}_2^6 + r^2 \max_{1 \leq i \neq j \leq r} \big|  \left\langle \bm{u}_{i}^{\star}, \bm{u}_{j}^{\star} \right\rangle^3 \big| \\
& \leq r  \lambda_{\max}^{\star 2} + r^2  \( \frac{\mu}{d} \)^{3/2}  \lambda_{\max}^{\star 2} \\
&\overset{(\mathrm{ii})}{\leq}  2r  \lambda_{\max}^{\star 2},
\end{align*}
where we use the fact that $\inprod{\bm{u}^{\ot 3} }{\bm{v}^{\ot 3}} = \inprod{\bm{u}}{\bm{v}}^3$ in (i), and (ii) arises due to the condition that $3r \leq \sqrt{d / \mu} \leq (d / \mu)^{3/2}$. Using a similar argument, we also know that $\|\bm{T}^{\star}\|_{\mathrm{F}}^2 \geq r  \lambda_{\min}^{\star 2} / 2$. This combined with the incoherence condition in (\ref{asmp_max}) yields
\begin{align*}
\norm{\bm{A}^\star}_\infty = \norm{\bm{T}^\star}_\infty \leq \sqrt{\frac{\mu}{d^3}} \norm{\bm{T}^\star}_{\mathrm{F}} \leq \frac{\sqrt{2 \mu r}}{d^{3/2}} \lambda_{\max}^{\star}.
\end{align*}

\item For any $1\leq i \leq d$, the $\ell_2$ norm of the $i$-th row of $\bm{A}^{\star}$ can be bounded by
\begin{align}
\big\| \bm{A}^{\star}_{i, :} \big\|_2^2 &= \Big\| \sum\nolimits_{1 \leq s \leq r} \big( \bm{u}_s^\star \big)_i \big( \bm{u}_s^\star \ot \bm{u}_s^\star \big)^\top \Big\|_2^2 \nonumber\\
& = \sum_{1 \leq s \leq r} \big( \bm{u}_s^\star \big)_i^2 \norm{\bm{u}_s^\star \ot \bm{u}_s^\star}_2^2 
+ \sum_{1 \leq s_1 \neq s_2 \leq r} \big( \bm{u}_{s_1}^\star \big)_i \big( \bm{u}_{s_2}^\star \big)_i \inprod{\bm{u}_{s_1}^\star \ot \bm{u}_{s_1}^\star }{ \bm{u}_{s_2}^\star \ot \bm{u}_{s_2}^\star } \nonumber\\
& \leq r  \,  \max_{s \in [r]}\| \bm{u}^\star_s \|_\infty^2  \max_{s \in [r]}\| \bm{u}^\star_s \|_2^4 + r^2  \,  \max_{1 \leq s \leq r}\| \bm{u}^\star_s \|_\infty^2  \max_{s_1 \neq s_2} \inprod{\bm{u}_{s_1}^\star}{\bm{u}_{s_2}^\star}^2 \nonumber\\
& \leq \lambda_{\max}^{\star 2} \(\frac{\mu r}{d} + \frac{\mu^2 r^2}{d^2} \)  \leq \frac{2 \mu r \lambda_{\max}^{\star 2} }{d},  \label{A_row_norm}
\end{align}
where the first inequality follows from the fact that $\inprod{\bm{u} \ot \bm{u}}{\bm{v} \ot \bm{v}} = \inprod{\bm{u}}{\bm{v}}^2$, the second inequality holds true due to \eqref{asmp_incoh} and  \eqref{asmp_corr}, and the last inequality holds as long as $r \leq d / \mu$. This immediately yields the advertised bound on 
$\big\| \bm{A}^{\star} \big\|_{2,\infty}$.

\item For any $\vec{j} \in \[d\]^2$ (which corresponds to $(j_1,j_2)$), the $\ell_2$ norm of the $\vec{j}$-th column of $\bm{A}^{\star}$ can be upper bounded similarly by
\begin{align}
\Big\| \bm{A}^{\star}_{:, \vec{j}} \Big\|_2^2 &= \Big\| \sum\nolimits_{1 \leq s \leq r} \big( \bm{u}_s^\star \ot \bm{u}_s^\star \big)_{\vec{j}} \bm{u}_s^\star \Big\|_2^2 \nonumber \\
& = \sum_{1 \leq s \leq r} \big( \bm{u}_s^\star \big)_{j_1}^2 \big( \bm{u}_s^\star \big)_{j_2}^2 \| \bm{u}^\star_s \|_2^2
+ \sum_{s_1 \neq s_2} \big( \bm{u}_{s_1}^\star \big)_{j_1} \big( \bm{u}_{s_1}^\star \big)_{j_2} \big( \bm{u}_{s_2}^\star \big)_{j_1} \big( \bm{u}_{s_2}^\star \big)_{j_2} \inprod{\bm{u}_{s_1}^\star }{ \bm{u}_{s_2}^\star } \nonumber \\
& \leq  r \, \max_{1 \leq s \leq r}\| \bm{u}^\star_s \|_\infty^4 \max_{s \in [r]}\| \bm{u}^\star_s \|_2^2 + r^2 \, \max_{1 \leq s \leq r}\| \bm{u}^\star_s \|_\infty^4 \max_{1 \leq s_1 \neq s_2 \leq r} \big| \inprod{\bm{u}_{s_1}^\star}{\bm{u}_{s_2}^\star} \big|  \nonumber \\
& \leq \lambda_{\max}^{\star 2} \( \frac{\mu^2 r }{d^2} + \frac{\mu^{5/2} r^2 }{d^{5/2}} \) 
 \leq  \frac{2 \mu^2 r \lambda_{\max}^{\star 2} }{d^2},
\label{A_col_norm}
\end{align}
 where the second inequality is valid due to \eqref{asmp_incoh} and \eqref{asmp_corr}, and the last inequality holds as long as $r \leq \sqrt{d / \mu}$. This yields the claimed bound regarding $\| \bm{A}^{\star \top} \|_{2, \infty}$. 
 
 \item Regarding the spectrum of $\bm{A}^\star$, $\bm{B}^\star$, $\bm{U}_\orth^\star$ and $\overline{\bm{U}}^\star$, we refer the reader to the proof of \cite[Corollary~1]{cai2019subspace}.

\item
We now move on to $\norm{\bm{B}^\star}_{2, \infty}$. For any $i \in \[d\]$, it is seen that
\begin{align*}
\norm{\bm{B}^\star_{i, :}}_2^2 
&= \sum_{j \in [d]} \( \bm{A}_{i, :}^{\star} \bm{A}_{j, :}^{\star \top} \)^2 
%= \sum_{j \in [d]} \bm{A}_{i, :}^{\star} \bm{A}_{j, :}^{*\top} \bm{A}_{j, :}^{\star} \bm{A}_{i, :}^{*\top}  
	= \bm{A}_{i, :}^{\star} \Big( \sum_{j \in [d]}  \bm{A}_{j, :}^{*\top} \bm{A}_{j, :}^{\star} \Big)  \bm{A}_{i, :}^{*\top}  
%&=  \bm{A}_{i, :}^{*\top} \bm{A}_{i, :}^{\star} \bm{A}^{\star \top}  \bm{A}^{\star} \) 
	 = \bm{A}_{i, :}^{\star} \( \bm{A}^{\star \top}  \bm{A}^{\star} \) \bm{A}_{i, :}^{*\top}  \\
	& \leq \left\| \bm{A}^{\star \top}  \bm{A}^{\star} \right\| \norm{\bm{A}_{i, :}^{\star}}_2^2 
 \leq 2 \lambda_{\max}^{\star 2} \norm{\bm{A}_{i, :}^{\star}}_2^2.
\end{align*}
Here, the last line makes use of the bound $\|\bm{A}^{\star} \|\leq \lambda_{\max}^{\star} (1+O(r\sqrt{\mu/d}))  \leq 2\lambda_{\max}^{\star}$, which holds if  $ r\sqrt{\mu/d} \leq c_1$ for some sufficiently small constant $c_1>0$. 
It then follows from \eqref{A_row_norm} that
\begin{align}
\norm{\bm{B}^\star}_{2, \infty} \leq \sqrt{2} \, \lambda_{\max}^{\star} \norm{\bm{A}^{\star}}_{2, \infty} \leq 2 \lambda_{\max}^{\star 2} \sqrt{\frac{\mu r}{d}}.
\end{align}

%\item 
%When it comes to the last claim, recall that
%$\bm{U}_{\mathsf{orth}}^{\star}=\overline{\bm{U}}^{\star}\big(\overline{\bm{U}}^{\star\top}\overline{\bm{U}}^{\star}\big)^{-1/2}$. 
%%
%This allows us to bound
%%
%\begin{align*}
%\big\|\bm{U}_{\mathsf{orth}}^{\star}\big\|_{2,\infty} & \leq\big\|\overline{\bm{U}}^{\star}\big\|_{2,\infty}\big\|\big(\overline{\bm{U}}^{\star\top}\overline{\bm{U}}^{\star}\big)^{-1/2}\big\|\leq\sqrt{\frac{\mu r}{d}}\sqrt{\frac{1}{\lambda_{\min}\big(\overline{\bm{U}}^{\star\top}\overline{\bm{U}}^{\star}\big)}}\leq\sqrt{\frac{\mu r}{d}}\sqrt{\frac{1}{1-1/3}}\leq \sqrt{\frac{2\mu r}{d}},
%\end{align*}
%%
%where the penultimate inequality results from \eqref{eq:Ubar-Ubarstar-spectrum} and the assumption that $3r\leq \sqrt{d/\mu}$. 

% %
% \begin{align*}
% B^\star_{i,i} = \lambda_i ( \bm{B}^\star) \, \big\| \big(\bm{U}^\star_\orth \big)_{i, :} \big\|_2^2 = \big\| \big( \bm{A}^\star \big)_{i, :} \big\|_2^2
% \end{align*}
% %
% for each $i \in \[d\]$. It follows from \eqref{A_row_norm} and \eqref{eigval_B} that
% %
% \begin{align}
% \left\| \bm{U}^\star_\orth \right\|_{2, \infty} \asymp   \sqrt{\frac{\mu r}{d}}.
% \end{align}

\end{enumerate}

%%%%%%%%%%%%%%%%%%%%%%%%%%%%%%%%%%%%%%%%%%%%%%%%%%%%%%%%%%%%%%%%%%%

\subsection{Proof of Lemma~\ref{lemma:T_op_norm} and Corollary \ref{cor:bound-PE}}
\label{pf:T_op_norm}

\subsubsection{Proof of Lemma~\ref{lemma:T_op_norm}}

We start by making the following simple observation: the tensor spectral norm is a $1$-Lipschitz function (w.r.t.~the Frobenius norm) of the entries of the tensor. This follows since $\big| \|\bm{T}\| - \| \bm{R} \| \big| \leq \| \bm{T} - \bm{R} \| \leq \| \bm{T} - \bm{R} \|_\frob$ holds for any tensor $\bm{T}, \bm{R} \in \R^{d \times d \times d}$. This allows us to invoke standard concentration results regarding functions of independent random variables.  

We shall first develop an upper bound on the mean $\E \big[ \big\| \PP_\Omega ( \bm{R} ) \big\| \big]$. 
In view of \cite[Corollary 4]{nguyen2015tensor} and Jensen's inequality, one has
\begin{align}
	& \E \big[ \left\| \bm{R} \right\| \big]  \leq \sqrt{\E \big[ \left\| \bm{R} \right\|^2 \big] }  \nonumber\\
	& \quad \lesssim \Bigg( \E \Bigg[ \max_{j, k \in [d]} \sum_{i \in [d]} R_{i, j, k}^2 + \max_{i,j \in [d]} \sum_{k \in [d]} R_{i, j, k}^2 + \max_{i, k \in [d]} \sum_{j \in [d]} R_{i, j, k}^2 \Bigg] \Bigg)^{1/2} \log^{5/2} d.
\label{eq:bound-nguyen2015tensor}
\end{align}
We then need to bound the quantity presented in \eqref{eq:bound-nguyen2015tensor}.

For some $\beta > 0$ to be specified later, one can upper bound
\begin{align}
\E \Big[ \max\nolimits_{j, k \in [d]} \sum\nolimits_{i \in [d]} R_{i, j, k}^2  \Big] & = \int_0^\infty \P \Big\{ \max_{j, k \in [d]} \sum\nolimits_{i \in [d]} R_{i, j, k}^2  > t \Big\} \, \mathrm{d} t  \nonumber\\
& \leq \beta + \int_\beta^\infty \P \Big\{ \max_{j, k \in [d]} \sum\nolimits_{i \in [d]} R_{i, j, k}^2  > t \Big\} \, \mathrm{d} t \nonumber \\
& \leq \beta + d^2 \int_{\beta}^\infty \P \Big\{ \sum\nolimits_{i \in [d]} R_{i, j, k}^2  > t \Big\} \, \mathrm{d} t . 
\label{eq:Emax-entries-UB}
\end{align}
We shall resort to the Bernstein inequality to bound $\P \big\{ \sum\nolimits_{i \in [d]} R_{i, j, k}^2 > t \big\}$. It is straightforward to compute that
\begin{align*}
M &:= \sum_{i\in \[ d \]} \E \[ R_{i, j, k}^2 \] \leq \sigma_\mode^2, \\
L &:=  \max_{i \in \[ d \]} \left| R_{i, j, k}^2\right| \leq B^2,  \\
S^2 &:= \sum_{i\in \[ d \]} \E \[ R_{i, j, k}^4 \] \leq B^2 \sigma_\mode^2.
\end{align*}
The Bernstein inequality then tells us that 
\begin{align}
	\label{eq:Berstein-100}
\mathbb{P}\left\{ \sum\nolimits_{i \in [d]} R_{i, j, k}^2 - M >t \right\} \leq \exp\left(-\frac{3}{8}\min\left\{ \frac{t^{2}}{S^{2}},\frac{t}{L}\right\} \right),\quad t>0. 
\end{align}
In particular, this implies that with probability exceeding $1-O\left(d^{-20}\right)$, 
\begin{align*}
\sum\nolimits_{i \in [d]} R_{i, j, k}^2 &\lesssim M + L \log d + S \sqrt{\log d} \lesssim \sigma_\mode^2 + B^2 \log d + \sqrt{ B^2 \sigma_\mode^2 \log d} \\
& \asymp \sigma_\mode^2 + B^2 \log d,
\end{align*}
where we have used the AM-GM inequality in the last step.
Therefore, by taking
\begin{align*}
	\beta:=C \( \sigma_\mode^2 + B^2 \log d \)^{1/2}
\end{align*}
for some sufficiently large constant $C>0$, we arrive at
\begin{align}
	\beta\geq\frac{C}{3}\left(M + L\log d+S\sqrt{\log d}\right) \gg M + L\log d+S\sqrt{\log d}. \label{eq:beta-LB1}
\end{align}
Given that $\beta \gg M$,  for any $t \geq \beta$ one has the following relations about several events
\begin{align}
\left\{ \sum\nolimits_{i \in [d]} R_{i, j, k}^2  > t  \right\} = \left\{ \sum\nolimits_{i \in [d]} R_{i, j, k}^2   - M > t - M  \right\} \subset \left\{ \sum\nolimits_{i \in [d]} R_{i, j, k}^2  - M > t - \beta / 2 \right\}
	\label{eq:relation-3events}
\end{align}
In addition, it is easily seen that 
\begin{align}
	\label{eq:t-S-LB}
	\min\left\{ t^{2}/S^{2},t/L\right\} \geq \frac{t}{ \max\left\{ S/\sqrt{\log d},L\right\}}  
\end{align}
for any $t\geq\beta$ (with $\beta$ obeying \eqref{eq:beta-LB1}). As a result, one can bound
 \begin{align*}
 \int_{\beta}^\infty \P \Big\{ \sum\nolimits_{i \in [d]} R_{i, j, k}^2  > t \Big\} \, \mathrm{d} t 
  & \overset{\mathrm{(i)}}{\leq}  \int_{\beta}^\infty \P \Big\{ \sum\nolimits_{i \in [d]} R_{i, j, k}^2 -M > t - \beta/2 \Big\} \, \mathrm{d} t \\
  & = \int_{\beta/2}^\infty \P \Big\{ \sum\nolimits_{i \in [d]} R_{i, j, k}^2 -M > t \Big\} \, \mathrm{d} t \\
  &  \overset{\mathrm{(ii)}}{\leq}  \int_{\beta / 2}^\infty  \exp\left(-\frac{3}{8}\min\left\{ \frac{t^{2}}{S^{2}},\frac{t}{L}\right\} \right)  \mathrm{d} t \\
  & \overset{\mathrm{(iii)}}{\leq}  \int_{\beta / 2}^{\infty}\exp\left(-\frac{3}{8}\frac{t}{\max\left\{ S/\sqrt{\log d},L\right\} }\right)\mathrm{d}t\\
  & \lesssim   \max\left\{ S/\sqrt{\log d},L\right\} \exp\left(-\frac{3}{16}\frac{\beta}{\max\left\{ S/\sqrt{\log d},L\right\} }\right)\\
  & \overset{\mathrm{(iv)}}{\lesssim}   \beta \exp\left(-\frac{1}{16} C \log d\right) \ll \beta / d^2, 
 \end{align*}
 where (i) follows from \eqref{eq:relation-3events}, (ii) comes from \eqref{eq:Berstein-100}, (iii) is a consequence of \eqref{eq:t-S-LB}, and (iv) holds true when $C>0$ is sufficiently large. 
 Consequently, 
 \begin{align*}
 \E \Big[ \max\nolimits_{j, k \in [d]} \sum\nolimits_{i \in [d]} R_{i, j, k}^2 \Big] \lesssim \beta \lesssim B \sqrt{\log d} + \sigma_\mode.
 \end{align*}
Clearly, the same bound holds for $\E \Big[ \max\nolimits_{i, k \in [d]} \sum\nolimits_{j \in [d]} R_{i, j, k}^2 \Big]$ and $\E \Big[ \max\nolimits_{i,j \in [d]} \sum\nolimits_{k \in [d]} R_{i, j, k}^2 \Big]$.

Substitution into \eqref{eq:bound-nguyen2015tensor} yields 
\begin{align}
\E \big[ \left\| \bm{R} \right\| \big] \lesssim B \log^3 d + \sigma_\mode \log^{5/2} d.
\end{align}
Recognizing that the magnitudes of all entries of $ \bm{R} $ are bounded by $B$, we can invoke Talagrand's concentration inequality \cite[Theorem 5.2.16]{vershynin2018high} for convex Lipschitz functions of independent bounded random variables to show that  with probability $1-O(d^{-10})$,
\begin{align*}
\left| \left\| \bm{R} \right\|- \mathbb{E}\big[\left\| \bm{R} \right\|\big] \right| \lesssim B \sqrt{\log d}
\end{align*}
and, therefore, 
\begin{align}
 \label{eq:T_op_norm_temp}
  \left\| \bm{R} \right\|  \lesssim B \log^3 d + \sigma_\mode \log^{5/2} d.
\end{align}

\subsubsection{Proof of Corollary~\ref{cor:bound-PE}}

Now we apply Lemma~\ref{lemma:T_op_norm} to our concrete setting. We first look at $p^{-1} \PP_{\Omega} \( \bm{T}^\star \) - \bm{T}^\star$ and treat it as $\bm{R}$ in Lemma~\ref{lemma:T_op_norm}. With the help of Lemma~\ref{lemma:incoh}, it is straightforward to compute that
\begin{align*}
\max_{i, j, k \in \[ d \]} \left| T_{i, j, k}^{\star} \( p^{-1} \chi_{i,j,k} - 1 \) \right| & \lesssim \frac{1}{p} \left\|\bm{A}^\star \right\|_\infty \lesssim \frac{\sqrt{\mu r} \, \lambda_{\max}^\star}{d^{3/2} p},
\end{align*}
and
\begin{align*}
 \max_{i, j \in \[ d \]}\sum_{k \in \[ d \]} \E \big[ T_{i, j, k}^{\star 2} ( p^{-1} \chi_{i,j,k} - 1 )^2 \big] &\lesssim \frac{1}{p}\norm{\bm{A}^{\star \top}}_{2,\infty}^2 \lesssim \frac{\mu^2 r \lambda_{\max}^{\star 2}}{d^2 p}.
\end{align*}
Clearly, $\max\nolimits_{i, k \in [ d ]}\sum\nolimits_{j \in [ d ]} \E \big[ T_{i, j, k}^{\star 2} ( p^{-1} \chi_{i,j,k} - 1 )^2 \big]$ and $\max\nolimits_{j, k \in \[ d \]}\sum\nolimits_{i \in \[ d \]} \E \big[T_{i, j, k}^{\star 2} ( p^{-1} \chi_{i,j,k} - 1 )^2 \big]$ can be controlled in the same way.
Substitution into (\ref{eq:T_op_norm_temp}) proves the claim \eqref{claim:T_missing_op_norm}.

We then turn to  $\PP_\Omega \( \bm{E} \)$. Recognizing that the entries of $\bm{E}$ might be unbounded, we invoke the following truncation trick to cope with this unboundedness issue.  Specifically,  define $\widetilde{\bm{E}} = [\widetilde{E}_{i,j,k}]_{1\leq i,j,k\leq d}$ where
\begin{align}
	\widetilde{E}_{i,j,k} := {E}_{i,j,k} \ind \big\{ \left|{E}_{i,j,k} \right| \leq c_1\sigma  \sqrt{\log d}  \big\}, \qquad 1\leq i,j,k \leq d
\end{align}
for some some sufficiently large constant $c_1>0$. Moreover, $\widetilde{E}_{i,j,k}$ is zero-mean because we assume that the distribution of $E_{i,j,k}$ is symmetric about 0. Standard concentration inequalities reveal that: with probability exceeding $1-O \(d^{-10}\)$, one has $\bm{E} = \widetilde{\bm{E}}$. Hence, it suffices to bound $\big\| \PP_\Omega ( \widetilde{\bm{E}} ) \big\|$. Towards this end, simple calculation reveals that
\begin{align*}
	B &= \max_{i, j, k \in \[ d \]} \big| \widetilde{E}_{i,j,k} \chi_{i,j,k} \big|  \lesssim  \big\|  \widetilde{\bm{E}} \big\|_\infty \lesssim \sigma \sqrt{\log d} , \\
	\sigma_\mode^2 & \leq \max_{i, j \in \[ d \]}\sum_{k \in \[ d \]} \E \big[ {E}^2_{i,j,k} \chi_{i,j,k} \big] + \max_{i, k \in \[ d \]}\sum_{j \in \[ d \]} \E \big[ {E}^2_{i,j,k} \chi_{i,j,k} \big] + \max_{ j,k \in \[ d \]}\sum_{i \in \[ d \]} \E \big[ {E}^2_{i,j,k} \chi_{i,j,k} \big]  \lesssim p \sigma^2 d .
 \end{align*}
 %
%where the second inequality also holds for $\max\nolimits_{i, k \in \[ d \]}\sum\nolimits_{j \in \[ d \]} \E \big[ \widetilde{E}^2_{i,j,k} \chi_{i,j,k} \big]$ and $\max\nolimits_{j, k \in \[ d \]}\E \big[ \widetilde{E}^2_{i,j,k} \chi_{i,j,k} \big]$.  
This together with \eqref{eq:T_op_norm_temp} as well as the high-probability event $\bm{E}=\widetilde{\bm{E}}$  completes the proof. 

%%%%%%%%%%%%%%%%%%%%%%%%%%%%%%%%%%%%%%%%%%%%%%%%%

\subsection{Proof of Lemma~\ref{lemma:T_loss_times3_op_norm}} \label{sec:proof-lemma-T_loss_times3_op_norm}

For notational simplicity, let us denote
\begin{align*}
\bm{X} & :=\left(p^{-1}\bm{T}-\bm{T}^{\star}\right)\times_{3}\bm{w}.
\end{align*}
Observe that $\bm{X}$ is a zero-mean random matrix in $\mathbb{R}^{d\times d}$
with independent entries
\begin{align*}
X_{i,j}=\sum_{k\in\left[d\right]}w_{k}\left\{ T_{i,j,k}^{\star}\left(p^{-1}\chi_{i,j,k}-1\right)+p^{-1}E_{i,j,k}\chi_{i,j,k}\right\} ,\quad\left(i,j\right)\in\left[d\right]^{2} .
\end{align*}
We shall apply the truncated matrix Bernstein inequality to control
the spectral norm of $\bm{X}$.
\begin{itemize}
\item First, it is straightforward to bound
\begin{align*}
V & :=\max\Big\{\max_{i\in\left[d\right]}\sum_{j\in\left[d\right]}\mathbb{E}\left[X_{i,j}^{2}\right],\max_{j\in\left[d\right]}\sum_{i\in\left[d\right]}\mathbb{E}\left[X_{i,j}^{2}\right]\Big\}\\
 & \,=\max\Big\{\max_{i\in\left[d\right]}\sum_{j,k\in\left[d\right]}p^{-1}w_{k}^{2}\left(T_{i,j,k}^{\star2}+\mathbb{E}\left[E_{i,j,k}^{2}\right]\right),\max_{j\in\left[d\right]}\sum_{i,k\in\left[d\right]}p^{-1}w_{k}^{2}\left(T_{i,j,k}^{\star2}+\mathbb{E}\left[E_{i,j,k}^{2}\right]\right)\Big\}\\
 & \,\leq\frac{1}{p}\left(\left\Vert \bm{w}\right\Vert _{\infty}^{2}\left\Vert \bm{A}^{\star}\right\Vert _{2,\infty}^{2}+\left\Vert \bm{w}\right\Vert _{2}^{2}\sigma^{2}d\right).
\end{align*}
\item Second, using the same truncation argument as in the proof of Lemma~\ref{lemma:T_op_norm} in Appendix~\ref{pf:T_op_norm}, we can assume $| E_{i,j,k} | \lesssim \sigma \sqrt{\log d}$ for all $1 \leq i,j,k \leq d$ (which holds with very high probability). The Bernstein
inequality reveals that
\begin{align*}
\mathbb{P}\left\{ \left|X_{i,j}\right|>t\right\} \leq2\exp\left(-\frac{3}{8}\min\left\{ \frac{t^{2}}{S^{2}},\frac{t}{L}\right\} \right),\quad t>0
\end{align*}
for each $\left(i,j\right)\in\left[d\right]^{2}$, where
\begin{align*}
	L & :=\max_{k\in\left[d\right]} \left\{ \left|w_{k}\right|\left|T_{i,j,k}^{\star}\left(p^{-1}\chi_{i,j,k}-1\right)+p^{-1}E_{i,j,k}\chi_{i,j,k}\right| \right\}
	\lesssim\frac{1}{p}\left\Vert \bm{w}\right\Vert _{\infty}\left(\left\Vert \bm{A}^{\star}\right\Vert _{\infty}+\sigma\sqrt{\log d}\right);\\
S^{2} & :=\mathbb{E}\left[X_{i,j}^{2}\right] \asymp \sum_{k\in\left[d\right]}p^{-1}w_{k}^{2}\left(T_{i,j,k}^{\star2}+\mathbb{E}\left[E_{i,j,k}^{2}\right]\right)\leq\frac{1}{p}\left(\left\Vert \bm{w}\right\Vert _{\infty}^{2}\left\Vert \bm{A}^{\star\top}\right\Vert _{2,\infty}^{2}+\left\Vert \bm{w}\right\Vert _{2}^{2}\sigma^{2}\right).
\end{align*}
This implies that with probability exceeding $1-O\left(d^{-20}\right)$,
\begin{align*}
	\max_{i,j\in\left[d\right]}\left|X_{i,j}\right|& \lesssim L\log d+S\sqrt{\log d} \\
	& \lesssim\frac{\left\Vert \bm{w}\right\Vert _{\infty}\log d}{p}\left(\left\Vert \bm{A}^{\star}\right\Vert _{\infty}+\sigma\sqrt{\log d}\right)+\sqrt{\frac{\log d}{p}}\left(\left\Vert \bm{w}\right\Vert _{\infty}\left\Vert \bm{A}^{\star\top}\right\Vert _{2,\infty}+\left\Vert \bm{w}\right\Vert _{2}\sigma\right).
\end{align*}
Therefore, if we choose
\begin{align*}
\beta:=C\left\{ \frac{\left\Vert \bm{w}\right\Vert _{\infty}\log d}{p}\left(\left\Vert \bm{A}^{\star}\right\Vert _{\infty}+\sigma\sqrt{\log d}\right)+\sqrt{\frac{\log d}{p}}\left(\left\Vert \bm{w}\right\Vert _{\infty}\left\Vert \bm{A}^{\star\top}\right\Vert _{2,\infty}+\left\Vert \bm{w}\right\Vert _{2}\sigma\right)\right\} 
\end{align*}
for some sufficiently large constant $C>0$, then one has
\begin{align*}
\beta\geq\frac{C}{2}\left(L\log d+S\sqrt{\log d}\right).
\end{align*}
\item Third, it is easy to bound
\begin{align*}
\mathbb{E}\big[\left|X_{i,j}\right|\mathds{1}\left\{ \left|X_{i,j}\right|\geq\beta\right\} \big] & \leq\beta\cdot\mathbb{P}\left\{ \left|X_{i,j}\right|\geq\beta\right\} +\int_{\beta}^{\infty}\mathbb{P}\left\{ \left|X_{i,j}\right|\geq t\right\} \mathrm{d}t\\
 & \leq\beta\cdot O\left(d^{-20}\right)+\int_{\beta}^{\infty}\mathbb{P}\left\{ \left|X_{i,j}\right|\geq t\right\} \mathrm{d}t.
\end{align*}
In view of our choice of $\beta$, we know that $\min\left\{ t^{2}/S^{2},t/L\right\} \geq t/\max\left\{ S/\sqrt{\log d},L\right\} $
for any $t\geq\beta$. As a result, for  $d$ sufficiently large, we
have
\begin{align*}
\int_{\beta}^{\infty}\mathbb{P}\left\{ \left|X_{i,j}\right|\geq t\right\} \mathrm{d}t & \leq2\int_{\beta}^{\infty}\exp\left(-\frac{3}{8}\min\left\{ \frac{t^{2}}{S^{2}},\frac{t}{L}\right\} \right)\mathrm{d}t\\
 & \leq2\int_{\beta}^{\infty}\exp\left(-\frac{3}{8}\frac{t}{\max\left\{ S/\sqrt{\log d},L\right\} }\right)\mathrm{d}t\\
 & \lesssim\max\left\{ S/\sqrt{\log d},L\right\} \exp\left(-\frac{3}{8}\frac{\beta}{\max\left\{ S/\sqrt{\log d},L\right\} }\right)\\
 & \lesssim\max\left\{ S/\sqrt{\log d},L\right\} \exp\left(-\frac{3}{8}C\log d\right)
	\ll\frac{\beta}{d^{2}}.
\end{align*}
Consequently, we have established that 
\begin{align*}
	q:=  \sum_{i,j} \mathbb{E}\big[\left|X_{i,j}\right|\mathds{1}\left\{ \left|X_{i,j}\right|\geq\beta\right\} \big]\ll \beta .
\end{align*}
\end{itemize}

Invoke the  matrix Bernstein inequality to demonstrate that
with probability  $1-O\left(d^{-10}\right)$,
\begin{align*}
\left\Vert \bm{X}\right\Vert  & \lesssim q+\beta\log d+\sqrt{V\log d}\asymp\beta\log d+\sqrt{V\log d}\\
 & \lesssim\frac{\left\Vert \bm{w}\right\Vert _{\infty}\log^{2}d}{p}\left(\left\Vert \bm{A}^{\star}\right\Vert _{\infty}+\sigma\sqrt{\log d}\right)+\frac{\log^{3/2}d}{\sqrt{p}}\left(\left\Vert \bm{w}\right\Vert _{\infty}\left\Vert \bm{A}^{\star\top}\right\Vert _{2,\infty}+\sigma\left\Vert \bm{w}\right\Vert _{2}\right)\\
 & \quad+\sqrt{\frac{\log d}{p}}\left(\left\Vert \bm{w}\right\Vert _{\infty}\left\Vert \bm{A}^{\star}\right\Vert _{2,\infty}+\left\Vert \bm{w}\right\Vert _{2}\sigma\sqrt{d}\right)\\
	& \asymp\left\Vert \bm{w}\right\Vert _{\infty} \left\{ \frac{\left\Vert \bm{A}^{\star}\right\Vert _{\infty}\log^{2}d}{p}+\frac{\left\Vert \bm{A}^{\star\top}\right\Vert _{2,\infty}\log^{3/2}d}{\sqrt{p}}+\frac{\left\Vert \bm{A}^{\star}\right\Vert _{2,\infty}\sqrt{\log d}}{\sqrt{p}}  + \frac{\sigma \log^{5/2} d}{ p} \right\} +\left\Vert \bm{w}\right\Vert _{2}\sigma\sqrt{\frac{d\log d}{p}}\\
	& \overset{\left(\mathrm{i}\right)}{\lesssim} \left\Vert \bm{w}\right\Vert _{\infty}\left\{ \frac{\sqrt{\mu r}\log^{2}d}{d^{3/2}p} + \sqrt{\frac{\mu^2 r \log^3 d}{d^2 p}}  +\sqrt{\frac{\mu r\log d}{dp}}  \right\} \lambda_{\max}^{\star} +   \left\Vert \bm{w}\right\Vert _{\infty} \frac{\sigma \log^{5/2} d}{ p}
	+ \left\Vert \bm{w}\right\Vert _{2}\sigma\sqrt{\frac{d\log d}{p}}\\
 & \overset{\left(\mathrm{ii}\right)}{\lesssim}\left\Vert \bm{w}\right\Vert _{\infty}\sqrt{\frac{\mu r\log d}{dp}}\,\lambda_{\max}^{\star}+ \left\Vert \bm{w}\right\Vert _{\infty} \frac{\sigma \log^{5/2} d}{ p} + \left\Vert \bm{w}\right\Vert _{2}\sigma\sqrt{\frac{d\log d}{p}},
\end{align*}
where (i) is due to Lemma~\ref{lemma:incoh}, and (ii) follows as long as $p\gtrsim d^{-2}\log^{3}d$ and $\mu \log^2 d \lesssim d$.

%%%%%%%%%%%%%%%%%%%%%%%%%%%%%%%%%%%%%%%%%%%%%%%%

\subsection{Proof of Lemma~\ref{lemma:indep_gaussian_gap}}
\label{pf:indep_gaussian_gap}
Recall that for a standard Gaussian random variable $Z \sim \N \( 0, 1\)$, one has
\begin{align}
	\frac{1}{5t\sqrt{2\pi}} \exp\( - t^2 / 2 \) ~\leq~   \( \frac{1}{t} - \frac{1}{t^3} \) \frac{1}{\sqrt{2\pi}} \exp\( - t^2 / 2 \) ~\leq~  \P \left\{ Z \geq t  \right\} ~\leq~ \frac{1}{t}\frac{1}{\sqrt{2\pi}} \exp\( - t^2 /2 \) \label{eq:gaussian_tail}
\end{align}
for all $t > \sqrt{5/4}$.  
Observing that $\kappa\sqrt{2\log r} + \Delta \geq \sqrt{5/4}$ since $\kappa \geq 1$ and $r \geq 2$, we can invoke the above tail bound to deduce that
\begin{align}
\mathbb{P} \left\{ X_{1, j} \geq \kappa\sqrt{2\log r} + \Delta \right\} & > \frac{1}{5 \sqrt{2\pi} \( \kappa\sqrt{2\log r} + \Delta\)}\exp \(-\big(\kappa\sqrt{2\log r} + \Delta \big)^2/2 \) \nonumber \\
& \geq \frac{1}{5 \sqrt{2\pi}  \( \kappa\sqrt{2\log r} + \Delta\)  r^{2\kappa^2} \exp(\Delta^2)},
\label{eq:normal_LB}
\end{align}
where we use the elementary inequality $\big(\kappa\sqrt{2\log r} + \Delta )^2 \leq 4\kappa^2 \log r + 2\Delta^2$. In addition,  it follows from the union bound that
\begin{align}
	\mathbb{P} \left\{ \max_{1< i \leq r} \left| X_{i, j}  \right| < \sqrt{2\log r} \right\} & \geq 1 - r \, \mathbb{P} \left\{ \left| X_{i,j}  \right| > \sqrt{2\log r} \right\} \geq 1 - r \left\{ \frac{1}{2\sqrt{\pi \log r}} \exp(-\log r) \right\}  \nonumber\\
	& \geq 
	1- \frac{1}{\sqrt{\pi \log r}} \geq 1- \frac{1}{\sqrt{\pi \log 2}}. \label{eq:normal_UB}
\end{align}
To prove the claim, it is sufficient to choose $L$ such that
\begin{align*}
\mathbb{P}\Big\{\forall j:X_{1,j}<\kappa\sqrt{2\log r}+\Delta\text{ or }\max\nolimits _{1<i\leq r}\big|X_{i,j}\big|\geq\sqrt{2\log r}\Big\}\leq\delta, 
\end{align*}
or equivalently,
\begin{align} \label{eq:X1j-Xij-bound-L}
\Big(1 - \mathbb{P} \big\{ X_{1, j} \geq \kappa\sqrt{2\log r} + \Delta \big\} \mathbb{P} \big\{ \max\nolimits_{1<i \leq r} \left| X_{i, j} \right| < \sqrt{2\log r} \big\} \Big)^L \leq \delta.
\end{align}
Note that $\log (1-x) \leq - 1/(2x)$ for $0 < x < 1/4$. In view of (\ref{eq:normal_LB}) and (\ref{eq:normal_UB}), one can verify that the above inequality \eqref{eq:X1j-Xij-bound-L}  as long as
\begin{align}
L \geq C \big(\kappa\sqrt{\log r} + \Delta\big)r^{2\kappa^2}\exp(\Delta^2) \,\log \frac{1}{\delta}, \label{eq:L_LB}
\end{align}
where $C>0$ is some universal constant.

To prove the second claim, recall the definitions that
\begin{align*}
\Delta_{j} := X_{1, j} - \max_{1 < i \leq r} \kappa X_{i, j}, \qquad 1 \leq j \leq L.
\end{align*}
and that $\Delta_{(1)} \geq \Delta_{(2)} \geq \dots, \geq \Delta_{(L)}$ denote $\left\{ \Delta_j \right\}_{j = 1}^L$ in descending order. For any $\epsilon > 0$, one has
\begin{align*}
\P \left\{ \Delta_{(1)} - \Delta_{(2)} < \epsilon \right\} & = \sum_{1\leq j \leq L} \P \left\{ \Delta_j - \max\nolimits_{k : k \neq j} \Delta_k < \epsilon \,\big|\, \Delta_j = \Delta_{(1)}\right\} \P \left\{ \Delta_j = \Delta_{(1)}\right\} \\
& = \P \left\{ \Delta_1 - \max\nolimits_{k : k \neq 1} \Delta_k > \epsilon \,\big|\, \Delta_1 = \Delta_{(1)} \right\},
\end{align*}
where the last line holds because the distribution of $\Delta_j - \max\nolimits_{k : k \neq j} \Delta_k$ conditional on $\Delta_j  = \max\nolimits_{1\leq k \leq r} \Delta_k$ is identical for all $1 \leq j \leq L$. In addition, it is straightforward to see that
\begin{align*}
	\Delta_1 = \Delta_{(1)} \quad & \Longleftrightarrow \quad \Delta_1 \geq \max_{k : k \neq 1} \Delta_k \\
	&\Longleftrightarrow  \quad X_{1, 1} \geq  \max_{1 < i \leq r} \kappa X_{i, 1} + \max_{k : k \neq 1} \Delta_k =: \, Y_{1, 1}.
\end{align*}
Hence, we have
\begin{align*}
\P \left\{ \Delta_{(1)} - \Delta_{(2)} < \epsilon \right\} =  \P \left\{ \Delta_1 - \max\nolimits_{k : k \neq 1} \Delta_k < \epsilon \,\big|\, \Delta_1 = \Delta_{(1)} \right\} 
  =  \P \left\{ X_{1, 1} - Y_{1, 1} < \epsilon \,\big|\, X_{1, 1} \geq Y_{1, 1} \right\}.
\end{align*}

Next, observe that $X_{1, 1}$ is independent of $Y_{1, 1}$, and hence we have
\begin{align*}
\P \left\{ X_{1, 1} - Y_{1, 1} < \epsilon \,\big|\, X_{1, 1} \geq Y_{1, 1}, Y_{1, 1} = x \right\} & = \P \left\{ X_{1, 1} < x + \epsilon \,\big|\, X_{1, 1} \geq y, Y_{1, 1} = x \right\} \\
& = \P \left\{ X_{1, 1} < x + \epsilon \,\big|\, X_{1, 1} \geq x \right\}
 = \frac{\P \left\{ x \leq X_{1, 1} < x + \epsilon \right\}}{\P \left\{ X_{1, 1} \geq x \right\}}
\end{align*}
for any $x \geq 0$. In order to study this function, we define 
$f_\epsilon(x) := \frac{ \P \left\{ x \leq Z \leq x+\epsilon \right\} }{\P \left\{ Z \geq x \right\}}$ with $Z\sim \mathcal{N}(0,1)$. 
Taking the derivative of $f_{\epsilon}(\cdot)$ w.r.t.~$x$ gives: for any $\epsilon >0$, 
\begin{align*}
	f'_\epsilon(x) &= \frac{\big\{ \exp (- (x+\epsilon )^2/2 ) - \exp (- x^2/2 ) \big\} \P \left\{ Z \geq x \right\} + \exp \(- x^2/2 \) \P \left\{ x \leq Z \leq x+\epsilon \right\}}{\sqrt{2\pi }  \( \P \left\{ Z \geq x \right\} \)^2} \\
& = \frac{\exp \(-(x+\epsilon )^2/2 \)\P \left\{ Z \geq x \right\} - \exp(-\frac{x^2}{2}) \P \left\{ Z \geq x+\epsilon \right\} }{\sqrt{2\pi } \( \P \left\{ N \geq x \right\} \)^2} \\
& = \frac{\exp \( x^2/2 \)\P \left\{ Z \geq x \right\} - \exp \( (x+\epsilon )^2 / 2 \) \P \left\{ Z \geq x+\epsilon \right\} }{\sqrt{2\pi }\exp\( (x+\epsilon )^2 / 2 \) \exp \( x^2/2 \) \( \P \left\{ Z \geq x \right\} \)^2} \\
& = \frac{\int_0^{\infty} \( \exp \(-\frac{t^2 + 2tx}{2} \) - \exp \(-\frac{t^2 + 2t (x+\epsilon )}{2} \) \) \, \mathrm{d} t}{2\pi \exp \( (x+\epsilon )^2 / 2 \) \exp \(x^2 / 2 \) \( \P \left\{ Z \geq x \right\} \)^2} > 0 .
\end{align*}
In other words, $f_{\epsilon}(x)$ is monotonically increasing in $x$ for any given $\eps > 0$. Therefore, for any $0 \leq x < B$ for some sufficiently large $B > \sqrt{5 / 4}$ (to be specified later), the above bounds taken together give
\begin{align*}
\P \left\{ X_{1, 1} - Y_{1, 1} < \epsilon \,\big|\, X_{1, 1} \geq Y_{1, 1}, Y_{1, 1} = x \right\} 
\overset{(\mathrm{i})}{\leq} \frac{\P \left\{ B \leq X_{1, 1} < B + \epsilon \right\}}{\P \left\{ X_{1, 1} \geq B \right\}} 
\overset{(\mathrm{ii})}{\leq} \frac{ \epsilon \exp\left(- B^2 / 2\right)  }{\frac{1}{5B}\exp\left(-B^2 / 2\right)} = 5 \epsilon B,
\end{align*}
where (i) arises from the monotonicity of $f_{\epsilon}(\cdot)$, and (ii) relies on \eqref{eq:gaussian_tail}. By taking $\epsilon =  \delta / ( 5B)$, we obtain 
\begin{align*}
\P \left\{ X_{1, 1} - Y_{1, 1} < \epsilon \,\big|\, X_{1, 1} \geq Y_{1, 1}, Y_{1, 1} = x \right\}  \leq \delta
\end{align*}
for any $0 \leq x \leq B$. Recall that $Y_{1, 1} = \max_{1 < i \leq r} \kappa X_{i, 1} + \max_{k : k \neq 1} \Delta_k$. By standard Gaussian concentration inequalities, with probability at least $1- \delta$ one has
\begin{align*}
 Y_{1, 1} \lesssim \kappa \sqrt{\log r} + \sqrt{\log L} + \sqrt{\log (1/\delta) } \asymp \sqrt{\log L  } + \sqrt{\log (1/\delta) }.
\end{align*}
where the last step arises from the lower bound on $L$ in \eqref{eq:L_LB}. If we choose $B = C \big( \sqrt{\log L  } + \sqrt{\log (1/\delta) } \big)$ for some sufficiently large universal constant $C > 0$, then this immediately implies that 
\begin{align*}
\P \left\{ \Delta_{(1)} - \Delta_{(2)} < \epsilon \right\} & = \P \left\{ X_{1, 1} - Y_{1, 1} < \epsilon \,\big|\, X_{1, 1} \geq Y_{1, 1} \right\}  \\
	& \leq\mathbb{P}\big\{ Y_{1, 1} > B \big\}+\P \left\{ X_{1, 1} - Y_{1, 1} < \epsilon \,\big|\, X_{1, 1} \geq Y_{1, 1},  Y_{1, 1} = B \right\}  \leq 2\delta.
\end{align*}
%
%$\P \big\{ \left| X \right|_{(1)} - \left| X \right|_{(2)} < \Delta \big\} \leq 2\delta$.

We have therefore concluded the proof.

\subsection{Proof of Lemma \ref{lemma:subspace_basis_dist}}
\label{pf:subspace_basis_dist}

To begin with, it is self-evident that
\begin{align*}
\mathcal{P}_{\bm{V}}(\bm{u}_{0}) & =\bm{V}\bm{V}^{\top}\bm{u}_{0}=\bm{U}\bm{U}^{\top}\bm{u}_{0}+\left(\bm{V}\bm{V}^{\top}-\bm{U}\bm{U}^{\top}\right)\bm{u}_{0}=\bm{u}_{0}+\left(\bm{V}\bm{V}^{\top}-\bm{U}\bm{U}^{\top}\right)\bm{u}_{0},
\end{align*}
where the last identity follows since $\bm{u}_0$ is assumed to lie within $\mathsf{span}(\bm{U})$. 
As a result,
\begin{align*}
\mathcal{P}_{\bm{V}^{\perp}}(\bm{u}_{0}) & =\bm{u}_{0}-\mathcal{P}_{\bm{V}}(\bm{u}_{0})= - \left(\bm{V}\bm{V}^{\top}-\bm{U}\bm{U}^{\top}\right)\bm{u}_{0}
\end{align*}
\begin{align*}
\Longrightarrow\qquad\big\|\mathcal{P}_{\bm{V}^{\perp}}(\bm{u}_{0})\big\|_{2} \leq \big\|\bm{V}\bm{V}^{\top}-\bm{U}\bm{U}^{\top}\big\|\cdot\|\bm{u}_{0}\|_{2}\leq\delta.
\end{align*}
The Pythagorean theorem then gives $\big\|\mathcal{P}_{\bm{V}}(\bm{u}_{0})\big\|_{2}=\sqrt{ \| \bm{u}_0 \|_2 - \big\|\mathcal{P}_{\bm{V}^{\perp}}(\bm{u}_{0})\big\|_{2}^{2}}\geq\sqrt{1-\delta^{2}}$.

%%%%%%%%%%%%%%%%%%%%%%%%%%%%%%%%%%%%%%%%%%%%%%%%%%%%%%%%%%%%%%%%%%%

\subsection{Proof of Lemma~\ref{lemma:sum_ij_square_cont_ineq}}
\label{pf:sum_ij_square_cont_ineq}

By virtue of Lemma \ref{lemma:incoh}, we can compute
\begin{align*}
\sum_{i,j\in[d]}\mathbb{E}\left[T_{i,j,k}^{\star2}(p^{-1}\chi_{i,j,k}-1)^{2}\right] & \leq\frac{1}{p}\sum_{i,j\in[d]}T_{i,j,k}^{\star2}\leq\frac{1}{p}\|\bm{A}^{\star}\|_{2,\infty}^{2}\lesssim\frac{\mu r\lambda_{\max}^{\star2}}{dp}=:M;\\
\Big|T_{i,j,k}^{\star2}(p^{-1}\chi_{i,j,k}-1)^{2}\Big| & \leq\frac{1}{p^{2}}\|\bm{T}^{\star}\|_{\infty}^{2} = \frac{1}{p^{2}}\|\bm{A}^{\star}\|_{\infty}^{2} \lesssim\frac{\mu r\lambda_{\max}^{\star2}}{d^{3}p^{2}}=:L;\\
\sum_{i,j\in[d]}\mathsf{Var}\Big(T_{i,j,k}^{\star2}(p^{-1}\chi_{i,j,k}-1)^{2}\Big) & \lesssim\frac{1}{p^{3}}\sum_{i,j\in[d]}T_{i,j,k}^{\star4}\leq\frac{1}{p^{3}}\|\bm{A}^{\star}\|_{\infty}^{2}\|\bm{A}^{\star}\|_{2,\infty}^{2}\lesssim\frac{\mu^2 r^{2}\lambda_{\max}^{\star 4}}{d^{4}p^{3}}=:V.
\end{align*}
Invoke the Bernstein inequality to show that: with probability exceeding
$1-O(d^{-20})$,
\begin{align*}
\sum_{i,j\in[d]}T_{i,j,k}^{\star 2}(p^{-1}\chi_{i,j,k}-1)^{2} & \lesssim M+L\log d +\sqrt{V\log d}\\
& \lesssim  \frac{\mu r \lambda_{\max}^{\star 2}}{dp} +\frac{\mu r\lambda_{\max}^{\star2}\log d}{d^{3}p^{2}} + \sqrt{\frac{\mu^2 r^{2}\lambda_{\max}^{\star4} \log d}{d^{4}p^{3}}} \\
 & \asymp\frac{\mu r\lambda_{\max}^{\star2}}{dp},
\end{align*}
where the last line holds with the proviso that $p\gtrsim d^{-2}\log d$.

%%%%%%%%%%%%%%%%%%%%%%%%%%%%%%%%%%%%%%%%%%%%%%%%%%%%%%%%%%%%%%

\subsection{Proof of Lemma~\ref{lemma:sum_ij_square_cont_ineq_subg}}
\label{pf:sum_ij_square_cont_ineq_subg}

Since the $E_{i,j,k}$'s are independent sub-Gaussian random variables with variance at most $\sigma^2$, one has
\begin{align*}
\sum_{i,j\in[d]}\mathbb{E}\left[(E_{i,j,k}\chi_{i,j,k})^{2}\right] & \lesssim \sigma^{2}d^{2}p=:M;\\
\big\|(E_{i,j,k}\chi_{i,j,k})^{2}\big\|_{\psi_{1}} & \lesssim\sigma^{2}=:L;\\
\sum_{i,j\in[d]}\mathsf{Var}\left[(E_{i,j,k}\chi_{i,j,k})^{2}\right] & \lesssim\sigma^{4} d^2 p=:V.
\end{align*}
Here, $\|\cdot \|_{\psi_1}$ denotes the sub-exponential norm \cite{vershynin2010introduction}. 
Taken together with the Bernstein inequality, these yield that with
probability exceeding $1-O(d^{-20})$,
\begin{align*}
\sum_{i,j\in[d]}(E_{i,j,k}\chi_{i,j,k})^{2} & \lesssim M+\sqrt{V\log d}+L\log^2 d \lesssim\sigma^{2}d^{2}p+\sqrt{\sigma^{4} d^2 p\log d}+\sigma^{2}\log^2 d
	\asymp\sigma^{2}d^{2}p,
\end{align*}
provided that $p\gtrsim d^{-2}\log^2 d$.

%%%%%%%%%%%%%%%%%%%%%%%%%%%%%%%%%%%%%%%%%%%%%%%%%%%%%%%%%%%%%%%%%%%

\subsection{Proof of Lemma~\ref{lemma:sum_jk_square_cont_ineq}}
\label{pf:sum_jk_square_cont_ineq}

Fix an arbitrary $1\leq i \leq d$. We first define a sequence of independent zero-mean random variables $\{ X_{j}\}_{j \in [d]}$ as follows
\begin{align*}
X_{j} := \sum\nolimits_{k \in [d]} T^\star_{i, j, k} w_{k} ( p^{-1} \chi_{i, j, k} - 1 ).
\end{align*}
%
%With the assistance of Lemma \ref{lemma:incoh}, 
One can easily show that
\begin{align*}
& \max_{k \in [d]} \Big| T^\star_{i, j, k} w_k ( p^{-1} \chi_{i, j, k} - 1 ) \Big| \leq \frac{1}{p} \norm{\bm{A}^\star}_\infty \norm{\bm{w}}_\infty =: L, \\
&\E [ X_{j}^2 ] = \sum_{k \in [d]} T_{i, j, k}^{\star 2} w_k^2 \E \big[ ( p^{-1} \chi_{i, j, k} - 1 )^2 \big] 
	\leq \frac{1}{p} \norm{\bm{w}}_\infty^2 \sum_{k \in [d]} T_{i, j, k}^{\star 2} = \frac{1}{p} \norm{\bm{w}}_\infty^2 \|\bm{A}^{\star\top}\|_{2,\infty}^2 =: V.
\end{align*}
The Bernstein inequality indicates that: with probability at least $1- O(d^{-20} )$,
\begin{align}
\begin{split}
\label{X_ij_abs_ub}
	| X_{j} | \lesssim L \log d + \sqrt{V \log d} 
	\lesssim  \frac{1}{p} \norm{\bm{A}^\star}_\infty \norm{\bm{w}}_\infty \log d + \sqrt{\frac{\log d}{p}}  \norm{\bm{A}^{\star \top}}_{2,\infty} \norm{\bm{w}}_\infty  := L_{j} .
\end{split}
\end{align}
Moreover, we can also bound the variance of $X_j$ as follows
\begin{align*}
	\var ( X_{j}^2 ) \leq	\E [ X_{j}^4 ]  &\lesssim \frac{1}{p^3}\sum_{k \in [d]} T_{i, j, k}^{\star 4} w_k^4
+  \frac{1}{p^2} \sum_{ k_1 \neq k_2 } T_{i, j, k_1}^{\star 2} T_{i, j, k_2}^{\star 2} w_{k_1}^2 w_{k_2}^2  \\
	 %& \lesssim p^{-3} \sum_{k \in [d]} T_{i, j, k}^{\star 4} w_k^4 +  p^{-2} \sum_{ k_1 \neq k_2 } T_{i, j, k_1}^{\star 2} T_{i, j, k_2}^{\star 2} w_{k_1}^2 w_{k_2}^2 \\
%& \leq ( p^{-3} + 2 d p^{-2} )\norm{\bm{w}}_\infty^4  \norm{\bm{T}^\star}_\infty^2\sum_{k \in [d]} T_{i, j, k}^{\star 2}
& \lesssim \frac{1}{p^3}\norm{\bm{w}}_\infty^4  \norm{\bm{A}^\star}_\infty^2\sum_{k \in [d]} T_{i, j, k}^{\star 2} + \frac{1}{p^2} \norm{\bm{w}}_\infty^4 \norm{\bm{A}^{\star\top}}_{2,\infty}^2  \sum_{k \in [d]} T_{i, j, k}^{\star 2}.
\end{align*}

Given that $X_{j}$ might be overly large in some rare case, we introduce a sequence $\{Y_j\}$, where 
we denote by $ Y_{j} $ the truncated version of $X_{j}$ as follows
\begin{align*}
	Y_{j} := X_{j} \ind \{ |X_{j}| \lesssim L_j \}.
%\begin{cases} 
 %  X_{i,k}, \quad  & \text{if }   \left| X_{j} \right| \leq L_{j}; \\
 %  L_{j},   & \text{if }   X_{ik} > L_{j}; \\
 %  -L_{j},   & \text{if }   X_{ik} < -L_{j}.
%\end{cases}
\end{align*}
We have learn from \eqref{X_ij_abs_ub} and the union bound that with probability at least $1- O(d^{-15})$, one has $Y_{j} = X_{j}$ for all $j \in [d]$.

Using the above bounds on the $X_j$'s, one observes that $\{ Y_{j} \}_{j \in [d]}$ is a sequence of independent random variables satisfying
\begin{align*}
%& \E [ Y_{j}^2 ] \leq \E [ X_{j}^2 ] \leq p^{-1} \norm{\bm{w}}_\infty^2 \sum_{k \in [d]} T_{i, j, k}^{\star 2} \\
%& \var ( Y_{j}^2 ) \leq \var ( X_{j}^2 ) \leq ( p^{-3} + 2 d p^{-2})\norm{\bm{w}}_\infty^4  \norm{\bm{T}^\star}_\infty^2\sum_{k \in [d]} T_{i, j, k}^{\star 2}
	\sum_{j \in [d]} \E [ Y_{j}^2 ] & \leq \sum_{j \in [d]} \E [ X_{j}^2 ] 
%\leq p^{-1} \norm{\bm{w}}_\infty^2 \sum_{j, k\in [d]} T_{i, j, k}^{\star 2} 
\leq \frac{1}{p} \norm{\bm{A}^\star}_{2,\infty}^2  \norm{\bm{w}}_\infty^2 ; \\
	\max_{j \in [d]} Y_{j}^2 & \leq  \max_{j \in [d]} L_{j}^2  
	\lesssim  \frac{1}{p^2} \norm{\bm{A}^\star}_\infty^2 \norm{\bm{w}}_\infty^2 \log d + \frac{\log d}{p}  \norm{\bm{A}^{\star \top}}_{2,\infty}^2 \norm{\bm{w}}_\infty^2;\\
	\sum_{j \in [d]} \var ( Y_{j}^2 ) &\leq \sum_{j \in [d]} \E [ X_{j}^4 ] 
%\lesssim \norm{\bm{w}}_\infty^4 p^{-3} \left\{ 1 \vee dp \right\}  \norm{\bm{T}^\star}_\infty^2 \sum_{j, k \in [d]} T_{i, j, k}^{\star 2} 
\lesssim  \frac{1}{p^3}  \norm{\bm{A}^\star}_\infty^2 \norm{\bm{A}^\star}_{2, \infty}^2 \norm{\bm{w}}_\infty^4 + \frac{1}{p^2}  \norm{\bm{A}^{\star\top}}_{2,\infty}^2  \norm{\bm{A}^{\star}}_{2,\infty}^2 \norm{\bm{w}}_\infty^4.
\end{align*}
We can apply the Bernstein inequality  to conclude that: with probability greater than $1- O( d^{-15}  )$,
\begin{align*}
\sum_{j \in [d]} Y_{j}^2 & \lesssim \sum_{j \in [d]} \E [ Y_{j}^2 ] + \max_{j \in [d]} Y_{j}^2 \log d + \sqrt{\sum\nolimits_{j \in [d]} \var ( Y_{j}^2 ) \log d } \\
& \lesssim \frac{1}{p} \norm{\bm{A}^\star}_{2,\infty}^2  \norm{\bm{w}}_\infty^2 + \frac{\log^2 d}{p^2} \norm{\bm{A}^\star}_\infty^2 \norm{\bm{w}}_\infty^2 + \frac{\log^2 d}{p}  \norm{\bm{A}^{\star \top}}_{2,\infty}^2 \norm{\bm{w}}_\infty^2 \\
& \quad + \sqrt{\frac{\log d}{p^3}  \norm{\bm{A}^\star}_\infty^2 \norm{\bm{A}^\star}_{2, \infty}^2 \norm{\bm{w}}_\infty^4 + \frac{\log d}{p^2} \norm{\bm{w}}_\infty^4 \norm{\bm{A}^{\star\top}}_{2,\infty}^2  \norm{\bm{A}^{\star}}_{2,\infty}^2} \\
& \overset{(\mathrm{i})}{\asymp} \frac{1}{p} \norm{\bm{A}^\star}_{2,\infty}^2  \norm{\bm{w}}_\infty^2 + \frac{\log^2 d}{p^2} \norm{\bm{A}^\star}_\infty^2 \norm{\bm{w}}_\infty^2 + \frac{\log^2 d}{p}  \norm{\bm{A}^{\star \top}}_{2,\infty}^2 \norm{\bm{w}}_\infty^2 \\
& \overset{(\mathrm{ii})}{\lesssim} \(  \frac{\mu r \lambda_{\max}^{\star 2}}{dp} + \frac{\mu r \lambda_{\max}^{\star 2} \log^2 d}{d^3 p^2} + \frac{\mu^2 r \lambda_{\max}^{\star 2} \log^2 d}{d^2 p} \) \norm{\bm{w}}_\infty^2 \\
& \asymp \frac{\mu r \lammax^{\star 2}}{dp} \norm{\bm{w}}_\infty^2, 
\end{align*}
where (i) is due to the AM-GM inequality, 
(ii) makes use of Lemma \ref{lemma:incoh}, 
and the last line follows under the conditions  $p \gtrsim d^{-2} \log^2 d $ and $\mu \log^2 d \lesssim d$.
This together with the high-probability fact $Y_j = X_j$ ($\forall j \in [d]$) concludes the proof.

%%%%%%%%%%%%%%%%%%%%%%%%%%%%%%%%%%%%%%%%%%%%%%%%%%%%%%%%%%%%%%%%%%%

\subsection{Proof of Lemma~\ref{lemma:sum_jk_square_cont_ineq_subg}}
\label{pf:sum_jk_square_cont_ineq_subg}

% Preview source code from paragraph 2 to 3

Fix any $1\leq i\leq d$. To begin with, define
\begin{equation*}
Z_{j} := \sum\nolimits_{k\in[d]}w_{k}E_{i,j,k}\chi_{i,j,k},
\end{equation*}
which is a zero-mean random variable. In order to bound $Z_{j}$,
one observes that
\begin{align*}
\|w_{k}E_{i,j,k}\chi_{i,j,k}\|_{\psi_{1}} & \lesssim \sigma \left\|\bm{w} \right\|_{\infty}=:L;\\
\sum\nolimits_{k\in[d]}\mathsf{Var}\big(w_{k}E_{i,j,k}\chi_{i,j,k}\big) & \leq \sigma^{2}p \sum\nolimits_{k\in[d]}w_{k}^{2}
	=\sigma^{2}p\left\|\bm{w} \right\|_{2}^{2}=:V, 
\end{align*}
where $\|\cdot\|_{\psi_1}$ denotes the sub-exponential norm. 
Apply the Bernstein inequality for the sum of sub-exponential random
variables to obtain
\begin{equation}
|Z_{j}|\lesssim\sqrt{V\log d}+L\log^2 d \lesssim\sigma\left\|\bm{w} \right\|_{2}\sqrt{p\log d}+\sigma\left\|\bm{w} \right\|_{\infty} \log^2 d =: L_{j}
\end{equation}
with probability exceeding $1-O(d^{-20})$. Further, given that $Z_{j}$
is not necessarily bounded, we introduce a sequence of truncated random
variables as follows
\begin{equation}
Y_{j} := Z_{j}\ind\{|Z_{j}|\lesssim L_{j}\}.
\end{equation}
According to the above bound, one has $Y_{j}=Z_{j}$ ($\forall j$)
with probability at least $1-O(d^{-19})$.

We then turn attention to bounding $\sum_{j\in[d]}Y_{j}^{2}$. To
this end, observe that
\begin{align*}
\sum_{j\in[d]}\mathbb{E}[Y_{j}^{2}] & \leq\sum_{j\in[d]}\mathbb{E}\Big[\Big(\sum\nolimits_{k\in[d]}w_{k}E_{i,j,k}\chi_{i,j,k}\Big)^{2}\Big]\lesssim\sigma^{2}p\sum_{j\in[d]}\sum_{k\in[d]}w_{k}^{2}\\
 & =\sigma^{2}pd\left\|\bm{w} \right\|_{2}^{2}=:M_{0}.
\end{align*}
Additionally, 
\begin{align*}
\sum_{j\in[d]}\mathbb{E}\left[Y_{j}^{4}\right] & \leq\sum_{j\in[d]}\mathbb{E}\Big[\Big(\sum\nolimits_{k\in[d]}w_{k}E_{i,j,k}\chi_{i,j,k}\Big)^{4}\Big]\\
 & \leq\sum_{j\in[d]}\sum_{k\in[d]}\mathbb{E}\Big[w_{k}^{4}E_{i,j,k}^{4}\chi_{i,j,k}^{4}\Big]+\sum_{j\in[d]}\sum_{k_{1}\neq k_{2}}\mathbb{E}\Big[w_{k_{1}}^{2}w_{k_{2}}^{2}E_{i,j,k_{1}}^{2}E_{i,j,k_{2}}^{2}\chi_{i,j,k_{1}}^{2}\chi_{i,j,k_{2}}^{2}\Big]\\
 & \lesssim\sigma^{4}p\sum_{j\in[d]}\sum_{k\in[d]}w_{k}^{4}+\sigma^{4}p^{2}\sum_{j\in[d]}\sum_{1 \leq k_{1}\neq k_{2} \leq d}w_{k_{1}}^{2}w_{k_{2}}^{2}\\
 & \lesssim\sigma^{4}pd\left\|\bm{w} \right\|_{2}^{2}\left\|\bm{w} \right\|_{\infty}^{2}+\sigma^{4}p^{2}d\left\|\bm{w} \right\|_{2}^{4}=:V_{0}.
\end{align*}
Invokde the Bernstein inequality to arrive at: with probability at
least $1-O(d^{-20})$, 
\begin{align*}
\sum_{j\in[d]}Y_{j}^{2} & \lesssim M_{0}+\sqrt{V_{0}\log d}+\max_{j\in[d]}L_{j}^{2}\log d\\
 & \lesssim\sigma^{2}pd\left\|\bm{w} \right\|_{2}^{2}+\sqrt{\sigma^{4}pd\left\|\bm{w} \right\|_{2}^{2}\left(p\left\|\bm{w} \right\|_{2}^{2} + \left\|\bm{w} \right\|_{\infty}^{2}\right)\log d}+\left(\sigma^2 p \left\|\bm{w} \right\|_{2}^2 \log^2 d +\sigma^2 \left\|\bm{w} \right\|_{\infty}^2 \log^5 d\right)\\
 & \asymp\sigma^{2}pd\left\|\bm{w} \right\|_{2}^{2}+\sigma^{2}\left\|\bm{w} \right\|_{\infty}^{2}\log^5 d.
\end{align*}
This together with the high-probability fact $Y_{j}=Z_{j}$ ($\forall j$)
completes the proof.

\section{Extension to asymmetric tensors}
\label{sec:asymmetry}

Thus far, we have focused on the case where the tensor of interest is symmetric. 
In this section, we discuss how to generalize our algorithm
and analysis to accommodate asymmetric tensors. 

\subsection{Problem settings}

Suppose that the unknown tensor $\bm{T}^{\star}\in\mathbb{R}^{d_{1}\times d_{2}\times d_{3}}$
is a rank-$r$ tensor with CP decomposition
\begin{equation}
\bm{T}^{\star}=\sum_{i=1}^{r}\bm{u}_{i}^{\star}\otimes\bm{v}_{i}^{\star}\otimes\bm{w}_{i}^{\star},\label{eq:asymmetric-tensor-truth}
\end{equation}
where $\bm{u}_{i}^{\star}\in\mathbb{R}^{d_{1}}$, $\bm{v}_{i}^{\star}\in\mathbb{R}^{d_{2}}$,
$\bm{w}_{i}^{\star}\in\mathbb{R}^{d_{3}}$ represent the tensor factors
of interest. Apparently, there is an unavoidable global scaling ambiguity
issue (for instance, multiplying $\bm{u}_{i}^{\star}$ by a constant $c$ and multiplying 
$\bm{v}_{i}^{\star}$
by $1/c$ accordingly result in the same tensor). Without
loss of generality, we shall assume throughout that 
\begin{equation}
\|\bm{u}_{i}^{\star}\|_{2}=\|\bm{v}_{i}^{\star}\|_{2}=\|\bm{w}_{i}^{\star}\|_{2},\qquad1\leq i\leq r.\label{eq:norm-constraint-tensor-factors}
\end{equation}
In addition, we assume that each entry $(j,k,l)$ is included in the sampling set $\Omega$ independently with probability $p$, and that each observed entry $T_{j,k,l}^\star$ is corrupted by an independent zero-mean sub-Guassian noise $E_{j,k,l}$ (cf.~Assumption~\ref{asmp:random-noise}). Our goal is to (1) estimate $\{\bm{u}_{i}^{\star},\bm{v}_{i}^{\star},\bm{w}_{i}^{\star}\}_{i=1}^{r}$
faithfully, modulo global permutation and global signs, and (2) estimate $\bm{T}^{\star}$ in a reliable manner. 

\subsection{Algorithms}

%\yxc{Defin $\bm{U}$, $\bm{V}$, $\bm{W}$}

We now move on to present an extension of our nonconvex algorithm to handle the noisy scenario.

First of all, setting
\begin{align*}
\bm{U}:=[\bm{u}_{1},\cdots,\bm{u}_{r}]\in\mathbb{R}^{d_{1}\times r}, \quad \bm{V}=[\bm{v}_{1},\cdots,\bm{v}_{r}]\in\mathbb{R}^{d_{2}\times r} \quad \text{and} \quad \bm{W}=[\bm{w}_{1},\cdots,\bm{w}_{r}]\in\mathbb{R}^{d_{3}\times r}, 
\end{align*}
we can define the following regularized squared loss function 
\begin{align}
	g(\bm{U},\bm{V},\bm{W}) & :=\frac{1}{6p}\,\Big\|\mathcal{P}_{\Omega}\Big(\sum_{i=1}^{r}\bm{u}_{i}\otimes\bm{v}_{i}\otimes\bm{w}_{i}-\bm{T}\Big)\Big\|_{\mathrm{F}}^{2}+\mathsf{reg}\big(\bm{U},\bm{V},\bm{W}\big),
	\label{eq:loss-function-asym}
\end{align}
where the regularization term $\mathsf{reg}\big(\bm{U},\bm{V},\bm{W}\big)$
is given by
\begin{equation}
\mathsf{reg}\big(\bm{U},\bm{V},\bm{W}\big):=\frac{1}{24}\sum_{i=1}^{r}\alpha_{i}\left\{ \big(\|\bm{u}_{i}\|_{2}^{2}-\|\bm{v}_{i}\|_{2}^{2}\big)^{2}+\big(\|\bm{u}_{i}\|_{2}^{2}-\|\bm{w}_{i}\|_{2}^{2}\big)^{2}+\big(\|\bm{v}_{i}\|_{2}^{2}-\|\bm{w}_{i}\|_{2}^{2}\big)^{2}\right\} 
	\label{eq:reg-function-asym}
\end{equation}
for some positive regularization parameters $\{\alpha_{i}\}_{i=1}^{r}$ to be specified
momentarily. In contrast to the symmetric case, the addition term $\mathsf{reg}\big(\bm{U},\bm{V},\bm{W}\big)$
is included to help ensure that the sizes of $\bm{U},\bm{V}$ and $\bm{W}$
stay close --- an algorithmic trick that has proved useful  in other problems like nonconvex
rectangular matrix recovery \cite{tu2016low,zheng2016convergence,chen2019nonconvex}. 

We are now ready to present our nonconvex algorithm
that accommodates the case with asymmetric tensors. As before, the
proposed algorithm is initialized by a spectral method, followed by
gradient descent designed to minimize the regularized loss function
(\ref{eq:loss-function-asym}). The precise procedure is described
in Algorithm \ref{alg:gd-asym} (which invokes Algorithms \ref{alg:init-asym}-\ref{alg:localization-asym}).

Before proceeding, we find it helpful to record closed-form expressions for the gradients, which are a crucial part when implementing the nonconvex gradient descent algorithm. 
Specifically, the gradients of $g(\bm{U},\bm{V},\bm{W})$
can be computed as follows
\begin{subequations}\label{eq:grad-asym}
\begin{align}
\nabla_{\bm{u}_{i}}g\big(\bm{U},\bm{V},\bm{W}\big) & =\frac{1}{3p}\,\mathcal{P}_{\Omega}\Big(\sum_{i=1}^{r}\bm{u}_{i}\otimes\bm{v}_{i}\otimes\bm{w}_{i}-\bm{T}\Big)\times_{2}\bm{v}_{i}\times_{3}\bm{w}_{i}+\frac{1}{6}\alpha_{i}\big(2\,\|\bm{u}_{i}\|_{2}^{2}-\|\bm{v}_{i}\|_{2}^{2}-\|\bm{w}_{i}\|_{2}^{2}\big)\,\bm{u}_{i},\label{eq:grad-u-asym}\\
\nabla_{\bm{v}_{i}}g\big(\bm{U},\bm{V},\bm{W}\big) & =\frac{1}{3p}\,\mathcal{P}_{\Omega}\Big(\sum_{i=1}^{r}\bm{u}_{i}\otimes\bm{v}_{i}\otimes\bm{w}_{i}-\bm{T}\Big)\times_{1}\bm{u}_{i}\times_{3}\bm{w}_{i}+\frac{1}{6}\alpha_{i}\big(2\,\|\bm{v}_{i}\|_{2}^{2}-\|\bm{u}_{i}\|_{2}^{2}-\|\bm{w}_{i}\|_{2}^{2}\big)\,\bm{v}_{i}\label{eq:grad-v-asym},\\
\nabla_{\bm{w}_{i}}g\big(\bm{U},\bm{V},\bm{W}\big) & =\frac{1}{3p}\,\mathcal{P}_{\Omega}\Big(\sum_{i=1}^{r}\bm{u}_{i}\otimes\bm{v}_{i}\otimes\bm{w}_{i}-\bm{T}\Big)\times_{1}\bm{u}_{i}\times_{2}\bm{v}_{i}+\frac{1}{6}\alpha_{i}\big(2\,\|\bm{w}_{i}\|_{2}^{2}-\|\bm{u}_{i}\|_{2}^{2}-\|\bm{v}_{i}\|_{2}^{2}\big)\,\bm{w}_{i}\label{eq:grad-w-asym}
\end{align}
\end{subequations}
for each $1\leq i\leq r$, where $\times_{1},\times_{2}$
and $\times_{3}$ have been defined in Section~\ref{sec:Notation-1}.

\noindent 
\begin{algorithm}[t]
\caption{Gradient descent for nonconvex tensor completion (asymmetric case)}
\label{alg:gd-asym} \begin{algorithmic}[1] \State Generate initial
estimates $\bm{U}^{0}\in\mathbb{R}^{d_{1}\times r}$, $\bm{V}^{0}\in\mathbb{R}^{d_{2}\times r}$,
$\bm{W}^{0}\in\mathbb{R}^{d_{3}\times r}$ via Algorithm \ref{alg:init-asym}. 

\For{$t=0,1,\dots,t_{0}-1$}
\begin{align*}
\bm{U}^{t+1} & =\bm{U}^{t}-\eta_{t}\nabla_{\bm{U}}g\big(\bm{U}^{t},\bm{V}^{t},\bm{W}^{t}\big),\\
\bm{V}^{t+1} & =\bm{V}^{t}-\eta_{t}\nabla_{\bm{V}}g\big(\bm{U}^{t},\bm{V}^{t},\bm{W}^{t}\big),\\
\bm{W}^{t+1} & =\bm{W}^{t}-\eta_{t}\nabla_{\bm{W}}g\big(\bm{U}^{t},\bm{V}^{t},\bm{W}^{t}\big),
\end{align*}
where the gradients are given in (\ref{eq:grad-asym}). \EndFor \end{algorithmic}
\end{algorithm}

\begin{algorithm}[t]
\caption{Spectral initialization for nonconvex tensor completion (asymmetric
case)}
\label{alg:init-asym} \begin{algorithmic}[1] \State Let $\bm{U}\bm{\Lambda}\bm{U}^{\top}$
be the rank-$r$ eigen-decomposition of $\mathcal{P}_{\mathsf{off}\text{-}\mathsf{diag}}(\bm{A}\bm{A}^{\top})$
where $\bm{A}=\mathsf{unfold}(\bm{T})$ is the mode-1 matricization
of $\bm{T}$, and $\mathcal{P}_{\mathsf{off}\text{-}\mathsf{diag}}(\bm{Z})$
extracts out the off-diagonal entries of $\bm{Z}$. 
	%\yxc{So you don't need to perform this for mode-2 / mode-3 matricization?} \cxc{see discussion in E.4}

\State \textbf{Output:} initial estimates $\bm{U}^{0}\in\mathbb{R}^{d_{1}\times r}$,
$\bm{V}^{0}\in\mathbb{R}^{d_{2}\times r}$, $\bm{W}^{0}\in\mathbb{R}^{d_{3}\times r}$
on the basis of $\bm{U}\in\mathbb{R}^{d_{1}\times r}$ using Algorithm~\ref{alg:localization-asym}.
\end{algorithmic}
\end{algorithm}

\begin{algorithm}[t]
\caption{Retrieval of low-rank tensor factors from a given subspace estimate
(asymmetric case)}
\label{alg:localization-asym} \begin{algorithmic}[1] 

\State \textbf{Input:} number of restarts $L$, pruning threshold
$\epsilon_{\mathsf{th}}$, subspace estimate $\bm{U}\in\mathbb{R}^{d_{1}\times r}$
given by Algorithm~\ref{alg:init-asym}. \For{$\tau=1,\dots,L$} 

\State Generate an independent Gaussian vector $\bm{g}^{\tau}\sim\mathcal{N}(0,\bm{I}_{d_{1}})$.
\State 
\begin{align*}
\big(\bm{\nu}^{(1),\tau},\bm{\nu}^{(2),\tau},\bm{\nu}^{(3),\tau},\lambda_{\tau},\mathsf{gap}_{\tau}\big)\gets\Call{Retrieve-one-tensor-factor-asym}{\bm{T},p,\bm{U},\bm{g}^{\tau}}.
\end{align*}

\EndFor 

\State Generate tensor factor estimates 
\begin{align*}
\big\{(\bm{u}^{1},\bm{v}^{1},\bm{w}^{1},\lambda_{1}),\dots,(\bm{u}^{r},\bm{v}^{r},\bm{w}^{r},\lambda_{r})\big\}\gets\Call{Prune-asym}{\big\{\big(\bm{\nu}^{(1),\tau},\bm{\nu}^{(2),\tau},\bm{\nu}^{(3),\tau},\lambda_{\tau},\mathsf{gap}_{\tau}\big)\big\}_{\tau=1}^{L},\epsilon_{\mathsf{th}}}.
\end{align*}

\State \textbf{Output:} initial estimate 
\begin{align*}
\bm{U}^{0}=\big[\lambda_{1}^{1/3}\bm{u}^{1},\dots,\lambda_{r}^{1/3}\bm{u}^{r}\big],\qquad\bm{V}^{0}=\big[\lambda_{1}^{1/3}\bm{v}^{1},\dots,\lambda_{r}^{1/3}\bm{v}^{r}\big],\qquad\text{and}\qquad\bm{W}^{0}=\big[\lambda_{1}^{1/3}\bm{w}^{1},\dots,\lambda_{r}^{1/3}\bm{w}^{r}\big].
\end{align*}

\end{algorithmic}
\end{algorithm}

\begin{algorithm}[t]
\label{alg:one-factor-asym} \begin{algorithmic}[1]\Function{Retrieve-one-tensor-factor-asym}{$\bm{T},p,\bm{U},\bm{g}$} 

\State Compute \begin{subequations} 
\begin{align}
\bm{\theta} & =\bm{U}\bm{U}^{\top}\bm{g}=:\mathcal{P}_{\bm{U}}(\bm{g}),\\
\bm{M} & =p^{-1}\bm{T}\times_{1}\bm{\theta},\label{eq:defn-Mtau}
\end{align}
\end{subequations} where $\times_{1}$ is defined in Section~\ref{sec:Notation-1}. 

\State Let $\bm{\nu}^{(2)}$ (resp.~$\bm{\nu}^{(3)}$) be the leading
left (resp.~right) singular vector of $\bm{M}$. Let $\bm{\nu}^{(1)}$
= $p^{-1}\bm{T}\times_{2}\bm{\nu}^{(2)}\times_{3}\bm{\nu}^{(3)}$
and set $\lambda=\|\bm{\nu}^{(1)}\|_{2}^{1/3}$. 

\State \Return $\big(\bm{\nu}^{(1)},\bm{\nu}^{(2)},\bm{\nu}^{(3)},\lambda,\sigma_{1}(\bm{M})-\sigma_{2}(\bm{M})\big)$. 

\EndFunction \end{algorithmic}
\end{algorithm}

\begin{algorithm}[t]
\label{alg:prune-asym} \begin{algorithmic}[1]\Function{Prune-asym}{$\big\{\big(\bm{\nu}^{(1),\tau},\bm{\nu}^{(2),\tau},\bm{\nu}^{(3),\tau},\lambda_{\tau},\mathsf{gap}_{\tau}\big)\big\}_{\tau=1}^{L},\epsilon_{\mathsf{th}}$} 

\State Set $\Theta=\big\{\big(\bm{\nu}^{(1),\tau},\bm{\nu}^{(2),\tau},\bm{\nu}^{(3),\tau},\lambda_{\tau},\mathsf{gap}_{\tau}\big)\big\}_{\tau=1}^{L}.$ 

\For{$i=1,\dots,r$} 

\State Choose $(\bm{\nu}^{(1),\tau},\bm{\nu}^{(2),\tau},\bm{\nu}^{(3),\tau},\lambda_{\tau},\mathsf{gap}_{\tau})$
from $\Theta$ with the largest $\mathsf{gap}_{\tau}$; set $\bm{u}^{i}=\bm{\nu}^{(1),\tau},\bm{v}^{i}=\bm{\nu}^{(2),\tau},\bm{w}^{i}=\bm{\nu}^{(3),\tau}$
and $\lambda_{i}=\lambda_{\tau}$. 

\State Update $\Theta\gets\Theta\setminus\left\{ \big(\bm{\nu}^{(1),\tau},\bm{\nu}^{(2),\tau},\bm{\nu}^{(3),\tau},\lambda_{\tau},\mathsf{gap}_{\tau}\big)\in\Theta:|\langle\bm{\nu}^{(2),\tau},\bm{v}^{i}\rangle|>1-\epsilon_{\mathsf{th}}\right\} $. 

\EndFor 

\State \Return $\big\{(\bm{u}^{1},\bm{v}^{1},\bm{w}^{1},\lambda_{1}),\dots,(\bm{u}^{r},\bm{v}^{r},\bm{w}^{r},\lambda_{r})\big\}.$ 

\EndFunction \end{algorithmic}
\end{algorithm}

\subsection{Numerical experiments}

In order to validate the effectiveness of the proposed algorithm,
we conduct a series of numerical experiments. 

To begin with, let us generate the
true tensor $\bm{T}^{\star}=\sum_{1\leq i\leq r}\bm{u}_{i}^{\star}\otimes\bm{v}_{i}^{\star}\otimes\bm{w}_{i}^{\star}$
via the following procedure: (1) generate $\widehat{\bm{u}}_{i}^{\star}\overset{\mathsf{i.i.d.}}{\sim}\mathcal{N}(\bm{0},\bm{I}_{d_{1}})$,
$\widehat{\bm{v}}_{i}^{\star}\overset{\mathsf{i.i.d.}}{\sim}\mathcal{N}(\bm{0},\bm{I}_{d_{2}})$
and $\widehat{\bm{w}}_{i}^{\star}\overset{\mathsf{i.i.d.}}{\sim}\mathcal{N}(\bm{0},\bm{I}_{d_{3}})$,
and (2) set $\lambda_{i}^{\star}:=\|\widehat{\bm{u}}_{i}^{\star}\|_{2}\|\widehat{\bm{v}}_{i}^{\star}\|_{2}\|\widehat{\bm{w}}_{i}^{\star}\|_{2}$,
$\bm{u}_{i}^{\star}=\lambda_{i}^{\star1/3}\widehat{\bm{u}}_{i}^{\star}$,
$\bm{v}_{i}^{\star}=\lambda_{i}^{\star1/3}\widehat{\bm{v}}_{i}^{\star}$
and $\bm{w}_{i}^{\star}=\lambda_{i}^{\star1/3}\widehat{\bm{w}}_{i}^{\star}$.
Akin to the symmetric case, we choose the algorithmic parameters to
be $L=r^{2}$, $\epsilon_{\mathsf{th}}=0.4$ and $t_{0}=100$. The
noise components are generated as i.i.d.~Gaussians, namely, $E_{i,j,k}\overset{\mathsf{i.i.d.}}{\sim}\mathcal{N}(0,\sigma^{2})$. The signal-to-noise ratio (SNR) is defined as $\text{SNR}=\|\bm{T}^{\star}\|_{\mathrm{F}}^{2}/(\sigma^{2}d_{1}d_{2}d_{3})$,
and $\text{SNR}=\infty$ stands for the noiseless case (i.e.~the case with $\sigma=0$).
The stepsize is set to be $\eta_{t}\equiv\eta=1/\big(2\max_{i}\|\bm{u}_{i}^{0}\|_{2}^{4/3}\big)$,
where $\bm{U}^{0}=[\bm{u}_{1}^{0},\cdots,\bm{u}_{r}^{0}]$ is the
initial estimate generated by Algorithm \ref{alg:init-asym}. 
Figure~\ref{fig:convergence-asym} illustrates the numerical convergence
rates of the proposed algorithm, where we set $d_{1}=100$, $d_{2}=150$, $d_{3}=200$,
$r=4$ and $p=0.05$. Since one can only recover the tensor factors
up to global signs and global permutation, the relative estimation
errors of $\bm{U}^{t}$ 
are defined to be
\begin{align*}
 & \min_{\bm{\Pi}\in\mathsf{perm}_{r},\,\bm{S}=\mathsf{diag}(s_{i})\in\mathbb{R}^{r\times r},s_{i}=\pm1}\left\Vert \bm{U}^{t}\bm{\Pi}\bm{S}-\bm{U}^{\star}\right\Vert _{\mathrm{F}}/\left\Vert \bm{U}^{\star}\right\Vert _{\mathrm{F}},\\
 & \min_{\bm{\Pi}\in\mathsf{perm}_{r},\,\bm{S}=\mathsf{diag}(s_{i})\in\mathbb{R}^{r\times r},s_{i}=\pm1}\left\Vert \bm{U}^{t}\bm{\Pi}\bm{S}-\bm{U}^{\star}\right\Vert _{\mathrm{2,\infty}}/\left\Vert \bm{U}^{\star}\right\Vert _{\mathrm{2,\infty}},
\end{align*}
where $\mathsf{perm}_{r}$ stands for the set of $r\times r$ permutation
matrices. The error metrics for $\bm{V}^{t}$ and $\bm{W}^{t}$ can be defined analogously.  The relative Euclidean and $\ell_{2,\infty}$ estimation
errors of $\bm{T}$ are defined as $\|\bm{T}^{t}-\bm{T}^{\star}\|_{\mathrm{F}}/\|\bm{T}^{\star}\|_{\mathrm{F}}$
and $\|\bm{T}^{t}-\bm{T}^{\star}\|_{2,\infty}/\|\bm{T}^{\star}\|_{2,\infty}$,
respectively, where $\bm{T}^{t}=\sum_{1\leq i\leq r}\bm{u}_{i}^{t}\otimes\bm{v}_{i}^{t}\otimes\bm{w}_{i}^{t}$.
As can be seen from the plots, the estimation errors decay geometrically
fast in the noiseless case. In the noisy case, the numerical estimation
errors of the algorithm also converge geometrically fast until an
error floor is hit. 

\begin{figure}[t]
\label{label:convergence-asym}\centering

\begin{tabular}{cc}
\includegraphics[width=0.45\textwidth]{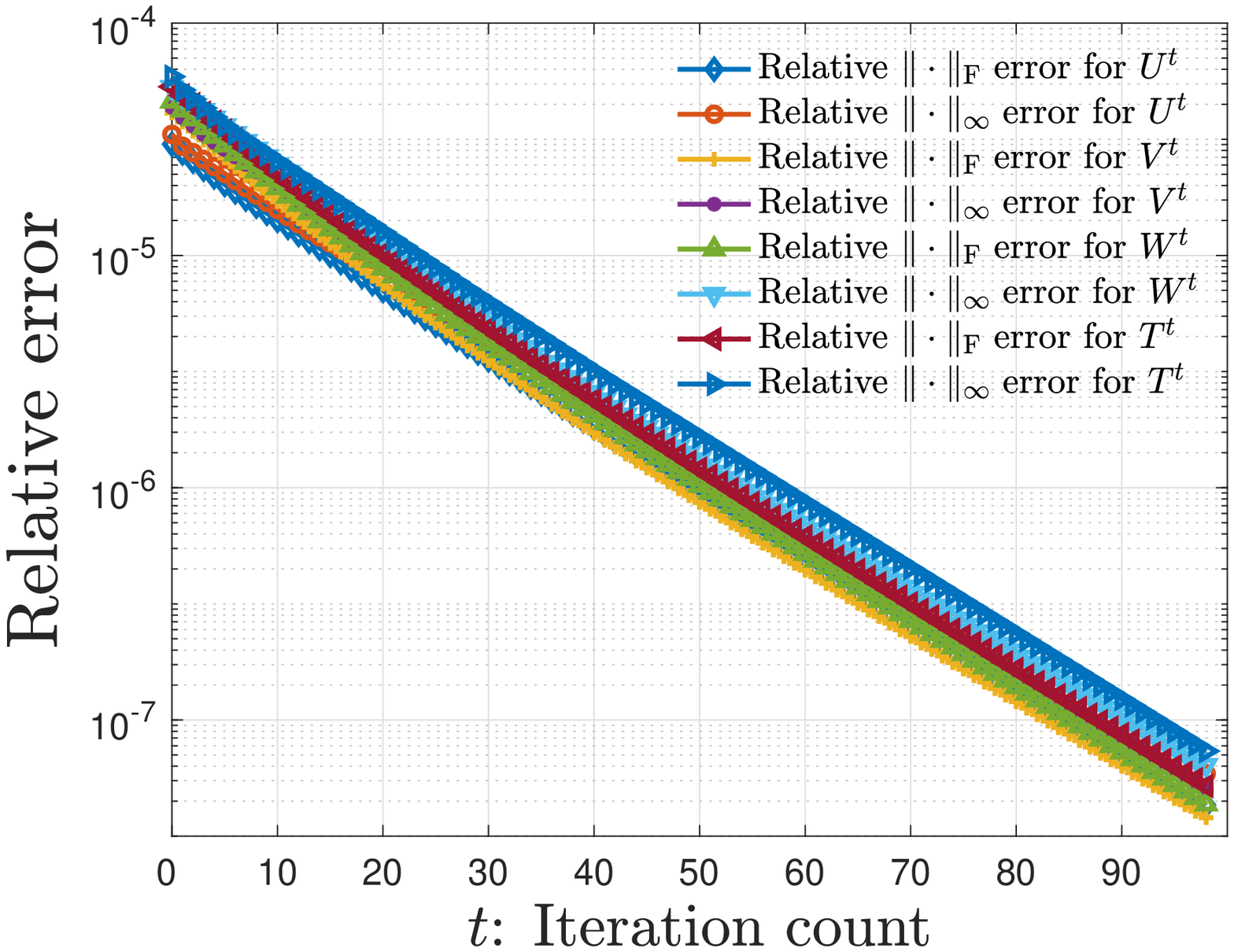} & \includegraphics[width=0.45\textwidth]{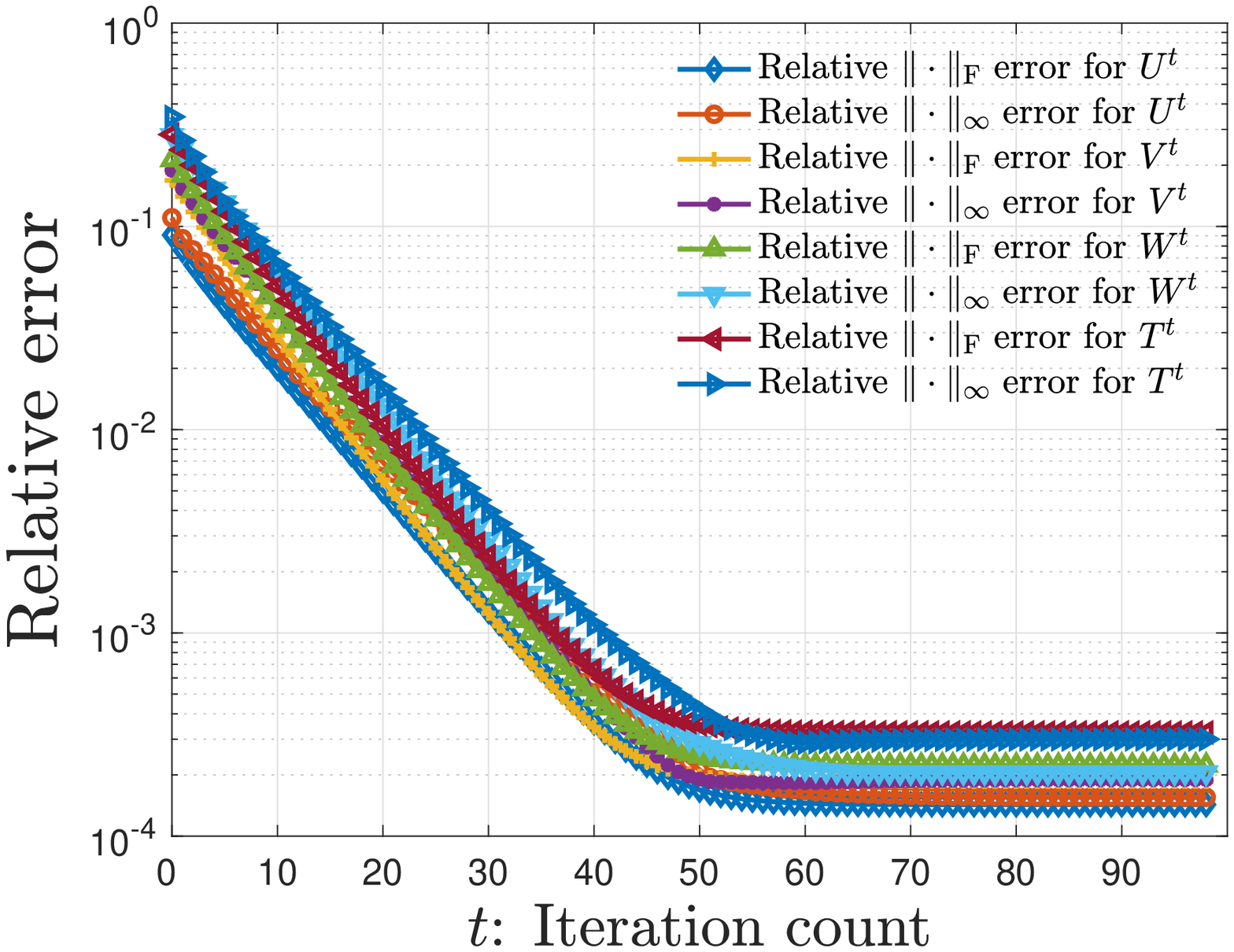}\tabularnewline
(a) & (b)\tabularnewline
\end{tabular}

\caption{(a) relative errors of the estimates $\bm{U}^{t},\bm{V}^{t},\bm{W}^{t}$
and $\bm{T}^{t}$ vs. iteration count $t$ for noiseless tensor completion,
where $d_{1}=100$, $d_{2}=150$, $d_{3}=200$, $r=4$, $p=0.05$;
(b) relative errors of the estimates $\bm{U}^{t},\bm{V}^{t},\bm{W}^{t}$
and $\bm{T}^{t}$ vs. iteration count $t$ for noisy tensor completion,
where $d_{1}=100$, $d_{2}=150$, $d_{3}=200$, $r=4$, $p=0.05$,
$\text{SNR}=10$.\label{fig:convergence-asym}}
\end{figure}

Moving beyond the above synthetic data, we apply our methods to an simulated MRI brain
image dataset \cite{cocosco1997brainweb}, which is available online at the McGill University McConnell Brain Imaging
Centre and has also been studied in prior work \cite{xia2017statistically}. The database consists of pre-computed simulated brain data, and a set of parameters can be set to generate the data accordingly. In this series of experiments, we choose the parameters to be T1 modality, $1$mm
slice thickness, $1\%$ noise and $20\%$ RF. The resulting
data is a three-order tensor in $\mathbb{R}^{181\times217\times181}$,
where each slice in any mode corresponds to a brain image. We use the tensor decomposition algorithm in \cite{anandkumar2015learning} to decompose the original data, and keep the top-$36$ components whose energy accounts for $85\%$ of the original tensor. Hence, the resulting low-rank tensor $\bm{T}^\star$ preserves the major information of the original tensor, and is used as the ground truth in the simulation. We sample
each entry of $\bm{T}^{\star}$ independently with probability $p$,
and then inject i.i.d.~Gaussian noise $\mathcal{N}(0,\sigma^{2})$ to each
observed entry (so as to simulate the noisy scenario). Based on the incomplete set of noisy samples, we apply the proposed algorithm to reconstruct $\bm{T}^\star$; the algorithmic parameters are chosen as above, and we denote by $\bm{T}$ the resulting tensor estimate. Table~\ref{table:real-data}
 reports the relative errors $\|\bm{T}-\bm{T}^{\star}\|_{\mathrm{F}}/\|\bm{T}^{\star}\|_{\mathrm{F}}$
and $\|\bm{T}-\bm{T}^{\star}\|_{\infty}/\|\bm{T}^{\star}\|_{\infty}$. As can be seen in the table, the performance of our algorithm is quite
favorable for low-rank tensor reconstruction. In particular, given that these real data tensor examples might not satisfy the assumptions imposed in our main theorems, these numerical experiments hint at the applicability of our algorithm to a broader set of problems. 

\begin{table}[t]
\centering \caption{Relative $\ell_{2}$ and $\ell_{2,\infty}$ errors for varying $p$
and SNR  in the MRI data experiments.\label{table:real-data}}
\begin{tabular}{c|c|c}
\hline 
$(p,\text{SNR})$ & $\|\bm{T}-\bm{T}^{\star}\|_{\mathrm{F}}/\|\bm{T}^{\star}\|_{\mathrm{F}}$ & $\|\bm{T}-\bm{T}^{\star}\|_{\infty}/\|\bm{T}^{\star}\|_{\infty}$\tabularnewline
\hline 
$(0.05,10)$ & $0.2231$ & $0.8774$\tabularnewline
\hline 
$(0.05,\infty)$ & $0.2149$ & $0.7856$\tabularnewline
\hline 
$(0.1,10)$ & $0.1472$ & $0.5343$\tabularnewline
\hline 
$(0.1,\infty)$ & $0.1398$ & $0.3991$\tabularnewline
\hline 
$(0.2,10)$ & $0.0841$ & $0.1417$\tabularnewline
\hline 
$(0.2,\infty)$ & $0.0686$ & $0.1193$\tabularnewline
\hline 
\end{tabular}
\end{table}

\subsection{Analysis ideas}

Before describing the proof ideas, we define the following incoherence
parameters and condition number that, similar  to the symmetric case,  play a crucial role in our theoretical
development.

\begin{subequations}

\begin{definition}\label{definition:incoherence-asymmetric}Define
the incoherence parameters and the condition number of $\bm{T}^{\star}$ as follows 
	%\yxc{typo in the definition of $\mu_1$}
\begin{align}
\mu_{0} & :=\frac{d_{1}d_{2}d_{3}\|\bm{T}^{\star}\|_{\infty}^{2}}{\|\bm{T}^{\star}\|_{\mathrm{F}}^{2}},\label{def:incoh-T-inf-asym}\\
\mu_{1} & :=\max\left\{ \frac{ d_{1} \|\bm{u}_{i}^{\star}\|_{\infty}^{2}}{\|\bm{u}_{i}^{\star}\|_{2}^{2}},\,\frac{d_2 \|\bm{v}_{i}^{\star}\|_{\infty}^{2}}{\|\bm{v}_{i}^{\star}\|_{2}^{2}},\,\frac{d_3 \|\bm{w}_{i}^{\star}\|_{\infty}^{2}}{\|\bm{w}_{i}^{\star}\|_{2}^{2}}\right\} ,\label{def:incoh-u-inf-asym}\\
\mu_{2} & :=\max\left\{ \frac{d_1 \langle\bm{u}_{i}^{\star},\bm{u}_{j}^{\star}\rangle^{2}}{\|\bm{u}_{i}^{\star}\|_{2}^{2}\|\bm{u}_{j}^{\star}\|_{2}^{2}},\,\frac{d_2 \langle\bm{v}_{i}^{\star},\bm{v}_{j}^{\star}\rangle^{2}}{\|\bm{v}_{i}^{\star}\|_{2}^{2}\|\bm{v}_{j}^{\star}\|_{2}^{2}},\,\frac{d_{3} \langle\bm{w}_{i}^{\star},\bm{w}_{j}^{\star}\rangle^{2}}{\|\bm{w}_{i}^{\star}\|_{2}^{2}\|\bm{w}_{j}^{\star}\|_{2}^{2}}\right\} ,\label{def:incoh-orth-asym}\\
\kappa & :=\frac{\max_{i}\left\{ \|\bm{u}_{i}^{\star}\|_{2}\|\bm{v}_{i}^{\star}\|_{2}\|\bm{w}_{i}^{\star}\|_{2}\right\} }{\min_{i}\left\{ \|\bm{u}_{i}^{\star}\|_{2}\|\bm{v}_{i}^{\star}\|_{2}\|\bm{w}_{i}^{\star}\|_{2}\right\} }.\label{def:kappa-asym}
\end{align}
\end{definition}
\noindent For notational convenience, we shall set
\begin{equation}
\mu:=\max\left\{ \mu_{0},\mu_{1},\mu_{2}\right\} ,\quad d_{\mathrm{min}}:=\min\{d_{1},\,d_{2},\,d_{3}\}\quad\text{and}\quad d_{\mathrm{max}}:=\max\{d_{1},\,d_{2},\,d_{3}\}.\label{eq:defn-mu-dmin-dmax-asymmetry}
\end{equation}
\end{subequations}

% \yxc{Give the reader an idea what you're going to do below. Are you going to prove a theorem or are you just going to outline some important steps (if so, why does this step need special attention but not others?) Also, properly connect these with the corresponding lemmas / theorems in the symmetric case. Fill in missing English here (and throughout). }

As has been made clear in the symmetric case, at the heart of our analysis  lie two crucial components: (1) the
geometric property of a noiseless version of the loss function,
and (2) reasonably well initial estimates for tensor factors (in the entrywise sense). 
Rather than providing a complete analysis for the asymmetric case (which would be very long), we shall only point out the important steps needed to extend these two parts for the asymmetric case.
%that differ from  the symmetric counterpart.
%

\paragraph{1. Local optimization landscape.}
Similar to the symmetric counterpart, the key step of the local convergence analysis lies in establishing the favorable geometric property (i.e.~local strong convexity and smoothness) of the following noiseless regularized loss function
%
%$g_{\mathsf{clean}}\big(\bm{U},\bm{V},\bm{W}\big)$ defined as
\begin{equation}
g_{\mathsf{clean}}\big(\bm{U},\bm{V},\bm{W}\big)=\frac{1}{6p}\,\Big\|\mathcal{P}_{\Omega}\Big(\sum_{i=1}^{r}\bm{u}_{i}\otimes\bm{v}_{i}\otimes\bm{w}_{i}-\bm{T}^{\star}\Big)\Big\|_{\mathrm{F}}^{2}+\mathsf{reg}\big(\bm{U},\bm{V},\bm{W}\big),
	\label{eq:loss-function-clean-asym}
\end{equation}
where the regularization term $\mathsf{reg}\big(\bm{U},\bm{V},\bm{W}\big)$
is defined in (\ref{eq:reg-function-asym}).
In words, this is a simplified version of the original loss function \eqref{eq:loss-function-asym} by dropping the influence of the noise.  
Lemma~\ref{lemma:RIC-asym}
below demonstrates that $g_{\mathsf{clean}}$ is locally strongly
convex and smooth in the neighborhood of the ground truth. \begin{lemma}[{\bf
Local strong convexity and smoothness}] \label{lemma:RIC-asym}
Suppose that
\begin{align}
\label{eq:RIC-asym-asmp}
p\geq c_{0}\max\left\{ \frac{\log^{3}d_{\max}}{\sqrt{d_{1}d_{2}d_{3}}},\frac{\mu^{2}r^{2}d_{\max}\log^{5}d_{\max}}{d_{1}d_{2}d_{3}}\right\} ,\qquad r\leq c_{1}\sqrt{\frac{d_{\min}}{\mu}}
\end{align}
and that the regularization parameter obeys $\big|\alpha_{i}-\lambda_{i}^{\star2/3}\big|\leq c_{2}\lambda_{\min}^{\star2/3}$
for some sufficiently large (resp.~small) constant $c_{0}>0$ (resp.~$c_{1},c_{2}>0$).
Then with probability greater than $1-O\big(d_{\min}^{-10}\big)$,
\begin{align}
\mathsf{vec}\left(\begin{bmatrix}\bm{X}\\
\bm{Y}\\
\bm{Z}
\end{bmatrix}\right)^{\top}\nabla^{2}g_{\mathsf{clean}}(\bm{U},\bm{V},\bm{W})\,\mathsf{vec}\left(\begin{bmatrix}\bm{X}\\
\bm{Y}\\
\bm{Z}
\end{bmatrix}\right) & \leq4\lambda_{\max}^{\star4/3}\left(\left\Vert \bm{X}\right\Vert _{\mathrm{F}}^{2}+\left\Vert \bm{Y}\right\Vert _{\mathrm{F}}^{2}+\left\Vert \bm{Z}\right\Vert _{\mathrm{F}}^{2}\right),\\
\mathsf{vec}\left(\begin{bmatrix}\bm{X}\\
\bm{Y}\\
\bm{Z}
\end{bmatrix}\right)^{\top}\nabla^{2}g_{\mathsf{clean}}(\bm{U},\bm{V},\bm{W})\,\mathsf{vec}\left(\begin{bmatrix}\bm{X}\\
\bm{Y}\\
\bm{Z}
\end{bmatrix}\right) & \geq\frac{1}{2}\lambda_{\min}^{\star4/3}\left(\left\Vert \bm{X}\right\Vert _{\mathrm{F}}^{2}+\left\Vert \bm{Y}\right\Vert _{\mathrm{F}}^{2}+\left\Vert \bm{Z}\right\Vert _{\mathrm{F}}^{2}\right)
\end{align}
holds simultaneously for all $\bm{X}\in\mathbb{R}^{d_{1}\times r},\bm{Y}\in\mathbb{R}^{d_{2}\times r},\bm{Z}\in\mathbb{R}^{d_{3}\times r}$
and all $\bm{U}\in\mathbb{R}^{d_{1}\times r},\bm{V}\in\mathbb{R}^{d_{2}\times r},\bm{W}\in\mathbb{R}^{d_{3}\times r}$
obeying
\begin{subequations}\label{eq:RIC-asym-asymp}
\begin{align}
\|\bm{U}-\bm{U}^{\star}\bm{S}^{(1)}\|_{\mathrm{F}} & \leq\delta\left\Vert \bm{U}^{\star}\right\Vert _{\mathrm{F}}\quad\text{and}\quad\left\Vert \bm{U}-\bm{U}^{\star}\bm{S}^{(1)}\right\Vert _{2,\infty}\leq\delta\left\Vert \bm{U}^{\star}\right\Vert _{2,\infty};\label{eq:RIC-asym-asymp-U}\\
\|\bm{V}-\bm{V}^{\star}\bm{S}^{(2)}\|_{\mathrm{F}} & \leq\delta\left\Vert \bm{V}^{\star}\right\Vert _{\mathrm{F}}\quad\text{and}\quad\left\Vert \bm{V}-\bm{V}^{\star}\bm{S}^{(2)}\right\Vert _{2,\infty}\leq\delta\left\Vert \bm{V}^{\star}\right\Vert _{2,\infty};\label{eq:RIC-asym-asymp-V}\\
\|\bm{W}-\bm{W}^{\star}\bm{S}^{(3)}\|_{\mathrm{F}} & \leq\delta\left\Vert \bm{V}^{\star}\right\Vert _{\mathrm{F}}\quad\text{and}\quad\left\Vert \bm{W}-\bm{W}^{\star}\bm{S}^{(3)}\right\Vert _{2,\infty}\leq\delta\left\Vert \bm{W}^{\star}\right\Vert _{2,\infty}.\label{eq:RIC-asym-asymp-W}
\end{align}
\end{subequations}
Here, $\delta\leq c_{3}/(\mu^{3/2}r)$ for some
sufficiently small constant $c_{3}>0$, and $\bm{S}^{(1)}$, $\bm{S}^{(2)}$
and $\bm{S}^{(3)}$ are some diagonal matrices in $\mathbb{R}^{r\times r}$
such that for each $1\leq i\leq r$, two of $S_{i,i}^{(1)}$, $S_{i,i}^{(2)}$
and $S_{i,i}^{(3)}$ equal to $-1$ with the remaining one equal to $1$. \end{lemma}
\begin{proof}See Appendix~\ref{subsec:Proof-of-Lemma:RIC-asym}.
\end{proof}
%\yxc{Explain what this lemma means; and explain a bit the importance of regularization.}
In a nutshell, Lemma~\ref{lemma:RIC-asym} confirms local strong convexity and smoothness of  the noiseless regularized loss  $g_{\mathsf{clean}}(\bm{U},\bm{V},\bm{W})$,
provided that \eqref{eq:RIC-asym-asmp} holds and that the matrices $\bm{U}$,
$\bm{V}$ and $\bm{W}$ are sufficiently close to the ground truth in every single
row. This is similar to the property of the symmetric counterpart $f_{\mathsf{clean}} (\bm{U})$ (cf.~Lemma~\ref{lemma:RIC} in Appendix~\ref{sec:analysis-GD}), except that we need to deal with asymmetric tensor factors as well as additional regularization terms. In particular, when $d_1 \asymp d_2 \asymp d_3 \asymp d$ and $\mu, r \asymp 1$, Condition \eqref{eq:RIC-asym-asmp} reduces to $p \gg d^{-3/2} \log^3 d$ and $r \ll \sqrt{d}$, which resembles  \eqref{eq:RIC-asymp-sym} in Lemma~\ref{lemma:RIC} derived for the symmetric case.

Having established the preceding local geometric properties of $g_{\mathsf{clean}}$, one can then argue similarly as in Lemmas~\ref{lemma:U_frob_loss}
and \ref{lemma:U_2inf_loss} in Appendix~\ref{sec:analysis-GD} to
prove that gradient descent converges linearly, as long as it is provided with an initial estimate satisfying the condition~\eqref{eq:RIC-asym-asymp}. Here, we emphasize that the regularization term \eqref{eq:reg-function-asym}, which essentially balances the sizes of the three tensor factors, is crucial for the local strong convexity and smoothness of $g_{\mathsf{clean}}$ to hold. 
% The gradient descent will not converge to the global optimum without the regularization term.

\paragraph{2. Guaranteeing a reasonably good initialization.} Another crucial ingredient lies in guaranteeing an initial estimate with sufficiently good accuracy.  Recall that our initialization scheme consists of two stages: (1) subspace estimation, and (2) retrieval of individual tensor factors. 
\begin{itemize}
\item 
	The subspace estimation part remains largely unchanged:  we shall unfold the observed tensor along the $1$-st mode to estimate the subspace spanned by tensor factors $\{ \bm{u}_i^\star\}_{i=1}^r$,  and the $\ell_{2,\infty}$ subspace estimation accuracy can be established by invoking the main theorems of our companion paper \cite{cai2019subspace}.

\item Regarding the retrieval of individual tensor factors, the key observation is that: the random vector $\bm{g}^{\tau}$ we generate satisfies 
\begin{align*}
\bm{T}^\star \times_{1}\mathcal{P}_{\bm{U}^\star} (\bm{g}^{\tau})  
	= \sum_{i=1}^r \lambda_i^\star \langle \overline{\bm{u}}_{i}^{\star}, \mathcal{P}_{\bm{U}^\star} (\bm{g}^\tau) \rangle  \overline{\bm{v}}_{i}^{\star} \overline{\bm{w}}_{i}^{\star\top}
	= \sum_{i=1}^r \lambda_i^\star \langle \overline{\bm{u}}_{i}^{\star}, \bm{g}^\tau \rangle  \overline{\bm{v}}_{i}^{\star} \overline{\bm{w}}_{i}^{\star\top},
\end{align*}
where $\mathcal{P}_{\bm{U}^\star}$ is the projection onto the subspace spanned by $\{ \bm{u}_i^\star\}_{i=1}^r$, and
\begin{align}
	\overline{\bm{u}}_{i}^{\star}= \frac{1}{\|\bm{u}_{i}^{\star}\|_{2}} \bm{u}_{i}^{\star}, \qquad
	\overline{\bm{v}}_{i}^{\star}=  \frac{1}{\|\bm{v}_{i}^{\star}\|_{2}} \bm{v}_{i}^{\star}, \qquad
	and \qquad
	\overline{\bm{w}}_{i}^{\star}=  \frac{1}{\|\bm{w}_{i}^{\star}\|_{2}} \bm{w}_{i}^{\star}.
\end{align}
Given a sufficiently accurate subspace estimate $\bm{U}$ for $\bm{U}^\star$ obtained in the subspace estimation stage,
for each $1\leq i\leq r$, there exists at least a point $1\leq\tau\leq L$
such that the spectral gap of the population version of $\bm{T}\times_{1} \bm{\theta}^\tau = \bm{T}\times_{1}\mathcal{P}_{\bm{U}} (\bm{g}^{\tau})$
(with respect to the sampling and noise) is large enough and that the perturbation is sufficiently small. As a result, the top left (resp.~right) singular vector $\bm{\nu}^{(2)}$ (resp.~$\bm{\nu}^{(3)}$) of $\bm{T}\times_{1} \bm{\theta}^\tau$
 is close to the (normalized) tensor factor
$\overline{\bm{v}}_{i}^{\star}$ (resp.~$\overline{\bm{w}}_{i}^{\star}$)
		both in the $\ell_{2}$ and $\ell_{\infty}$ norm  (similar to Lemmas~\ref{lemma:init_2_loss}-\ref{lemma:init_loo_entry_loss} in Section~\ref{sec:analysis-retrieval} for the symmetric case). In turn, this further allows
us to reliably estimate $\overline{\bm{u}}_{i}^{\star}$ and the magnitude
$\lambda_{i}^{\star}$,  in a way similar to what we have done in Lemma~\ref{lemma:init_coeff} in Section~\ref{sec:analysis-retrieval}. By repeating the procedures with random restarts and invoking a pruning procedure similar to the symmetric case, we can hope to recover all tensor factors with high probability. 
\end{itemize}
Following the above strategies, one could adapt the proofs of Theorems~\ref{thm:init}-\ref{thm:init-loo}
in Section~\ref{sec:analysis-retrieval} to show that: our initial
estimates $\{\bm{u}_{i},\bm{v}_{i},\bm{w}_{i}\}_{i=1}^{r}$ are all exceedingly
close to the ground truth in the entrywise sense (up to global permutation
and global signs). This in turn confirms that the algorithm will enter a locally strongly convex and smooth region as characterized in  Lemma~\ref{lemma:RIC-asym}, thus leading to our performance guarantees for the entire algorithm.  Once again, while the analysis ideas for the asymmetric case bear much resemblance to the symmetric counterpart,  a complete proof has to be fairly long due to more clumsy notation compared to the symmetric case;  for the sake of brevity, we do not provide the full proof here.

\subsection{Proof of Lemma \ref{lemma:RIC-asym}}

\label{subsec:Proof-of-Lemma:RIC-asym}

Fix arbitrary matrices $\bm{X}=[\bm{x}_{1},\cdots,\bm{x}_{r}]\in\mathbb{R}^{d_{1}\times r}$,
$\bm{Y}=[\bm{y}_{1},\cdots,\bm{y}_{r}]\in\mathbb{R}^{d_{2}\times r}$
and $\bm{Z}=[\bm{z}_{1},\cdots,\bm{z}_{r}]\in\mathbb{R}^{d_{3}\times r}$.
Direct computation reveals that
\begin{align*}
 & \mathsf{vec}\left(\begin{bmatrix}\bm{X}\\
\bm{Y}\\
\bm{Z}
\end{bmatrix}\right)^{\top}\nabla^{2}g_{\mathsf{clean}}(\bm{U},\bm{V},\bm{W})\,\mathsf{vec}\left(\begin{bmatrix}\bm{X}\\
\bm{Y}\\
\bm{Z}
\end{bmatrix}\right)\\
 & \qquad=\frac{1}{3p}\,\Big\|\mathcal{P}_{\Omega}\Big(\sum_{i=1}^{r}\bm{x}_{i}\otimes\bm{v}_{i}\otimes\bm{w}_{i}+\bm{u}_{i}\otimes\bm{y}_{i}\otimes\bm{w}_{i}+\bm{u}_{i}\otimes\bm{v}_{i}\otimes\bm{z}_{i}\Big)\Big\|_{\mathrm{F}}^{2}\\
 & \qquad\quad+\frac{2}{3p}\,\Big\langle\sum_{i=1}^{r}\bm{x}_{i}\otimes\bm{y}_{i}\otimes\bm{w}_{i}+\bm{x}_{i}\otimes\bm{v}_{i}\otimes\bm{z}_{i}+\bm{u}_{i}\otimes\bm{y}_{i}\otimes\bm{z}_{i},\mathcal{P}_{\Omega}\Big(\sum_{i=1}^{r}\bm{u}_{i}\otimes\bm{v}_{i}\otimes\bm{w}_{i}-\bm{T}^{\star}\Big)\Big\rangle\\
 & \qquad\quad+\frac{1}{3}\sum_{i=1}^{r}\alpha_{i}\big(\langle\bm{x}_{i},\bm{u}_{i}\rangle-\langle\bm{y}_{i},\bm{v}_{i}\rangle\big)^{2}+\frac{1}{6}\sum_{i=1}^{r}\alpha_{i}\big(\|\bm{u}_{i}\|_{2}^{2}-\|\bm{v}_{i}\|_{2}^{2}\big)\big(\|\bm{x}_{i}\|_{2}^{2}-\|\bm{y}_{i}\|_{2}^{2}\big)\\
 & \qquad\quad+\frac{1}{3}\sum_{i=1}^{r}\alpha_{i}\big(\langle\bm{x}_{i},\bm{u}_{i}\rangle-\langle\bm{z}_{i},\bm{w}_{i}\rangle\big)^{2}+\frac{1}{6}\sum_{i=1}^{r}\alpha_{i}\big(\|\bm{u}_{i}\|_{2}^{2}-\|\bm{w}_{i}\|_{2}^{2}\big)\big(\|\bm{x}_{i}\|_{2}^{2}-\|\bm{z}_{i}\|_{2}^{2}\big)\\
 & \qquad\quad+\frac{1}{3}\sum_{i=1}^{r}\alpha_{i}\big(\langle\bm{y}_{i},\bm{v}_{i}\rangle-\langle\bm{z}_{i},\bm{w}_{i}\rangle\big)^{2}+\frac{1}{6}\sum_{i=1}^{r}\alpha_{i}\big(\|\bm{v}_{i}\|_{2}^{2}-\|\bm{w}_{i}\|_{2}^{2}\big)\big(\|\bm{y}_{i}\|_{2}^{2}-\|\bm{z}_{i}\|_{2}^{2}\big).
\end{align*}
For notational convenience, we shall define
\begin{align*}
\bm{\Delta}^{\bm{U}}=\big[\bm{\Delta}_{1}^{\bm{U}},\cdots,\bm{\Delta}_{r}^{\bm{U}}\big]:=\bm{U}\bm{S}^{(1)}-\bm{U}^{\star}\in\mathbb{R}^{d_{1}\times r},
\end{align*}
and define $\bm{\Delta}^{\bm{V}}\in\mathbb{R}^{d_{2}\times r}$ and
$\bm{\Delta}^{\bm{W}}\in\mathbb{R}^{d_{3}\times r}$ in an analogous
manner. Without loss of generality, it is assumed that $\bm{S}^{(1)}=\bm{I}_{r}$
and $\bm{S}^{(2)}=\bm{S}^{(3)}=-\bm{I}_{r}$.

\begin{comment}
Note that we can expand
\begin{align*}
\sum_{i=1}^{r}\alpha_{i}\big(\langle\bm{x}_{i},\bm{u}_{i}\rangle-\langle\bm{y}_{i},\bm{v}_{i}\rangle\big)^{2} & =\sum_{i=1}^{r}\alpha_{i}\big(\langle\bm{x}_{i},\bm{u}_{i}^{\star}+\bm{\Delta}_{i}^{\bm{U}}\rangle-\langle\bm{y}_{i},\bm{v}_{i}^{\star}+\bm{\Delta}_{i}^{\bm{V}}\rangle\big)^{2}\\
 & =\sum_{i=1}^{r}\alpha_{i}\big(\langle\bm{x}_{i},\bm{u}_{i}^{\star}\rangle-\langle\bm{y}_{i},\bm{v}_{i}^{\star}\rangle+\langle\bm{x}_{i},\bm{\Delta}_{i}^{\bm{U}}\rangle-\langle\bm{y}_{i},\bm{\Delta}_{i}^{\bm{V}}\rangle\big)^{2}\\
 & =\sum_{i=1}^{r}\alpha_{i}\big(\langle\bm{x}_{i},\bm{u}_{i}^{\star}\rangle-\langle\bm{y}_{i},\bm{v}_{i}^{\star}\rangle\big)^{2}+\sum_{i=1}^{r}\alpha_{i}\big(\langle\bm{x}_{i},\bm{\Delta}_{i}^{\bm{U}}\rangle-\langle\bm{y}_{i},\bm{\Delta}_{i}^{\bm{V}}\rangle\big)^{2}\\
 & \quad+2\sum_{i=1}^{r}\alpha_{i}\big(\langle\bm{x}_{i},\bm{u}_{i}^{\star}\rangle-\langle\bm{y}_{i},\bm{v}_{i}^{\star}\rangle\big)\big(\langle\bm{x}_{i},\bm{\Delta}_{i}^{\bm{U}}\rangle-\langle\bm{y}_{i},\bm{\Delta}_{i}^{\bm{V}}\rangle\big).
\end{align*}
\end{comment}
With the above notation in place, one can decompose
\begin{align*}
\mathsf{vec}\left(\begin{bmatrix}\bm{X}\\
\bm{Y}\\
\bm{Z}
\end{bmatrix}\right)^{\top}\nabla^{2}g_{\mathsf{clean}}(\bm{U},\bm{V},\bm{W})\,\mathsf{vec}\left(\begin{bmatrix}\bm{X}\\
\bm{Y}\\
\bm{Z}
\end{bmatrix}\right) & =\sum_{i=1}^{4}\beta_{i}
\end{align*}
where $\beta_{i},1\leq i\leq4$ are given respectively by \begin{subequations}
\begin{align}
\beta_{1} & :=\frac{1}{3}\,\Big\|\sum_{i=1}^{r}\bm{x}_{i}\otimes(-\bm{v}_{i}^{\star})\otimes(-\bm{w}_{i}^{\star})+\bm{u}_{i}^{\star}\otimes\bm{y}_{i}\otimes(-\bm{w}_{i}^{\star})+\bm{u}_{i}^{\star}\otimes(-\bm{v}_{i}^{\star})\otimes\bm{z}_{i}\Big\|_{\mathrm{F}}^{2}\nonumber \\
 & \,\quad+\frac{1}{3}\sum_{i=1}^{r}\lambda_{i}^{\star2/3}\Big[\big(\langle\bm{x}_{i},\bm{u}_{i}^{\star}\rangle-\langle\bm{y}_{i},-\bm{v}_{i}^{\star}\rangle\big)^{2}+\big(\langle\bm{x}_{i},\bm{u}_{i}^{\star}\rangle-\langle\bm{z}_{i},-\bm{w}_{i}^{\star}\rangle\big)^{2}+\big(\langle\bm{y}_{i},-\bm{v}_{i}^{\star}\rangle-\langle\bm{z}_{i},-\bm{w}_{i}^{\star}\rangle\big)^{2}\Big],\label{eq:Hessian-asym-beta1}\\
\beta_{2} & :=\frac{1}{3p}\,\Big\|\mathcal{P}_{\Omega}\Big(\sum_{i=1}^{r}\bm{x}_{i}\otimes\bm{v}_{i}\otimes\bm{w}_{i}+\bm{u}_{i}\otimes\bm{y}_{i}\otimes\bm{w}_{i}+\bm{u}_{i}\otimes\bm{v}_{i}\otimes\bm{z}_{i}\Big)\Big\|_{\mathrm{F}}^{2}\nonumber \\
 & \,\qquad-\frac{1}{3}\,\Big\|\sum_{i=1}^{r}\bm{x}_{i}\otimes(-\bm{v}_{i}^{\star})\otimes(-\bm{w}_{i}^{\star})+\bm{u}_{i}^{\star}\otimes\bm{y}_{i}\otimes(-\bm{w}_{i}^{\star})+\bm{u}_{i}^{\star}\otimes(-\bm{v}_{i}^{\star})\otimes\bm{z}_{i}\Big\|_{\mathrm{F}}^{2},\label{eq:Hessian-asym-beta2}\\
\beta_{3} & :=\frac{2}{3p}\,\Big\langle\sum_{i=1}^{r}\bm{x}_{i}\otimes\bm{y}_{i}\otimes\bm{w}_{i}+\bm{x}_{i}\otimes\bm{v}_{i}\otimes\bm{z}_{i}+\bm{u}_{i}\otimes\bm{y}_{i}\otimes\bm{z}_{i},\mathcal{P}_{\Omega}\Big(\sum_{i=1}^{r}\bm{u}_{i}\otimes\bm{v}_{i}\otimes\bm{w}_{i}-\bm{T}^{\star}\Big)\Big\rangle,\label{eq:Hessian-asym-beta3}\\
\beta_{4} & :=\frac{1}{3}\sum_{i=1}^{r}\big(\alpha_{i}-\lambda_{i}^{\star2/3}\big)\Big[\big(\langle\bm{x}_{i},\bm{u}_{i}^{\star}\rangle-\langle\bm{y}_{i},-\bm{v}_{i}^{\star}\rangle\big)^{2}+\big(\langle\bm{x}_{i},\bm{u}_{i}^{\star}\rangle-\langle\bm{z}_{i},-\bm{w}_{i}^{\star}\rangle\big)^{2}+\big(\langle\bm{y}_{i},-\bm{v}_{i}^{\star}\rangle-\langle\bm{z}_{i},-\bm{w}_{i}^{\star}\rangle\big)^{2}\Big]\nonumber \\
 & \,\quad+\frac{1}{3}\sum_{i=1}^{r}\alpha_{i}\Big[\big(\langle\bm{x}_{i},\bm{\Delta}_{i}^{\bm{U}}\rangle-\langle\bm{y}_{i},\bm{\Delta}_{i}^{\bm{V}}\rangle\big)^{2}+\big(\langle\bm{x}_{i},\bm{\Delta}_{i}^{\bm{U}}\rangle-\langle\bm{z}_{i},\bm{\Delta}_{i}^{\bm{W}}\rangle\big)^{2}+\big(\langle\bm{y}_{i},\bm{\Delta}_{i}^{\bm{V}}\rangle-\langle\bm{z}_{i},\bm{\Delta}_{i}^{\bm{W}}\rangle\big)^{2}\Big]\nonumber \\
 & \,\quad+\frac{2}{3}\sum_{i=1}^{r}\alpha_{i}\Big[\big(\langle\bm{x}_{i},\bm{u}_{i}^{\star}\rangle-\langle\bm{y}_{i},-\bm{v}_{i}^{\star}\rangle\big)\big(\langle\bm{x}_{i},\bm{\Delta}_{i}^{\bm{U}}\rangle-\langle\bm{y}_{i},\bm{\Delta}_{i}^{\bm{V}}\rangle\big)+\big(\langle\bm{x}_{i},\bm{u}_{i}^{\star}\rangle-\langle\bm{z}_{i},-\bm{w}_{i}^{\star}\rangle\big)\big(\langle\bm{x}_{i},\bm{\Delta}_{i}^{\bm{U}}\rangle-\langle\bm{z}_{i},\bm{\Delta}_{i}^{\bm{W}}\rangle\big)\nonumber \\
 & \,\quad\qquad+\big(\langle\bm{y}_{i},-\bm{v}_{i}^{\star}\rangle-\langle\bm{z}_{i},-\bm{w}_{i}^{\star}\rangle\big)\big(\langle\bm{y}_{i},\bm{\Delta}_{i}^{\bm{V}}\rangle-\langle\bm{z}_{i},\bm{\Delta}_{i}^{\bm{W}}\rangle\big)\Big]\nonumber \\
 & \,\quad+\frac{1}{6}\sum_{i=1}^{r}\alpha_{i}\Big[\big(\|\bm{u}_{i}\|_{2}^{2}-\|\bm{v}_{i}\|_{2}^{2}\big)\big(\|\bm{x}_{i}\|_{2}^{2}-\|\bm{y}_{i}\|_{2}^{2}\big)+\big(\|\bm{u}_{i}\|_{2}^{2}-\|\bm{w}_{i}\|_{2}^{2}\big)\big(\|\bm{x}_{i}\|_{2}^{2}-\|\bm{z}_{i}\|_{2}^{2}\big)\nonumber \\
 & \,\quad\qquad+\big(\|\bm{v}_{i}\|_{2}^{2}-\|\bm{w}_{i}\|_{2}^{2}\big)\big(\|\bm{y}_{i}\|_{2}^{2}-\|\bm{z}_{i}\|_{2}^{2}\big)\Big].\label{eq:Hessian-asym-beta4}
\end{align}
\end{subequations}

In what follows, we shall demonstrate that $\beta_{1}$ is the dominant
term, with the remaining terms being negligible compared to $\beta_{1}$.
Here, we note that the proof idea is almost identical to that of the
symmetric case (cf.~Lemma~\ref{lemma:RIC} in Appendix~\ref{sec:analysis-GD}). 
% \yxc{Give a pointer to the lemma in the symmetric case}. 
For the
sake of conciseness, we will focus only on the part where the symmetric
and the asymmetric cases differ, and omit the proof details when their
analyses are similar.

\paragraph{Bounding $\beta_{1}$}

Let us first expand
\begin{align*}
\beta_{1} & =\frac{1}{3}\,\Big\|\sum_{i=1}^{r}\bm{x}_{i}\otimes\bm{v}_{i}^{\star}\otimes\bm{w}_{i}^{\star}-\bm{u}_{i}^{\star}\otimes\bm{y}_{i}\otimes\bm{w}_{i}^{\star}-\bm{u}_{i}^{\star}\otimes\bm{v}_{i}^{\star}\otimes\bm{z}_{i}\Big\|_{\mathrm{F}}^{2}\\
 & \quad+\frac{1}{3}\sum_{i=1}^{r}\lambda_{i}^{\star2/3}\Big[\big(\langle\bm{x}_{i},\bm{u}_{i}^{\star}\rangle+\langle\bm{y}_{i},\bm{v}_{i}^{\star}\rangle\big)^{2}+\big(\langle\bm{x}_{i},\bm{u}_{i}^{\star}\rangle+\langle\bm{z}_{i},\bm{w}_{i}^{\star}\rangle\big)^{2}+\big(\langle\bm{y}_{i},\bm{v}_{i}^{\star}\rangle-\langle\bm{z}_{i},\bm{w}_{i}^{\star}\rangle\big)^{2}\Big]\\
 & =\underbrace{\frac{1}{3}\,\Big\|\sum_{i=1}^{r}\bm{x}_{i}\otimes\bm{v}_{i}^{\star}\otimes\bm{w}_{i}^{\star}\Big\|_{\mathrm{F}}^{2}+\frac{1}{3}\,\Big\|\sum_{i=1}^{r}\bm{u}_{i}^{\star}\otimes\bm{y}_{i}\otimes\bm{w}_{i}^{\star}\Big\|_{\mathrm{F}}^{2}+\frac{1}{3}\,\Big\|\sum_{i=1}^{r}\bm{u}_{i}^{\star}\otimes\bm{v}_{i}^{\star}\otimes\bm{z}_{i}\Big\|_{\mathrm{F}}^{2}}_{=:\,\gamma_{1}}\\
 & \quad+\underbrace{\frac{2}{3}\sum_{i=1}^{r}\lambda_{i}^{\star2/3}\langle\bm{x}_{i},\bm{u}_{i}^{\star}\rangle^{2}+\frac{2}{3}\sum_{i=1}^{r}\lambda_{i}^{\star2/3}\langle\bm{y}_{i},\bm{v}_{i}^{\star}\rangle^{2}+\frac{2}{3}\sum_{i=1}^{r}\lambda_{i}^{\star2/3}\langle\bm{z}_{i},\bm{w}_{i}^{\star}\rangle^{2}}_{=:\,\gamma_{2}}\\
 & \quad+\underbrace{-\frac{2}{3}\,\Big\langle\sum_{i=1}^{r}\bm{x}_{i}\otimes\bm{v}_{i}^{\star}\otimes\bm{w}_{i}^{\star},\sum_{i=1}^{r}\bm{u}_{i}^{\star}\otimes\bm{y}_{i}\otimes\bm{w}_{i}^{\star}\Big\rangle+\frac{2}{3}\sum_{i=1}^{r}\lambda_{i}^{\star2/3}\langle\bm{x}_{i},\bm{u}_{i}^{\star}\rangle\langle\bm{y}_{i},\bm{v}_{i}^{\star}\rangle}_{=:\,\gamma_{3}}\\
 & \quad+\underbrace{-\frac{2}{3}\,\Big\langle\sum_{i=1}^{r}\bm{x}_{i}\otimes\bm{v}_{i}^{\star}\otimes\bm{w}_{i}^{\star},\sum_{i=1}^{r}\bm{u}_{i}^{\star}\otimes\bm{v}_{i}^{\star}\otimes\bm{z}_{i}\Big\rangle+\frac{2}{3}\sum_{i=1}^{r}\lambda_{i}^{\star2/3}\langle\bm{x}_{i},\bm{u}_{i}^{\star}\rangle\langle\bm{z}_{i},\bm{w}_{i}^{\star}\rangle}_{=:\,\gamma_{4}}\\
 & \quad+\underbrace{\frac{2}{3}\,\Big\langle\sum_{i=1}^{r}\bm{u}_{i}^{\star}\otimes\bm{y}_{i}\otimes\bm{w}_{i}^{\star},\sum_{i=1}^{r}\bm{u}_{i}^{\star}\otimes\bm{v}_{i}^{\star}\otimes\bm{z}_{i}\Big\rangle-\frac{2}{3}\sum_{i=1}^{r}\lambda_{i}^{\star2/3}\langle\bm{y}_{i},\bm{v}_{i}^{\star}\rangle\langle\bm{z}_{i},\bm{w}_{i}^{\star}\rangle}_{=:\,\gamma_{5}},
\end{align*}
leaving us with five terms to control. 
\begin{enumerate}
\item Let us begin with $\gamma_{1}$. Observe that we can express
\begin{align*}
\Big\|\sum_{i=1}^{r}\bm{x}_{i}\otimes\bm{v}_{i}^{\star}\otimes\bm{w}_{i}^{\star}\Big\|_{\mathrm{F}}=\big\|\bm{X}\bm{H}^{\top}\big\|_{\mathrm{F}},\qquad\text{with}\quad\bm{H}=[\bm{v}_{1}^{\star}\otimes\bm{w}_{1}^{\star},\cdots,\bm{v}_{r}^{\star}\otimes\bm{w}_{r}^{\star}]\in\mathbb{R}^{d^{2}\times r}
\end{align*}
Arguing similarly as in the proof of Lemma~\ref{lemma:incoh}, 
%\yxc{which lemma? and where? give pointers}, 
one can derive
\begin{align*}
\big|\sigma_{\min}(\bm{H})-\lambda_{\min}^{\star2/3}\big|\ll\lambda_{\min}^{\star2/3}\qquad\text{and}\qquad\big|\sigma_{\max}(\bm{H})-\lambda_{\max}^{\star2/3}\big|\ll\lambda_{\min}^{\star2/3},
\end{align*}
provided that $r\ll d_{\min}/\mu$. This leads to the following inequalities
\begin{align*}
\frac{19}{20}\lambda_{\min}^{\star2/3}\left\Vert \bm{X}\right\Vert _{\mathrm{F}}\leq\sigma_{\min}(\bm{H})\left\Vert \bm{X}\right\Vert _{\mathrm{F}}\leq\Big\|\sum_{i=1}^{r}\bm{x}_{i}\otimes\bm{v}_{i}^{\star}\otimes\bm{w}_{i}^{\star}\Big\|_{\mathrm{F}}\leq\sigma_{\max}(\bm{H})\left\Vert \bm{X}\right\Vert _{\mathrm{F}}\leq\frac{11}{10}\lambda_{\max}^{\star2/3}\left\Vert \bm{X}\right\Vert _{\mathrm{F}}.
\end{align*}
Clearly, the same argument reveals that
\begin{align*}
 & \frac{19}{20}\lambda_{\min}^{\star2/3}\left\Vert \bm{Y}\right\Vert _{\mathrm{F}}\leq\Big\|\sum_{i=1}^{r}\bm{u}_{i}^{\star}\otimes\bm{y}_{i}\otimes\bm{w}_{i}^{\star}\Big\|_{\mathrm{F}}\leq\frac{11}{10}\lambda_{\max}^{\star2/3}\left\Vert \bm{Y}\right\Vert _{\mathrm{F}},\\
 & \frac{19}{20}\lambda_{\min}^{\star2/3}\left\Vert \bm{Z}\right\Vert _{\mathrm{F}}\leq\Big\|\sum_{i=1}^{r}\bm{u}_{i}^{\star}\otimes\bm{v}_{i}^{\star}\otimes\bm{z}_{i}\Big\|_{\mathrm{F}}\leq\frac{11}{10}\lambda_{\max}^{\star2/3}\left\Vert \bm{Z}\right\Vert _{\mathrm{F}}.
\end{align*}
Combining these bounds, we reach
\begin{align*}
\gamma_{1}\leq\frac{3}{2}\lambda_{\max}^{\star4/3}\left(\left\Vert \bm{X}\right\Vert _{\mathrm{F}}^{2}+\left\Vert \bm{Y}\right\Vert _{\mathrm{F}}^{2}+\left\Vert \bm{Z}\right\Vert _{\mathrm{F}}^{2}\right).
\end{align*}
\item We now move on to $\gamma_{2}$. Recall our assumption that $\|\bm{u}_{i}^{\star}\|_{2}=\|\bm{v}_{i}^{\star}\|_{2}=\|\bm{w}_{i}^{\star}\|_{2}=\lambda_{i}^{\star1/3}$
for all $1\leq i\leq r$. It follows from the Cauchy-Schwartz inequality
that
\begin{align*}
0\leq\gamma_{2} & \leq\frac{2}{3}\sum_{i=1}^{r}\lambda_{i}^{\star2/3}\|\bm{x}_{i}\|^{2}\|\bm{u}_{i}^{\star}\|_{2}^{2}+\frac{2}{3}\sum_{i=1}^{r}\lambda_{i}^{\star2/3}\|\bm{y}_{i}\|^{2}\|\bm{v}_{i}^{\star}\|_{2}^{2}+\frac{2}{3}\sum_{i=1}^{r}\lambda_{i}^{\star2/3}\|\bm{z}_{i}\|_{2}^{2}\|\bm{w}_{i}^{\star}\|_{2}^{2}\\
 & \leq\frac{2}{3}\lambda_{\max}^{\star4/3}\left(\left\Vert \bm{X}\right\Vert _{\mathrm{F}}^{2}+\left\Vert \bm{Y}\right\Vert _{\mathrm{F}}^{2}+\left\Vert \bm{Z}\right\Vert _{\mathrm{F}}^{2}\right).
\end{align*}
\item Turning to $\gamma_{3}$, one can straightforwardly bound
\begin{align*}
|\gamma_{3}| & \overset{(\mathrm{i})}{=}\frac{2}{3}\Big|\sum_{i_{1}\neq i_{2}}\langle\bm{x}_{i_{1}},\bm{u}_{i_{2}}^{\star}\rangle\langle\bm{v}_{i_{1}}^{\star},\bm{y}_{i_{2}}\rangle\langle\bm{w}_{i_{1}}^{\star},\bm{w}_{i_{2}}^{\star}\rangle\Big|\\
 & \overset{(\mathrm{ii})}{\leq}\max_{i_{1}\neq i_{2}}\big|\langle\bm{w}_{i_{1}}^{\star},\bm{w}_{i_{2}}^{\star}\rangle\big|\Big(\sum_{i=1}^{r}\|\bm{x}_{i}\|_{2}\|\bm{v}_{i}^{\star}\|_{2}\Big)\Big(\sum_{i=1}^{r}\|\bm{u}_{i}^{\star}\|_{2}\|\bm{y}_{i}\|_{2}\Big)\\
 & \overset{(\mathrm{iii})}{\leq}\max_{i_{1}\neq i_{2}}\big|\langle\bm{w}_{i_{1}}^{\star},\bm{w}_{i_{2}}^{\star}\rangle\big|\|\bm{U}^{\star}\|_{\mathrm{F}}\|\bm{V}^{\star}\|_{\mathrm{F}}\|\bm{X}\|_{\mathrm{F}}\|\bm{Y}\|_{\mathrm{F}}\\
 & \overset{(\mathrm{iv})}{\leq}r\sqrt{\frac{\mu}{d}}\,\lambda_{\max}^{\star4/3}\|\bm{X}\|_{\mathrm{F}}\|\bm{Y}\|_{\mathrm{F}}\\
 & \overset{(\mathrm{v})}{\ll}\lambda_{\max}^{\star4/3}\|\bm{X}\|_{\mathrm{F}}\|\bm{Y}\|_{\mathrm{F}}.
\end{align*}
Here, we have used the fact that $\|\bm{w}_{i}^{\star}\|_{2}=\lambda_{i}^{\star1/3}$
in (i); the inequalities (ii) and (iii) arise from the Cauchy-Schwartz
inequality; (iv) follows from (\ref{def:incoh-orth-asym}); (v) holds
as long as $r\ll\sqrt{d_{\min}/\mu}$. In a similar manner, one can
easily verify that
\begin{align*}
|\gamma_{4}|\ll\lambda_{\max}^{\star4/3}\|\bm{X}\|_{\mathrm{F}}\|\bm{Z}\|_{\mathrm{F}}\qquad\text{and}\qquad|\gamma_{5}|\ll\lambda_{\max}^{\star4/3}\|\bm{Y}\|_{\mathrm{F}}\|\bm{Z}\|_{\mathrm{F}}.
\end{align*}
It then follows from the AM-GM inequality that
\begin{align*}
|\gamma_{3}|+|\gamma_{4}|+|\gamma_{5}|\ll\lambda_{\max}^{\star4/3}\left(\left\Vert \bm{X}\right\Vert _{\mathrm{F}}^{2}+\left\Vert \bm{Y}\right\Vert _{\mathrm{F}}^{2}+\left\Vert \bm{Z}\right\Vert _{\mathrm{F}}^{2}\right).
\end{align*}
\item Putting the above bounds together allows us to bound $\beta_{1}$
as follows
\begin{align*}
\beta_{1} & \leq\gamma_{1}+\gamma_{2}+|\gamma_{3}|+|\gamma_{4}|+|\gamma_{5}|\leq\frac{7}{2}\lambda_{\max}^{\star4/3}\left(\left\Vert \bm{X}\right\Vert _{\mathrm{F}}^{2}+\left\Vert \bm{Y}\right\Vert _{\mathrm{F}}^{2}+\left\Vert \bm{Z}\right\Vert _{\mathrm{F}}^{2}\right);\\
\beta_{1} & \geq\gamma_{1}-|\gamma_{3}|-|\gamma_{4}|-|\gamma_{5}|\geq\frac{9}{10}\lambda_{\min}^{\star4/3}\left(\left\Vert \bm{X}\right\Vert _{\mathrm{F}}^{2}+\left\Vert \bm{Y}\right\Vert _{\mathrm{F}}^{2}+\left\Vert \bm{Z}\right\Vert _{\mathrm{F}}^{2}\right).
\end{align*}
\end{enumerate}

\paragraph{Bounding $\beta_{2}$}

To control $\beta_{2}$, we note that $\beta_{2}$ involves two quantities:
(1) the deviation of $\frac{1}{p}\,\Big\|\mathcal{P}_{\Omega}\Big(\sum_{i=1}^{r}\bm{x}_{i}\otimes\bm{v}_{i}\otimes\bm{w}_{i}+\bm{u}_{i}\otimes\bm{y}_{i}\otimes\bm{w}_{i}+\bm{u}_{i}\otimes\bm{v}_{i}\otimes\bm{z}_{i}\Big)\Big\|_{\mathrm{F}}^{2}$
from its expectation with respect to $\Omega$; (2) the distance between
$\sum_{i=1}^{r}\bm{x}_{i}\otimes\bm{v}_{i}\otimes\bm{w}_{i}+\bm{u}_{i}\otimes\bm{y}_{i}\otimes\bm{w}_{i}+\bm{u}_{i}\otimes\bm{v}_{i}\otimes\bm{z}_{i}$
and $\sum_{i=1}^{r}\bm{x}_{i}\otimes\bm{v}_{i}^{\star}\otimes\bm{w}_{i}-\bm{u}_{i}^{\star}\otimes\bm{y}_{i}\otimes\bm{w}_{i}^{\star}-\bm{u}_{i}^{\star}\otimes\bm{v}_{i}^{\star}\otimes\bm{z}_{i}$.
The first term can be shown to be exceedingly small under our sample
size condition with the help of \cite[Lemma 5]{yuan2016tensor},  
%\yxc{Given reference number rather than "5"},
while the second term is also guaranteed to be sufficiently small
by the assumptions of the error of $\bm{U},\bm{V}$ and $\bm{W}$
in (\ref{eq:RIC-asym-asymp}).
Therefore, one can apply an analogous argument for $\alpha_{1}$ and
$\alpha_{2}$ in the proof of Lemma~\ref{lemma:RIC} in Appendix~\ref{sec:proof-local-convergence}
to derive: with probability at least $1-O\big(d_{\min}^{-10}\big)$
one has
\begin{align*}
\big|\beta_{2}\big|\ll\lambda_{\min}^{\star4/3}\left(\left\Vert \bm{X}\right\Vert _{\mathrm{F}}^{2}+\left\Vert \bm{Y}\right\Vert _{\mathrm{F}}^{2}+\left\Vert \bm{Z}\right\Vert _{\mathrm{F}}^{2}\right),
\end{align*}
provided the sampling rate exceeds $p\gg\frac{\mu^{2}r^{2}d_{\max}\log^{2} d_{\max}}{d_{1}d_{2}d_{3}}$.

\paragraph{Bounding $\beta_{3}$}

By the definition of the operator norm, one
can bound
\begin{align*}
|\beta_{3}| & \lesssim\Big\| p^{-1}\mathcal{P}_{\Omega}\Big(\sum_{i=1}^{r}\bm{u}_{i}\otimes\bm{v}_{i}\otimes\bm{w}_{i}-\bm{T}^{\star}\Big)\Big\|\cdot\sum_{i=1}^{r}\Big(\left\Vert \bm{w}_{i}\right\Vert _{2}\|\bm{x}_{i}\|_{2}\|\bm{y}_{i}\|_{2}+\left\Vert \bm{v}_{i}\right\Vert _{2}\|\bm{x}_{i}\|_{2}\|\bm{z}_{i}\|_{2}+\left\Vert \bm{u}_{i}\right\Vert _{2}\|\bm{y}_{i}\|_{2}\|\bm{z}_{i}\|_{2}\Big)\\
 & \lesssim\lambda_{\max}^{2/3}\,\Big\|\sum_{i=1}^{r}\bm{u}_{i}\otimes\bm{v}_{i}\otimes\bm{w}_{i}-\bm{T}^{\star}\Big\|\Big\| p^{-1}\mathcal{P}_{\Omega}\big(\bm{1}_{d_{1}}\otimes\bm{1}_{d_{2}}\otimes\bm{1}_{d_{3}}\big)\Big\|\left(\left\Vert \bm{X}\right\Vert _{\mathrm{F}}^{2}+\left\Vert \bm{Y}\right\Vert _{\mathrm{F}}^{2}+\left\Vert \bm{Z}\right\Vert _{\mathrm{F}}^{2}\right).
\end{align*}
One can easily adapt the proof of Lemma~\ref{lemma:T_op_norm}
to show that, with probability at with probability at least $1-O\big(d_{\min}^{-10}\big),$
\begin{align*}
\big\|\big(p^{-1}\mathcal{P}_{\Omega}-\mathcal{I}\big)\big(\bm{1}_{d_{1}}\otimes\bm{1}_{d_{2}}\otimes\bm{1}_{d_{3}}\big)\big\|\lesssim\frac{\log^{3}d_{\max}}{p}+\sqrt{\frac{d_{\max}\log^{5}d_{\max}}{p}}\lesssim\sqrt{d_{1}d_{2}d_{3}}=\big\|\bm{1}_{d_{1}}\otimes\bm{1}_{d_{2}}\otimes\bm{1}_{d_{3}}\big\|
\end{align*}
holds as long as $p\gg\max\Big\{\frac{\log^{3}d_{\max}}{\sqrt{d_{1}d_{2}d_{3}}},\,\frac{d_{\max}\log^{5}d_{\max}}{d_{1}d_{2}d_{3}}\Big\}$.
We can then adapt the proof for bounding $\alpha_{3}$
in the proof of Lemma~\ref{lemma:RIC} in Appendix~\ref{sec:proof-local-convergence} to derive that with probability at least $1-O\big(d_{\min}^{-10}\big),$
\begin{align*}
\big|\beta_{3}\big|\ll\lambda_{\min}^{\star4/3}\left(\left\Vert \bm{X}\right\Vert _{\mathrm{F}}^{2}+\left\Vert \bm{Y}\right\Vert _{\mathrm{F}}^{2}+\left\Vert \bm{Z}\right\Vert _{\mathrm{F}}^{2}\right).
\end{align*}

\paragraph{Bounding $\beta_{4}$}

It remains to control $\beta_{4}$. By symmetry, it suffices to consider
the following terms:
\begin{align*}
\gamma_{1} & :=\sum_{i=1}^{r}\big(\alpha_{i}-\lambda_{i}^{\star2/3}\big)\big(\langle\bm{x}_{i},\bm{u}_{i}^{\star}\rangle+\langle\bm{y}_{i},\bm{v}_{i}^{\star}\rangle\big)^{2},\\
\gamma_{2} & :=\sum_{i=1}^{r}\alpha_{i}\big(\langle\bm{x}_{i},\bm{\Delta}_{i}^{\bm{U}}\rangle-\langle\bm{y}_{i},\bm{\Delta}_{i}^{\bm{V}}\rangle\big)^{2},\\
\gamma_{3} & :=\sum_{i=1}^{r}\alpha_{i}\big(\langle\bm{x}_{i},\bm{u}_{i}^{\star}\rangle+\langle\bm{y}_{i},\bm{v}_{i}^{\star}\rangle\big)\big(\langle\bm{x}_{i},\bm{\Delta}_{i}^{\bm{U}}\rangle-\langle\bm{y}_{i},\bm{\Delta}_{i}^{\bm{V}}\rangle\big),\\
\gamma_{4} & :=\sum_{i=1}^{r}\alpha_{i}\big(\|\bm{u}_{i}\|_{2}^{2}-\|\bm{v}_{i}\|_{2}^{2}\big)\big(\|\bm{x}_{i}\|_{2}^{2}-\|\bm{y}_{i}\|_{2}^{2}\big).
\end{align*}
In the following, we shall bound these four terms separately. By the
assumptions that $\bm{S}^{(1)}=\bm{I}_{r},\bm{S}^{(2)}=\bm{S}^{(3)}=-\bm{I}_{r}$,
$\delta\ll1$ and $\kappa\asymp1$, one has \begin{subequations}
\begin{align}
\|\bm{\Delta}_{i}^{\bm{U}}\|_{2} & \leq\|\bm{U}-\bm{U}^{\star}\|_{\mathrm{F}}\leq\delta\,\|\bm{U}^{\star}\|_{\mathrm{F}}\leq\delta\sqrt{r}\lambda_{\max}^{\star1/3}\ll\lambda_{\min}^{\star1/3}\leq\|\bm{u}_{i}^{\star}\|_{2};\label{eq:Hessian-asym-u-l2-loss}\\
\|\bm{\Delta}_{i}^{\bm{V}}\|_{2} & \leq\left\Vert -\bm{V}-\bm{V}^{\star}\right\Vert _{\mathrm{F}}\leq\delta\,\|\bm{V}^{\star}\|_{\mathrm{F}}\leq\delta\sqrt{r}\lambda_{\max}^{\star1/3}\ll\lambda_{\min}^{\star1/3}\leq\|\bm{v}_{i}^{\star}\|_{2};\label{eq:Hessian-asym-v-l2-loss}\\
\|\bm{\Delta}_{i}^{\bm{W}}\|_{2} & \leq\left\Vert -\bm{W}-\bm{W}^{\star}\right\Vert _{\mathrm{F}}\leq\delta\,\|\bm{W}^{\star}\|_{\mathrm{F}}\leq\delta\sqrt{r}\lambda_{\max}^{\star1/3}\ll\lambda_{\min}^{\star1/3}\leq\|\bm{w}_{i}^{\star}\|_{2}.\label{eq:Hessian-asym-w-l2-loss}
\end{align}
\end{subequations}In addition, for each $1\leq i\leq r$ , one has
\begin{equation}
\frac{9}{10}\lambda_{\min}^{\star2/3}\leq\lambda_{i}-\big|\alpha_{i}-\lambda_{i}^{\star2/3}\big|\leq\alpha_{i}\leq\lambda_{i}+\big|\alpha_{i}-\lambda_{i}^{\star2/3}\big|\leq\frac{11}{10}\lambda_{\max}^{\star2/3}\label{eq:Hessian-asym-alpha-bound}
\end{equation}
by virtue of the condition that $\big|\alpha_{i}-\lambda_{i}^{\star2/3}\big|\ll\lambda_{\min}^{\star2/3}$.
\begin{enumerate}
\item Let us start with $\gamma_{1}$. By the condition that $\max_{i}\big|\alpha_{i}-\lambda_{i}^{\star2/3}\big|\ll\lambda_{\min}^{\star2/3}$,
it is easily seen that
\begin{align*}
|\gamma_{1}| & \lesssim\sum_{i=1}^{r}\big|\alpha_{i}-\lambda_{i}^{\star2/3}\big|\big(\langle\bm{x}_{i},\bm{u}_{i}^{\star}\rangle^{2}+\langle\bm{y}_{i},\bm{v}_{i}^{\star}\rangle^{2}\big)\leq\max_{i}\big|\alpha_{i}-\lambda_{i}^{\star2/3}\big|\sum_{i=1}^{r}\big(\|\bm{x}_{i}\|_{2}^{2}\|\bm{u}_{i}^{\star}\|_{2}^{2}+\|\bm{y}_{i}\|_{2}^{2}\|\bm{v}_{i}^{\star}\|_{2}^{2}\big)\\
 & \leq\max_{i}\big|\alpha_{i}-\lambda_{i}^{\star2/3}\big|\left(\|\bm{u}_{i}^{\star}\|_{2}^{2}+\|\bm{v}_{i}^{\star}\|_{2}^{2}\right)\big(\|\bm{X}\|_{\mathrm{F}}^{2}+\|\bm{Y}\|_{\mathrm{F}}^{2}\big)\ll\lambda_{\min}^{\star4/3}\big(\|\bm{X}\|_{\mathrm{F}}^{2}+\|\bm{Y}\|_{\mathrm{F}}^{2}\big),
\end{align*}
with the proviso that $\kappa\asymp1$.
\item As for $\gamma_{2}$, one can bound
\begin{align*}
|\gamma_{2}| & \lesssim\sum_{i=1}^{r}\alpha_{i}\big(\langle\bm{x}_{i},\bm{\Delta}_{i}^{\bm{U}}\rangle^{2}-\langle\bm{y}_{i},\bm{\Delta}_{i}^{\bm{V}}\rangle^{2}\big)\leq\max_{i}\alpha_{i}\sum_{i=1}^{r}\big(\|\bm{x}_{i}\|_{2}^{2}\|\bm{\Delta}_{i}^{\bm{U}}\|_{2}^{2}+\|\bm{y}_{i}\|_{2}^{2}\|\bm{\Delta}_{i}^{\bm{V}}\|_{2}^{2}\big)\\
 & \leq\max_{i}\alpha_{i}\big(\|\bm{\Delta}_{i}^{\bm{U}}\|_{2}^{2}+\|\bm{\Delta}_{i}^{\bm{V}}\|_{2}^{2}\big)\big(\|\bm{X}\|_{\mathrm{F}}^{2}+\|\bm{Y}\|_{\mathrm{F}}^{2}\big)\ll\lambda_{\min}^{\star4/3}\big(\|\bm{X}\|_{\mathrm{F}}^{2}+\|\bm{Y}\|_{\mathrm{F}}^{2}\big).
\end{align*}
Here, we use (\ref{eq:Hessian-asym-u-l2-loss}), (\ref{eq:Hessian-asym-v-l2-loss}),
(\ref{eq:Hessian-asym-alpha-bound}) and $\kappa\asymp1$ in the last
step.
\item Turning to $\gamma_{3}$, one can develop a similar bound as follows
\begin{align*}
|\gamma_{3}| & \leq\sqrt{\sum_{i=1}^{r}\alpha_{i}\big(\langle\bm{x}_{i},\bm{u}_{i}^{\star}\rangle+\langle\bm{y}_{i},\bm{v}_{i}^{\star}\rangle\big)^{2}}\cdot\sqrt{\sum_{i=1}^{r}\alpha_{i}\big(\langle\bm{x}_{i},\bm{\Delta}_{i}^{\bm{U}}\rangle-\langle\bm{y}_{i},\bm{\Delta}_{i}^{\bm{V}}\rangle\big)^{2}}\\
 & \lesssim\max_{i}\alpha_{i}\sqrt{\left(\|\bm{u}_{i}^{\star}\|_{2}^{2}+\|\bm{v}_{i}^{\star}\|_{2}^{2}\right)\big(\|\bm{X}\|_{\mathrm{F}}^{2}+\|\bm{Y}\|_{\mathrm{F}}^{2}\big)}\cdot\sqrt{\big(\|\bm{\Delta}_{i}^{\bm{U}}\|_{2}^{2}+\|\bm{\Delta}_{i}^{\bm{V}}\|_{2}^{2}\big)\big(\|\bm{X}\|_{\mathrm{F}}^{2}+\|\bm{Y}\|_{\mathrm{F}}^{2}\big)}\\
 & \ll\lambda_{\min}^{\star4/3}\big(\|\bm{X}\|_{\mathrm{F}}^{2}+\|\bm{Y}\|_{\mathrm{F}}^{2}\big).
\end{align*}
\item With regards to $\gamma_{4}$, we can expand
\begin{align*}
|\gamma_{4}| & =\sum_{i=1}^{r}\alpha_{i}\big(\|\bm{u}_{i}^{\star}+\bm{\Delta}_{i}^{\bm{U}}\|_{2}^{2}-\left\Vert -\bm{v}_{i}^{\star}+\bm{\Delta}_{i}^{\bm{V}}\right\Vert _{2}^{2}\big)\big(\|\bm{x}_{i}\|_{2}^{2}-\|\bm{y}_{i}\|_{2}^{2}\big)\\
 & =\sum_{i=1}^{r}\alpha_{i}\big(\|\bm{u}_{i}^{\star}\|_{2}^{2}+2\langle\bm{u}_{i}^{\star},\bm{\Delta}_{i}^{\bm{U}}\rangle+\|\bm{\Delta}_{i}^{\bm{U}}\|_{2}^{2}-\|\bm{v}_{i}^{\star}\|_{2}^{2}+2\langle\bm{v}_{i}^{\star},\bm{\Delta}_{i}^{\bm{V}}\rangle-\|\bm{\Delta}_{i}^{\bm{V}}\|_{2}^{2}\big)\big(\|\bm{x}_{i}\|_{2}^{2}-\|\bm{y}_{i}\|_{2}^{2}\big)\\
 & =\sum_{i=1}^{r}\alpha_{i}\big(2\langle\bm{u}_{i}^{\star},\bm{\Delta}_{i}^{\bm{U}}\rangle+\|\bm{\Delta}_{i}^{\bm{U}}\|_{2}^{2}+2\langle\bm{v}_{i}^{\star},\bm{\Delta}_{i}^{\bm{V}}\rangle-\|\bm{\Delta}_{i}^{\bm{V}}\|_{2}^{2}\big)\big(\|\bm{x}_{i}\|_{2}^{2}-\|\bm{y}_{i}\|_{2}^{2}\big),
\end{align*}
where the last step follows from the assumption $\|\bm{u}_{i}^{\star}\|_{2}=\|\bm{v}_{i}^{\star}\|_{2}=\|\bm{w}_{i}^{\star}\|_{2}$.
It follows that
\begin{align*}
\Big|\sum_{i=1}^{r}\alpha_{i}\big(\|\bm{u}_{i}\|_{2}^{2}-\|\bm{v}_{i}\|_{2}^{2}\big)\big(\|\bm{x}_{i}\|_{2}^{2}-\|\bm{y}_{i}\|_{2}^{2}\big)\Big| & \lesssim\sum_{i=1}^{r}\alpha_{i}\big|\langle\bm{u}_{i}^{\star},\bm{\Delta}_{i}^{\bm{U}}\rangle+\|\bm{\Delta}_{i}^{\bm{U}}\|_{2}^{2}+\langle\bm{v}_{i}^{\star},\bm{\Delta}_{i}^{\bm{V}}\rangle+\|\bm{\Delta}_{i}^{\bm{V}}\|_{2}^{2}\big|\left(\|\bm{x}_{i}\|_{2}^{2}+\|\bm{y}_{i}\|_{2}^{2}\right)\\
 & \leq\max_{i}\alpha_{i}\big(\|\bm{u}_{i}^{\star}\|_{2}\|\bm{\Delta}_{i}^{\bm{U}}\|_{2}+\|\bm{v}_{i}^{\star}\|_{2}\|\bm{\Delta}_{i}^{\bm{V}}\|_{2}\big)\sum_{i=1}^{r}\left(\|\bm{x}_{i}\|_{2}^{2}+\|\bm{y}_{i}\|_{2}^{2}\right)\\
 & \ll\lambda_{\min}^{\star4/3}\left(\left\Vert \bm{X}\right\Vert _{\mathrm{F}}^{2}+\left\Vert \bm{Y}\right\Vert _{\mathrm{F}}^{2}\right),
\end{align*}
where we have used the conditions that $\max_{i}\alpha_{i}\lesssim\lambda_{\max}^{\star2/3}$
and $\kappa\asymp1$. 
\item It is not hard to check that similar bounds also hold for the remaining
terms in $\beta_{4}$. As a result, we arrive at
\begin{align*}
\big|\beta_{4}\big|\ll\lambda_{\min}^{\star4/3}\left(\left\Vert \bm{X}\right\Vert _{\mathrm{F}}^{2}+\left\Vert \bm{Y}\right\Vert _{\mathrm{F}}^{2}+\left\Vert \bm{Z}\right\Vert _{\mathrm{F}}^{2}\right).
\end{align*}
\end{enumerate}

\paragraph{Combining the previous bounds on $\beta_{1},\beta_{2},\beta_{3}$
and $\beta_{4}$}

Putting the above estimates together, we conclude that with probability
at least $1-O(d_{\min}^{-10})$, one has
\begin{align*}
\mathsf{vec}\left(\begin{bmatrix}\bm{X}\\
\bm{Y}\\
\bm{Z}
\end{bmatrix}\right)^{\top}\nabla^{2}g_{\mathsf{clean}}(\bm{U},\bm{V},\bm{W})\,\mathsf{vec}\left(\begin{bmatrix}\bm{X}\\
\bm{Y}\\
\bm{Z}
\end{bmatrix}\right) & \geq\beta_{1}-|\beta_{2}|-|\beta_{3}|-|\beta_{4}|\geq\frac{1}{2}\lambda_{\min}^{\star4/3}\left(\left\Vert \bm{X}\right\Vert _{\mathrm{F}}^{2}+\left\Vert \bm{Y}\right\Vert _{\mathrm{F}}^{2}+\left\Vert \bm{Z}\right\Vert _{\mathrm{F}}^{2}\right)\\
\mathsf{vec}\left(\begin{bmatrix}\bm{X}\\
\bm{Y}\\
\bm{Z}
\end{bmatrix}\right)^{\top}\nabla^{2}g_{\mathsf{clean}}(\bm{U},\bm{V},\bm{W})\,\mathsf{vec}\left(\begin{bmatrix}\bm{X}\\
\bm{Y}\\
\bm{Z}
\end{bmatrix}\right) & \leq\beta_{1}+|\beta_{2}|+|\beta_{3}|+|\beta_{4}|\leq4\lambda_{\max}^{\star4/3}\left(\left\Vert \bm{X}\right\Vert _{\mathrm{F}}^{2}+\left\Vert \bm{Y}\right\Vert _{\mathrm{F}}^{2}+\left\Vert \bm{Z}\right\Vert _{\mathrm{F}}^{2}\right)
\end{align*}
as claimed.

\bibliographystyle{alphaabbr}
\bibliography{bibfileNonconvex, TC}

\end{document}